\documentclass[11pt,a4paper]{article}
\usepackage[utf8]{inputenc}

\usepackage{hyperref,mathrsfs,amsthm,xcolor,verbatim,bbm,amsmath,amsfonts,amssymb,nicefrac,bm,mathtools,xparse,etoolbox,enumitem,derivative,cite,mdframed,aliascnt}
\usepackage[margin=1in]{geometry}
\usepackage[capitalise,sort,nameinlink]{cleveref} 

\crefname{enumi}{item}{items}
\crefname{equation}{}{}
\crefname{subsection}{Subsection}{Subsections}

\hypersetup{
    colorlinks,
    linkcolor={blue!80!black},
    citecolor={green},
    urlcolor={blue!80!black},
}

\newmdenv[
skipabove=\topsep,
skipbelow=\topsep,
backgroundcolor=gray!10,
linecolor=gray!80,
linewidth=1pt,
innertopmargin=4pt,
innerbottommargin=0.5\baselineskip,
innerleftmargin=0.5\baselineskip,
innerrightmargin=0.5\baselineskip
]{boxedthm}

\theoremstyle{plain}

\newaliascnt{lemma}{theorem}
\newtheorem{lemma}[lemma]{Lemma}
\aliascntresetthe{lemma}

\newaliascnt{prop}{theorem}
\newtheorem{prop}[prop]{Proposition}
\aliascntresetthe{prop}

\newaliascnt{cor}{theorem}

\aliascntresetthe{cor}

\crefname{theorem}{Theorem}{Theorems}
\Crefname{theorem}{Theorem}{Theorems}

\crefname{lemma}{Lemma}{Lemmas}
\Crefname{lemma}{Lemma}{Lemmas}

\crefname{prop}{Proposition}{Propositions}
\Crefname{prop}{Proposition}{Propositions}

\crefname{cor}{Corollary}{Corollaries}
\Crefname{cor}{Corollary}{Corollaries}

\newaliascnt{definition}{theorem}
\newtheorem{definition}[definition]{Definition}
\aliascntresetthe{definition}
\crefname{definition}{Definition}{Definitions}
\Crefname{definition}{Definition}{Definitions}

\theoremstyle{definition}

\DeclareMathAlphabet{\mathpzc}{OT1}{pzc}{m}{it}

\DeclareFontEncoding{LS1}{}{}
\DeclareFontSubstitution{LS1}{stix}{m}{n}
\DeclareMathAlphabet{\mathscr}{LS1}{stixscr}{m}{n}

\newcommand{\E}{\mathbb{E}}
\renewcommand{\P}{\mathbb{P}}

\newcommand{\R}{\mathbb{R}}
\newcommand{\N}{\mathbb{N}}
\newcommand{\Z}{\mathbb{Z}}

\renewcommand{\c}[1]{\mathfrak{c}^{#1}}

\newcommand{\proper}[2]{$#1$ is proper $#2$-conditional $\P$-integrable\cfadd{def:proper:cond:int}}

\newcommand{\improper}[2]{$#1$ is improper $#2$-conditional $\P$-integrable\cfadd{def:improper:cond:int}}

\newcommand{\cond}[3]{$#1$ is a $#2$-conditional $\P$-expectation of $#3$\cfadd{def:cond:exp}}

\newcommand{\co}[2]{\E\PR{#1|#2}\cfadd{gen:conditional:expectation}}
\newcommand{\cob}[2]{\E\PRb{#1\big|#2}\cfadd{gen:conditional:expectation}}

\newcommand{\cD}{\mathcal{D}}

\newcommand{\cF}{\mathcal{F}}
\newcommand{\cG}{\mathcal{G}}

\newcommand{\cL}{\mathcal{L}}

\newcommand{\fC}{\mathfrak{C}}

\newcommand{\fc}{\mathfrak{c}}

\renewcommand{\emptyset}{\varnothing}
\newcommand{\eps}{\varepsilon}

\newcommand{\dpa}{\textbf{x}}

\newcommand{\DPA}{\textbf{X}}
\newcommand{\act}{\mathbb{A}}

\DeclarePairedDelimiter{\norm}{\lVert}{\rVert}

\makeatletter
\def\@tvsp{\mathchoice{{}\mkern-4.5mu}{{}\mkern-4.5mu}{{}\mkern-2.5mu}{}}
\def\ltrivert{\left|\@tvsp\left|\@tvsp\left|}
\def\rtrivert{\right|\@tvsp\right|\@tvsp\right|}
\makeatother

\makeatletter
\ExplSyntaxOn
\seq_new:N \g_abbrs
\prop_new:N \g_abbr_counts
\tl_new:N \l_abbr_count_tl

\ExplSyntaxOn

\bool_new:N \g_forexample

\NewDocumentCommand{\eg}{ o }{
	\IfValueT{#1}{
		\str_if_eq:noTF {fe} {#1} {
			\bool_gset_true:N \g_forexample
		} {\bool_gset_false:N \g_forexample}
	}
	\bool_if:nTF { \g_forexample } {
		\bool_gset_false:N \g_forexample
		for~example
	}{
		\bool_gset_true:N \g_forexample
		for~instance
	}
}

\NewDocumentCommand{\abbr}{m m O{#1} m m O{#4} m}{
	\expandafter\newcommand\csname#3\endcsname[1][]{
		\seq_if_in:NnTF \g_abbrs {#1} {
			\prop_get:NnN \g_abbr_counts {#1} \l_abbr_count_tl
			\prop_gput:Nnx \g_abbr_counts {#1} {\int_eval:n {\l_abbr_count_tl + 1}}
			\hyperref[#1]{#7}
		} {
			\seq_gput_left:Nn \g_abbrs {#1}
			\prop_gput:Nnn \g_abbr_counts {#1} {1}
			\expandafter\gdef\csname#1@def\endcsname{#2}
			\phantomsection\label{#1}
			\str_if_eq:nnTF{##1}{}{\emph{#2}}{##1}~(\hyperref[#1]{#7})
		}
	}
	\expandafter\newcommand\csname#6\endcsname[1][]{
		\seq_if_in:NnTF \g_abbrs {#1} {
			\prop_get:NnN \g_abbr_counts {#1} \l_abbr_count_tl
			\prop_gput:Nnx \g_abbr_counts {#1} {\int_eval:n {\l_abbr_count_tl + 1}}
			\hyperref[#1]{#4}
		} {
			\expandafter\gdef\csname#1@def\endcsname{#5}
			\seq_gput_left:Nn \g_abbrs {#1}
			\prop_gput:Nnn \g_abbr_counts {#1} {1}
			\phantomsection\label{#1}
			\str_if_eq:nnTF{##1}{}{\emph{#5}}{##1}~(\hyperref[#1]{#4})
		}
	}
}

\ExplSyntaxOff
\makeatother

\abbr{ANN}{artificial neural network}{ANNs}{artificial neural networks}{ANN}
\abbr{ADAM}{adaptive moment estimation}{ADAMs}{adaptive moment estimations}{Adam}
\abbr{RMSprop}{root mean square propagation}{RMSprops}{root mean square propagation}{RMSprop}
\abbr{DNN}{deep artificial neural network}{DNNs}{deep artificial neural networks}{DNN}
\abbr{SGD}{stochastic gradient descent}{SGDs}{stochastic gradient descent}{SGD}
\abbr{iid}{independent and identically distributed}{i.i.d}{independent and identically distributed}{i.i.d.}
\abbr{DL}{Deep learning}{DLs}{Deep learning}{DL}
\abbr{AI}{artificial intelligence}{AIs}{artificial intelligence}{AI}
\abbr{LLM}{large language model}{LLMs}{large language models}{LLM}
\abbr{PDE}{partial differential equation}{PDEs}{partial differantial equations}{PDE}
\abbr{as}{almost surely}{ass}{almost surely}{a.s.}
\abbr{pas}{$\P$-almost surely}{pass}{almost surely}{$\P$-a.s.}
\abbr{ReLU}{rectified linear unit}{ReLUs}{rectified linear units}{ReLU}

\renewcommand{\d}{ \mathrm{d}}

\newcommand{\qandq}{\quad \text{and} \quad }
\newcommand{\qqandqq}{\qquad\text{and}\qquad}

\DeclarePairedDelimiter{\vass}{\lvert}{\rvert}
\DeclarePairedDelimiter{\pr}{(}{)}
\DeclarePairedDelimiter{\PR}{[}{]}
\DeclarePairedDelimiter{\pR}{\{}{\}}

\newcommand{\vertiii}[1]{\vass{\kern-0.25ex\vass{\kern-0.25ex\vass{#1}\kern-0.25ex}\kern-0.25ex}}

\newcommand{\prb}[1]{\pr[\big]{ #1 }}
\newcommand{\PRb}[1]{\PR[\big]{ #1 }}
\newcommand{\pRb}[1]{\pR[\big]{ #1 }}
\newcommand{\prbbb}[1]{\pr[\bigg]{ #1 }}
\newcommand{\PRbbb}[1]{\PR[\bigg]{ #1 }}
\newcommand{\pRbbb}[1]{\pR[\bigg]{ #1 }}

\newcommand{\PRbb}[1]{\PR[\Big]{ #1 }}

\newcommand{\indicator}[1]{\mathbbm{1}_{\smash{#1}}}

\newcommand{\realization}[1]{\mathscr{N}^{ #1 }}

\newcommand{\var}{\operatorname{Var}}

\newcommand{\cva}{\mathcal{V}}

\ExplSyntaxOn

\bool_new:N \g_noteobserve

\NewDocumentCommand{\setnote}{}{
  \bool_gset_true:N \g_noteobserve
}

\NewDocumentCommand{\setobserve}{}{
  \bool_gset_false:N \g_noteobserve
}

\NewDocumentCommand{\nobs}{ o }{
  \IfValueT{#1}{
    \str_if_eq:noTF {note} {#1} {
      \bool_gset_true:N \g_noteobserve
    } {
      \str_if_eq:noTF {Note} {#1} {
        \bool_gset_true:N \g_noteobserve
      } {
        \bool_gset_false:N \g_noteobserve
      }
    }
  }
  \bool_if:nTF { \g_noteobserve } {
    \bool_gset_false:N \g_noteobserve
    note
  } {
    \bool_gset_true:N \g_noteobserve
    observe
  }
  \IfValueF{#1}{~}
}

\NewDocumentCommand{\Nobs}{ o }{
  \IfValueT{#1}{
    \str_if_eq:noTF {note} {#1} {
      \bool_gset_true:N \g_noteobserve
    } {
      \str_if_eq:noTF {Note} {#1} {
        \bool_gset_true:N \g_noteobserve
      } {
        \bool_gset_false:N \g_noteobserve
      }
    }
  }
  \bool_if:nTF { \g_noteobserve } {
    \bool_gset_false:N \g_noteobserve
    Note
  } {
    \bool_gset_true:N \g_noteobserve
    Observe
  }
  \IfValueF{#1}{~}
}

\int_new:N \g_furthermore

\NewDocumentCommand{\Moreover}{ o o }{
  \IfValueT{#1}{
    \str_case:nn {#1} {
      {Furthermore} {\int_set:Nn {\g_furthermore} {0}}
      {Moreover} {\int_set:Nn {\g_furthermore} {1}}
      {In~addition} {\int_set:Nn {\g_furthermore} {2}}
      {note} {\bool_gset_true:N \g_noteobserve}
      {observe} {\bool_gset_false:N \g_noteobserve}
    }
    \IfValueT{#2}{
      \str_case:nn {#2} {
        {Furthermore} {\int_set:Nn {\g_furthermore} {0}}
        {Moreover} {\int_set:Nn {\g_furthermore} {1}}
        {In~addition} {\int_set:Nn {\g_furthermore} {2}}
        {note} {\bool_gset_true:N \g_noteobserve}
        {observe} {\bool_gset_false:N \g_noteobserve}
      }
    }
  }
  \int_case:nn { \int_mod:nn {\g_furthermore} {3} } {
    { 0 } { Furthermore,~\nobs that}
    { 1 } { Moreover,~\nobs that}
    { 2 } { In~addition,~\nobs that}
  }
  \int_incr:N \g_furthermore
  \IfValueF{#1}{~}
}

\bool_new:N \g_hencetherefore

\NewDocumentCommand{\hence}{}{
  \bool_if:nTF { \g_hencetherefore } {
    \bool_gset_false:N \g_hencetherefore
    hence~
  } {
    \bool_gset_true:N \g_hencetherefore
    therefore~
  }
}

\NewDocumentCommand{\Hence}{}{
  \bool_if:nTF { \g_hencetherefore } {
    \bool_gset_false:N \g_hencetherefore
    Hence,~we~obtain~
  } {
    \bool_gset_true:N \g_hencetherefore
    Therefore,~we~obtain~
  }
}

\seq_new:N \g_cflist_loaded
\seq_new:N \g_cflist_pending

\NewDocumentCommand{\cfadd}{ m }
{
  \seq_if_in:NnF \g_cflist_loaded { #1 } {
    \seq_if_in:NnF \g_cflist_pending { #1 } {
      \seq_gput_right:Nn \g_cflist_pending { #1 }
    }
  }
}

\NewDocumentCommand{\cfconsiderloaded}{ m }{
  \seq_gput_right:Nn \g_cflist_loaded {#1}
}

\NewDocumentCommand{\cfremove}{ m }
{
  \seq_gremove_all:Nn \g_cflist_pending { #1 }
}

\NewDocumentCommand{\cfload}{ o }
{
  \seq_if_empty:NTF \g_cflist_pending {\unskip} {
    (cf.\ \cref{\seq_use:Nn \g_cflist_pending {,}})\IfValueTF{#1}{#1~}{\unskip}
    \seq_gconcat:NNN \g_cflist_loaded \g_cflist_loaded \g_cflist_pending
    \seq_gclear:N \g_cflist_pending
  }
}

\NewDocumentCommand{\cfclear} {} {
  \seq_gclear:N \g_cflist_loaded
  \seq_gclear:N \g_cflist_pending
}

\NewDocumentCommand{\cfout}{ o }
{
  \seq_if_empty:NTF \g_cflist_pending {\unskip} {
    (cf.\ \cref{\seq_use:Nn \g_cflist_pending {,}})\IfValueTF{#1}{#1~}{\unskip}
    \seq_gclear:N \g_cflist_pending
  }
}

\NewDocumentCommand{\ifnocf} { m } {
  \seq_if_empty:NT \g_cflist_pending { #1 }
}

\ExplSyntaxOn
\prop_new:N \l__verbs
\prop_put:Nnn \l__verbs {show} {shows}
\prop_put:Nnn \l__verbs {imply} {implies}
\prop_put:Nnn \l__verbs {demonstrate} {demonstrates}
\prop_put:Nnn \l__verbs {prove} {proves}
\prop_put:Nnn \l__verbs {establish} {establishes}
\prop_put:Nnn \l__verbs {ensure} {ensures}
\prop_put:Nnn \l__verbs {assure} {assures}

\seq_const_from_clist:Nn \g_prove_mru {
	establish,
	demonstrate,
	prove,
	show,
	imply,
	ensure
}

\tl_new:N \g_wordtmp
\seq_new:N \l_mytmps

\cs_generate_variant:Nn \str_if_in:nnTF { nVTF }

\NewDocumentCommand{\prove}{ o }{
	\IfValueTF{#1}{
		\seq_clear:N \l_mytmps
		\seq_map_inline:Nn \g_prove_mru {
			\str_if_eq:nnTF {##1} {ensure} {
				\str_set:Nn \l_temps {n}
			} {
				\str_set:Nx \l_temps {\str_head_ignore_spaces:n {##1}}
			}
			\str_if_in:nVTF {#1} \l_temps {
				\seq_put_right:Nn \l_mytmps {##1}
			} { }
		}
		\seq_get_right:NN \l_mytmps \g_wordtmp
	} {
		\seq_get_right:NN \g_prove_mru \g_wordtmp
	}
	\tl_use:N \g_wordtmp
	\seq_gput_left:NV \g_prove_mru \g_wordtmp
	\seq_gremove_duplicates:N \g_prove_mru
}

\NewDocumentCommand{\proves}{ o }{
	\IfValueTF{#1}{
		\seq_clear:N \l_mytmps
		\seq_map_inline:Nn \g_prove_mru {
			\str_if_eq:nnTF {##1} {ensure} {
				\str_set:Nn \l_temps {n}
			} {
				\str_set:Nx \l_temps {\str_head_ignore_spaces:n {##1}}
			}
			\str_if_in:nVTF {#1} \l_temps {
				\seq_put_right:Nn \l_mytmps {##1}
			} { }
		}
		\seq_get_right:NN \l_mytmps \g_wordtmp
	} {
		\seq_get_right:NN \g_prove_mru \g_wordtmp
	}
	\str_set:NV \l_tmpa_str \g_wordtmp
	\prop_get:NVN \l__verbs \l_tmpa_str \l_tmpa_tl
	\tl_use:N \l_tmpa_tl
	\seq_gput_left:NV \g_prove_mru \g_wordtmp
	\seq_gremove_duplicates:N \g_prove_mru
}

\ExplSyntaxOff

\NewDocumentEnvironment{cproof}{m}
{\begin{proof}[Proof of \cref{#1}]}%
{\noindent The proof of \cref{#1} is thus complete.
\end{proof}}

\NewDocumentEnvironment{cproof2}{m}
{\begin{proof}[Proof of \cref{#1}]}%
{\noindent This completes the proof of \cref{#1}.
\end{proof}}

\ExplSyntaxOff

\ExplSyntaxOn
\global\def\loc{dummy}

\NewDocumentEnvironment {athm} {m m o} {%
		\IfValueT{#3}{\begin{#1}[#3]}
			\IfValueF{#3}{\begin{#1}}
				\label{#2}\global\def\loc{#2}%
			}{%
			\end{#1}%
		}
		
		\NewDocumentEnvironment{aproof} {} {%
			\begin{proof}[Proof~of~\cref{\loc}]%
			}{%
				\global\def\loc{dummy}\end{proof}%
		}
		
		\ExplSyntaxOff

		\newcommand{\lref}[1]{\cref{\loc.#1}}
		
		\newcommand{\llabel}[1]{\label{\loc.#1}}
		
		\newcommand{\finishproofthus}{The proof of \cref{\loc} is thus complete.}

	\ExplSyntaxOn

	\bool_new:N \g_arg_start_bool
	\bool_gset_true:N \g_arg_start_bool
	
	\NewDocumentCommand{\startnewargseq}{}{\bool_gset_true:N \g_arg_start_bool \tl_set:Nn \g_label_tl {}}
	
	\cs_generate_variant:Nn \seq_if_in:NnTF {NxTF}
	\cs_generate_variant:Nn \seq_remove_all:Nn {Nx}
	
	\int_new:N \l_random_int
	
	\cs_generate_variant:Nn \tl_if_head_eq_catcode:nNTF {oNTF}
	
	\cs_generate_variant:Nn \tl_if_head_eq_catcode:nNTF {eNTF}
	\cs_generate_variant:Nn \tl_log:n {o}
	\cs_generate_variant:Nn \tl_log:n {f}
	\cs_generate_variant:Nn \tl_log:n {x}
	\cs_generate_variant:Nn \tl_log:n {e}
	
	\cs_generate_variant:Nn \tl_if_in:nnTF {onTF}
	\cs_generate_variant:Nn \tl_if_in:NnTF {NeTF}
	
	\cs_generate_variant:Nn \tl_if_head_eq_meaning:nNTF {VNTF}

	\seq_const_from_clist:Nn \g_arg_mru_this {
		Ahpr,
		Tapr,
		Ctapr,
		H
	}
	
	\seq_const_from_clist:Nn \g_arg_mru_nothis {
		Ia,
		Nwc,
		N,
		Itns,
		Fm,
		Itnswc,
		Mo
	}

	\seq_new:N \l_arg_seq
	\tl_new:N \l_cons_tl
	\tl_new:N \l_dummy_tl

	\bool_new:N \g_debug_bool
	\bool_gset_false:N \g_debug_bool
	
	\bool_new:N \l_insidearg_bool
	
	\bool_new:N \g_firstargletter_bool
	
	\sys_gset_rand_seed:n {0903}
	
	\bool_new:N \l_plural_bool
	\tl_new:N \l_arg_verbs_tl

	\NewDocumentCommand{\argument}{mom}{
		\bool_set_false:N \l_plural_bool
		\tl_set:Nn \l_arg_verbs_tl {sindep}
		\keys_define:nn { benno/argument } {
			plural .value_forbidden:n = true,
			plural .code:n = {\bool_set_true:N \l_plural_bool},
			verbs .value_required:n = false,
			verbs .tl_set:N = \l_arg_verbs_tl,
		}
		\IfValueT{#2}{
			\keys_set:nn { benno/argument } {#2}
		}
		\bool_log:N \l_plural_bool
		\bool_gset_true:N \l_insidearg_bool
		\seq_set_split:Nnn \l_arg_seq ; {#1}
		\seq_remove_all:Nn \l_arg_seq { }
		\seq_log:N \l_arg_seq
		\tl_set:Nn \l_cons_tl {#3}
		\tl_trim_spaces:N \l_cons_tl
		\seq_if_in:NxTF \l_arg_seq {\lref{\g_label_tl}} {
			\seq_remove_all:Nx \l_arg_seq {\lref{\g_label_tl}}
			\seq_get_left:NNTF \l_arg_seq \l_dummy_tl {
				\tl_trim_spaces:N \l_dummy_tl
				\bool_gset_false:N \g_firstargletter_bool
				\tl_if_head_eq_catcode:VNTF \l_dummy_tl a {
					\bool_gset_true:N \g_firstargletter_bool
				} {
					\tl_if_head_eq_meaning:VNTF \l_dummy_tl {\cref} {
						\tl_set:Nx \l_tmpa_tl {\tl_tail:N \l_dummy_tl}
						\tl_set:Nx \l_tmpb_tl {\tl_head:N \l_tmpa_tl}
						\bool_gset_true:N \g_firstargletter_bool
						\tl_if_in:NeTF \l_tmpb_tl {lem\c_colon_str} {} {
							\tl_if_in:NeTF \l_tmpb_tl {thm\c_colon_str} {} {
								\tl_if_in:NeTF \l_tmpb_tl {prop\c_colon_str} {} {
									\tl_if_in:NeTF \l_tmpb_tl {cor\c_colon_str} {} {
										\bool_gset_false:N \g_firstargletter_bool
									}
								}
							}
						}
					} {
					}
				}
				\bool_if:NTF \g_firstargletter_bool {
					\seq_set_eq:NN \l_tmpa_seq \g_arg_mru_this
					\seq_remove_all:Nn \l_tmpa_seq {H}
					\seq_get_right:NN \l_tmpa_seq \l_tmpa_tl
					\int_case:nnF {\seq_count:N \l_arg_seq} {
						{1} {
							\str_case:VnF {\l_tmpa_tl} {
								{Ahpr} {
									\bool_if:NT \g_debug_bool {C1.1}
									\seq_gput_left:Nn \g_arg_mru_this {Ahpr}
									\seq_gremove_duplicates:N \g_arg_mru_this
									\enum:nxnn {#1} {\lref{\g_label_tl}} {-} {\BooleanTrue}
									\hence~
									\bool_if:NTF \l_plural_bool {
										\prove[\l_arg_verbs_tl]~\ignorespaces #3
									} {
										\proves[\l_arg_verbs_tl]~\ignorespaces #3
									}
								}
								{Tapr} {
									\bool_if:NT \g_debug_bool {C1.2}
									\seq_gput_left:Nn \g_arg_mru_this {Tapr}
									\seq_gremove_duplicates:N \g_arg_mru_this
									\enum[\lref{\g_label_tl}]{
										This;
										#1
									}[\l_arg_verbs_tl]\ignorespaces #3
								}
								{Ctapr} {
									\bool_if:NT \g_debug_bool {C1.3}
									\seq_gput_left:Nn \g_arg_mru_this {Ctapr}
									\seq_gremove_duplicates:N \g_arg_mru_this
									Combining~
									\enum[\lref{\g_label_tl}]{
										this;
										#1
									} \proves[\l_arg_verbs_tl]~\ignorespaces #3
								}
							} {}
						}
					} {
						\str_case:VnF {\l_tmpa_tl} {
							{Ahpr} {
								\bool_if:NT \g_debug_bool {C2.1}
								\seq_gput_left:Nn \g_arg_mru_this {Ahpr}
								\seq_gremove_duplicates:N \g_arg_mru_this
								\enum:nxnn {#1} {\lref{\g_label_tl}} {-} {\BooleanTrue}
								\hence~
								\prove[\l_arg_verbs_tl]~\ignorespaces #3
							}
							{Tapr} {
								\bool_if:NT \g_debug_bool {C2.2}
								\seq_gput_left:Nn \g_arg_mru_this {Tapr}
								\seq_gremove_duplicates:N \g_arg_mru_this
								\enum[\lref{\g_label_tl}]{
									This;
									#1
								}[\l_arg_verbs_tl]\ignorespaces #3
							}
							{Ctapr} {
								\int_case:nn {\int_rand:nn {0} {1}} {
									{0} {
										\bool_if:NT \g_debug_bool {C2.3}
										\seq_gput_left:Nn \g_arg_mru_this {Ctapr}
										\seq_gremove_duplicates:N \g_arg_mru_this
										Combining~
										\enum[\lref{\g_label_tl}]{
											this;
											#1
										} \proves[\l_arg_verbs_tl]~\ignorespaces #3
									}
									{1} {
										\bool_if:NT \g_debug_bool {C2.4}
										\seq_gput_left:Nn \g_arg_mru_this {Ctapr}
										\seq_gremove_duplicates:N \g_arg_mru_this
										Combining~
										\enum:nxnn {#1} {\lref{\g_label_tl}} {-} {\BooleanFalse}
										\hence~
										\proves[\l_arg_verbs_tl]~\ignorespaces #3
									}
								}
							}
						} {}
					}
				} {
					\seq_set_eq:NN \l_tmpa_seq \g_arg_mru_this
					\seq_remove_all:Nn \l_tmpa_seq {H}
					\seq_remove_all:Nn \l_tmpa_seq {Ahpr}
					\seq_get_right:NN \l_tmpa_seq \l_tmpa_tl
					\int_case:nnF {\seq_count:N \l_arg_seq} {
						{1} {
							\str_case:VnF {\l_tmpa_tl} {
								{Tapr} {
									\bool_if:NT \g_debug_bool {C3.1}
									\seq_gput_left:Nn \g_arg_mru_this {Tapr}
									\seq_gremove_duplicates:N \g_arg_mru_this
									\enum[\lref{\g_label_tl}]{
										This;
										#1
									}[\l_arg_verbs_tl]\ignorespaces #3
								}
								{Ctapr} {
									\bool_if:NT \g_debug_bool {C3.2}
									\seq_gput_left:Nn \g_arg_mru_this {Ctapr}
									\seq_gremove_duplicates:N \g_arg_mru_this
									Combining~
									\enum[\lref{\g_label_tl}]{
										this;
										#1
									} \proves[\l_arg_verbs_tl]~\ignorespaces #3
								}
							} {}
						}
					} {
						\str_case:VnF {\l_tmpa_tl} {
							{Tapr} {
								\bool_if:NT \g_debug_bool {C4.1}
								\seq_gput_left:Nn \g_arg_mru_this {Tapr}
								\seq_gremove_duplicates:N \g_arg_mru_this
								\enum[\lref{\g_label_tl}]{
									This;
									#1
								}[\l_arg_verbs_tl]\ignorespaces #3		
							}
							{Ctapr} {
								\int_case:nn {\int_rand:nn {0} {1}} {
									{0} {
										\bool_if:NT \g_debug_bool {C4.2}
										\seq_gput_left:Nn \g_arg_mru_this {Ctapr}
										\seq_gremove_duplicates:N \g_arg_mru_this
										Combining~
										\enum[\lref{\g_label_tl}]{
											this;
											#1
										} \proves[\l_arg_verbs_tl]~\ignorespaces #3		
									}
									{1} {
										\bool_if:NT \g_debug_bool {C4.3}
										\seq_gput_left:Nn \g_arg_mru_this {Ctapr}
										\seq_gremove_duplicates:N \g_arg_mru_this
										Combining~
										\enum:nxnn {#1} {\lref{\g_label_tl}} {-} {\BooleanFalse}
										\hence~
										\proves[\l_arg_verbs_tl]~\ignorespaces #3    
									}
								}
							}
						} {}
					}
				}
			} {
				\tl_if_head_eq_catcode:oNTF \l_cons_tl a {
					\seq_set_eq:NN \l_tmpa_seq \g_arg_mru_this
					\seq_remove_all:Nn \l_tmpa_seq {Ctapr}
					\seq_remove_all:Nn \l_tmpa_seq {Ahpr}
					\seq_get_right:NN \l_tmpa_seq \l_tmpa_tl
					\str_case:VnF {\l_tmpa_tl} {
						{H} {
							\bool_if:NT \g_debug_bool {C5.1}
							\seq_gput_left:Nn \g_arg_mru_this {H}
							\seq_gremove_duplicates:N \g_arg_mru_this
							Hence,~we~obtain~\ignorespaces #3
						}
						{Tapr} {
							\bool_if:NT \g_debug_bool {C5.2}
							\seq_gput_left:Nn \g_arg_mru_this {Tapr}
							\seq_gremove_duplicates:N \g_arg_mru_this
							This~\proves[\l_arg_verbs_tl]~\ignorespaces #3
						}
					} {}
				} {
					\bool_if:NT \g_debug_bool {C6.1}
					\seq_gput_left:Nn \g_arg_mru_this {Tapr}
					\seq_gremove_duplicates:N \g_arg_mru_this
					This~\proves[\l_arg_verbs_tl]~\ignorespaces #3
				}
			} 
		} {
			\int_compare:nNnTF {\seq_count:N \l_arg_seq} = {0} {
				\bool_if:NTF \g_arg_start_bool {
					\bool_if:NT \g_debug_bool {C7.1}
					\Nobs\unskip
					#3
				} {
					\bool_if:NT \g_debug_bool {C7.2}
					\Moreover~
					#3
				}
			} {
				\bool_if:NTF \g_arg_start_bool {
					\bool_if:NT \g_debug_bool {C8.1}
					\tl_log:N \l_arg_verbs_tl
					\Nobs~that~
					\enum{
						#1
					}[\l_arg_verbs_tl]\ignorespaces #3
				} {
					\int_compare:nNnTF {\seq_count:N \l_arg_seq} = {1} {
						\seq_set_eq:NN \l_tmpa_seq \g_arg_mru_nothis
						\seq_remove_all:Nn \l_tmpa_seq {Nwc}
						\seq_remove_all:Nn \l_tmpa_seq {Itnswc}
						\seq_get_right:NN \l_tmpa_seq \l_tmpa_tl
					} {
						\seq_get_right:NN \g_arg_mru_nothis \l_tmpa_tl
					}
					\str_case:VnF {\l_tmpa_tl} {
						{Mo} {
							\bool_if:NT \g_debug_bool {C9.1}
							\seq_gput_left:Nn \g_arg_mru_nothis {Mo}
							\seq_gremove_duplicates:N \g_arg_mru_nothis
							Moreover,~\nobs~that~
							\enum{
								#1
							}[\l_arg_verbs_tl]\ignorespaces #3		
						}
						{Fm} {
							\bool_if:NT \g_debug_bool {C9.2}
							\seq_gput_left:Nn \g_arg_mru_nothis {Fm}
							\seq_gremove_duplicates:N \g_arg_mru_nothis
							Furthermore,~\nobs~that~
							\enum{
								#1
							}[\l_arg_verbs_tl]\ignorespaces #3		
						}
						{Ia} {
							\bool_if:NT \g_debug_bool {C9.3}
							\seq_gput_left:Nn \g_arg_mru_nothis {Ia}
							\seq_gremove_duplicates:N \g_arg_mru_nothis
							In~addition,~\nobs~that~
							\enum{
								#1
							}[\l_arg_verbs_tl]\ignorespaces #3		
						}
						{N} {
							\bool_if:NT \g_debug_bool {C9.4}
							\seq_gput_left:Nn \g_arg_mru_nothis {N}
							\seq_gremove_duplicates:N \g_arg_mru_nothis
							Next,~\nobs~that~
							\enum{
								#1
							}[\l_arg_verbs_tl]\ignorespaces #3		
						}
						{Itns} {
							\bool_if:NT \g_debug_bool {C9.5}
							\seq_gput_left:Nn \g_arg_mru_nothis {Itnswc}
							\seq_gput_left:Nn \g_arg_mru_nothis {Itns}
							\seq_gremove_duplicates:N \g_arg_mru_nothis
							In~the~next~step~we~\nobs~that~
							\enum{
								#1
							}[\l_arg_verbs_tl]\ignorespaces #3		
						}
						{Nwc} {
							\bool_if:NT \g_debug_bool {C9.6}
							\seq_gput_left:Nn \g_arg_mru_nothis {Nwc}
							\seq_gremove_duplicates:N \g_arg_mru_nothis
							Next~we~combine~
							\enum{
								#1
							}to~obtain~\ignorespaces #3
						}
						{Itnswc} {
							\bool_if:NT \g_debug_bool {C9.7}
							\seq_gput_left:Nn \g_arg_mru_nothis {Itns}
							\seq_gput_left:Nn \g_arg_mru_nothis {Itnswc}
							\seq_gremove_duplicates:N \g_arg_mru_nothis
							In~the~next~step~we~combine~
							\enum{
								#1
							}to~obtain~\ignorespaces #3
						}
					} {}
				}
			}
		}
		\bool_gset_false:N \g_arg_start_bool
		\bool_gset_false:N \l_insidearg_bool
		\cfload[.]
	}
	
	\tl_new:N \g_label_tl
	\tl_gset:Nn \g_label_tl { }
	
	\NewDocumentCommand{\savelabel}{m}{
		\bool_if:NTF \l_insidearg_bool {
			\tl_gset:Nn \g_label_tl {#1}
		} {
			\tl_gset:Nn \g_label_tl { }
		}
	}
	
	\ExplSyntaxOff

\title{Non-convergence of Adam and other adaptive stochastic gradient descent optimization methods for non-vanishing learning rates}

\author{Steffen Dereich$^{1}$, Robin Graeber$^{2}$, and Arnulf Jentzen$^{3,4}$
	\bigskip
	\\
	\small{$^1$ Applied Mathematics: Institute for Mathematical Stochastics, Faculty of Mathematics and}
	\vspace{-0.1cm}\\
	\small{Computer Science, University of M{\"u}nster, Germany, e-mail: \texttt{steffen.dereich@uni-muenster.de}}
	\smallskip
	\\
	\small{$^2$ School of Data Science,	The Chinese University of Hong Kong, Shenzhen}
	\vspace{-0.1cm}\\
	\small{(CUHK-Shenzhen), China, e-mail:
		\texttt{223040041@link.cuhk.edu.cn}}
	\smallskip
	\\
\small{$^3$ School of Data Science, The Chinese University of Hong Kong, Shenzhen}
\vspace{-0.1cm}\\
\small{(CUHK-Shenzhen), China, e-mail: \texttt{ajentzen@cuhk.edu.cn}}
\smallskip
	\\
	\small{$^4$ Applied Mathematics: Institute for Analysis and Numerics, Faculty of Mathematics and}
	\vspace{-0.1cm}\\
	\small{Computer Science, University of M{\"u}nster, Germany, e-mail: \texttt{ajentzen@uni-muenster.de}}
	\smallskip
	\\
}

\date{\today}

\begin{document}

\maketitle

\begin{abstract}
	Deep learning (DL) approximation algorithms -- typically consisting of a class of deep neural networks (DNNs) trained by a stochastic gradient descent (SGD) optimization method -- are nowadays the key ingredients in many artificial intelligence (AI) systems and have revolutionized our ways of working and living in modern societies.

It is known that the standard SGD method fails to converge even in the situation of several convex optimization problems if the learning rates do not converge to zero, for example, in the case where the learning rates are kept constant. 
However, in many practical relevant training scenarios, often not the  standard SGD method but instead \emph{adaptive SGD optimization methods} such as 
the Adamax method,
the AMSGrad method,
the RMSprop method,
and the famous Adam method are considered.

This naturally raises the question of whether such \emph{adaptive} optimizers, in which the learning rates are modified adaptively during the training process, \emph{do converge} 
in the situation where the learning rates do not converge to zero. 
In this work we answer this question negatively by showing that, when
the mini-batch sizes are bounded and the learning rates do not converge
to zero, the iterates of each of these optimizers can fail to converge to any random vector with respect to the mean squared error.
For Adam we also give conditions under which every subsequence of a given component does not converge to any real-valued random variable with respect to the mean absolute error.
\end{abstract}


\vspace{8pt}


\newpage
\tableofcontents

\newcommand{\inc}{\Phi}
\newcommand{\pars}{\mathfrak{d}}
\renewcommand{\dim}{d}
\newcommand{\DIM}{\delta}
\newcommand{\mom}[1][]{\mathbb{M}_{#1}}
\newcommand{\MOMalt}[1][]{\textbf{m}_{#1}}
\newcommand{\MOMALT}[1][]{\mathbb{V}_{#1}}
\newcommand{\MOM}{\mathbb{V}}
\newcommand{\MOMs}[1][]{\mathfrak{M}_{#1}}
\newcommand{\UB}{\mathbb{L}}
\newcommand{\BSCL}{R}
\newcommand{\betaalt}{\sigma}
\newcommand{\ascl}{\lambda}
\newcommand{\bscl}{\kappa}
\newcommand{\Cst}{\rho}
\newcommand{\cst}{\eta}
\newcommand{\CST}{B}
\newcommand{\batch}{J}
\newcommand{\Batch}{K}
\newcommand{\BATCH}{L}
\newcommand{\loss}{\ell}
\newcommand{\fil}{\mathbb{F}}
\newcommand{\gil}{\mathcal{G}}
\newcommand{\uv}{e}
\newcommand{\call}{K}
\newcommand{\mm}{m}

\newpage
\section{Introduction}
\label{section:introduction}

\DL\ approximation algorithms -- typically consisting of a class of deep \ANNs\ trained by a \SGD\ optimization method -- are nowadays the key ingredients in many \AI\ systems and have revolutionized our ways of working and living in modern societies.
For example, \SGD\ methods are used to train powerful \LLMs\ such as \textsc{GPT-3} (cf.\ \cite{brown2020languagemodelsfewshotlearners})
and \textsc{Gemini} (cf.\ \cite{geminiteam2024geminifamilyhighlycapable}), \SGD\ methods are employed to create successful generative \AI\ based text-to-image creation models such as \textsc{Midjourney}, 
\textsc{DALL-E} (cf.\  \cite{ramesh2021zeroshottexttoimagegeneration}), and \textsc{Stable Diffusion} 
(cf.\ \cite{esser2024scalingrectifiedflowtransformers}), but \SGD\ methods are also used to train deep \ANNs\ to approximately solve scientific models such as \PDE\ models from physics and biology (cf., \eg, 
\cite{E_2017,Han_2018,li2021fourierneuraloperatorparametric,RAISSI2019686,Sirignano_2018},
the review articles \cite{WhatsNext,Beck_2023,MR4356985,karniadakis2021physics,Blechschmidt},
and the references mentioned therein) and optimal control and stopping problems 
(cf., \eg,
\cite{becker2020deepoptimalstopping,han2016deeplearningapproximationstochastic}, the review articles \cite{germain2021neural,ruf2020neuralnetworksoptionpricing}, and the references mentioned therein) from engineering.

It is well known that the error of the plain vanilla standard \SGD\ method does not converge if the step sizes, the so-called learning rates, do not converge to zero; see, \eg,\! \ \cite[Subsection~7.2.2.2]{jentzen2023mathematical}. 
To better illustrate this elementary fact, we present within this introductory section in the following result, \cref{intro:thm:1} below, a special case of the non-convergence result in Lemma~7.2.13 in \cite[Subsection~7.2.2.2]{jentzen2023mathematical}.
\Cref{intro:thm:1} considers 
the standard \SGD\ method applied to
a very simple exemplary quadratic stochastic optimization problem 
where $\pars\in\N$ represents the dimensionality
of the stochastic optimization problem,
where the data 
of the stochastic optimization problem are represented through $\R^\pars$-valued \iid\ random variables $X_{n,m}\colon\Omega\to\R^\pars$ for $n,m \in \N$ on a probability space $ ( \Omega, \cF, \P ) $ (cf.\ \eqref{Thm1:disp:1} below),
where the learning rates 
of the \SGD\ method are represented by the sequence $\gamma=(\gamma_n)_{n\in\N}\colon\N\to(0,\infty)$ (cf.\ \eqref{Thm1:disp:1} below), and
where the sizes of the mini-batches 
of the \SGD\ method are represented by the sequence $\batch=(\batch_n)_{n\in\N}\colon\N\to\N$ (cf.\ \eqref{Thm1:disp:1} below).

\begin{athm}{theorem}{intro:thm:1}[Non-convergence of \SGD]
	\begin{boxedthm}
	Let $\pars\in\N$, $\xi\in\R^{\pars}$,
	let $ ( \Omega, \cF, \P ) $ be a probability space,
	let	$
	X_{n,\mm}\colon \Omega \to \R^{\pars}
	$, $(n,\mm)\in\N^2$, be \iid\ random variables,
	let 
	$\loss\colon\R^{\pars}\times\R^{\pars}\to\R$,
	$\batch\colon\N\to\N$, and $\gamma\colon\N\to(0,\infty)$ satisfy\footnotemark\
	for all 
	$\theta,x\in\R^\pars$
	that
	\begin{equation}
		\label{Thm1:disp:1}
		\textstyle
		\loss(\theta,x)=\norm{\theta-x}^2
		\qqandqq
		\limsup_{n\to\infty}\PRb{(\gamma_n)^2(\batch_n)^{-1}}>0,
	\end{equation}
	let
	$ \Theta=(\Theta^{(1)},\ldots,\Theta^{(\pars)}) \colon\N_0\times\Omega \to \R^{\pars}$ satisfy\footnotemark\
	for all $n\in\N$ that
	\begin{equation}
		\label{Thm1:disp:2}
		\begin{split}
			\textstyle
			\Theta_0=\xi
			\qqandqq
			\Theta_{ n }
			=\Theta_{ n-1  }-\gamma_n\PRb
			{\frac{1}{\batch_n}\sum_{\mm=1}^{\batch_n}\pr{\nabla_{\theta}\loss}(\Theta_{n-1},X_{n,\mm})},
		\end{split}
	\end{equation}
	and
	assume $\E\PR{\norm{X_{1,1}}}<\infty$ and $\operatorname{Trace}( \operatorname{Cov}( X_{ 1, 1 } ) ) > 0$.
	Then 
			\begin{equation}
		\label{Thm1:disp:result}
		\begin{split}
			\inf_{\vartheta\in\R^{\pars}}\PRbbb{\limsup_{ n \to \infty }\E\PRb{\norm{ \Theta_n - \vartheta }^2 }}
			>0
			.
		\end{split}
	\end{equation}
		\end{boxedthm}
\end{athm}
\footnotetext[1]{\MakeLinkTarget*{Hfootnote.1}Note that for all $d\in\N$, $v=(v_1,\dots,v_d)$, $w=(w_1,\dots,w_d)\in\R^d$ it holds that
	$\langle v,w \rangle=\sum_{i=1}^d v_i w_i$ and $\norm{v}=\pr{\langle v,v \rangle}^{\nicefrac{1}{2}}$.}
\footnotetext{Note that for all $\pars,d\in\N$, every differentiable $\loss
	\colon\R^{\pars}\times\R^{d}\to\R
	$, 
	and all 
	$\theta=(\theta_1,\dots,\theta_\pars)\in\R^\pars$,
	$x=(x_1,\dots,x_d)\in\R^d$
	it holds that
		$(\nabla\loss)(\theta,x)
		=\prb{\frac{\partial\loss}{\partial\theta_1}(\theta,x),\dots,\frac{\partial\loss}{\partial\theta_{\pars}}(\theta,x),\frac{\partial\loss}{\partial x_1}(\theta,x),\dots,\frac{\partial\loss}{\partial x_{d}}(\theta,x)}$
		and $(\nabla_{\theta}\loss)(\theta,x)=\prb{\frac{\partial\loss}{\partial\theta_1}(\theta,x),\dots,\frac{\partial\loss}{\partial\theta_{\pars}}(\theta,x)}$.}
\Cref{intro:thm:1} is an immediate consequence of Lemma~7.2.13 in \cite[Subsection~7.2.2.2]{jentzen2023mathematical}. 
\Cref{intro:thm:1} considers the stochastic optimization problem to minimize the function $\R^\pars\ni\theta\mapsto \E\PR{\ell\pr{\theta,X_{1,1}}}\in[0,\infty]$
(with $\ell$ specified in \eqref{Thm1:disp:1} above).
For this optimization problem \cref{intro:thm:1} ensures that the standard \SGD\ method in \eqref{Thm1:disp:2} fails to converge to any possible point $\vartheta\in\R^\pars$ if the product of the squared learning rates $\gamma=(\gamma_n)_{n\in\N}\colon\N\to(0,\infty)$ and the reciprocal of the mini-batch sizes $\batch=(\batch_n)_{n\in\N}\colon\N\to\N$ in \eqref{Thm1:disp:1} do not converge to zero (cf.\ \eqref{Thm1:disp:1} above).
%


\newpage
\subsection{Non-convergence of Adam}
\label{subsection:Adam}

In many practical relevant training scenarios, often not the standard \SGD\ method (cf.\ \eqref{Thm1:disp:2} above) but instead adaptive \SGD\ methods such as 
the Adamax (cf.\ \cite{KingmaBa2024_Adam}),
the AMSgrad (cf.\ \cite{ReddiKale2019}),
the \RMSprop\ (cf.\ \cite{Hinton24_RMSprop}),
and the popular \ADAM\ (cf.\ \cite{KingmaBa2024_Adam})
optimizers, in which the learning rates are modified adaptively during the training process, are employed (for details and references on further variations of \SGD\ optimization methods 
we also refer to the overview articles \cite{ruder2017overviewgradientdescentoptimization,Sun2019} 
and the monograph \cite{jentzen2023mathematical}). 
This naturally raises the question of whether such adaptive optimizers, in which the learning rates are modified adaptively during the training process, do converge in the situation of non-vanishing learning rates.
In this work we answer this question negatively by proving that adaptive \SGD\ methods such as the popular \ADAM\ optimizer (cf.\ \cite{KingmaBa2024_Adam}) can fail to converge to any possible random point if the learning rates do not converge to zero.
Specifically, \cref{cor:limsup:simple} in \cref{sec:results} below, which is a main result of this work, shows under suitable assumptions that various adaptive optimizers, such as \ADAM, \RMSprop, and Adamax, fail to converge to any random point $\vartheta\colon\Omega\to\R^\pars$
if the sizes of the mini-batches are bounded from above and if the learning rates do not converge to zero.
To better illustrate the contribution of this work, within this introductory section, we now specialize the conclusion of \cref{cor:limsup:simple} to the situation in which \ADAM\ is applied to the very simple exemplary quadratic stochastic optimization problem in \eqref{Thm1:disp:1} from \cref{intro:thm:1} above.

\begin{athm}{theorem}{whole:vec:thm:sup}[Non-convergence of \ADAM]
	\begin{boxedthm}
	Let $\pars\in\N$, 
	$\xi\in\R^{\pars}$,
	$a\in\R$,
	$b\in(a,\infty)$,
	$\eps\in(0,\infty)$,
	$\alpha,\beta\in[0,1)$,
	let $ ( \Omega, \cF, \P ) $ be a probability space,
	let	$
	X_{n,m}\colon \Omega \to [ a, b ]^{\pars}
	$, $(n,m)\in\N^2$, be \iid\ random variables,
	let $\loss\colon\R^{\pars}\times\R^{\pars}\to\R$,
	$\batch\colon\N\to\N$, and $\gamma\colon\N\to[0,\infty)$ satisfy
	for all 
	$\theta,x\in\R^\pars$
	that
	\begin{equation}
		\label{Thm2:disp:4}
		\textstyle
		\loss(\theta,x)=\norm{\theta-x}^2
		\qqandqq
		\limsup_{n\to\infty}\PRb{(\gamma_n)^2(\batch_n)^{-1}}>0,
	\end{equation} 
	let
	$ \Theta=(\Theta^{(1)},\ldots,\Theta^{(\pars)}) \colon\N_0\times\Omega \to \R^{\pars}$,
	$\mom=(\mom^{(1)},\ldots,\mom^{(\pars)})\colon\N_0\times\Omega\to\R^{\pars}$, and 
	$\MOM=(\MOM^{(1)},\ldots,\MOM^{(\pars)})\colon\N_0\times\Omega\to\R^{\pars}$ satisfy
	for all $n\in\N$, $i\in\{1,2,\dots,\pars\}$ that
	\begin{equation}
		\label{eq:whole:vec:2.1}
		\begin{split}
						\mom[0]=0, 
						\qquad 
			\mom[n]&\textstyle=\alpha \mom[n-1]+(1-\alpha)\PRb
			{\frac{1}{\batch_n}\sum_{m=1}^{\batch_n}\pr{\nabla_{\theta}\loss}(\Theta_{n-1},X_{n,m})},
		\end{split}
	\end{equation}
	\begin{equation}
		\label{eq:whole:vec:2.2}
		\begin{split}
						\MOM^{(i)}_0=0,
						\qquad
			\MOM_{n}^{(i)}&\textstyle=\beta \MOM_{n-1}^{(i)}+(1-\beta)\PRb
			{\frac{1}{\batch_n}\sum_{m=1}^{\batch_n}\prb{\frac{\partial\loss}{\partial\theta_i}}(\Theta_{n-1},X_{n,m})}^2,
		\end{split}
	\end{equation}
	\begin{equation}
		\label{Thm2:disp:7}
		\begin{split}
			\Theta_{ 0 }=\xi,
			\qqandqq
			\Theta_{ n }^{(i)}
			&=\Theta_{ n-1  }^{(i)}-\gamma_n \prb{\eps+\PRb{(1-\beta^n)^{-1} \MOM_{n}^{(i)}}^{\nicefrac{1}{2}}}^{-1}\mom[n]^{(i)},
		\end{split}
	\end{equation}
	and
	assume $\operatorname{Trace}( \operatorname{Cov}( X_{ 1, 1 } ) ) > 0$.
	Then 
	\begin{equation}
		\label{Thm2:disp:result}
		\begin{split}
			\inf_{\substack{\vartheta\colon \Omega \to \R^{\pars}\\\text{measurable}}}\PRbbb{\limsup_{ n \to \infty }\E\PRb{\norm{ \Theta_n - \vartheta }^2 }}
			>0
			.
		\end{split}
	\end{equation}
	\end{boxedthm}
\end{athm}

\Cref{whole:vec:thm:sup}
is a direct consequence of
\cref{prop:prop:non_convergence_modified_Adam:easy:mat}	in \cref{section:4} below.
\Cref{prop:prop:non_convergence_modified_Adam:easy:mat}, in turn,
follows from
\cref{cor:limsup:simple}, which is a main result of this article.
\Cref{whole:vec:thm:sup}, as well as \cref{intro:thm:1}, considers the stochastic optimization problem to minimize the function $\R^\pars\ni\theta\mapsto \E\PR{\ell\pr{\theta,X_{1,1}}}\in\R$
(with $\ell$ specified in \eqref{Thm2:disp:4} above).
For this optimization problem \cref{whole:vec:thm:sup} shows that \ADAM, which is set up in \eqref{eq:whole:vec:2.1}, \eqref{eq:whole:vec:2.2}, and \eqref{Thm2:disp:7}, fails to converge to any random vector $\vartheta\colon\Omega\to\R^\pars$ if the product of the squared learning rates $\gamma=(\gamma_n)_{n\in\N}\colon\N\to[0,\infty)$ and the reciprocal of the mini-batch sizes $\batch=(\batch_n)_{n\in\N}\colon\N\to\N$ in \eqref{Thm2:disp:4} do not converge to zero (cf.\ \eqref{Thm2:disp:4} above).
We note that \cref{whole:vec:thm:sup} covers \ADAM\ with the first moment bias adjustment although it is not explicitly stated as such.
To illustrate this issue better to the reader we also present the corresponding result in \cref{prop:weak:opt:error:non-cvg}.

Moreover, \cref{thm:new:rob:cond} in \cref{sec:results} below, which is
the other main result of this work, shows under suitable assumptions
that each component satisfying the respective hypotheses fails to
converge in mean absolute error to any real-valued random variable
$\vartheta\colon\Omega\to\R$ if the mini-batch sizes and the learning
rates are bounded from above and the learning rates do not converge
to zero. This applies, in particular, to \ADAM, \RMSprop, and AMSGrad.
For \ADAM, \cref{cor:ADAM} also establishes non-accumulation in this
sense if the learning rates are asymptotically bounded away from zero
and the mini-batch sizes are at least two asymptotically.

\subsection{Non-convergence of Adam in the training of deep artificial neural networks (ANNs)}
\label{subsection:ANNs}

Despite the fact that \cref{whole:vec:thm:sup} shows non-convergence of \ADAM\ in the simple convex setting in \cref{Thm2:disp:4}, the main success of \ADAM\ is based on its superior performance when employed in the training of deep \ANNs.
This raises the question of whether there is a simple setting in which \ADAM\ fails to converge when training a deep \ANN.
In \cref{thm:1.3} we answer this question affirmatively by showing that, when \ADAM\ is employed to learn an input-output relation represented by \iid\ random variables $X_{n,m}\colon\Omega\to\pr{\R^{\dim}\times\R^{\DIM}}$ for $n,m\in\N$ it does not converge to
any random point $\vartheta\colon\Omega\to\R^\pars$ 
if the optimal risk is strictly positive and the product of the squared learning rates $\gamma=(\gamma_n)_{n\in\N}\colon\N\to[0,\infty)$ and the reciprocal of the mini-batch sizes $\batch=(\batch_n)_{n\in\N}\colon\N\to\N$ in \eqref{eq:thm:1.3.b} does not converge to zero (cf.\ \eqref{eq:thm:1.3.b} below).

\begin{athm}{theorem}{thm:1.3}[Non-convergence of \ADAM\ in the training of deep \ANNs]
	\begin{boxedthm}
	Let $\dim,\DIM,\pars,L\in\N$, 
	$\xi\in\R^\pars$,
	$\eps,\fc\in(0,\infty)$,
	$l_0,l_1,l_2,\ldots\in\N$,
	$s_0,s_1,s_2,\ldots\in\N_0$
	satisfy $l_0=\dim$, $l_L=\DIM$, $s_L=\pars$, and $\sum_{k=0}^L \vass{s_k-\sum_{j=1}^k l_j(l_{j-1}+1)}=0$,
	let $\act\in C^1(\R,\R)$ satisfy for all $x\in\R$ that $\vass{\act'(x)}\leq\fc(1+\vass{x})^{\fc}$,
	for every $\theta=(\theta_1,\dots,\theta_{\pars})\in\R^{\pars}$
	let 
	$\realization{\theta}_{k}=(\realization{\theta}_{k,1},\dots,\realization{\theta}_{k,l_k})\colon\R^\dim\to\R^{l_k}$, $k\in\N_0$, satisfy for all $k\in\{0,1,\dots,L-1\}$, $x=(x_1,\dots,x_\dim)\in\R^\dim$, $j\in\{1,2,\dots,l_{k+1}\}$ that
	\begin{equation}
		\label{eq:thm:1.3.a}
		\realization{\theta}_{k+1,j}(x)
		=
		\theta_{j+l_{k+1}l_{k}+s_{k}}
		+
		\sum_{i=1}^{l_{k}}
		\theta_{i+(j-1) l_{k}+s_{k}}
		\prb{
			\act\prb{\realization{\theta}_{k,i}(x)}\indicator{\N}(k)
			+x_{\min\{i,d\}}\indicator{\{0\}}(k)
		},
	\end{equation}
		let $\loss
	\colon\R^{\pars}\times\pr{\R^{\dim}\times\R^\DIM}\to\R$, 
	$\batch\colon\N\to\N$, and
	$\gamma\colon\N\to[0,\infty)$
	satisfy
	for all 
	$\theta\in\R^\pars$, 
	$v\in\R^\dim$,
	$w\in\R^\DIM$
	that
	\begin{equation}
		\label{eq:thm:1.3.b}	
		\textstyle
		\loss(\theta,(v,w))=\norm{\realization{\theta}_{L}(v)-w}^2
		\qqandqq
		\limsup_{n\to\infty}\PRb{(\gamma_n)^2\pr{\batch_n}^{-1}}>0,
	\end{equation} 
	let $ ( \Omega, \cF, \P ) $ be a probability space,
	for every $n,m\in\N$ let $X_{n,m}\colon\Omega \allowbreak\to \pr{\R^{\dim}\times\R^{\DIM}}$
	be a bounded random variable,
	let 
	$\alpha,\beta\in[0,1)$,
	let
	$ \Theta=(\Theta^{(1)},\ldots,\Theta^{(\pars)}) \colon\N_0\times\Omega \to \R^{\pars}$,
	$\mom=(\mom^{(1)},\ldots,\mom^{(\pars)})\colon\N_0\times\Omega\to\R^{\pars}$, and 
	$\MOM=(\MOM^{(1)},\ldots,\MOM^{(\pars)})\colon\N_0\times\Omega\to\R^{\pars}$ satisfy
	for all $n\in\N$, $i\in\{1,2,\dots,\pars\}$ that
	\begin{equation}
		\label{eq:thm:1.3.c}
		\begin{split}
												\mom[0]=0, 
									\qquad 
			\mom[n]&\textstyle=\alpha \mom[n-1]+(1-\alpha)\PRb
			{\frac{1}{\batch_{n}}\sum_{m=1}^{\batch_{n}}\pr{\nabla_\theta\loss}(\Theta_{n-1},X_{n,m})},
		\end{split}
	\end{equation}
	\begin{equation}
		\label{eq:thm:1.3.d}
		\begin{split}
						\MOM^{(i)}_0=0,
						\qquad
			\MOM_n^{(i)}&\textstyle=\beta \MOM_{n-1}^{(i)}+(1-\beta)\PRb
			{\frac{1}{\batch_{n}}\sum_{m=1}^{\batch_{n}}\prb{\tfrac{\partial\loss}{\partial\theta_{i}}}(\Theta_{n-1},X_{n,m})}^2,
		\end{split}
	\end{equation}
	\begin{equation}
		\label{eq:thm:1.3.e}
		\begin{split}
			\Theta_{ 0 }=\xi,
			\qqandqq
			\Theta_{ n }^{(i)}
			&=\Theta_{ n-1  }^{(i)}-\gamma_n\prb{\eps+\PRb{(1-\beta^n)^{-1}\MOM_n^{(i)}}^{\nicefrac{1}{2}}}^{-1}
			\mom[n]^{(i)}
			,
		\end{split}
	\end{equation}
	assume $\inf_{\theta\in\R^\pars}
	\E\PR{\loss(\theta,X_{1,1})}
	>0$,
	and 
	assume that $X_{n,m}$, $(n,m)\in\{(i,j)\in\N^2\colon j\leq\batch_i\}$, are \iid\
	Then 
	\begin{equation}
		\label{eq:thm:1.3.g}
		\begin{split}
			\inf_{\substack{\vartheta\colon \Omega \to \R^\pars\\\text{measurable}}}
			\PRbbb{\limsup_{ n \to \infty }\E\PRb{ \norm{ \Theta_n - \vartheta }^2 }}
			>
			0
			.
		\end{split}
	\end{equation}
	\end{boxedthm}
\end{athm}
\Cref{thm:1.3} follows immediately from \cref{ANNS:ADAM} in \cref{section:4} below. \cref{ANNS:ADAM}, in turn, follows from \cref{lem:for:EX:ANNs:1} in \cref{sec:results}. \cref{lem:for:EX:ANNs:1} is a direct consequence of \cref{cor:limsup:simple}.
\Cref{thm:1.3} considers the stochastic optimization problem to minimize the function $\R^\pars\ni\theta\mapsto \E\PR{\ell\pr{\theta,X_{1,1}}}\in\R$
(with $\ell$ specified in \eqref{eq:thm:1.3.b} above).
For this optimization problem \cref{thm:1.3} shows that \ADAM, which is set up in \eqref{eq:thm:1.3.c}, \eqref{eq:thm:1.3.d}, and \eqref{eq:thm:1.3.e}, fails to converge to any random vector $\vartheta\colon\Omega\to\R^\pars$ 
if the optimal risk value is strictly positive and the product of the squared learning rates $\gamma=(\gamma_n)_{n\in\N}\colon\N\to[0,\infty)$ and the reciprocal of the batch sizes $\batch=(\batch_n)_{n\in\N}\colon\N\to\N$ in \eqref{eq:thm:1.3.b} does not converge to zero (cf.\ \eqref{eq:thm:1.3.b} above).

\subsection{Literature review}
\label{subsection:literature}

In the following we provide a very brief review on research findings in the literature related to 
the non-convergence result in \cref{whole:vec:thm:sup} above 
and its generalizations and extensions in \cref{sec:results,section:4}.
Further lower bound, non-convergence, and divergence results for \SGD\ optimization methods can, \eg, be found in 
\cite{CHERIDITO2021101540,jentzen2024nonconvergenceglobalminimizersadam,Lu_2020,gallon2022blowphenomenagradientdescent,ReddiKale2019}. 
In particular, \cite{CHERIDITO2021101540} establishes a strictly positive
lower bound for the expected error in the training of deep \ReLU\ \ANNs\
by \SGD\ if the depth is sufficiently large relative to the width and
the number of random initializations is sufficiently small relative
to the network size.
In \cite{Lu_2020}, the probability that a randomly initialized \ReLU\
\ANN\ represents a constant function at initialization
is studied analytically and numerically.
It is also shown that, on this
event, the realization remains constant as a function of the input
throughout gradient-based training.
Moreover, in certain shallow \ANN\ training scenarios,
\cite{jentzen2024nonconvergenceglobalminimizersadam} shows that,
for a class of \SGD\ optimization methods including \ADAM, the
probability that the iterates fail to converge to any global minimizer
tends to one as the width of the hidden layer tends to infinity.
For deep \ANNs, related results on non-convergence to global minimizers
and on non-convergence in probability of the true risk to its optimal
value are established in \cite{DoHannibalJentzen2026} and
\cite{DoJentzenRiekert2025}, respectively.
In addition, in \ANN\ training scenarios where there do not exist global minimizers in the optimization landscape
it is shown in \cite{gallon2022blowphenomenagradientdescent} (cf.\ also \cite{PetersenRaslanVoigtlaender2021})
that the norms of suitable gradient based optimization processes fail to converge but diverge to infinity.
Furthermore, \cite[Theorem~3]{ReddiKale2019} provides an explicit
example of a convex stochastic optimization problem in which the
\ADAM\ optimizer fails to converge to the optimal solution.
In the deterministic setting, \cite{BockWeiss2019} establishes the
existence of two-cycles for the \ADAM\ optimizer without bias correction
for a scalar quadratic objective function.
Moreover, \cite{Toint2025} gives a one-dimensional example in which
the deterministic \ADAM\ iterates diverge to infinity.
For the quadratic loss function considered in \cref{intro:thm:1},
upper and lower error bounds for the standard \SGD\ method with
polynomially decreasing learning rates are established in
\cite{JentzenVonWurstemberger2020}.
For a class of strongly convex stochastic optimization problems and a sufficiently
small constant learning rate, \cite{DieuleveutDurmusBach2020} shows that
the standard \SGD\ iterates converge in distribution to the unique
stationary distribution.
For \ADAM\ with a constant learning rate,
\cite[Theorem~4.5]{Barakat_2021_cvg} shows that the long-run expected
fraction of iterates outside any fixed-radius neighborhood of the set
of critical points tends to zero as the learning rate tends to zero.
At any fixed positive learning rate and fixed mini-batch size,
\cref{whole:vec:thm:sup} shows that there is a positive constant
such that, for every random target vector, the mean square error
of the \ADAM\ iterates for the quadratic problem considered there
exceeds this constant infinitely often.
While \cref{whole:vec:thm:sup} follows from \cref{cor:limsup:simple} in
\cref{sec:results} below, in \cref{thm:new:rob:cond} we prove that
under stronger assumptions no subsequence of every component
converges in mean absolute error to any real-valued random variable.
Besides lower bound, non-convergence, and divergence results, 
we also refer, \eg, to \cite{
	DereichGraeberJentzenRiekert2025,
	dereich2024convergenceratesadamoptimizer,
	DereichDoJentzen2026,
	GodichonBaggioni2023,
	dereich2024learningrateadaptivestochastic,
	ZhangChen2022,
	Barakat_2021_cvg,
	li2023convergenceadamrelaxedassumptions,
	ReddiKale2019,
	hong2024revisitingconvergenceadagradrelaxed,
	ZouShen2019,
	Defossez2022
	}
for upper bound and convergence results for \ADAM\ algorithms and other adaptive \SGD\ optimization methods.
For further investigations on \SGD\ optimization methods we also refer, \eg, to 
\cite{ruder2017overviewgradientdescentoptimization,Sun2019,jentzen2023mathematical} and the references mentioned therein.

\subsection{Structure of this article}
\label{subsection:structure}

The remainder of this article is organized as follows. 
In \cref{section:factorization:lemma} we present and study a generalized variant of the standard concepts of conditional expectations of a random variable.
 In \cref{sec:results} we employ the findings of \cref{section:factorization:lemma} to establish suitable non-convergence results for \ADAM\ as well as other adaptive and non-adaptive \SGD\ optimization methods.
In particular, \cref{cor:limsup:simple} establishes non-convergence
of the optimization process, and \cref{thm:new:rob:cond} establishes
non-convergence and, under stronger assumptions, non-accumulation
for individual components.
These are the two main results of this article.
%
In \cref{section:4} we utilize the results of \cref{sec:results} to prove non-convergence results for various specific non-adaptive and adaptive \SGD\ optimization methods.
\Cref{whole:vec:thm:sup} above is an immediate consequence of the non-convergence result in \cref{prop:prop:non_convergence_modified_Adam:easy:mat}, which, in turn, follows from \cref{cor:limsup:simple}.
Similarly, \cref{thm:1.3} follows from \cref{ANNS:ADAM}, which is a consequence of \cref{lem:for:EX:ANNs:1}.

\newpage

\section{Factorization lemmas for generalized conditional expectations}
\label{section:factorization:lemma}

In this section we present and study a generalized variant of the standard concepts of conditional expectations of a random variable.
To be more specific, in the literature 
for every probability space $(\Omega,\cF,\P)$,
every sigma-algebra $\gil\subseteq\cF$ on $\Omega$,
and
every random variable $X\colon\Omega\to[-\infty,\infty]$
 with $\E\PR{\vass{X}}<\infty$ (proper integrability of $X$)
the concept of the expectation of $X$ conditioned on $\gil$ is presented, investigated, and used; cf., \eg, 
\cite[Section~8.2]{klenkeprobability}, 
 \cite[Chapter~8]{Kallenberg}, 
 \cite[Chapter~10]{Dudley2002}, and
\cite[Chapter~12]{Krishna2006}.
It is also standard in the literature to extend this conditional expectation concept to random variables which are only improper integrable.
 Specifically, for every probability space $(\Omega,\cF,\P)$,
 every sigma-algebra $\gil\subseteq\cF$ on $\Omega$, and
 every random variable $X\colon\Omega\to[-\infty,\infty]$ with $\min\{\E\PR{\max\{X,0\}},\E\PR{\max\{-X,0\}}\}<\infty$ (improper integrability of $X$) the concept of the improper expectation of $X$ conditioned on $\gil$ is presented, studied, and employed; cf., \eg, 
\cite[Definition~12.1.3]{Krishna2006},
 \cite[Remark~8.16]{klenkeprobability}, 
 \cite[Exercise~5~in~Chapter~8]{Kallenberg}, and
 \cite[Exercise~7~in~Section~10.1]{Dudley2002}.

However, in our proof of the non-convergence results for \ADAM\ and other adaptive \SGD\ optimization methods in \cref{sec:results} we employ a more general concept of conditional expectations beyond the situation of improper integrable random variables.
This is the reason why we present and study in this section such a generalized variant of the standard concepts of conditional expectations.
In particular, in \cref{def:cond:exp} and \cref{gen:conditional:expectation} we present for every probability space $(\Omega,\cF,\P)$,
every sigma-algebra $\gil\subseteq\cF$ on $\Omega$, and
every random variable $X\colon\Omega\to[-\infty,\infty]$ 
with the property that there exist $A_n\in\gil$, $n\in\N$, such that $\Omega=\cup_{n\in\N}A_n$ and $\forall\,n\in\N\colon\min_{z\in\{-1,1\}}\E\PR{\max\{zX,0\}\indicator{A_n}}<\infty$
 a generalized variant of the improper expectation of $X$ conditioned on $\gil$ (cf.\ \cref{def:improper:cond:int}).
As it seems to be difficult to find a reference in the literature in which such a concept of generalized conditional expectations is presented and studied,
 we introduce and investigate this conceptionality within this section in detail and also develop a factorization lemma for such generalized conditional expectations in \cref{lem:2.9:plus:minus} below and
a factorization lemma for the associated generalized conditional variances in \cref{lem:factorization2:spec} below.
We employ the factorization lemma for generalized conditional variances in \cref{lem:factorization2:spec} to prove the non-convergence results for \ADAM\ and other adaptive \SGD\ optimization methods in \cref{sec:results}.

\subsection{Generalized conditional expectations}
\label{subsection:3.1}
\newcommand{\codom}{D}

\begin{definition}[Proper conditional integrable]
	\label{def:proper:cond:int}
	Let $ ( \Omega, \cF, \P ) $ be a probability space,
	let $ \codom\subseteq[-\infty,\infty]$ be a set,
	let $X\colon\Omega\to \codom$ be a random variable, and
	let $\gil\subseteq\cF$ be a sigma-algebra on $\Omega$.
	Then we say that $X$ is proper $\gil$-conditional $\P$-integrable  if and only if there 
	exist $A_n\in\gil$, $n\in\N$, such that 
	\begin{enumerate}[label=(\roman*)] 
		\item it holds that $\Omega=\cup_{n\in\N}A_n$ and
		\item it holds for all $n\in\N$ that $\E\PR{\vass{X}\indicator{A_n}}<\infty$.
	\end{enumerate}
\end{definition}
\cfclear

\begin{definition}[Improper conditional integrable]
	\label{def:improper:cond:int}
	Let $ ( \Omega, \cF, \P ) $ be a probability space,
	let $ \codom\subseteq[-\infty,\infty]$ be a set,
	let $X\colon\Omega\to  \codom$ be a random variable, and
	let $\gil\subseteq\cF$ be a sigma-algebra on $\Omega$.
	Then we say that \improper{X}{\gil} if and only if there exist $A_n\in\gil$, $n\in\N$, such that 
	\begin{enumerate}[label=(\roman*)] 
		\item it holds that $\Omega=\cup_{n\in\N}A_n$ and
		\item it holds for all $n\in\N$ that $\min_{z\in\{-1,1\}}\E\PR{\max\{zX,0\}\indicator{A_n}}
		<\infty$.
	\end{enumerate}
\end{definition}

\cfclear

\begin{athm}{lemma}{lem:cond:int:triv:sig:alg}
	Let $ ( \Omega, \cF, \P ) $ be a probability space,
	let $ \codom\subseteq[-\infty,\infty]$ be a set, and
	let $X\colon\Omega\to  \codom$ be a random variable.
	Then 
	\begin{enumerate}[label=(\roman*)] 
		\item \label{it:triv:proper}it holds that \proper{X}{\{\emptyset,\Omega\}} if and only if $\E\PR{\vass{X}}<\infty$ and
		\item \label{it:triv:improper}it holds that \improper{X}{\{\emptyset,\Omega\}} if and only if $\min_{z\in\{-1,1\}}\allowbreak\E\PR{\max\{zX,0\}}
		<\infty$ 
	\end{enumerate}
	\cfout.
\end{athm}

\begin{proof}[Proof of \cref{lem:cond:int:triv:sig:alg}]
	\Nobs that for all $A_n\in\{\emptyset,\Omega\}$, $n\in\N$, with $\Omega=\cup_{n\in\N}A_n$ there exists $m\in\N$ such that for all $n\in\N$ it holds that
	\begin{equation}
		\label{eq:only:trivial:coverings}
		A_m=\Omega
		\qqandqq
		A_n\in\{\emptyset,\Omega\}.
	\end{equation}
	This \proves\ that \proper{X}{\{\emptyset,\Omega\}} if and only if
	\begin{equation}
		\E\PR{\vass{X}\indicator{\Omega}}=\E\PR{\vass{X}}<\infty
	\end{equation}
	\cfload.
	This \proves\ \cref{it:triv:proper}.
	\Nobs that \eqref{eq:only:trivial:coverings} \proves\ that \improper{X}{\{\emptyset,\Omega\}} if and only if
	\begin{equation}
		\textstyle
		\min_{z\in\{-1,1\}}\E\PR{\max\{zX,0\}\indicator{\Omega}}=\min_{z\in\{-1,1\}}\E\PR{\max\{zX,0\}}<\infty
	\end{equation}
	\cfload. This \proves\ \cref{it:triv:improper}.
	\finishproofthus
\end{proof}

\cfclear

\begin{athm}{lemma}{lem:cond:int:whole:sig:alg}
	Let $ ( \Omega, \cF, \P ) $ be a probability space,
	let $ \codom\subseteq[-\infty,\infty]$ be a set, and
	let $X\colon\Omega\to  \codom$ be a random variable.
	Then
	\begin{enumerate}[label=(\roman*)] 
		\item \label{it:whole:proper}it holds that \proper{X}{\cF} if and only if $\P(\vass{X}<\infty)=1$ and
		\item \label{it:whole:improper}it holds that \improper{X}{\cF}
	\end{enumerate}
	\cfout.
\end{athm}

\begin{proof}[Proof of \cref{lem:cond:int:whole:sig:alg}]
	\Nobs for every random variable $Y\colon\Omega\to\codom$ it holds that
	\begin{equation}
		\label{eq:covering:omega}
		\cup_{n\in\N}\{|Y|\leq n\}
		=\Omega\backslash\{\vass{Y}=\infty\}
		\qqandqq
		\PRb{\cup_{n\in\N}\{\{|Y|\leq n\},\{\vass{Y}=\infty\}\}}\subseteq\cF.
	\end{equation}
	This \proves\ that for every random variable $Y\colon\Omega\to\codom$ and every $n\in\N$ it holds that
	\begin{equation}
		\label{eq:fin:proper:exp}
		\E\PRb{\vass{Y}\indicator{\{|Y|\leq n\}}}
		\leq
		\E\PRb{n\indicator{\{|Y|\leq n\}}}
		=
		n\P\pr{|Y|\leq n}
		\leq
		n
		<\infty.
	\end{equation}
	\Moreover\ for every random variable $Y\colon\Omega\to\codom$ with $\P(\vass{Y}<\infty)=1$ it holds that
	\vspace{-0.3cm}
	\begin{equation}
		\begin{split}
			\E\PRb{\vass{Y}\indicator{\{\vass{Y}=\infty\}}}
			=
			0.
		\end{split}
	\end{equation}
	This, 
	\eqref{eq:fin:proper:exp}, and
	\eqref{eq:covering:omega} \prove\ that for every random variable $Y\colon\Omega\to\codom$ with $\P(\vass{Y}<\infty)=1$ it holds that
	$
		\text{\proper{Y}{\cF}}
	$
	\cfload.
	\Moreover for random variable $Y\colon\Omega\to\codom$,  every $A_n\in\cF$, $n\in\N$, with $\Omega=\cup_{n\in\N} A_n$ and it holds that
	\begin{equation}
		\label{eq:implication:proper:as:finite}
		\begin{split}
			\P(\vass{Y}=\infty)
			=
			\P\prb{\cup_{n\in\N}\PR{A_n\cap\{\vass{Y}=\infty\}}}
			&\textstyle\leq \sum_{n=1}^{\infty}\P\pr{A_n\cap\{\vass{Y}=\infty\}}.
		\end{split}
	\end{equation}
	\Moreover\ for every random variable $Y\colon\Omega\to\codom$ and every $A\in\cF$ with $\E\PR{\vass{Y}\indicator{A}}<\infty$ it holds that
	\begin{equation}
		\begin{split}
			\P(A\cap\{\vass{Y}=\infty\})
			\leq
			\E\PR{\vass{Y}\indicator{A\cap\{\vass{Y}=\infty\}}}
			=0.
		\end{split}
	\end{equation}
	This and \eqref{eq:implication:proper:as:finite} \prove\ that for every proper $\cF$-conditional $\P$-integrable random variable $Y\colon\Omega\to\codom$ it holds that $\P(\vass{Y}=\infty)=0$.
	\Moreover for all $k\in\{-1,1\}$ it holds that
	\begin{equation}
		\textstyle
		\min_{z\in\{-1,1\}}\E\PR{\max\{zX,0\}\indicator{\pR{kX=\infty}}}
		\leq\E\PR{\max\{-kX,0\}\indicator{\pR{kX=\infty}}}
		=0
		.
	\end{equation}
	Combining this, 
	\eqref{eq:fin:proper:exp}, and
	\eqref{eq:covering:omega} with the fact that $\pRb{\pR{X=\infty},\pR{-X=\infty}}\subseteq\cF$ \proves\ that \improper{X}{\cF} \cfload.
	This \proves\ \cref{it:whole:improper}.
	\finishproofthus
\end{proof}

\cfclear

\cfclear

\begin{athm}{lemma}{lem:cond:int:incr:sig:alg}
	Let $ ( \Omega, \cF, \P ) $ be a probability space,
	let $\codom\subseteq[-\infty,\infty]$ be a set,
	let
	$\gil_1$ and $\gil_2$ be sigma-algebras on $\Omega$, and assume $\gil_1\subseteq\gil_2\subseteq\cF$. Then
	\begin{enumerate}[label=(\roman*)] 
		\item \label{it:cond:incr:sig:proper}it holds for every proper $\gil_1$-conditional $\P$-integrable random variable $X\colon\Omega\to\codom$ that \proper{X}{\gil_2} and
		\item \label{it:cond:incr:sig:improper}it holds for every improper $\gil_1$-conditional $\P$-integrable random variable $X\colon\Omega\to\codom$ that \improper{X}{\gil_2}.
	\end{enumerate}\cfout.
\end{athm}

\begin{proof}[Proof of \cref{lem:cond:int:incr:sig:alg}]
	Throughout this proof let $X\colon\Omega\to\codom$ and $Y\colon\Omega\to\codom$ be random variables and assume that \proper{X}{\gil_1} and that \improper{Y}{\gil_1} \cfload.
	\Nobs that the assumption that \proper{X}{\gil_1} \proves\ that there exist $A_n\in\gil_1$, $n\in\N$, such that 
	\begin{enumerate}[label=(\Roman*)] 
		\item it holds that $\Omega=\cup_{n\in\N}A_n$ and
		\item it holds for all $n\in\N$ that $\E\PR{\vass{X}\indicator{A_n}}<\infty$.
	\end{enumerate}
	Combining this with the fact that $\gil_1\subseteq\gil_2$ \proves\ that \proper{X}{\gil_2}.
	This \proves\ \cref{it:cond:incr:sig:proper}.
	\Nobs that the assumption that \improper{Y}{\gil_1} \proves\ that there exist $A_n\in\gil_1$, $n\in\N$, such that 
	\begin{enumerate}[label=(\Alph*)] 
		\item it holds that $\Omega=\cup_{n\in\N}A_n$ and
		\item it holds for all $n\in\N$ that $\min_{z\in\{-1,1\}}\E\PR{\max\{zY,0\}\indicator{A_n}}
		<\infty$.
	\end{enumerate}
	Combining this with the fact that $\gil_1\subseteq\gil_2$ \proves\ that \improper{Y}{\gil_2}.
	This \proves\ \cref{it:cond:incr:sig:improper}.
	\finishproofthus
\end{proof}

\cfclear

\begin{definition}[Generalized conditional expectation]
	\label{def:cond:exp}
	Let $ ( \Omega, \cF, \P ) $ be a probability space,
	let $\codom\subseteq[-\infty,\infty]$ be a set,
	let $X\colon\Omega\to \codom$ and  $Y\colon\Omega\to[-\infty,\infty]$ be random variables, and
	let $\gil\subseteq\cF$ be a sigma-algebra on $\Omega$.
	Then we say that \cond{Y}{\gil}{X} if and only if there exist $A_n\in\gil$, $n\in\N$, such that 
	\begin{enumerate}[label=(\roman*)] 
		\item it holds that $\Omega=\cup_{n\in\N}A_n$,
		\item it holds that $Y$ is $\gil$-measurable,
		\item it holds for all $n\in\N$ that
		$\min_{z\in\{-1,1\}}\E\PRb{\pr{\max\{zX,0\}+\max\{zY,0\}}\indicator{A_n}}<\infty$, and
		\item it holds for all $n\in\N$, $B\in\gil$ that $\E\PR{X\indicator{A_n\cap B}}=\E\PR{Y\indicator{A_n\cap B}}$.
	\end{enumerate}
\end{definition}

In the following result, \cref{prop:lem:cond:exp:ex} below, we recall the well-known fact that for every probability space $ ( \Omega, \cF, \P ) $ and every sigma-algebra $\gil\subseteq\cF$ on $\Omega$ it holds that every non-negative random variable has a conditional expectation with respect to $\gil$ (cf., \eg,   \cite[Remark~8.16]{klenkeprobability} and\cite[Remark~12.1.3]{Krishna2006}).
Only for completeness we include here in this section a detailed proof for \cref{prop:lem:cond:exp:ex}.

\cfclear
\begin{athm}{lemma}{prop:lem:cond:exp:ex}
	Let $ ( \Omega, \cF, \P ) $ be a probability space,
	let $\codom\subseteq[0,\infty]$ be a set,
	let $X\colon\Omega\to \codom$ be a random variable, and
	let $\gil\subseteq\cF$ be a sigma-algebra on $\Omega$.
	Then there exists a $\gil$-measurable function $Y\colon\Omega\to[0,\infty]$ such that for all $A\in\gil$ it holds that $\E\PR{Y\indicator{A}}=\E\PR{X\indicator{A}}$.
\end{athm}

\begin{proof}[Proof of \cref{prop:lem:cond:exp:ex}]
	Throughout this proof let $\mu_{n}\colon\gil\to[0,\infty]$, $n\in\N\cup\{\infty\}$, satisfy for all $A\in\gil$, $n\in\N$ that
	\begin{equation}
		\label{eq:new:measures}
		\mu_{n}(A)=\E\PRb{X\indicator{A\cap\{n-1< X\leq n\}}}
		\qandq
		\mu_{\infty}(A)=\P\pr{A\cap\{X=\infty\}}.
	\end{equation}
	\Nobs that \eqref{eq:new:measures} \proves\ that for all $n\in\N\cup\{\infty\}$ it holds that $\mu_{n}$ 
	is a finite measure on the measurable space $(\Omega,\gil)$ and
$\mu_{n}$ is absolutely continuous on $(\Omega,\gil)$ with respect to $\P|_\gil$.
	This and the Radon-Nikodym theorem (see, \eg,  \cite[Corollary~7.34]{klenkeprobability})
	\prove\ that for every $n\in\N\cup\{\infty\}$
	 there exists a $\gil$-measurable function $Z_{n}\colon\Omega\to[0,\infty]$ which satisfies for all $A\in\gil$ that
	\begin{equation}
		\label{eq:radon:nikodym}
		\mu_{n}(A)=\E\PR{Z_{n}\indicator{A}}
		.
	\end{equation}
	This, 
	the fact that for all $\omega\in\Omega$ it holds that $(\sum_{n=1}^k Z_n(\omega))_{k\in\N}$ is non-decreasing, and
	the monotone convergence theorem for non-negative measurable functions prove that for all $A\in\gil$ it holds that
	$\sum_{n=1}^{\infty} Z_n$ is $\gil$-measurable and
	\begin{equation}
		\label{eq:measurability:sum}
		\textstyle
		\sum_{n=1}^{\infty} \E\PR{Z_n\indicator{A}}=\E\PRb{\sum_{n=1}^{\infty} Z_n\indicator{A}}.
	\end{equation}
	This, \eqref{eq:new:measures}, and 
	\eqref{eq:radon:nikodym} \prove\ that for all $z\in[0,\infty]$, $A\in\gil$ with $\P(A\cap\{X=\infty\})=0$ it holds that
	\begin{equation}
		\label{eq:radon:nikodym:finite:rv}
		\begin{split}
		\E\PR{X\indicator{A}}
		=
		\E\PRb{X\indicator{A\cap\{X<\infty\}}}
		&\textstyle=
		z\P\pr{A\cap\{X=\infty\}}+
		\sum_{n=1}^{\infty}\E\PRb{X\indicator{A\cap\{n-1< X\leq n\}}}
		\\&\textstyle=
		z\mu_{\infty}(A)+\sum_{n=1}^{\infty}\mu_{n}(A)
		\\&\textstyle=
		z\E\PR{Z_\infty\indicator{A}}+\sum_{n=1}^{\infty}\E\PR{Z_{n}\indicator{A}}
		\\&\textstyle=
		\E\PR{zZ_\infty\indicator{A}}+\E\PRb{\sum_{n=1}^{\infty}Z_{n}\indicator{A}}
		\\&\textstyle=\E\PRb{\prb{zZ_\infty+\sum_{n=1}^{\infty}Z_{n}}\indicator{A}}.
		\end{split}
	\end{equation}
			\Moreover 
			\eqref{eq:new:measures},
			\eqref{eq:radon:nikodym}, and
			\eqref{eq:measurability:sum} \prove\ that all $z\in\{\infty\}$, $A\in\gil$ with $\P(A\cap\{X=\infty\})>0$ it holds that
	\begin{equation}
		\begin{split}
			\textstyle\E\PRb{\prb{zZ_\infty+\sum_{n=1}^{\infty}Z_{n}}\indicator{A}}
			&\geq
			\textstyle\E\PRb{zZ_\infty\indicator{A}}
			=z\mu_\infty(A)
			=\infty
			=
			\E\PR{X\indicator{A}}.
		\end{split}
	\end{equation}
	This and \eqref{eq:radon:nikodym:finite:rv} \prove\ that for all $z\in\{\infty\}$, $A\in\gil$ it holds that
	\begin{equation}
		\label{eq:cond:exp:right:expectations:on:gil}
		\textstyle\E\PRb{\prb{zZ_\infty+\sum_{n=1}^{\infty}Z_{n}}\indicator{A}}=\E\PR{X\indicator{A}}.
	\end{equation}
	\Moreover \eqref{eq:radon:nikodym} and \eqref{eq:measurability:sum} \prove\ that for all $z\in[0,\infty]$ it holds that
	$zZ_\infty+\sum_{n=1}^{\infty}Z_{n}$ is $\gil$-measurable.
	This and \eqref{eq:cond:exp:right:expectations:on:gil} \prove\ that there exists a $\gil$-measurable function $Y\colon\Omega\to[0,\infty]$ which satisfies for all $A\in\gil$ that
	\begin{equation}
		\E\PR{X\indicator{A}}=\E\PR{Y\indicator{A}}.
	\end{equation}
	\finishproofthus
\end{proof}
\cfclear
\begin{athm}{prop}{lem:cond:exp:ex}
	Let $ ( \Omega, \cF, \P ) $ be a probability space,
	let $\codom\subseteq[-\infty,\infty]$ be a set,
	let $X\colon\Omega\to \codom$ be a random variable,
	let $\gil\subseteq\cF$ be a sigma-algebra on $\Omega$, and assume that \improper{X}{\gil} \cfload.
	Then there exists a random variable $Y\colon\Omega\to[-\infty,\infty]$ such that \cond{Y}{\gil}{X}
	\cfout.
\end{athm}

\begin{proof}[Proof of \cref{lem:cond:exp:ex}]
	Throughout this proof
	let $A_n\in\gil$, $n\in\N$, satisfy that
	\begin{enumerate}[label=(\roman*)] 
		\item \label{it:covering:for:improper:ex}it holds that $\Omega=\cup_{n\in\N}A_n$ and
		\item \label{it:improper:finiteness}it holds for all $n\in\N$ that $\min_{z\in\{-1,1\}}\E\PR{\max\{zX,0\}\indicator{A_n}}
		<\infty$
	\end{enumerate}
	(cf.\ \cref{def:improper:cond:int}),
	let $B_n\in\gil$, $n\in\N$, satisfy for all $n\in\N$ that
	\begin{equation}
		\label{eq:setup:disjoint:covering}
		B_n=A_n\backslash\PR{\cup_{k=1}^{n-1}A_k}.
	\end{equation}
	\Nobs that \cref{prop:lem:cond:exp:ex} \proves\ that for every $n\in\N$, $z\in\{-1,1\}$ there exists a $\gil$-measurable function $Z_{n,z}\colon\Omega\to[0,\infty]$ which satisfies for all $S\in\gil$ that 
	\begin{equation}
		\label{eq:conditional:exp:on:composition}
		\E\PR{\pr{\max\{zX,0\}\indicator{B_n}}\indicator{S}}=\E\PR{Z_{n,z}\indicator{S}}
		.
	\end{equation} 
	This, 
	the fact that for all $z\in\{-1,1\}$, $\omega\in\Omega$ it holds that $(\sum_{n=1}^k Z_{n,z}(\omega))_{k\in\N}$ is non-decreasing, and
	the monotone convergence theorem for non-negative measurable functions prove that for all $z\in\{-1,1\}$, $S\in\gil$ it holds that $\sum_{n=1}^{\infty} Z_{n,z}$ is $\gil$-measurable and
	\begin{equation}
		\label{eq:measurability:sum2}
		\textstyle
		\sum_{n=1}^{\infty} \E\PR{Z_{n,z}\indicator{S}}=\E\PRb{\sum_{n=1}^{\infty} Z_{n,z}\indicator{S}}.
	\end{equation}
	\Moreover \eqref{eq:conditional:exp:on:composition} and \cref{it:improper:finiteness} \prove\ that for all $n\in\N$ it holds that
	\begin{equation}
		\min_{z\in\{-1,1\}}\E\PR{Z_{n,z}\indicator{B_n}}
		=\min_{z\in\{-1,1\}}\E\PR{\max\{zX,0\}\indicator{B_n}}
		\leq\min_{z\in\{-1,1\}}\E\PR{\max\{zX,0\}\indicator{A_n}}
		<\infty.
	\end{equation}
	This,
	\eqref{eq:setup:disjoint:covering}, and
	\eqref{eq:conditional:exp:on:composition} \prove\ that for all $n\in\N$, $S\in\gil$
	it holds that
	\begin{equation}
		\label{eq:one:side:is:finite}
		\begin{split}
			&\textstyle\min_{z\in\{-1,1\}}\E\PRb{\pr{\max\{zX,0\}+\max\pRb{\sum_{m=1}^{\infty} Z_{m,z},0}}\indicator{B_n}}
			\\&\leq\textstyle 2\min_{z\in\{-1,1\}}\E\PRb{\sum_{m=1}^{\infty} Z_{m,z}\indicator{B_n}}
			\\&=\textstyle 2\min_{z\in\{-1,1\}}\sum_{m=1}^{\infty}\E\PRb{ Z_{m,z}\indicator{B_n}}
			\\&=\textstyle 2\min_{z\in\{-1,1\}}\E\PRb{ Z_{n,z}\indicator{B_n}}
			<\infty.
		\end{split}
	\end{equation}
	This,
	\eqref{eq:setup:disjoint:covering},
	\eqref{eq:conditional:exp:on:composition}, and
	 the fact that for all random variables $Z\colon\Omega\to[0,\infty]$ with $\E\PR{Z}=0$ it holds that $\P(Z>0)=0$ \prove\ that
	\begin{equation}
		\label{eq:both:sums:infinite:zero:set}
		\begin{split}
		&\textstyle\P\prb{\min\pRb{\sum_{m=1}^{\infty}Z_{m,1},\sum_{m=1}^{\infty}Z_{m,-1}}=\infty}
		\\&\textstyle=\sum_{n=1}^{\infty}\P\prb{\{\min\pRb{\sum_{m=1}^{\infty}Z_{m,1},\sum_{m=1}^{\infty}Z_{m,-1}}=\infty\}\cap B_n}
		\\&\textstyle=\sum_{n=1}^{\infty}\P\prb{\min\pRb{\sum_{m=1}^{\infty}Z_{m,1}\indicator{B_n},\sum_{m=1}^{\infty}Z_{m,-1}\indicator{B_n}}=\infty}
		\\&\textstyle=\sum_{n=1}^{\infty}\P\prb{\min\pRb{Z_{n,1}\indicator{B_n},Z_{n,-1}\indicator{B_n}}=\infty}
		=0.
		\end{split}
	\end{equation}
	This,
	\eqref{eq:setup:disjoint:covering}, 
	\eqref{eq:conditional:exp:on:composition},
	\eqref{eq:measurability:sum2}, 
	\eqref{eq:one:side:is:finite}, and \cref{it:improper:finiteness}
	 \prove\ that for every $\gil$-measurable random variable $Y\colon\Omega\to[-\infty,\infty]$ and every $n\in\N$, $S\in\gil$ with \begin{equation}
		\textstyle
		\P\prb{\pRb{\min\pRb{\sum_{m=1}^{\infty}Z_{m,1},\sum_{m=1}^{\infty}Z_{m,-1}}<\infty}\cap\pRb{Y=\sum_{m=1}^{\infty}\pr{Z_{m,1}-Z_{m,-1}}}}=1
	\end{equation}
		 it holds that
	\begin{equation}
		\begin{split}
			\E\PR{X\indicator{S\cap B_n}}
			&=\textstyle\E\PR{\pr{\max\{X,0\}-\max\{-X,0\}}\indicator{S\cap B_n}}
			\\&=\textstyle\E\PR{\max\{X,0\}\indicator{S\cap B_n}}-\E\PR{\max\{-X,0\}\indicator{S\cap B_n}}
			\\&=\textstyle\E\PR{Z_{n,1}\indicator{S\cap B_n}}-\E\PR{Z_{n,-1}\indicator{S\cap B_n}}
			\\&=\textstyle\PRb{\sum_{m=1}^\infty\E\PR{Z_{m,1}\indicator{S\cap B_n}}}
			-\PRb{\sum_{m=1}^\infty\E\PR{Z_{m,-1}\indicator{S\cap B_n}}}
			\\&=\textstyle\E\PRb{\sum_{m=1}^\infty Z_{m,1}\indicator{S\cap B_n}}
			-\E\PRb{\sum_{m=1}^\infty Z_{m,-1}\indicator{S\cap B_n}}
			\\&=\textstyle\E\PRb{\sum_{m=1}^\infty\pr{Z_{m,1}-Z_{m,-1}}\indicator{S\cap B_n}}
			=\textstyle\E\PRb{ Y\indicator{S\cap B_n}}.
		\end{split}
	\end{equation}
		This,
		\eqref{eq:setup:disjoint:covering}, \eqref{eq:one:side:is:finite}, and \cref{it:covering:for:improper:ex} \prove\ that there exists a random variable $Y\colon\Omega\to[-\infty,\infty]$ such that \cond{Y}{\gil}{X} \cfload.
	\finishproofthus
\end{proof}
\cfclear

In the next result, \cref{lem:cond:uniq} below, we show that
for every probability space $ ( \Omega, \cF, \P ) $ and
every sigma-algebra $\gil\subseteq\cF$ on $\Omega$
we have that $\gil$-conditional $\P$-expectations of a random variable are \as\ unique with respect to $\P$.
Our proof of \cref{lem:cond:uniq} is strongly based on the proof of the well-known fact that 
for every probability space $ ( \Omega, \cF, \P ) $ and
every sigma-algebra $\gil\subseteq\cF$ on $\Omega$
we have that standard conditional expectations of a random variable are $\P$-\as\ unique (cf., \eg,
\cite[Theorem~12.1.4]{Krishna2006},
\cite[Theorem~8.12]{klenkeprobability}, and  \cite[Theorem~8.1]{Kallenberg}).

\begin{athm}{prop}{lem:cond:uniq}
	Let $ ( \Omega, \cF, \P ) $ be a probability space,
	let $\codom\subseteq[-\infty,\infty]$ be a set,
	let $X\colon\Omega\to \codom$,
	$Y_1\colon\Omega\to[-\infty,\infty]$, and 
	$Y_2\colon\Omega\to[-\infty,\infty]$ be random variables,
	let $\gil\subseteq\cF$ be a sigma-algebra on $\Omega$, and assume for all $k\in\{1,2\}$ that \cond{Y_k}{\gil}{X} \cfload.
	Then $\P(Y_1=Y_2)=1$\cfout.
\end{athm}

\begin{proof}[Proof of \cref{lem:cond:uniq}]
	Throughout this proof let $A_n\in\gil$, $n\in\N$, satisfy that 
	\begin{enumerate}[label=(\roman*)] 
		\item \label{it:Y_1:cond:covering}it holds that $\Omega=\cup_{n\in\N}A_n$,
		\item \label{it:Y_1:cond:measurable}it holds that $Y_1$ is $\gil$-measurable,
		\item \label{it:Y_1:cond:integrable}it holds for all $n\in\N$ that
		$\min_{z\in\{-1,1\}}\E\PRb{\pr{\max\{zX,0\}+\max\{zY_1,0\}}\indicator{A_n}}<\infty$, and
		\item \label{it:Y_1:cond:equal}it holds for all $n\in\N$, $S\in\gil$ that $\E\PR{X\indicator{A_n\cap S}}=\E\PR{Y_1\indicator{A_n\cap S}}$
	\end{enumerate}
	and let $B_n\in\gil$, $n\in\N$, satisfy that 
	\begin{enumerate}[label=(\Roman*)] 
		\item \label{it:Y_2:cond:covering}it holds that $\Omega=\cup_{n\in\N}B_n$,
		\item \label{it:Y_2:cond:measurable}it holds that $Y_2$ is $\gil$-measurable,
		\item \label{it:Y_2:cond:integrable}it holds for all $n\in\N$ that
		$\min_{z\in\{-1,1\}}\E\PRb{\pr{\max\{zX,0\}+\max\{zY_2,0\}}\indicator{B_n}}<\infty$, and
		\item \label{it:Y_2:cond:equal}it holds for all $n\in\N$, $S\in\gil$ that $\E\PR{X\indicator{B_n\cap S}}=\E\PR{Y_2\indicator{B_n\cap S}}$
	\end{enumerate}
		(cf.\ \cref{def:cond:exp}). \Nobs that \cref{it:Y_1:cond:measurable} and \cref{it:Y_2:cond:measurable} \prove\ that for all $k\in\{1,2\}$ it holds that		
		\begin{equation}
			\label{eq:interesting:sets:gil:mb}
			\pRb{\{Y_k=\infty\},
				\{Y_k=-\infty\},
				\{\vass{Y_k}<\infty\},
			\{\max\{\vass{Y_1},\vass{Y_{2}}\}<\infty\},
			\{Y_{3-k} < Y_k\}}\subseteq\gil.
		\end{equation}
		This, \cref{it:Y_1:cond:equal}, and \cref{it:Y_2:cond:equal} \prove\ that for all $n,m,p\in\N$, $k\in\{1,2\}$ it holds that
		\begin{equation}
			\begin{split}
			\E\PR{Y_1\indicator{A_n\cap B_m\cap\{\max\{\vass{Y_1},\vass{Y_{2}}\}<p,Y_{3-k} < Y_k\}}}
			&=
			\E\PR{X\indicator{A_n\cap B_m\cap\{\max\{\vass{Y_1},\vass{Y_{2}}\}<p,Y_{3-k} < Y_k\}}}
			\\&=
			\E\PR{Y_2\indicator{A_n\cap B_m\cap\{\max\{\vass{Y_1},\vass{Y_{2}}\}<p,Y_{3-k} < Y_k\}}}
			.
			\end{split}
		\end{equation}
		This \proves\ that 
		for all $n,m,p\in\N$, $k\in\{1,2\}$ it holds that
		\begin{equation}
			\begin{split}
				\E\PR{Y_1\indicator{A_n\cap B_m\cap\{\max\{\vass{Y_1},\vass{Y_{2}}\}<p,Y_{3-k} < Y_k\}}-Y_2\indicator{A_n\cap B_m\cap\{\max\{\vass{Y_1},\vass{Y_{2}}\}<p,Y_{3-k} < Y_k\}}}=0
				.
			\end{split}
		\end{equation}
		This \proves\ that 
		for all $n,m,p\in\N$, $k\in\{1,2\}$ it holds that
		\begin{equation}
			\begin{split}
				\P\pr{A_n\cap B_m\cap\{\max\{\vass{Y_1},\vass{Y_{2}}\}<p,Y_{3-k} < Y_k\}}=0
				.
			\end{split}
		\end{equation}
		This \proves\ that 
		\begin{equation}
			\label{eq:finite:Y_1:and:Y_2}
			\begin{split}
			&\P\pr{\max\{\vass{Y_1},\vass{Y_{2}}\}<\infty,Y_1=Y_{2}}
			\\&=
			\P\pr{\max\{\vass{Y_1},\vass{Y_{2}}\}<\infty}
			-\sum_{k=1}^2\P\pr{\{\max\{\vass{Y_1},\vass{Y_{2}}\}<\infty,Y_{3-k} < Y_k\}}
			\\&\geq
			\P\pr{\max\{\vass{Y_1},\vass{Y_{2}}\}<\infty}
			\\&\quad-\sum_{k=1}^2\sum_{n=1}^{\infty}\sum_{m=1}^{\infty}\sum_{p=1}^\infty\P\pr{A_n\cap B_m\cap\{\max\{\vass{Y_1},\vass{Y_{2}}\}<p,Y_{3-k} < Y_k\}}
			\\&=\P\pr{\max\{\vass{Y_1},\vass{Y_{2}}\}<\infty}
			.
		\end{split}
		\end{equation}
		\Moreover 
		\cref{it:Y_1:cond:equal}, \cref{it:Y_2:cond:equal}, and
		\eqref{eq:interesting:sets:gil:mb} \prove\ that for all $n,m,p\in\N$, $k\in\{1,2\}$, $z\in\{-1,1\}$ it holds that
		\begin{equation}
			\begin{split}
			\E\PR{Y_k\indicator{A_n\cap B_m\cap\{zY_k=\infty\}\cap\{zY_{3-k}<p\}}}
			&=
			\E\PR{X\indicator{A_n\cap B_m\cap\{zY_k=\infty\}\cap\{zY_{3-k}<p\}}}
			\\&=
			\E\PR{Y_{3-k}\indicator{A_n\cap B_m\cap\{zY_k=\infty\}\cap\{zY_{3-k}<p\}}}.
			\end{split}
		\end{equation}
		This \proves\ that for all $n,m,p\in\N$, $k\in\{1,2\}$, $z\in\{-1,1\}$ it holds that
		\begin{equation}
			\begin{split}
				\P\pr{A_n\cap B_m\cap\{zY_k=\infty\}\cap\{zY_{3-k}<p\}}
				=0.
			\end{split}
		\end{equation}
		This \proves\ that 
		\begin{equation}
			\begin{split}
				&\P\pr{\max\{\vass{Y_1},\vass{Y_{2}}\}=\infty,Y_1= Y_2}
				\\&=\textstyle\P\pr{\max\{\vass{Y_1},\vass{Y_{2}}\}=\infty}
				-\sum_{k=1}^{2}\P\pr{\max\{\vass{Y_1},\vass{Y_{2}}\}=\infty,Y_k< Y_{3-k}}
				\\&\geq\textstyle\P\pr{\max\{\vass{Y_1},\vass{Y_{2}}\}=\infty}
				-\sum_{k=1}^{2}\sum_{z\in\{-1,1\}}\P\pr{\{zY_k=\infty\}\cap\{zY_{3-k}<\infty\}}
				\\&\geq
				\P\pr{\max\{\vass{Y_1},\vass{Y_{2}}\}=\infty}
				\\&\quad-\sum_{k=1}^2\sum_{z\in\{-1,1\}}\sum_{m=1}^{\infty}\sum_{n=1}^\infty\P\pr{A_n\cap B_m\cap\{zY_k=\infty\}\cap\{zY_{3-k}<\infty\}}
				\\&=\P\pr{\max\{\vass{Y_1},\vass{Y_{2}}\}=\infty}
				.
			\end{split}
		\end{equation}
		This and \eqref{eq:finite:Y_1:and:Y_2} \prove\ that
		\begin{equation}
			\begin{split}
			\P(Y_1=Y_2)
			&=
			\P\pr{\max\{\vass{Y_1},\vass{Y_{2}}\}<\infty,Y_1= Y_2}
			+
			\P\pr{\max\{\vass{Y_1},\vass{Y_{2}}\}=\infty,Y_1= Y_2}
			\\&=
			\P\pr{\max\{\vass{Y_1},\vass{Y_{2}}\}<\infty}
			+
			\P\pr{\max\{\vass{Y_1},\vass{Y_{2}}\}=\infty}
			=1
			.
			\end{split}
		\end{equation}
		 \finishproofthus
\end{proof}

\begin{definition}[Generalized conditional expectation]
	\label{gen:conditional:expectation}
	Let $ ( \Omega, \cF, \P ) $ be a probability space,
	let $\codom\subseteq[-\infty,\infty]$ be a set,
	let $X\colon\Omega\to\codom$ be a random variable,
	let $\gil\subseteq\cF$ be a sigma-algebra on $\Omega$, and assume that \improper{X}{\gil} \cfload.
	Then we denote by $\E[X|\gil]$ the set given by
	\begin{equation}
		\E[X|\gil]
		=\pR*{Y\colon\Omega\to[-\infty,\infty]\colon
			\!\!\PR*{\!\!
				\begin{array}{c}
		\pr{Y\text{ is $\cF$-measurable}}
			\wedge \pr{Y\text{ is a}\\
			\text{$\gil$-conditional $\P$-expectation of }X}
		\end{array}
	\!\!}
	}
	\end{equation}
	(cf.\ \cref{def:cond:exp,lem:cond:exp:ex}).
\end{definition}

\cfclear
\begin{athm}{lemma}{lem:basic:improper}
	Let $ ( \Omega, \cF, \P ) $ be a probability space,
	let $ \codom\subseteq[-\infty,\infty]$ be a set, and let 
	$\gil\subseteq\cF$ be a sigma-algebra on $\Omega$.
	Then
	\begin{enumerate}[label=(\roman*)] 
		\item \label{it:basic:improper:1}it holds for every proper $\gil$-conditional $\P$-integrable random variable $X\colon\Omega\to\codom$ that \improper{X}{\gil} and
		\item \label{it:basic:improper:2}it holds\footnote{In this work we do, as usual, not distinguish between random variables and equivalence classes of random variables and, in particular, we observe that for every probability space $ ( \Omega, \cF, \P ) $,
			every $D\subseteq[-\infty,\infty]$,
			every random variable $X\colon\Omega\to D$,
			every sigma-algebra $\gil\subseteq\cF$ on $\Omega$,
			every $Y\in\E[X|\gil]$, 
			and every $A\in\mathcal{B}([-\infty,\infty])$ it holds that $\P(\co{X}{\gil}\in A)=\P(Y\in A)$.}  for every  random variable $X\colon\Omega\to\codom$ with $\P(X\geq 0)=1$ that \improper{X}{\gil} and $\P(\co{X}{\gil}\geq 0)=1$.
	\end{enumerate}\cfout.
\end{athm}

\begin{proof}[Proof of \cref{lem:basic:improper}]
	\Nobs that the fact that for every random variable $X\colon\Omega\to D$, every $A\in\gil$, and every $z\in\{-1,1\}$ it holds that 
	$\vass{X}\indicator{A}\geq \max\{zX,0\}\indicator{A}$ \proves\ that
	for every proper $\gil$-conditional $\P$-integrable random variable $X\colon\Omega\to\codom$ it holds that \improper{X}{\gil} \cfload.
	This \proves\ \cref{it:basic:improper:1}.
	\Nobs that for every random variable $X\colon\Omega\to\codom$ with $\P(X\geq 0)=1$ it holds that
	\begin{equation}
		\begin{split}
		\min\pRb{\E\PR{\max\{X,0\}},\E\PR{\max\{-X,0\}}}
		=\E\PR{\max\{-X,0\}}
		= 0.
		\end{split}
	\end{equation}
	This, \cref{it:triv:improper} in \cref{lem:cond:int:triv:sig:alg}, and
	\cref{it:cond:incr:sig:improper} in \cref{lem:cond:int:incr:sig:alg} \prove\ that for every 
	random variable $X\colon\Omega\to\codom$ with $\P(X\geq 0)=1$
	it holds that 
	\improper{X}{\gil}.
	\Moreover for every 
	random variable $X\colon\Omega\to\codom$,
	every $A\in\gil$
	with $\P(X\geq 0)=1$ and $\E\PRb{\co{X}{\gil}\indicator{A\cap\{\co{X}{\gil}<0\}}}
	=\E\PR{X\indicator{A\cap\{\co{X}{\gil}<0\}}}$ it holds that
	\begin{equation}
		\E\PRb{\co{X}{\gil}\indicator{A\cap\{\co{X}{\gil}<0\}}}
		=\E\PR{X\indicator{A\cap\{\co{X}{\gil}<0\}}}
		=
		\E\PR{X\indicator{\{X\geq 0\}\cap A\cap\{\co{X}{\gil}<0\}}}\geq0
	\end{equation}
	\cfload.
	This and the fact that for every 
	random variable $X\colon\Omega\to\codom$ with $\P(X\geq 0)=1$
	it holds that 
	\improper{X}{\gil} \prove\ that
	for every random variable $X\colon\Omega\to\codom$ with $\P(X\geq 0)=1$ it holds that \improper{X}{\gil} and $\P(\co{X}{\gil}\geq 0)=1$. This \proves[ei]\ \cref{it:basic:improper:2}. \finishproofthus
\end{proof}
\cfclear

\begin{athm}{lemma}{lem:cond:int:P:crit}
	Let $ ( \Omega, \cF, \P ) $ be a probability space,
	let $\codom\subseteq[-\infty,\infty]$ be a set,
	let $X\colon\Omega\to\codom$ be a random variable,
	and
	let $\gil\subseteq\cF$ be a sigma-algebra on $\Omega$.
	Then
	\begin{enumerate}[label=(\roman*)]
		\item \label{it:proper:P:crit}it holds that \proper{X}{\gil} if and only if
		\begin{equation}
			\P\prb{\cob{\vass{X}}{\gil}<\infty}=1,
		\end{equation}
		\item \label{it:improper:P:crit}it holds that \improper{X}{\gil} if and only if
		\begin{equation}
			\textstyle
			\P\prb{ \min_{ z \in \{-1,1\} } \co{\max\{ z X, 0 \} }{ \gil } < \infty } = 1,
		\end{equation}
				and
		\item \label{it:alt:for:improper:P:crit}it holds that
		\begin{equation}
			\textstyle
			\begin{split}
			&\textstyle\P\prb{ \min_{ z \in \{-1,1\} } \co{\max\{ z X, 0 \} }{ \gil } < \infty } 
			\\
			&= 
			\P\prb{ \pR{ \co{\max\{ X, 0 \}}{\gil} < \infty } \cup \pR{ \co{\max\{- X, 0 \}}{\gil} < \infty }}
			\end{split}
		\end{equation}
		\end{enumerate}
		\cfout.

\end{athm}

\begin{proof}[Proof of \cref{lem:cond:int:P:crit}]
	\Nobs that \cref{it:basic:improper:2} in \cref{lem:basic:improper} \proves\ that for every random variable $Y\colon\Omega\to\codom$ it holds that 	
	\improper{\vass{Y}}{\gil}
\cfload. This \proves\ that for every random variable $Y\colon\Omega\to\codom$, every $A\in\gil$ with $\E\PRb{\cob{\vass{Y}}{\gil}\allowbreak\indicator{A\cap\{\co{\vass{Y}}{\gil}=\infty\}}}
	=\E\PR{\vass{Y}\indicator{A\cap \{\co{\vass{Y}}{\gil}=\infty\}}}$ and $\E\PR{\vass{Y}\indicator{A}}<\infty$ it holds that
	\begin{equation}
		\E\PRb{\cob{\vass{Y}}{\gil}\indicator{A\cap \{\co{\vass{Y}}{\gil}=\infty\}}}
		=\E\PR{\vass{Y}\indicator{A\cap \{\co{\vass{Y}}{\gil}=\infty\}}}
		\leq\E\PR{\vass{Y}\indicator{A}}<\infty.
	\end{equation}
	This \proves\ that for every random variable $Y\colon\Omega\to\codom$ and every $A\in\gil$ with $\E\PRb{\cob{\vass{Y}}{\gil}\allowbreak\indicator{A\cap \{\co{\vass{Y}}{\gil}=\infty\}}}
	=\E\PR{\vass{Y}\indicator{A\cap \{\co{\vass{Y}}{\gil}=\infty\}}}$ and $\E\PR{\vass{Y}\indicator{A}}<\infty$ it holds that
	\begin{equation}
			\P\prb{A\cap\pRb{\cob{\vass{Y}}{\gil}=\infty}}=0.
	\end{equation}
	This \proves\ that for every proper $\gil$-conditional $\P$-integrable random variable $Y\colon\Omega\to\codom$ it holds that
	\cfadd{def:proper:cond:int}
	\begin{equation}
		\label{it:1:first:implication}
	\P\prb{\cob{\vass{Y}}{\gil}<\infty}=	1-\P\prb{\cob{\vass{Y}}{\gil}=\infty}=1
	\end{equation}
	\cfload.
	\Moreover for every random variable $Y\colon\Omega\to\codom$ and every $A\in\gil$, $n\in\N$ with $\E\PRb{\cob{\vass{Y}}{\gil}\allowbreak\indicator{A\cap \{\co{\vass{Y}}{\gil}\leq n\}}}
	=\E\PR{\vass{Y}\indicator{A\cap \{\co{\vass{Y}}{\gil}\leq n\}}}$ it holds that
			\begin{equation}
		\label{eq:proper:on:intersected:covers}
		\E\PR{\vass{Y}\indicator{A\cap\{\co{\vass{Y}}{\gil}\leq n\}}}
		=
		\E\PR{\cob{\vass{Y}}{\gil}\indicator{A\cap\{\co{\vass{Y}}{\gil}\leq n\}}}
		\leq
		\E\PR{n\indicator{A\cap\{\co{\vass{Y}}{\gil}\leq n\}}}
		\leq n <\infty
		.
	\end{equation}
			\Moreover for every random variable $Y\colon\Omega\to\codom$ and  every $A\in\gil$
			 with $\P\prb{\cob{\vass{Y}}{\gil}<\infty}=1$ 
			 and 
			 $\E\PRb{\cob{\vass{Y}}{\gil}\allowbreak\indicator{A\cap\{\co{\vass{Y}}{\gil}=\infty\}}}
			 =\E\PR{\vass{Y}\indicator{A\cap \{\co{\vass{Y}}{\gil}=\infty\}}}$
			  it holds  that
	\begin{equation}
		\begin{split}
			\E\PRb{\vass{Y}\indicator{A\cap\{\co{\vass{Y}}{\gil}= \infty\}}}
			=
			\E\PRb{\cob{\vass{Y}}{\gil}\indicator{A\cap\{\co{\vass{Y}}{\gil}= \infty\}}}
			=0.
		\end{split}
	\end{equation}
	This and \eqref{eq:proper:on:intersected:covers} \prove\ that for every random variable $Y\colon\Omega\to\codom$ 
	with $\P\prb{\cob{\vass{Y}}{\gil}<\infty}=1$ it holds that \proper{Y}{\gil}. This and \cref{it:1:first:implication} \prove\ \cref{it:proper:P:crit}.
	\Moreover for every random variable $Y\colon\Omega\to D$ it holds that
	$\max\{Y,0\}$ and $\max\{-Y,0\}$ are improper $\gil$-conditional $\P$-integrable.
		This \proves\ that for every random variable $Y\colon\Omega\to D$,
	there exist $A_n\in\gil$, $n\in\N$, such that for all $n\in\N$, $z\in\{-1,1\}$, $B\in\gil$ it holds that
	\begin{equation}
		\label{eq:max:parts:improper}
		\begin{split}
			\textstyle
			\Omega=\cup_{k\in\N}A_k
			\qqandqq
			\E\PRb{\co{\max\{zY,0\}}{\gil}\indicator{A_n\cap B}}=\E\PR{\max\{zY,0\}\indicator{A_n\cap B}}.
		\end{split}
	\end{equation}
	\Moreover for every improper $\gil$-conditional $\P$-integrable random variable $Y\colon\Omega\to D$,
	there exist $A_n\in\gil$, $n\in\N$, such that for all $n\in\N$, $z\in\{-1,1\}$, $B\in\gil$ it holds that
	\begin{equation}
		\begin{split}
			\textstyle
			\Omega=\cup_{k\in\N}A_k,
			\qquad
			\min_{k\in\{-1,1\}}\E\PR{\max\{kY,0\}\indicator{A_n}}<\infty,
		\end{split}
	\end{equation}
	\begin{equation}
		\qqandqq
		\E\PRb{\co{Y}{\gil}\indicator{A_n\cap B}}=\E\PR{Y\indicator{A_n\cap B}}.
	\end{equation}
	This and \eqref{eq:max:parts:improper} \prove\ that for every improper $\gil$-conditional $\P$-integrable random variable $Y\colon\Omega\to D$,
	there exist $A_n\in\gil$, $n\in\N$, such that for all $n\in\N$, $z\in\{-1,1\}$, $B\in\gil$ it holds that
	\begin{equation}
		\begin{split}
			\label{eq:suitable:decomposition}
			\textstyle
			\Omega=\cup_{k\in\N}A_k,
			\qquad
			\min_{k\in\{-1,1\}}\E\PR{\max\{kY,0\}\indicator{A_n}}<\infty,
		\end{split}
	\end{equation}
	\begin{equation}
		\label{eq:suitable:decomposition2}
		\qqandqq
		\E\PRb{\co{\max\{zY,0\}}{\gil}\indicator{A_n\cap B}}=\E\PR{\max\{zY,0\}\indicator{A_n\cap B}}.
	\end{equation}
	\Moreover for every random variable $Y\colon\Omega\to D$ and
	every $A\in\gil$ with
	\begin{equation}
		\begin{split}
		&\textstyle\min_{ z \in \{-1,1\} }\E\PRb{\co{\max\{ z Y, 0 \} }{ \gil }\indicator{A\cap\{ \co{\max\{Y, 0 \} }{ \gil }=\infty\}\cap\{ \co{\max\{-Y, 0 \} }{ \gil }=\infty\}}}
		\\&\textstyle=
		\min_{ z \in \{-1,1\} }\E\PRb{ \max\{zY, 0 \} \indicator{A\cap\{ \co{\max\{Y, 0 \} }{ \gil }=\infty\}\cap\{ \co{\max\{-Y, 0 \} }{ \gil }=\infty\}}}
		\end{split}
	\end{equation}
	and $\min_{k\in\{-1,1\}}\E\PR{\max\{kY,0\}\indicator{A}}<\infty$
	 it holds that
	 \begin{equation}
	 	\begin{split}
	 		\textstyle
	 		&\textstyle\E\PRb{\min_{ z \in \{-1,1\} } \co{\max\{ z Y, 0 \} }{ \gil }\indicator{A\cap\{ \co{\max\{Y, 0 \} }{ \gil }=\infty\}\cap\{ \co{\max\{-Y, 0 \} }{ \gil }=\infty\}}}
	 		\\&\leq\textstyle\min_{ z \in \{-1,1\} }\E\PRb{ \co{\max\{ z Y, 0 \} }{ \gil }\indicator{A\cap\{ \co{\max\{Y, 0 \} }{ \gil }=\infty\}\cap\{ \co{\max\{-Y, 0 \} }{ \gil }=\infty\}}}
	 		\\&\textstyle=\min_{ z \in \{-1,1\} }\E\PRb{ \max\{zY, 0 \} \indicator{A\cap\{ \co{\max\{Y, 0 \} }{ \gil }=\infty\}\cap\{ \co{\max\{-Y, 0 \} }{ \gil }=\infty\}}}
	 		\\&\textstyle\leq\min_{ z \in \{-1,1\} }\E\PRb{ \max\{zY, 0 \} \indicator{A}}<\infty.
	 	\end{split}
	 \end{equation}
	 This, \eqref{eq:suitable:decomposition}, and \eqref{eq:suitable:decomposition2} \prove\ that for every improper $\gil$-conditional $\P$-integrable random variable $Y\colon\Omega\to D$ it holds that
	 \begin{equation}
	 	\label{eq:ii:implication}
	 	\begin{split}
	 	&\textstyle\P\prb{ \min_{ z \in \{-1,1\} } \co{\max\{ z Y, 0 \} }{ \gil } < \infty } 
	 	\\&= 1- \P\prb{  \co{\max\{Y, 0 \} }{ \gil }=\infty, \co{\max\{-Y, 0 \} }{ \gil }=\infty } 
	 	= 1.
	 	\end{split}
	 \end{equation}
	 \Moreover for every random variable $Y\colon\Omega\to D$ and
	 every $A\in\gil$, $n\in\N$, $k\in\{-1,1\}$
		  with
		  \begin{equation}
		  	\begin{split}
		  	\E\PR{\co{\max\{Y,0\}}{\gil}\indicator{A\cap\{\co{\max\{kY,0\}}{\gil}\leq n\}}}
		  	&=\E\PR{\max\{Y,0\}\indicator{A\cap\{\co{\max\{kY,0\}}{\gil}\leq n\}}}
		  	\end{split}
		  \end{equation}
		  and
		  \begin{equation}
		  	\begin{split}
		  		\E\PR{\co{\max\{-Y,0\}}{\gil}\indicator{A\cap\{\co{\max\{kY,0\}}{\gil}\leq n\}}}
		  		&=\E\PR{\max\{-Y,0\}\indicator{A\cap\{\co{\max\{kY,0\}}{\gil}\leq n\}}}
		  	\end{split}
		  \end{equation}
	it holds that
	 \begin{equation}
	 	\label{eq:finite:for:improper:on:finite}
	 	\begin{split}
	 		&\textstyle\min_{z\in\{-1,1\}}\E\PR{\max\{zY,0\}\indicator{A\cap\{\co{\max\{kY,0\}}{\gil}\leq n\}}}
	 		\\&=\textstyle\min_{z\in\{-1,1\}}\E\PR{\co{\max\{zY,0\}}{\gil}\indicator{A\cap\{\co{\max\{kY,0\}}{\gil}\leq n\}}}
	 		\\&\leq
	 			\E\PR{\co{\max\{kY,0\}}{\gil}\indicator{A\cap\{\co{\max\{kY,0\}}{\gil}\leq n\}}}
	 		\leq
	 		\E\PR{n}<\infty.
	 	\end{split}
	 \end{equation}
	 \Moreover the fact that for every random variable $Y\colon\Omega\to[-\infty,\infty]$ with $\P(\vass{Y}<\infty)=1$ it holds that $\E\PR{\vass{Y}\indicator{\{\vass{Y}=\infty\}}}=0$ \prove\ that for every random variable $Y\colon\Omega\to D$ with $\P\prb{ \min_{ z \in \{-1,1\} } \co{\max\{ z Y, 0 \} }{ \gil } < \infty } = 1$
	 it holds that
	 \begin{equation}
	 	\label{eq:finite:for:improper:on:infinite}
	 	\begin{split}
	 		\textstyle\min_{z\in\{-1,1\}}\E\PR{\max\{zY,0\}\indicator{\{\co{\max\{Y,0\}}{\gil}=\co{\max\{-Y,0\}}{\gil}=\infty\}}}
	 		=0.
	 	\end{split}
	 \end{equation}
	 Combining this and
	 \eqref{eq:max:parts:improper} with
	  \eqref{eq:finite:for:improper:on:finite} 
	 \proves\ that for every random variable $Y\colon\Omega\to D$ with $\P\prb{ \min_{ z \in \{-1,1\} } \co{\max\{ z Y, 0 \} }{ \gil } < \infty } = 1$ it holds that \improper{Y}{\gil}.
	This and \eqref{eq:ii:implication} \prove[ei]\ \cref{it:improper:P:crit}.
	\Nobs that for all random variables $Y\colon\Omega\to[-\infty,\infty]$ and $Z\colon\Omega\to[-\infty,\infty]$ it holds that
	\begin{equation}
		\begin{split}
			\textstyle\pRb{\min\{Y,Z\}< \infty}
			=\pRb{Y< \infty}\cup\pRb{Z< \infty}.
		\end{split}
	\end{equation}
This \proves\ \cref{it:alt:for:improper:P:crit}.
	\finishproofthus
	\end{proof}

The following result, \cref{lem:norm:cond:exp} below, relates the concept of generalized conditional expectations in \cref{gen:conditional:expectation} to the concept of standard conditional expectations (cf., \eg, 
\cite[Definition~12.1.3]{Krishna2006},
\cite[Definition~8.11]{klenkeprobability}, 
\cite[Theorem~8.1]{Kallenberg}, and
\cite[Chapter~10]{Dudley2002}). 
More specifically, \cref{it:localizations:cond:exp2} in \cref{lem:norm:cond:exp} shows in the situation where the considered random variable is proper integrable that 
the generalized conditional expectation in \cref{gen:conditional:expectation} (cf.\ \cref{def:cond:exp})
 coincides with 
 the standard conditional expectation (cf., \eg, 
\cite[Definition~12.1.3]{Krishna2006},
\cite[Definition~8.11]{klenkeprobability}, 
\cite[Theorem~8.1]{Kallenberg}, and
\cite[Chapter~10]{Dudley2002})
 and
\cref{it:localizations:cond:exp} in \cref{lem:norm:cond:exp} proves in the situation where the random variable is improper integrable that 
the generalized conditional expectation in \cref{gen:conditional:expectation} (cf.\ \cref{def:cond:exp})
coincides with 
the standard conditional expectation (cf., \eg, 
\cite[Definition~12.1.3]{Krishna2006},
\cite[Remark~8.16]{klenkeprobability}, 
\cite[Exercise~5~in~Chapter~8]{Kallenberg}, and
\cite[Exercise~7~in~Section~10.1]{Dudley2002}).

\cfclear
\begin{lemma}
	\label{lem:norm:cond:exp}
	Let $ ( \Omega, \cF, \P ) $ be a probability space,
	let $\codom\subseteq[-\infty,\infty]$ be a set,
	let $X\colon\Omega\to\codom$ be a random variable,
	let $\gil\subseteq\cF$ be a sigma-algebra on $\Omega$, and assume that \improper{X}{\gil}
	 \cfload.
	Then
	\begin{enumerate}[label=(\roman*)] 
		\item \label{it:localizations:cond:exp2}it holds for all $A\in\gil$ with $\E\PR{\vass{X}\indicator{A}}<\infty$ that $\E\PR{X\indicator{A}}=\E\PR{\co{X}{\gil}\indicator{A}}$
		 and
		\item \label{it:localizations:cond:exp}it holds for all $A\in\gil$ with $\min_{z\in\{-1,1\}}\E\PR{\max\{zX,0\}\indicator{A}}<\infty$ that $\E\PR{X\indicator{A}}=\E\PR{\co{X}{\gil}\indicator{A}}$
	\end{enumerate}
	\cfout.
\end{lemma}

\begin{proof}[Proof of \cref{lem:norm:cond:exp}]
	Throughout this proof 
	let $B_n\in\gil$, $n\in\N$, satisfy that
	\begin{enumerate}[label=(\Roman*)] 
		\item\label{it:norm:cond:exp:1} it holds that $\Omega=\cup_{n\in\N}B_n$, 
		\item\label{it:norm:cond:exp:2} it holds for all $i\in\N$, $j\in\N\backslash\{i\}$ that $B_i\cap B_j=\emptyset$, and
		\item\label{it:norm:cond:exp:4} it holds for all $n\in\N$, $A\in\gil$ that $\E\PR{X\indicator{B_n\cap A}}=\E\PR{\co{X}{\gil}\indicator{B_n\cap A}}$
	\end{enumerate}
	 \cfadd{def:cond:exp,lem:cond:exp:ex}\cfload. 
	\Nobs that for all $A\in\gil$, $z\in\{-1,1\}$, $n\in\N$ with $\E\PR{\max\{zX,0\}\indicator{A}}<\infty$ it holds that
	\begin{equation}
		\begin{split}
		\E\PR{\max\{z\co{X}{\gil},0\}\indicator{B_n\cap A}}
		&=
		z\E\PR{\co{X}{\gil}\indicator{B_n\cap A\cap\{z\co{X}{\gil}>0\}}}
		\\&=z\E\PR{X\indicator{B_n\cap A\cap\{z\co{X}{\gil}>0\}}}
		\leq \E\PR{\max\{zX,0\}\indicator{B_n\cap A}}
		<\infty.
		\end{split}
	\end{equation}
	\Nobs that \cref{it:norm:cond:exp:1,it:norm:cond:exp:2} hence \prove\ that for all $A\in\gil$, $z\in\{-1,1\}$ with $\E\PR{\max\{zX,0\}\indicator{A}}<\infty$ it holds that
	\begin{equation}
		\begin{split}
		\E\PR{\max\{z\co{X}{\gil},0\}\indicator{A}}
		=
		\sum_{n\in\N}
		\E\PR{\max\{z\co{X}{\gil},0\}\indicator{B_n\cap A}}
		&\leq
		\sum_{n\in\N}
		\E\PR{\max\{zX,0\}\indicator{B_n\cap A}}
		\\&=
		\E\PR{\max\{zX,0\}\indicator{A}}
		.
		\end{split}
	\end{equation}
	This, \cref{it:norm:cond:exp:1}, \cref{it:norm:cond:exp:2}, and \cref{it:norm:cond:exp:4} \prove[ps]\ that for all $A\in\gil$ with $\min_{z\in\{-1,1\}}\E\PR{\max\{zX,0\}\indicator{A}}\allowbreak<\infty$ it holds that
	\begin{equation}
		\textstyle
		\E\PR{X\indicator{A}}
		=
		\sum_{n=1}^{\infty}\E\PR{X\indicator{B_n\cap A}}
		=
		\sum_{n=1}^{\infty}\E\PR{\co{X}{\gil}\indicator{B_n\cap A}}
		=\E\PR{\co{X}{\gil}\indicator{A}}.
	\end{equation}
	This \proves\ \cref{it:localizations:cond:exp}.
		\Nobs that \cref{it:localizations:cond:exp} \proves\ \cref{it:localizations:cond:exp2}.
	The proof of \cref{lem:norm:cond:exp} is thus complete.
	\end{proof}

\subsection{Factorization lemma for generalized conditional expectations}
\label{subsection:3.2} 

In the next result, \cref{lem:2.9:from:1609.07031} below, we combine the well-known factorization lemma for conditional expectations for non-negative functions (cf., \eg, \cite[Lemma~2.9]{JentzenPusnik_2018} and  \cite[Proposition~1.12]{Prato}) with \cref{it:localizations:cond:exp} in \cref{lem:norm:cond:exp} to reformulate the factorization lemma in the situation of generalized conditional expectations for non-negative functions.

\begin{lemma}[{Factorization lemma for conditional expectations for non-negative functions}]
	\label{lem:2.9:from:1609.07031}
	Let $ ( \Omega, \cF, \P ) $ be a probability space,
	let $(D,\mathcal{D})$ and $(E,\mathcal{E})$ be measurable spaces,
	let $\Phi\colon D\times E\to[0,\infty]$ be measurable,
	let $\gil\subseteq\cF$ be a sigma-algebra on $\Omega$, 
	let $X\colon\Omega\to D$ be $\gil$-measurable,
	let $Y\colon\Omega\to E$ be a random variable,
	assume that $Y$ is independent\footnote{Note that 
		for every probability space
		$ ( \Omega, \cF, \P ) $, 
		every sigma-algebra $\cG\subseteq\cF$ on $\Omega$,
		every measurable space $(E,\mathcal{E})$,
		and every random variable $Y\colon \Omega\to E$ it holds that $Y$ is independent of $\cG$ if and only if
		it holds for all $A\in\mathcal{E}$, $B\in\cG$ that $\P(\{Y\in A\}\cap B)=\P(Y\in A)\P(B)$.} of $\gil$, and
	let $\phi\colon D\to[0,\infty]$ satisfy for all $x\in D$ that
	\begin{equation}
		\phi(x)=\E\PR{\Phi(x,Y)}.
	\end{equation}
	Then 
	\begin{enumerate}[label=(\roman*)]
		\item \label{it1:lem.29}it holds that $\phi$ is measurable and
		\item \label{it2:lem.29}it holds $\P$-\as\ that $\co{\Phi(X,Y)}{\gil}=\phi(X)$.
	\end{enumerate}
\end{lemma}

\begin{cproof}{lem:2.9:from:1609.07031}
	\Nobs \cref{lem:norm:cond:exp} and \cite[Lemma~2.9]{JentzenPusnik_2018} \prove[p]\ \cref{it1:lem.29,it2:lem.29}. 
\end{cproof}

\cfclear
\begin{athm}{prop}{lem:2.9:plus:minus}[Factorization lemma for conditional expectations for general functions]
	Let $ ( \Omega, \cF, \P ) $ be a probability space,
	let $(D,\mathcal{D})$ and $(E,\mathcal{E})$ be measurable spaces,
	let $\Phi\colon D\times E\to[-\infty,\infty]$ be measurable,
	let $\gil\subseteq\cF$ be a sigma-algebra  on $\Omega$, 
	let $X\colon\Omega\to D$ be $\gil$-measurable,
	let $Y\colon\Omega\to E$ be a random variable,
	assume that $Y$ is independent of $\gil$, 
	assume for all $x\in D$ that $\min_{z\in\{-1,1\}}\E\PR{\max\{z\Phi(x,Y),0\}}<\infty$,
	and 
	let $\phi\colon D\to[-\infty,\infty]$ satisfy for all $x\in D$ that
	\begin{equation}
		\label{eq:pos:neg:fact:exp}
		\phi(x)=\E\PR{\Phi(x,Y)}.
	\end{equation}
	Then 
	\begin{enumerate}[label=(\roman*)]
		\item \label{it:pos:neg:mb} it holds that $\phi$ is measurable,
		\item \label{it:pos:neg:improper} it holds that \improper{\Phi(X,Y)}{\gil}, and
		\item \label{it:pos:neg:fact} it holds $\P$-\as\ that $\co{\Phi(X,Y)}{\gil}
		=\phi(X)$
	\end{enumerate}
	\cfout.
\end{athm}

\begin{proof}[Proof of \cref{lem:2.9:plus:minus}]
	Throughout this proof let $G\colon D\times E\to[0,\infty]$ and $H\colon D\times E\to[0,\infty]$ satisfy for all $x\in D$, $y\in E$ that
	\begin{equation}
		\label{eq:plus:minus:setup}
		G(x,y)=\max\{\Phi(x,y),0\}
		\qqandqq
		H(x,y)=\max\{-\Phi(x,y),0\},
	\end{equation}
	let $g\colon D\to[0,\infty]$ and $h\colon D\to[0,\infty]$ satisfy for all $x\in D$ that
	\begin{equation}
		\label{eq:plus:minus:setup2}
		g(x)=\E[G(x,Y)]
		\qqandqq
		h(x)=\E[H(x,Y)],
	\end{equation}
	and for every $n\in\N$ let $A_n\subseteq \Omega$ and $B_n\subseteq \Omega$ satisfy 
	\begin{equation}
		\label{eq:covering:gil:mb:sets}
		A_n=\{g(X)\leq n\}
		\qqandqq
		B_n=\{h(X)\leq n\}.
	\end{equation}
	\Nobs that \eqref{eq:plus:minus:setup} and the assumption that $\Phi$ is measurable \prove\ that 
	\begin{equation}
		\label{eq:pos:min:parts:mb}
		\text{$G$ and $H$ are measurable}.
	\end{equation}
	This, \eqref{eq:plus:minus:setup2}, and \cref{lem:2.9:from:1609.07031} (applied with
	$( \Omega, \cF, \P )\curvearrowleft( \Omega, \cF, \P )$,
	$( D,\mathcal{D} )\curvearrowleft( D,\mathcal{D} )$,
	$( E,\mathcal{E} )\curvearrowleft( E,\mathcal{E} )$,
	$\Phi\curvearrowleft G$,
	$\gil\curvearrowleft\gil$,
	$X\curvearrowleft X$,
	$Y\curvearrowleft Y$,
	$\phi\curvearrowleft g$ in the notation of \cref{lem:2.9:from:1609.07031})
	\prove\ that
	\begin{enumerate}[label=(\Roman*)]
		\item \label{it:pos:mb} it holds that $g$ is measurable and
		\item \label{it:pos:cond:exp} it holds $\P$-\as\ that $\co{G(X,Y)}{\gil}
		=g(X)$
	\end{enumerate}
	\cfload. \Moreover \eqref{eq:plus:minus:setup2}, \eqref{eq:pos:min:parts:mb}, and \cref{lem:2.9:from:1609.07031} (applied with
	$( \Omega, \cF, \P )\curvearrowleft( \Omega, \cF, \P )$,
	$( D,\mathcal{D} )\curvearrowleft( D,\mathcal{D} )$,
	$( E,\mathcal{E} )\curvearrowleft( E,\mathcal{E} )$,
	$\Phi\curvearrowleft H$,
	$\gil\curvearrowleft\gil$,
	$X\curvearrowleft X$,
	$Y\curvearrowleft Y$,
	$\phi\curvearrowleft h$ in the notation of \cref{lem:2.9:from:1609.07031})
	\prove\ that
	\begin{enumerate}[label=(\Alph*)]
		\item \label{it:neg:mb} it holds that $h$ is measurable and
		\item \label{it:neg:cond:exp} it holds $\P$-\as\ that $\co{H(X,Y)}{\gil}
		=h(X)$.
	\end{enumerate}
	\Moreover 
	\eqref{eq:pos:neg:fact:exp}, 
	\eqref{eq:plus:minus:setup},
	\eqref{eq:plus:minus:setup2}, and the assumption that for all $x\in D$ it holds that $\min\{\E\PR{G(x,Y)},\E\PR{H(x,Y)}\}=\min_{z\in\{-1,1\}}\E\PR{\max\{z\Phi(x,Y),0\}}<\infty$ \prove\ that for all $x\in D$ it holds that
	\begin{equation}
		\label{eq:decomp:in:pos:neg}
		\phi(x)=\E\PR{\Phi(x,Y)}=\E\PR{G(x,Y)-H(x,Y)}=\E\PR{G(x,Y)}-\E\PR{H(x,Y)}=g(x)-h(x).
	\end{equation}
	This, \cref{it:pos:mb}, and \cref{it:neg:mb} \prove\ that $\phi$ is measurable.
	This \proves[pei] \cref{it:pos:neg:mb}. 
	\Nobs that 
	\eqref{eq:covering:gil:mb:sets} and the assumption that for all $x\in D$ it holds that $\min_{z\in\{-1,1\}}\E\PR{\max\{z\Phi(x,Y),0\}}<\infty$ \prove\ that for all $\omega\in\Omega$ it holds that
	\begin{equation}
		\begin{split}
		\min\{g(X(\omega)),h(X(\omega))\}
		&=\min\pRb{\E\PR{G(X(\omega),Y)},\E\PR{H(X(\omega),Y)}}
		\\&=\min\pRb{\E\PR{\max\{\Phi(X(\omega),Y),0\}},\E\PR{\max\{-\Phi(X(\omega),Y),0\}}}
		\\&<\infty
		.
		\end{split}
	\end{equation}
	This, \cref{it:pos:mb},
	\cref{it:neg:mb}, and the assumption that $X$ is $\gil$-measurable \prove\ that for all $m\in\N$ it holds that
	\begin{equation}
		\label{eq:covering:max:cond:exp}
		\Omega=\PRb{\cup_{n\in\N}A_n}\cup\PRb{\cup_{n\in\N}B_n},
		\qquad
		A_m\in\gil,
		\qqandqq
		B_m\in\gil
		.
	\end{equation}
		\Moreover
		\cref{it:pos:cond:exp},
		\cref{it:neg:cond:exp}, and
		\cref{it:localizations:cond:exp} in \cref{lem:norm:cond:exp}
		 \prove\ that for all $C\in\gil$ it holds that
	\begin{equation}
		\label{eq:localization:on:any:gil}
		\begin{split}
				\E\PR{G(X,Y)\indicator{C}}=\E\PR{g(X)\indicator{C}}
				\qqandqq
				\E\PR{H(X,Y)\indicator{C}}=\E\PR{h(X)\indicator{C}}
				.
		\end{split}
	\end{equation}
			This and \eqref{eq:covering:gil:mb:sets} \prove\ that for all $n\in\N$ it holds that
	\begin{equation}
		\label{eq:one:is:finite}
		\begin{split}
			\E\PR{G(X,Y)\indicator{A_n}}
			=
			\E\PR{g(X)\indicator{A_n}}
			\leq n
			\qqandqq
				\E\PR{H(X,Y)\indicator{B_n}}
			=
			\E\PR{h(X)\indicator{B_n}}
			\leq n
			.
		\end{split}
	\end{equation}
	This,
	\eqref{eq:plus:minus:setup},
	\eqref{eq:covering:max:cond:exp}, and \eqref{eq:localization:on:any:gil} \prove\ that for all $n\in\N$, $C\in\{A_n,B_n\}$ it holds that
	\begin{equation}
		\label{eq:Phi:max:indeed:improper}
		\begin{split}
		&\textstyle
		\min_{z\in\{-1,1\}}\E\PR{\pr{\max\{z\Phi(X,Y),0\}+\max\{z\phi(X),0\}}\indicator{C}}
		\\&\leq
		\textstyle
		\min\{
		\E\PR{(G(X,Y)+g(X))\indicator{C}}
		,
		\E\PR{(H(X,Y)+h(X))\indicator{C}}
		\}
		\leq
		2n
		<\infty
		.
		\end{split}
		\end{equation}
	Combining this with \eqref{eq:covering:max:cond:exp} \proves\ that \improper{\Phi(X,Y)}{\gil}.
	This \proves\ \cref{it:pos:neg:improper}.
	\Nobs that 
	\eqref{eq:decomp:in:pos:neg},
	\eqref{eq:covering:max:cond:exp},
	\eqref{eq:localization:on:any:gil},
	 \eqref{eq:one:is:finite}, and
	 \eqref{eq:Phi:max:indeed:improper} \prove\ that for all $n\in\N$, $C_1\in\{A_n,B_n\}$, $C_2\in\gil$ it holds that
	\begin{equation}
		\begin{split}
			\E\PR{\phi(X)\indicator{C_1\cap C_2}}
			=\E\PR{(g(X)-h(X))\indicator{C_1\cap C_2}}
			&=
			\E\PR{g(X)\indicator{C_1\cap C_2}}-\E\PR{h(X)\indicator{C_1\cap C_2}}
			\\&=
			\E\PR{G(X,Y)\indicator{C_1\cap C_2}}-\E\PR{H(X,Y)\indicator{C_1\cap C_2}}
			\\&=
			\E\PR{(G(X,Y)-H(X,Y))\indicator{C_1\cap C_2}}
			\\&=
			\E\PR{\Phi(X,Y)\indicator{C_1\cap C_2}}
			.
		\end{split}
	\end{equation}
	Combining
	this, \eqref{eq:covering:max:cond:exp},
	and \eqref{eq:Phi:max:indeed:improper}
	with the fact that $\phi(X)$ is $\gil$-measurable
	\proves\ that it holds $\P$-\as\ that $\phi(X)=\co{\Phi(X,Y)}{\gil}$. This \proves[pei] \cref{it:pos:neg:fact}. \finishproofthus
\end{proof}

\subsection{Factorization lemma for generalized conditional variances}
\label{subsection:3.3}

\cfclear
\begin{athm}{prop}{lem:factorization2}
	Let $ ( \Omega, \cF, \P ) $ be a probability space,
	let $\gil\subseteq\cF$ be a sigma-algebra on $\Omega$,
	let $m,n\in\N$,
	let $\Phi\colon\R^m\times\R^n\to\R$ be measurable,
	let $X\colon\Omega\to\R^m$ be $\gil$-measurable,
	and $Y\colon\Omega\to\R^n$ be a random variable,
	assume that $Y$ is independent of $\gil$,
	assume for all $x\in\R^m$ that $\E\PR{\vass{\Phi(x,Y)}}<\infty$,
	and	let $\psi\colon\R^m\to[0,\infty]$ satisfy for all $x\in\R^m$ that
	\begin{equation}
		\label{eq:setup:var:function:2:var}
		\psi(x)
		=\var(\Phi(x,Y))
		.
	\end{equation}
	Then
	\begin{enumerate}[label=(\roman*)]
		\item \label{it:mb:fact:for:var0}it holds that $\psi$ is measurable,
		\item \label{it:mb:fact:improper} it holds that \improper{\Phi(X,Y)}{\gil},
		\item \label{it:mb:fact:for:var01}it holds $\P$-\as\ that $\cob{(\Phi(X,Y)-\co{\Phi(X,Y)}{\gil})^2}{\gil}=\psi(X)$, and
		\item \label{it:mb:fact:for:var02}it holds that $\E\PRb{(\Phi(X,Y)-\E\PR{\Phi(X,Y)|\gil})^2}=\E\PR{\psi(X)}$
	\end{enumerate}
	\cfout.
\end{athm}

\begin{proof}[Proof of \cref{lem:factorization2}]
	Throughout this proof let $\phi\colon\R^m\to\R$, $f\colon\R^m\to[0,\infty]$, and $F\colon\R^m\times\R^n\to[0,\infty]$ satisfy for all $x\in\R^m$, $y\in\R^n$ that
	\begin{equation}
		\label{eq:setup:for:fact1}
		\phi(x)=\E\PR{\Phi(x,Y)},
		\qquad
		f(x)=\E\PR{\pr{\Phi(x,Y)}^2},
		\qqandqq
		F(x,y)=(\Phi(x,y))^2.
	\end{equation}
	\Nobs that \eqref{eq:setup:var:function:2:var} and \eqref{eq:setup:for:fact1} \prove\ that for all $x\in\R^m$ it holds that
	\begin{equation}
		\label{eq:decomp:psi:in:phi:and:fact}
		\begin{split}
			\psi(x)
			&=\E\PR{\pr{\Phi(x,Y)-\E\PR{\Phi(x,Y)}}^2}
			\\&=\E\PR{\pr{\Phi(x,Y)}^2}-2\E\PR{\Phi(x,Y)}\E\PR{\Phi(x,Y)}+(\E\PR{\Phi(x,Y)})^2
			\\&=\E\PR{\pr{\Phi(x,Y)}^2}-(\E\PR{\Phi(x,Y)})^2
			=f(x)-(\phi(x))^2.
		\end{split}
	\end{equation}
	\Moreover \eqref{eq:setup:for:fact1},
	the assumption that for all $x\in\R^m$ it holds that $\E\PR{\vass{\Phi(x,Y)}}<\infty$ and \cref{lem:2.9:plus:minus} (applied with 
	$( \Omega, \cF, \P )\curvearrowleft( \Omega, \cF, \P )$,
	$( D,\mathcal{D} )\curvearrowleft( \R^m,\mathcal{B}(\R^m) )$,
	$( E,\mathcal{E} )\curvearrowleft( \R^n,\mathcal{B}(\R^n)  )$,
	$\Phi\curvearrowleft \Phi$,
	$\gil\curvearrowleft\gil$,
	$X\curvearrowleft X$,
	$Y\curvearrowleft Y$,
	$\phi\curvearrowleft \phi$
	in the notation of \cref{lem:2.9:plus:minus}) \prove\ that
	\begin{enumerate}[label=(\Roman*)]
		\item \label{it:fact:mb} it holds that $\phi$ is measurable,
		\item \label{it:fact:improper} it holds that \improper{\Phi(X,Y)}{\gil}, and
		\item \label{it:fact:cond:exp} it holds $\P$-\as\ that $\co{\Phi(X,Y)}{\gil}
		=\phi(X)$
	\end{enumerate}
	\cfload. \Nobs that \cref{it:fact:cond:exp} and the assumption that for all $x\in\R^m$ it holds that $\E\PR{\vass{\Phi(x,Y)}}<\infty$ \prove\ that $\P$-\as\ it holds that
	\begin{equation}
		\label{eq:finite:first:cond:exp:mom}
		\vass{\co{\Phi(X,Y)}{\gil}}=\vass{\phi(X)}<\infty.
	\end{equation}
	\Moreover \eqref{eq:setup:for:fact1} and the assumption that $\Phi$ is measurable \prove\ that $F$ is measurable and that for all $x\in\R^m$ it holds that
	\begin{equation}
		f(x)
		=\E\PR{\pr{\Phi(x,Y)}^2}
		=\E\PR{F(x,Y)}.
	\end{equation}
	This and \cref{lem:2.9:from:1609.07031} (applied with 
	$( \Omega, \cF, \P )\curvearrowleft( \Omega, \cF, \P )$,
	$( D,\mathcal{D} )\curvearrowleft( \R^m,\mathcal{B}(\R^m) )$,
	$( E,\mathcal{E} )\curvearrowleft( \R^n,\mathcal{B}(\R^n)  )$,
	$\Phi\curvearrowleft F$,
	$\gil\curvearrowleft\gil$,
	$X\curvearrowleft X$,
	$Y\curvearrowleft Y$,
	$\phi\curvearrowleft f$
	in the notation of \cref{lem:2.9:from:1609.07031}) \prove\ that
	\begin{enumerate}[label=(\Alph*)]
		\item \label{it:fact2:mb}  it holds that $f$ is measurable and
		\item \label{it:fact2:cond:exp} it holds $\P$-\as\ that $\E\PR{F(X,Y)|\gil}
		=f(X)$.
	\end{enumerate}
	\Nobs that 
	\eqref{eq:decomp:psi:in:phi:and:fact},
	\cref{it:fact:mb}, 
	and \cref{it:fact2:mb} \prove\ that $\psi$ is measurable.
	This \proves[pei] \cref{it:mb:fact:for:var0}.
	\Nobs that 
	\cref{it:fact:cond:exp},
	\cref{it:fact2:cond:exp},
	\eqref{eq:setup:for:fact1}, and
	\eqref{eq:decomp:psi:in:phi:and:fact} \prove\ that there exist $A_n\in\gil$, $n\in\N$, with $\Omega=\cup_{n\in\N}A_n$ such that for all $n\in\N$, $B\in\gil$ it holds that
	\begin{equation}
		\label{eq:cond:var:core:est}
		\begin{split}
			\E\PR{\psi(X)\indicator{A_n\cap B}}
			&=
			\E\PRb{\prb{f(X)-(\phi(X))^2}\indicator{A_n\cap B}}
			\\&=
			\E\PRb{\prb{\co{F(X,Y)}{\gil}-(\co{\Phi(X,Y)}{\gil})^2}\indicator{A_n\cap B}}
			\\&=
			\E\PRb{\prb{F(X,Y)-(\co{\Phi(X,Y)}{\gil})^2}\indicator{A_n\cap B}}
			\\&=
			\E\PRb{\prb{\pr{\Phi(X,Y)}^2-(\co{\Phi(X,Y)}{\gil})^2}\indicator{A_n\cap B}}
			.
		\end{split}
	\end{equation}
\Nobs that \eqref{eq:setup:for:fact1},
\eqref{eq:decomp:psi:in:phi:and:fact},
\cref{it:fact:mb,it:fact:cond:exp}, and
\cref{it2:lem.29} in \cref{lem:2.9:from:1609.07031}
(applied with
	$( \Omega, \cF, \P )\curvearrowleft( \Omega, \cF, \P )$,
	$( D,\mathcal{D} )\curvearrowleft( \R^m,\mathcal{B}(\R^m) )$,
	$( E,\mathcal{E} )\curvearrowleft( \R^n,\mathcal{B}(\R^n)  )$,
	$\Phi\curvearrowleft
	\pr{(\R^m\times\R^n)\ni(x,y)\mapsto
		\pr{\Phi(x,y)-\phi(x)}^2\in[0,\infty)}$,
	$\gil\curvearrowleft\gil$,
	$X\curvearrowleft X$,
	$Y\curvearrowleft Y$,
	$\phi\curvearrowleft \psi$
	in the notation of \cref{lem:2.9:from:1609.07031}) \prove\  that $\P$-\as\ it holds that
	\begin{equation}
		\llabel{eq:pre:exp}
		\cob{(\Phi(X,Y)-\co{\Phi(X,Y)}{\gil})^2}{\gil}=\psi(X).
	\end{equation}
	This \proves[pei] \cref{it:mb:fact:for:var01}. 
	\Nobs that \lref{eq:pre:exp}, \cref{it:localizations:cond:exp} in \cref{lem:norm:cond:exp}, and the fact that for all $x\in\R^m$ it holds that $\psi(x)\geq 0$ \prove\ that
	\begin{equation}
		\E\PRb{(\Phi(X,Y)-\co{\Phi(X,Y)}{\gil})^2}=\E\PR{\psi(X)}.
	\end{equation}
	This \proves[pei] \cref{it:mb:fact:for:var02}.
	\finishproofthus
\end{proof}

\cfclear

\begin{athm}{cor}{lem:factorization2:spec}
	Let $ ( \Omega, \cF, \P ) $ be a probability space,
	let $\gil\subseteq\cF$ be a sigma-algebra on $\Omega$,
	let $m,n\in\N$,
	$D\in\mathcal{B}(\R^m)$, $E\in\mathcal{B}(\R^n)$,
	let $\Phi\colon D\times E\to\R$ be measurable,
	let $X\colon\Omega\to D$ be $\gil$-measurable, 
	let $Y\colon\Omega\to E$ be a random variable,
	assume that $Y$ is independent of $\gil$,
	assume for all $x\in D$ that $\E\PR{\vass{\Phi(x,Y)}}<\infty$,
	and	let $\psi\colon D\to[0,\infty]$ satisfy for all $x\in D$ that
	\begin{equation}
		\label{eq:setup:var:function:1:var}
		\psi(x)=\var(\Phi(x,Y))
		.
	\end{equation}
	Then
	\begin{enumerate}[label=(\roman*)]
		\item \label{it:mb:fact:for:var}it holds that $\psi$ is measurable,
		\item \label{it:mb:fact:for:improper} it holds that \improper{\Phi(X,Y)}{\gil},
		\item \label{it:mb:fact:for:var1}it holds $\P$-\as\ that $\cob{(\Phi(X,Y)-\co{\Phi(X,Y)}{\gil})^2}{\gil}=\psi(X)$, and
		\item \label{it:mb:fact:for:var2}it holds that $\E\PR{(\Phi(X,Y)-\co{\Phi(X,Y)}{\gil})^2}=\E\PR{\psi(X)}$
	\end{enumerate}
	\cfout.
\end{athm}

\begin{proof}[Proof of \Cref{lem:factorization2:spec}]
	Throughout this proof let $f\colon\R^m\to[0,\infty]$ and $F\colon\R^m\times\R^n\to\R$ satisfy for all $x\in\R^m$, $y\in\R^n$ that
	\begin{equation}
		\label{eq:codom:exten:fct}
		f(x)=
		\begin{cases}
			\psi(x)&\colon x\in D\\
			0&\colon x\notin D\\		
		\end{cases}
		\qqandqq
		F(x,y)=
		\begin{cases}
			\Phi(x,y)&\colon (x,y)\in D\times E\\
			0&\colon (x,y)\notin D\times E
		\end{cases}
	\end{equation}
	and let $S\colon\Omega\to\R^m$ and $T\colon\Omega\to\R^n$ satisfy for all $\omega\in\Omega$ that
	\begin{equation}
		\label{eq:codom:exten:var}
		S(\omega)=X(\omega)
		\qqandqq
		T(\omega)=Y(\omega).
	\end{equation}
	\Nobs that \eqref{eq:codom:exten:fct}, \eqref{eq:codom:exten:var}, and the assumption that $\Phi\colon D\times E\to\R$ is measurable \prove\ that $F$ is measurable and that it holds $\P$-\as\ that
	\begin{equation}
		\label{eq:extension:a.s}
		\psi(X)=f(S)
		\qqandqq
		F(S,T)=\Phi(X,Y).
	\end{equation}
	This, 
	\eqref{eq:codom:exten:fct}, and the assumption that for all $x\in D$ it holds that $\E\PR{\vass{\Phi(x,Y)}}<\infty$ \prove\ that for all $x\in D$, $y\in \R^m\backslash D$ it holds that
	$\E\PR{\vass{F(y,T)}}=0$ and
	\begin{equation}
		\label{eq:cond:finite:check:expanded}
		\begin{split}
			\E\PR{\vass{F(x,T)}}
			=
			\E\PR{\vass{F(x,Y)}}
			&=
			\E\PR{\vass{\Phi(x,Y)}}
			<\infty.
		\end{split}
	\end{equation}
	This,  
	\eqref{eq:setup:var:function:1:var}, 
	\eqref{eq:codom:exten:fct}, and
	\eqref{eq:codom:exten:var} \prove\ that for all  $x\in D$, $y\in \R^m\backslash D$ it holds that $f(y)=0=\var(F(y,Y))$ and
	\begin{equation}
		\begin{split}
			f(x)
			=\psi(x)
			&=\var(\Phi(x,Y))
			=\var(F(x,Y)).
		\end{split}
	\end{equation}
	This,
	\eqref{eq:codom:exten:var},
	\eqref{eq:cond:finite:check:expanded}, and
	\cref{lem:factorization2} (applied with 
	$( \Omega, \cF, \P )\curvearrowleft( \Omega, \cF, \P )$,
	$\gil\curvearrowleft\gil$,
	$m\curvearrowleft m$,
	$n\curvearrowleft n$,
	$\Phi\curvearrowleft F$,
	$X\curvearrowleft S$,
	$Y\curvearrowleft T$,
	$\psi\curvearrowleft f$
	in the notation of \cref{lem:factorization2}) \prove\ that
	\begin{enumerate}[label=(\Roman*)]
		\item \label{it:mb:fact:for:var0:app}it holds that $f$ is measurable,
		\item \label{it:mb:fact:var00:improper} it holds that \improper{F(S,T)}{\gil},
		\item \label{it:mb:fact:for:var01:app}it holds $\P$-\as\ that $\cob{(F(S,T)-\E\PR{F(S,T)|\gil})^2}{\gil}=f(S)$, and
		\item \label{it:mb:fact:for:var02:app}it holds that $\E\PRb{(F(S,T)-\E\PR{F(S,T)|\gil})^2}=\E\PR{f(S)}$
	\end{enumerate}
	\cfload.
	\Nobs that \eqref{eq:codom:exten:fct} and \cref{it:mb:fact:for:var0:app} \prove\ \cref{it:mb:fact:for:var}.
	\Nobs that \eqref{eq:extension:a.s} and
	\cref{it:mb:fact:var00:improper} \prove\ \cref{it:mb:fact:for:improper}.
	This and \eqref{eq:extension:a.s} \prove\ that it holds $\P$-\as\ that
	\begin{equation}
		\label{eq:extension:cond:exp}
		\co{F(S,T)}{\gil}=\co{\Phi(X,Y)}{\gil}.
	\end{equation}
	This,
	\eqref{eq:extension:a.s}, and
	\cref{it:mb:fact:for:var01:app} \prove\ that it holds $\P$-\as\ that
	\begin{equation}
		\begin{split}
			\psi(X)
			=f(S)
			&=\E\PR{(F(S,T)-\E\PR{F(S,T)|\gil})^2|\gil}
			=\E\PR{(\Phi(X,Y)-\E\PR{\Phi(X,Y)|\gil})^2|\gil}
			.
		\end{split}
	\end{equation}
	This \proves[pei] \cref{it:mb:fact:for:var1}.
	\Nobs that \eqref{eq:extension:a.s} and
	\cref{it:mb:fact:for:var02:app} \prove\ that
	\begin{equation}
		\begin{split}
			\E\PR{\psi(X)}
			=\E\PR{f(S)}
			&=\E\PR{(F(S,T)-\E\PR{F(S,T)|\gil})^2}
			=\E\PR{(\Phi(X,Y)-\E\PR{\Phi(X,Y)|\gil})^2}
			.
		\end{split}
	\end{equation}
	This \proves[pei] \cref{it:mb:fact:for:var2}. \finishproofthus
\end{proof}

\section{Non-convergence for general stochastic gradient descent (SGD) methods}
\label{sec:results}
The main goal of this section is to establish suitable non-convergence results for \ADAM\ and other adaptive \SGD\ optimization methods.
In particular, \cref{cor:limsup:simple} in \cref{subsection:3.4}, the first main result of this article, shows that
for every random variable $\xi\colon\Omega\to\R^{\pars}$
the optimization process $\Theta_n\colon\Omega\to\R^\pars$, $n\in\N_0$, generated by the employed adaptive \SGD\ optimization method does not converge to $\xi$ in the sense that  $\limsup_{ n \to \infty }
\E\PRb{\norm{\Theta_n - \xi}^{2} }
>0$
 if the product of the squared learning rates $\gamma=(\gamma_n)_{n\in\N}\colon\N\to[0,\infty)$ and the reciprocal of the mini-batch sizes $\batch=(\batch_n)_{n\in\N}\colon\N\to\N$ do not converge to zero
 (cf.\ \eqref{eq:LR:simple} in \cref{cor:limsup:simple}).

Moreover, \cref{thm:new:rob:cond} in \cref{subsection:4.3}, the second main result of this article, implies that
 for every component $i\in\{1,2,\dots,\pars\}$ of the considered adaptive \SGD\ optimization process $\Theta_n=\prb{\Theta_n^{(1)},\dots,\Theta_n^{(\pars)}}\colon\Omega\to\R^\pars$, $n\in\N_0$, and every scalar random variable $\xi\colon\Omega\to\R$ we have that the error of the employed adaptive \SGD\ optimization method does not vanish in the sense that
 $\limsup_{ n \to \infty }
 \E\PRb{| \Theta_n^{(i)} - \xi | }
 >0$
 if the sizes of the mini-batches $\batch=(\batch_n)_{n\in\N}\colon\N\to\N$ are bounded from above, 
 if the learning rates $\gamma=(\gamma_n)_{n\in\N}\colon\N\to[0,\infty)$ are bounded from above, and
 if the learning rates do not converge to zero (cf.\ \eqref{eq:bounded:batches:and:LR} in \cref{thm:new:rob:cond}).
Additionally, \cref{thm:new:rob:cond} shows that
in the case of non-vanishing learning rates in the sense that $\liminf_{n\to\infty}\gamma_n>0$ and 
asymptotically non-trivial mini-batches in the sense that
$\liminf_{n\to\infty}\batch_n>1$
 we have that  $\liminf_{ n \to \infty }
\E\PRb{| \Theta_n^{(i)} - \xi | }
>0$ (cf.\ \eqref{eq:bounded:batches:and:LR} in \cref{thm:new:rob:cond}).

\subsection{Lower bounds for expectations of appropriate random variables}

In \cref{lem:variance:cont} we establish suitable lower bounds for variances of appropriately scaled random variables.
\Cref{lem:variance:cont} is employed in our proof of  \cref{prop:prop:non_convergence_modified_Adam:specific:setup:SCOPE3}.
Our proof of \cref{lem:variance:cont} employs the elementary and well-known representation for the variance of a random variable in \cref{lem:variance:representations:0}.
\Cref{lem:variance:representations:0}, in turn, is based on an application of the elementary and well-known symmetrization identity for the squared differences of \iid\ random variables in \cref{lem:variance:representations:1}.
Only for completeness we include in this subsection detailed proofs for \cref{lem:variance:representations:1} and \cref{lem:variance:representations:0}.
\label{subsection:4.1}

\begin{athm}{lemma}{lem:variance:representations:1}
	Let $ ( \Omega, \cF, \P ) $ be a probability space
	and
	let $X\colon\Omega\to\R$ and $Y\colon\Omega\to\R$ be \iid\ random variables.
	Then
	\begin{equation}
		\begin{split}
			\label{eq:lem:variance:representations:1}
			\tfrac{1}{2}\E\PR{(X-Y)^2}=\E\PR{(X-Y)^2\indicator{\pR{X\leq Y}}}
			.
		\end{split}
	\end{equation}
\end{athm}

\begin{proof}[Proof of \cref{lem:variance:representations:1}]
	\Nobs the fact that $\E\PR{(X-Y)^2\indicator{\pR{X = Y}}}=0$ and the assumption that $X$ and $Y$ are \iid\ \prove\ that
	\begin{equation}
		\begin{split}
			&\E\PR{(X-Y)^2\indicator{\pR{X\leq Y}}}
			\\&=
			\tfrac{1}{2}\prb{\E\PR{(X-Y)^2\indicator{\pR{X\leq Y}}}+\E\PR{(X-Y)^2\indicator{\pR{X\leq Y}}}}
			\\&=
			\tfrac{1}{2}\prb{\E\PR{(X-Y)^2\indicator{\pR{X < Y}}}
				+2\,\E\PR{(X-Y)^2\indicator{\pR{X = Y}}}
				+\E\PR{(X-Y)^2\indicator{\pR{ X < Y}}}}
			\\&=
			\tfrac{1}{2}\prb{\E\PR{(X-Y)^2\indicator{\pR{X < Y}}}
				+2\,\E\PR{(X-Y)^2\indicator{\pR{X = Y}}}
				+\E\PR{(X-Y)^2\indicator{\pR{ X > Y}}}}
			\\&=
			\tfrac{1}{2}\prb{\E\PR{(X-Y)^2}
				+\E\PR{(X-Y)^2\indicator{\pR{X = Y}}}}
				=
			\tfrac{1}{2}\E\PR{(X-Y)^2}.
		\end{split}
	\end{equation}
	Hence we obtain \eqref{eq:lem:variance:representations:1}.
	\finishproofthus
\end{proof}

\begin{athm}{lemma}{lem:variance:representations:0}
	Let $ ( \Omega, \cF, \P ) $ be a probability space,
	let $X\colon\Omega\to\R$ and $Y\colon\Omega\to\R$ be \iid\ random variables, and assume $\E\PR{\vass{X}}<\infty$.
	Then
	\begin{equation}
		\begin{split}
			\label{eq:lem:variance:representations:0}
			\var\pr{X}
			=\tfrac{1}{2}\E\PR{(X-Y)^2}
			=\E\PR{(X-Y)^2\indicator{\pR{X\leq Y}}}
			.
		\end{split}
	\end{equation}
\end{athm}

\begin{proof}[Proof of \cref{lem:variance:representations:0}]
	\Nobs the fact that $\E\PR{(X-Y)^2\indicator{\pR{X = Y}}}=0$, the assumption that $\E\PR{\vass{X}}<\infty$, and the assumption that $X$ and $Y$ are \iid\ \prove\ that
	\begin{equation}
		\begin{split}
			\E\PR{(X-Y)^2}
			=
			\E\PR{X^2-2XY+Y^2}
			&=
			\E\PR{X^2+Y^2}-\E\PR{2XY}
			\\&=
			\E\PR{X^2}-2\E\PR{X}\E\PR{Y}+\E\PR{Y^2}
			\\&=
			2\E\PR{X^2}-2\pr{\E\PR{X}}^2
			=
			2\var(X)
			.
		\end{split}
	\end{equation}
	This and \cref{lem:variance:representations:1} \prove\ \eqref{eq:lem:variance:representations:0}.
	\finishproofthus
\end{proof}

\begin{athm}{prop}{lem:variance:cont}
	Let $\eps\in(0,\infty)$,
	let $ ( \Omega, \cF, \P ) $ be a probability space, 
	let $\MOM\colon\R\to[0,\infty)$ be measurable,
	let $X\colon\Omega\to \R$ be a bounded random variable, 
	assume $\sup_{\omega\in\Omega}\MOM(X(\omega))<\infty$, and
	assume for all 	$x\in\R$, $y\in(-\infty,x)$ that $y \MOM(x)\leq x \MOM(y)$.
	Then
	\begin{enumerate}[label=(\roman*)]
		\item \label{lem:variance:cont:it1}it holds that
		\begin{equation}
			\begin{split}
				\E\PRbbb{	\frac{\vass{X}} { \eps + \MOM(X)}}
				<
				\infty,
			\end{split}
		\end{equation}
		and
		\item \label{lem:variance:cont:it2}it holds that
		\begin{equation}
			\begin{split}
				\var\prbbb{	\frac{X} { \eps + \MOM(X)}} 
				&
				\geq
				\frac{\eps^2\var(X)}{\pr{\eps + \textstyle\sup_{\omega\in\Omega}\MOM(X(\omega))}^{4}}
				.
			\end{split}
		\end{equation}
	\end{enumerate}
\end{athm}

\begin{proof}[Proof of \cref{lem:variance:cont}]
	\Nobs that
	\begin{equation}
		\label{eq:finite:first:moment:scaled:RV}
		\begin{split}
			\E\PRbbb{	\frac{\vass{X}} { \eps + \MOM(X)}}
			\leq
			\E\PRbbb{	\frac{\vass{X}} { \eps }}
			=
			\eps^{-1}\E\PR{\vass{X}}
			<
			\infty
			.
		\end{split}
	\end{equation}
	This \proves[pei] \cref{lem:variance:cont:it1}.
	\Nobs that the assumption that for
	all $x\in\R$, $y\in(-\infty,x]$ it holds that $y \MOM(x)\leq x \MOM(y)$
	 \proves\ that for all $x\in\R$, $y\in(-\infty,x]$ it holds that
	\begin{equation}
		\label{eq:diff:super:variance:points}
		\begin{split}
			\frac{x }{\eps + \MOM(x)}-\frac{y }{\eps + \MOM(y)}
			=\frac{x \pr{\eps + \MOM(y)}-y \pr{\eps + \MOM(x)}  }{\pr{\eps + \MOM(x)}\pr{\eps + \MOM(y)} }
			&\geq\frac{\eps(x-y)  }{\pr{\eps + \MOM(x)}\pr{\eps + \MOM(y)} }.
		\end{split}
	\end{equation}
	This \proves[pei] that for all $x\in\R$, $y\in(-\infty,x)$ it holds that
	\begin{equation}
		\frac{x }{\eps + \MOM(x)}>\frac{y }{\eps + \MOM(y)}.
	\end{equation}
	This,
	\eqref{eq:finite:first:moment:scaled:RV},
	\eqref{eq:diff:super:variance:points},
	\cref{lem:variance:representations:0}, and the assumption that
	$\sup_{\omega\in\Omega}\vass{X(\omega)}<\infty$ \prove\ that for all $Y\colon\Omega\to\R$ with $X$ and $Y$ are \iid\ it holds that
	\begin{equation}
		\begin{split}
			\var\prbbb{	\frac{X} { \eps + \MOM(X)}}
			&=
			\E\PRbbb{\prbbb{\frac{X} { \eps + \MOM(X)}-\frac{Y} { \eps + \MOM(Y)}}^2\indicator{\pR{X\geq Y}}}
			\\&\geq
			\E\PRbbb{\prbbb{\frac{\eps(X-Y)  }{\pr{\eps + \MOM(X)}\pr{\eps + \MOM(Y)} }}^2\indicator{\pR{X\geq Y}}}
			\\&\geq
			\E\PRbbb{\frac{\eps^2(X-Y)^2  \indicator{\pR{X\geq Y}}}{\pr{\eps + \textstyle\sup_{\omega\in\Omega}\MOM(X(\omega))}^4 }}
			\\&=
			\eps^2\pr{\eps + \textstyle\sup_{\omega\in\Omega}\MOM(X(\omega))}^{-4} \E\PR{(X-Y)^2  \indicator{\pR{X\geq Y}}}
			\\&=
			\eps^2\pr{\eps + \textstyle\sup_{\omega\in\Omega}\MOM(X(\omega))}^{-4} \var(X).
		\end{split}
	\end{equation}
	This \proves[pei] \cref{lem:variance:cont:it2}.
	\finishproofthus
\end{proof}

In the next result, \cref{lem:cond:exp:vanish2} below, we recall, roughly speaking, a special case of the well-known $L^2$-best approximation property for conditional expectations (see, \eg, 
\cite[Corollary~8.17]{klenkeprobability} and 
\cite[Theorem~12.1.2]{Krishna2006}).

\cfclear
\begin{athm}{lemma}{lem:cond:exp:vanish2}
	Let $ ( \Omega, \cF, \P ) $ be a probability space, 
	let $X\colon\Omega\to\R$ be a random variable,
	let $\gil\subseteq\cF$ be a sigma-algebra on $\Omega$, and assume $\E\PR{\vass{X}}<\infty$. Then
	\begin{equation}
		\label{it:basic:cond:exp:22}
		\E\PRb{X^2}\geq \E\PRb{\pr{X-\co{X}{\gil}}^2}
	\end{equation}
	\cfout.
\end{athm}

\begin{proof}[Proof of \cref{lem:cond:exp:vanish2}]
	Throughout this proof assume without loss of generality that $\E\PR{X^2}<\infty$.
	\Nobs that \cite[Corollary~8.17]{klenkeprobability} (applied with
	$(\Omega,\mathcal{A},\P)\curvearrowleft( \Omega, \cF, \P )$,
	$\gil\curvearrowleft\gil$,
	$X\curvearrowleft X$,
	$Y\curvearrowleft \pr{\Omega\ni\omega\mapsto 0\in\R}$,
	in the notation of \cite[Corollary~8.17]{klenkeprobability}) \proves\ that
	\begin{equation}
		\E\PR{X^2}=\E\PR{(X-0)^2}\geq \E\PR{\pr{X-\co{X}{\gil}}^2}
	\end{equation}
	\cfload. This \proves\ \eqref{it:basic:cond:exp:22}.
	\finishproofthus
\end{proof}

In the following result, \cref{lem:cond:exp:vanish3} below, we reformulate the $L^2$-best approximation property for conditional expectations for merely integrable but not square integrable random variables (see, \eg, 
\cite[Corollary~8.17]{klenkeprobability} and 
\cite[Theorem~12.1.2]{Krishna2006}).
\Cref{lem:cond:exp:vanish3} is an immediate consequence of \cref{lem:cond:exp:vanish2}.

\cfclear
\begin{athm}{cor}{lem:cond:exp:vanish3}[Conditional expectation as projection]
	Let $ ( \Omega, \cF, \P ) $ be a probability space, 
	let $X\colon\Omega\to\R$ be a random variable,
	let $\gil\subseteq\cF$ be a sigma-algebra on $\Omega$,
	let $Y\colon\Omega\to\R$ be $\gil$-measurable, and assume $\E\PRb{\vass{X}+\vass{Y}}<\infty$. Then
	\begin{equation}
		\label{it:basic:cond:exp:23}
		\E\PRb{(X-Y)^2}\geq \E\PRb{\pr{X-\co{X}{\gil}}^2}
	\end{equation}
	\cfout.
\end{athm}
\begin{proof}[Proof of \cref{lem:cond:exp:vanish3}]
	\Nobs that
	\cref{lem:cond:exp:vanish2} (applied with 
	$ ( \Omega, \cF, \P ) \curvearrowleft  ( \Omega, \cF, \P ) $, 
	$X\curvearrowleft (X-Y)$,
	$\gil\curvearrowleft \gil$
	in the notation of \cref{lem:cond:exp:vanish2}) \proves\ that
	\begin{equation}
		\E\PR{(X-Y)^2}\geq \E\PR{\pr{(X-Y)-\co{X-Y}{\gil}}^2}
	\end{equation}
	\cfload. Combining
	 this,
	 \cref{lem:norm:cond:exp},
	  and 
	  \cite[Theorem~8.14]{klenkeprobability}
	  with the assumptions that $\E\PRb{\vass{X}+\vass{Y}}<\infty$
	and that $Y$ is $\gil$-measurable 
	\proves\ that
	\begin{equation}
		\E\PR{(X-Y)^2}
		\geq 
		\E\PR{\pr{X-Y-\E\PR{X-Y|\gil}}^2}
		=
		\E\PR{\pr{X-Y+Y-\E\PR{X|\gil}}^2}
		=\E\PR{\pr{X-\E\PR{X|\gil}}^2}
		.
	\end{equation}
	This \proves\ \eqref{it:basic:cond:exp:23}.
	\finishproofthus
\end{proof}

In the next result, \cref{metrix:cauchy:limsup} below, we present, roughly speaking, an elementary lower bound for the asymptotic distance of an arbitrary point to an arbitrary sequence of points in a generalized metric space.
 \Cref{metrix:cauchy:limsup} is essentially a direct consequence of the triangle inequality in metric spaces.

\begin{athm}{lemma}{metrix:cauchy:limsup}
	Let $E$ be a set,
	let $z\in\{-1,1\}$, $\lambda\in(0,\infty)$,
	let $d\colon E\times E\to[0,\infty]$ satisfy for all $u,v,w\in E$ with $d(u,v)<\infty$ that $d(u,w)\leq \lambda d(u,v)+\lambda d(v,w)$ and $d(u,v)=d(v,u)$, 
	and let $x=(x_n)_{n\in\N_0}\colon\N_0\to E$ be a function.
	Then
	\begin{equation}
		\begin{split}
			\label{eq:sup:cauch:result:pre}
			z\limsup_{n\to\infty}z\PR{d(x_0,x_n)}
			&\geq
			\tfrac{z}{2\lambda}\PRbb{\limsup_{n\to\infty}\limsup_{ m \to \infty }z\PR{d(x_n,x_m)}}
			\\&=
			\tfrac{z}{2\lambda}\PRbb{\limsup_{m\to\infty}\limsup_{ n \to \infty }z\PR{d(x_n,x_m)}}
			.
		\end{split}
	\end{equation}
\end{athm}

\begin{proof}[Proof of \cref{metrix:cauchy:limsup}]
	\Nobs that the assumption that for all $u,v,w\in E$ with $d(u,v)<\infty$ it holds that $d(u,w)\leq \lambda d(u,v)+\lambda d(v,w)$ and $d(u,v)=d(v,u)$ \proves\ that for all $u,v,w\in E$ it holds that
	\begin{equation}
		\label{eq:symmetric:metric:pre}
		d(u,w)\leq \lambda d(u,v)+\lambda d(v,w)
		\qqandqq
		d(u,v)=d(v,u).
	\end{equation}
	This \proves\ that for all $m,n\in\N$ it holds that
	\begin{equation}
		\begin{split}
			d(x_{n},x_{m})
			\leq
			\lambda d(x_{n},x_{0})+\lambda d(x_{0},x_{m})
			=
			\lambda d(x_{n},x_{0})+\lambda d(x_{m},x_{0})
			.
		\end{split}
	\end{equation}
	This \proves\ for all $n\in\N$ that
	\begin{equation}
		\begin{split}
			\textstyle
			z\limsup_{ m \to \infty }z\PR{d(x_n,x_m)}
			&\leq\textstyle
			z\limsup_{ m \to \infty }z\PR{\lambda d(x_{n},x_{0})+\lambda d(x_{m},x_{0})}
			\\&=\textstyle
			\lambda d(x_{n},x_{0})
			+
			z\limsup_{ m \to \infty }z\PR{\lambda d(x_{m},x_{0})}
			.
		\end{split}
	\end{equation}
	This and \eqref{eq:symmetric:metric:pre} \prove\ that
	\begin{equation}
		\label{eq:left:part}
		\begin{split}
			&\textstyle
			z\limsup_{ n \to \infty }
			\limsup_{ m \to \infty }z\PR{ d(x_m,x_n)}
			\\&=\textstyle
			z\limsup_{ n \to \infty }
			\limsup_{ m \to \infty }z\PR{ d(x_n,x_m)}
			\\&=\textstyle
			z\limsup_{ n \to \infty }
			z\PRb{z\limsup_{ m \to \infty }z\PR{ d(x_n,x_m)}}
			\\&\leq\textstyle
			z\limsup_{ n \to \infty }
			z\PRb{\lambda d(x_{n},x_{0})
				+
				z\limsup_{ m \to \infty }z\PR{\lambda d(x_{m},x_{0})}}
			\\&=\textstyle
				\PRb{z\limsup_{ n \to \infty }
				z\PR{\lambda d(x_{n},x_{0})}}
					+
					\PRb{z\limsup_{ m \to \infty }z\PR{\lambda d(x_{m},x_{0})}}
			\\&=\textstyle
			2\lambda z\limsup_{ n \to \infty }
			z\PR{ d(x_{n},x_{0})}
			.
		\end{split}
	\end{equation}
	This and \eqref{eq:symmetric:metric:pre} \prove\ \eqref{eq:sup:cauch:result:pre}.
	\finishproofthus
\end{proof}

In the next elementary result, \cref{norm:propterties:at:infty} below, we specialize \cref{metrix:cauchy:limsup} to the situation of real-valued random variables on a probability space.

\begin{athm}{cor}{norm:propterties:at:infty}
	Let $ ( \Omega, \cF, \P ) $ be a probability space,
	let $\pars\in\N$, $z\in\{-1,1\}$,
	for every $n\in\N$ let $X_n\colon\Omega\to\R^\pars$ be a random variable, and
	let $Y\colon\Omega\to\R^\pars$ be a random variable.
	Then
	\begin{equation}
		\textstyle
		z\prb{\limsup_{n\to\infty}z\E\PRb{\norm{Y-X_n}^2}}
		\geq \tfrac{z}{4}\prb{\limsup_{n\to\infty}\limsup_{m\to\infty}z\E\PRb{\norm{X_n-X_m}^2}}
		.
	\end{equation}
\end{athm}

\begin{proof}[Proof of \cref{norm:propterties:at:infty}]
		Throughout this proof let $E\subseteq\prb{\R^\pars}^\Omega$ satisfy \begin{equation}
		E=\{Z\colon\Omega\to\R^\pars\colon\text{$Z$ is measurable}\}.
	\end{equation} 
	\Nobs that for all $Z_1,Z_2,Z_3\in E$ it holds that
	\begin{equation}
		\begin{split}
			\norm{Z_1-Z_2}^2
			&\leq\prb{\norm{Z_1-Z_3}+\norm{Z_3-Z_2}}^2
			\\&=
			\norm{Z_1-Z_3}^2+2\norm{Z_1-Z_3}\norm{Z_3-Z_2}+\norm{Z_3-Z_2}^2
			\\&=
			2\norm{Z_1-Z_3}^2-\prb{\norm{Z_1-Z_3}-\norm{Z_3-Z_2}}^2+2\norm{Z_3-Z_2}^2
			\\&\leq
			2\norm{Z_1-Z_3}^2+2\norm{Z_3-Z_2}^2.
		\end{split}
	\end{equation}
	This \proves\ that for all $Z_1,Z_2,Z_3\in E$ it holds that
	\begin{equation}
		\begin{split}
			\E\PRb{\norm{Z_1-Z_2}^2}
			\leq
			2\E\PRb{\norm{Z_1-Z_3}^2}
			+
			2\E\PRb{\norm{Z_2-Z_3}^2}.
		\end{split}
	\end{equation}
	This and
	\cref{metrix:cauchy:limsup} (applied with 
	$E\curvearrowleft E$,
	$z\curvearrowleft z$,
	$\lambda\curvearrowleft 2$,
	$d\curvearrowleft \pr{(E\times E)\ni(V,W)\mapsto\E\PRb{\norm{V-W}^2}\in[0,\infty]}$,
	$x\curvearrowleft \pr{\N_0\ni n\mapsto Y\indicator{\{0\}}(n)+X_n\indicator{\N}(n)\in E}$
	in the notation of \cref{metrix:cauchy:limsup}) \proves\ that
		\begin{equation}
		\textstyle
		z\prb{\limsup_{n\to\infty}z\E\PRb{\norm{Y-X_n}^2}}
		\geq \tfrac{z}{4}\prb{\limsup_{n\to\infty}\limsup_{m\to\infty}z\E\PRb{\norm{X_n-X_m}^2}}
		.
	\end{equation}
	\finishproofthus
\end{proof}

\subsection{Lower bounds for non-adaptive SGD optimization methods}
\label{subsection:4.1.1}



\cfclear
\begin{athm}{prop}{prop:SGD:mom}
	Let $\pars,\dim\in\N$, 
	$\fc,a\in\R$,
	$b\in[a,\infty)$,
	$\Cst\in(0,\infty)$,
	$z\in\{-1,1\}$,
	let 
	$\gamma\colon\N\to[0,\infty)$, 
	$\batch\colon\N\to\N$, and $\alpha=(\alpha_{n,k})_{(n,k)\in(\N_0)^2}\colon(\N_0)^2\to[0,\infty)$ satisfy
	\begin{equation}
		\label{eq:setup:SGD:mom}
		\textstyle
		z\PRb{\limsup_{n\to\infty}\prb{z(\alpha_{n,n}\gamma_n)^2\pr{\batch_n}^{-1}}}>0,
	\end{equation}
	let $g=(g_1,\dots,g_{\pars})\colon\R^{\pars}\times[a,b]^{\dim}\to\R^{\pars}$ be measurable, 
	for every $n\in\N$ let $G_{n}=\pr{G_{n}^{(1)},\dots,G_{n}^{(\pars)}}\colon\R^\pars \allowbreak\times\pr{[a,b]^{\dim}}^{\batch_n}\to\R^\pars$ satisfy for all $\theta\in\R^{\pars}$, $x=(x_1,\dots,x_{\batch_n})\in\pr{[a,b]^{\dim}}^{\batch_n}$ that
	\begin{equation}
		\label{prop:eq:setup:gradient:SGD:mom}
		\textstyle
		G_n(\theta,x)=\frac{1}{\batch_{n}}\PRb{\sum_{j=1}^{\batch_{n}}g(\theta,x_j)}
		,
	\end{equation}
	let $ ( \Omega, \cF,(\mathbb{F}_n)_{n\in\N_0}, \P ) $ be a filtered probability space,
	for every $n\in\N$ let $X_n=(X_{n,1},\dots,X_{n,\batch_n})\colon\allowbreak\Omega \allowbreak\to \pr{[ a, b ]^{\dim}}^{\batch_n}$ be $\fil_n$-measurable,
	assume that $X_{n,j}$, $(n,j)\in\{(m,i)\in\N^2\colon i\leq\batch_m\}$, are \iid,
	assume for all $n\in\N$ that $X_n$ is independent of $\fil_{n-1}$,
	$ \Theta=(\Theta^{(1)},\ldots,\Theta^{(\pars)}) \colon\N_0\times\Omega \to \R^{\pars}$ and
	$\mom=(\mom[1],\ldots,\mom[\pars])\colon\Omega\to\R^{\pars}$
	satisfy
	for all $n\in\N$ that
	\begin{equation}
		\label{eq:incr:rule:MOM_SGD}
		\begin{split}
			\textstyle
			\Theta_{ n }
			&\textstyle=\Theta_{ n-1  }-\gamma_{n}  \PRb{\alpha_{n,0}\mom+\sum_{k=1}^n\alpha_{n,k}G_k(\Theta_{k-1},X_k)},
		\end{split}
	\end{equation}
	let $i\in\{1,2,\dots,\pars\}$, 
	assume that $\Theta_0$ and $\mom$ are $\fil_0$-measurable,
	assume for all $\theta=(\theta_1,\dots,\theta_\pars)\in \R^{\pars}$, $x\in[a,b]^{\dim}$ that
	\begin{equation}
		\label{eq:setup:starting:expectations:bounded:MOM_SGD}
		\max\pRb{\E\PRb{\vass{\Theta_0^{(i)}}},
			\E\PRb{\vass{\mom[i]}}}
		<\infty
		\qqandqq
		\textstyle
		\vass{g_i(\theta,x)}\leq \Cst(\vass{\theta_i}+\fc)
		,
	\end{equation}
	and 
	assume  $\inf_{\theta\in\R^{\pars}}\var\pr{g_i(\theta,X_{1,1})}>0$.
	Then 
	\begin{enumerate}[label=(\roman*)]
		\item \label{it:fist:moment:SGD:mom} it holds for all $n\in\N_0$ that $\E\PRb{\vass{\Theta_n^{(i)}}}<\infty$,
		\item \label{it:fist:moment:SGD:mom:for:fact} it holds for all $n\in\N$, $m\in\R$, $\theta_1,\dots,\theta_n\in\R^\pars$, $y_1\in\pr{[a,b]^{\dim}}^{\batch_1}$, 
		$y_2\in\pr{[a,b]^{\dim}}^{\batch_2}$, $\dots$, 
		$y_n\in\pr{[a,b]^{\dim}}^{\batch_n}$ that
		 \begin{equation}
		 	\textstyle
		\E\PRb{\vass{
			\alpha_{n,0}m
			+\alpha_{n,n}G_n^{(i)}(\theta_{n},X_n)
			+\sum_{k=1}^{n-1}\alpha_{n,k}G_k^{(i)}(\theta_k,y_k)}}<\infty,
		\end{equation}
		\item 
		\label{it:SGD:mom:incr}
		it holds for all $n\in\N$ that
			\begin{equation}
			\begin{split}
				\E\PRb{\prb{\Theta_{n}^{(i)}-\cob{\Theta_{n}^{(i)}}{\fil_{n-1}}}^2}
				\geq (\alpha_{n,n}\gamma_n)^2\pr{\batch_n}^{-1}
				\textstyle\inf_{\vartheta\in\R^{\pars}}\var\pr{g_i(\vartheta,X_{1,1})}
				,
			\end{split}
		\end{equation}
		and
		\item 
		\label{it:SGD:mom:lims}
		it holds that
		\begin{equation}
			\begin{split}
				\inf_{\substack{\xi\colon \Omega \to \R\\\text{measurable}}}
				\prbbb{z\PRbbb{\limsup_{ n \to \infty }\prb{z\E\PRb{ | \Theta_n^{(i)} - \xi |^2 }}}
				}
				>
				0
				.
			\end{split}
		\end{equation}
	\end{enumerate}

\end{athm}

\begin{proof}[Proof of \cref{prop:SGD:mom}]
	Throughout this proof
	for every $n\in\N$ 
	let $\phi_n\colon \R\times\pr{\R^{\pars}}^n\times\pr{[a,b]^{\dim}}^{\batch_1}\times\pr{[a,b]^{\dim}}^{\batch_2}\times\dots\times\pr{[a,b]^{\dim}}^{\batch_{n-1}}\to[0,\infty]$ and
	$\inc_n\colon \R\times\pr{\R^{\pars}}^n\times\pr{[a,b]^{\dim}}^{\batch_1}\times\pr{[a,b]^{\dim}}^{\batch_2}\times\dots\times\pr{[a,b]^{\dim}}^{\batch_n}\to\R$ satisfy for all  $m\in\R$, $\theta_{1},\theta_2,\dots,\theta_{n}\in\R^{\pars}$, $y_1\in\pr{[a,b]^{\dim}}^{\batch_1}$, 
	$y_2\in\pr{[a,b]^{\dim}}^{\batch_2}$, $\dots$, 
	$y_n\in\pr{[a,b]^{\dim}}^{\batch_n}$ that
	\begin{equation}
		\textstyle
		\inc_n(m,\theta_1,\dots,\theta_n,y_1,\dots,y_n)
		=
		-\gamma_{n}  \PRb{\alpha_{n,0}m+\sum_{k=1}^n\alpha_{n,k}G_k^{(i)}(\theta_{k},y_k)}
	\end{equation}
	\begin{equation}
		\qqandqq
		\textstyle
		\phi_n(m,\theta_1,\dots,\theta_n,y_1,\dots,y_{n-1})
		=
		\var\pr{\inc_n(m,\theta_1,\dots,\theta_n,y_1,\dots,y_{n-1},X_n)}.
	\end{equation}
	\Nobs that
	\eqref{prop:eq:setup:gradient:SGD:mom},
	\eqref{eq:incr:rule:MOM_SGD},
	\eqref{eq:setup:starting:expectations:bounded:MOM_SGD}, and the assumption that $\forall\,\theta=(\theta_1,\dots,\theta_\pars)\in\R^{\pars}$, $x\in[a,b]^{\dim}\colon\vass{g_i(\theta,x)}\leq \Cst(\vass{\theta_i}+\fc)$ \prove\ that 
	for all $n\in\N$ with $\max\pRb{\E\PRb{\vass{\Theta_0^{(i)}}},\E\PRb{\vass{\Theta_1^{(i)}}},\allowbreak\ldots,\allowbreak\E\PRb{\vass{\Theta_{n-1}^{(i)}}}}\allowbreak<\infty$ it holds that
	\begin{equation}
		\begin{split}
			\E\PRb{\vass{\Theta_n^{(i)}}}
			&=
			\textstyle
			\E\PRb{\vass{
			\Theta_{ n-1  }^{(i)}-\gamma_{n}  \PRb{\alpha_{n,0}\mom[i]+\sum_{k=1}^n\alpha_{n,k}G_k^{(i)}(\Theta_{k-1},X_k)}
			}
			}
			\\&=
			\textstyle
			\E\PRb{\vass{
					\Theta_{ n-1  }^{(i)}-\gamma_{n}  \PRb{\alpha_{n,0}\mom[i]+\sum_{k=1}^n\frac{\alpha_{n,k}}{\batch_{k}}\PRb{\sum_{j=1}^{\batch_{k}}g_i(\Theta_{k-1},X_{k,j})}}
				}
			}
			\\&\leq
			\textstyle
			\E\PRb{\vass{
					\Theta_{ n-1  }^{(i)}}}
					+
			\gamma_{n}\alpha_{n,0}\E\PRb{\vass{\mom[i]}}
			+\gamma_{n}
		 \sum_{k=1}^n\frac{\alpha_{n,k}}{\batch_{k}}
		 \sum_{j=1}^{\batch_{k}}\E\PRb{\vass{g_i(\Theta_{k-1},X_{k,j})}}
		 \\&\leq
		 \textstyle
		 \E\PRb{\vass{
		 		\Theta_{ n-1  }^{(i)}}}
		 +
		 \gamma_{n}\alpha_{n,0}\E\PRb{\vass{\mom[i]}}
		 +\gamma_{n}
		 \sum_{k=1}^n\alpha_{n,k}\Cst
		 \prb{\E\PRb{\vass{
		 		\Theta_{ k-1  }^{(i)}}}+\fc}		
	 		<\infty
	 		.
	 		\end{split}
	\end{equation}
	Combining this with \eqref{eq:setup:starting:expectations:bounded:MOM_SGD} and induction \proves\ \cref{it:fist:moment:SGD:mom}.
	\Nobs that \eqref{eq:incr:rule:MOM_SGD} and  \eqref{eq:setup:starting:expectations:bounded:MOM_SGD} \prove\ that for all $n\in\N$, $m\in\R$, $\theta_1,\dots,\theta_n\in\R^\pars$, $y_1\in\pr{[a,b]^{\dim}}^{\batch_1}$, 
	$y_2\in\pr{[a,b]^{\dim}}^{\batch_2}$, $\dots$, 
	$y_n\in\pr{[a,b]^{\dim}}^{\batch_n}$ it holds that
	\begin{equation}
		\label{eq:improper:integrable}
		\begin{split}
		\textstyle
		&\textstyle\E\PRb{\vass{
				\alpha_{n,0}m
				+\alpha_{n,n}G_n^{(i)}(\theta_{n},X_n)
				+\sum_{k=1}^{n-1}\alpha_{n,k}G_k^{(i)}(\theta_k,y_k)}}
		\\&\textstyle\leq\vass{
				\alpha_{n,0}m}
				+\alpha_{n,n}\E\PRb{\vass{G_n^{(i)}(\theta_{n},X_n)}}
				+\sum_{k=1}^{n-1}\alpha_{n,k}\vass{G_k^{(i)}(\theta_k,y_k)}
		\\&\textstyle\leq\vass{
			\alpha_{n,0}m}
		+\alpha_{n,n}\E\PRb{\sup_{\omega\in\Omega}\max_{j\in\{1,2,\dots,\batch_n\}}\vass{g_i(\theta_n,X_{n,j}(\omega))}}
		+\sum_{k=1}^{n-1}\alpha_{n,k}\vass{G_k^{(i)}(\theta_k,y_k)}
		\\&\textstyle\leq\vass{
			\alpha_{n,0}m}
		+\alpha_{n,n}\sup_{x\in[a,b]^d}\vass{g_i(\theta_n,x)}
			+\sum_{k=1}^{n-1}\alpha_{n,k}\vass{G_k^{(i)}(\theta_k,y_k)}
			<\infty
			.
			\end{split}
	\end{equation}
	This \proves[pei]\ \cref{it:fist:moment:SGD:mom:for:fact}.
	\Nobs that for all $n\in\N$, $m\in\R$, $\theta_1,\dots,\theta_n\in\R^\pars$, $y_1\in\pr{[a,b]^{\dim}}^{\batch_1}$, 
	$y_2\in\pr{[a,b]^{\dim}}^{\batch_2}$, $\dots$, 
	$y_n\in\pr{[a,b]^{\dim}}^{\batch_n}$ it holds that
		\begin{equation}
		\label{eq:improper:integrable2}
		\begin{split}
			\textstyle
			&\phi_n(m,\theta_1,\dots,\theta_n,y_1,\dots,y_{n-1})
			\\&=
			\var\pr{\inc_n(m,\theta_1,\dots,\theta_n,y_1,\dots,y_{n-1},X_n)}
			\\&=\textstyle\var\prb{\gamma_n\PRb{
					\alpha_{n,0}m
					+\alpha_{n,n}G_n^{(i)}(\theta_{n},X_n)
					+\sum_{k=1}^{n-1}\alpha_{n,k}G_k^{(i)}(\theta_k,y_k)}}
			\\&=(\alpha_{n,n}\gamma_n)^2\var\prb{G_n^{(i)}(\theta_{n},X_n)}
			\\&\textstyle\geq
			(\alpha_{n,n}\gamma_n)^2
			\inf_{\vartheta\in\R^{\pars}}\var\prb{G_n^{(i)}(\vartheta,X_n)}
			\\&=
			(\alpha_{n,n}\gamma_n)^2\pr{\batch_n}^{-1}
			\textstyle\inf_{\vartheta\in\R^{\pars}}\var\pr{g_i(\vartheta,X_{n,1})}
			\\&=
			(\alpha_{n,n}\gamma_n)^2\pr{\batch_n}^{-1}
			\textstyle\inf_{\vartheta\in\R^{\pars}}\var\pr{g_i(\vartheta,X_{1,1})}
			.
		\end{split}
	\end{equation}
	Combining this with \eqref{eq:improper:integrable} and
	\cref{lem:factorization2:spec} 
	(applied with
	$( \Omega, \cF, \P )\curvearrowleft ( \Omega, \cF, \P )$,
	$\gil\curvearrowleft\fil_{n-1}$,
	$n\curvearrowleft J_n\dim$,
	$m\curvearrowleft 1+n\pars+\dim\sum_{k=1}^{n-1}J_k$,
	$E\curvearrowleft \pr{[a,b]^{\dim}}^{\batch_n}$,
	$D\curvearrowleft\prb{\R\times\pr{\R^{\pars}}^n\times\pr{[a,b]^{\dim}}^{\batch_1}\times\pr{[a,b]^{\dim}}^{\batch_2}\times\dots\times\pr{[a,b]^{\dim}}^{\batch_{n-1}}}$,
	$X\curvearrowleft (\mom[i],\Theta_0,\dots,\Theta_{n-1},X_1,\dots,X_{n-1})$,
	$Y\curvearrowleft X_n$,
	$\Phi\curvearrowleft \inc_n$
	for $n\in\N$
	in the notation of \cref{lem:factorization2:spec}) \proves\ that for all $n\in\N$ it holds that
	\begin{align}
		\label{eq:LB:incr:SGD:mom}
			&\nonumber\E\PRb{\prb{\Theta_{n}^{(i)}-\cob{\Theta_{n}^{(i)}}{\fil_{n-1}}}^2}
			\\&=\E\PRb{\prb{\Theta_{n}^{(i)}-\Theta_{n-1}^{(i)}-\cob{\Theta_{n}^{(i)}-\Theta_{n-1}^{(i)}}{\fil_{n-1}}}^2}
			\\&\nonumber=\E\PRb{\pr{\inc_n(\mom[i],\Theta_{0},\dots,\Theta_{n-1},X_1,\dots,X_n)
					-\E\PR{\inc_n(\mom[i],\Theta_{0},\dots,\Theta_{n-1},X_1,\dots,X_n)|\fil_{n-1}}}^2}
			\\&\nonumber\geq (\alpha_{n,n}\gamma_n)^2\pr{\batch_n}^{-1}
			\textstyle\inf_{\vartheta\in\R^{\pars}}\var\pr{g_i(\vartheta,X_{1,1})}
		\end{align}
	\cfload.
	\Moreover
	the fact that for all $m\in\N$, $n\in\N\cap[m,\infty)$ it holds that $\Theta_{m}=\prb{\Theta_m^{(1)},\dots,\Theta_m^{(\pars)}}\colon\Omega\to\R^{\pars}$ is $\fil_{n}$-measurable
	and \cref{lem:cond:exp:vanish3} (applied with 
	$ ( \Omega, \cF, \P ) \curvearrowleft  ( \Omega, \cF, \P ) $, 
	$X\curvearrowleft \Theta_n^{(i)}$,
	$Y\curvearrowleft \Theta_m^{(i)}$,
	$\gil\curvearrowleft \fil_{n-1}$
	for $m\in\N$, $n\in\N\cap(m,\infty)$ 
	in the notation of \cref{lem:cond:exp:vanish3})
	\prove\ that for all $m\in\N$, $n\in\N\cap(m,\infty)$ 
	it holds that
	\begin{equation}
		\label{eq:cauchy:to:incr}
		\begin{split}
			\E\PRb{\prb{\Theta_n^{(i)}-\Theta_m^{(i)}}^2}
			&\geq
			\E\PRb{\prb{\Theta_n^{(i)}-\cob{\Theta_{n}^{(i)}}{\fil_{n-1}}}^2}
			.
		\end{split}
	\end{equation}
	This, \eqref{eq:LB:incr:SGD:mom},
	and
	\cref{norm:propterties:at:infty}
	(applied with 
	$ ( \Omega, \cF, \P ) \curvearrowleft  ( \Omega, \cF, \P ) $, 
	$\pars\curvearrowleft 1$,
	$z\curvearrowleft z$,
	$Y\curvearrowleft \xi$,
	$(X_n)_{n\in\N}\curvearrowleft (\Theta_n^{(i)})_{n\in\N}$
	for a random variable $\xi\colon\Omega\to\R$ in the notation of \cref{norm:propterties:at:infty})
	\prove\ that for all random variables $\xi\colon\Omega\to\R$ it holds that
	\begin{equation}
		\begin{split}
			&\textstyle z\PRb{\limsup_{ n \to \infty }z\E\PRb{ (\Theta_n^{(i)} - \xi)^2 }}
			\\&\geq\textstyle\tfrac{z}{4}\PRb{\limsup_{n\to\infty}\limsup_{m\to\infty}z\E\PRb{\prb{\Theta_n^{(i)}-\Theta_m^{(i)}}^2}}
			\\&=\textstyle\tfrac{z}{4}\PRb{\limsup_{n\to\infty}\limsup_{m\to\infty}z\E\PRb{\prb{\Theta_n^{(i)}-\Theta_{\max\{n+1,m\}}^{(i)}}^2}}
			\\&\geq\textstyle\tfrac{z}{4}\PRb{\limsup_{n\to\infty}\limsup_{m\to\infty}z\E\PRb{\prb{\Theta_{\max\{n+1,m\}}^{(i)}-\cob{\Theta_{\max\{n+1,m\}}^{(i)}}{\fil_{\max\{n+1,m\}-1}}}^2}}
			\\&=\textstyle\tfrac{z}{4}\PRb{\limsup_{n\to\infty}z\E\PRb{\prb{\Theta_{n}^{(i)}-\cob{\Theta_{n}^{(i)}}{\fil_{n-1}}}^2}}
			\\&\geq\textstyle
			\tfrac{z}{4}\limsup_{ n \to \infty }
			z\PRb{(\alpha_{n,n}\gamma_n)^2\pr{\batch_n}^{-1}
				\textstyle\inf_{\vartheta\in\R^{\pars}}\var\pr{g_i(\vartheta,X_{1,1})}
			}
			.
		\end{split}
	\end{equation}
	This, 
	\eqref{eq:setup:SGD:mom}
	the assumption that $\inf_{\theta\in\R^{\pars}}\var\pr{g_i(\theta,X_{1,1})}>0$ \prove\ that
	\begin{equation}
		\begin{split}
					&\inf_{\substack{\xi\colon \Omega \to \R\\\text{measurable}}}\prbbb{z\PRbbb{\limsup_{ n \to \infty }\prb{z\E\PRb{ | \Theta_n^{(i)} - \xi |^2 }}}
					}
			\\&\geq
			\inf_{\substack{\xi\colon \Omega \to \R\\\text{measurable}}}\prbbb{
			\tfrac{z}{4}\PRbbb{\limsup_{ n \to \infty }
			\PRb{z(\alpha_{n,n}\gamma_n)^2\pr{\batch_n}^{-1}
				\textstyle\prb{\inf_{\vartheta\in\R^{\pars}}\var\pr{g_i(\vartheta,X_{1,1})}}
			}}}
			\\&=\textstyle
			\tfrac{1}{4}\PRb{\inf_{\vartheta\in\R^{\pars}}\var\pr{g_i(\vartheta,X_{1,1})}}
						\PRb{z\prb{\limsup_{ n \to \infty }\PRb{z
					(\alpha_{n,n}\gamma_n)^2\pr{\batch_n}^{-1}
						\textstyle}}}
			>0		
			.
		\end{split}
	\end{equation}
	This \proves\ \cref{it:SGD:mom:lims}. \finishproofthus
\end{proof}

\subsection{Lower bounds for adaptive SGD optimization methods}
\label{subsection:4.2}

\cfclear
\begin{athm}{lemma}{lem:factorization:applied:Adam:class}
	Let $\pars,\dim,n\in\N$, $a\in\R$, $b\in[a,\infty)$,
	$\eps\in(0,\infty)$, $\gamma\in[0,\infty)$,
	let $(\Omega,\cF,(\mathbb{F}_k)_{k\in\N_0},\P)$ be a filtered probability space,
	let $\batch_k\in\N$, $k\in\{1,2,\dots,n\}$,
	for every $k\in\{1,2,\dots,n\}$ let
	$X_k=(X_{k,1},\dots,X_{k,\batch_k})\colon\Omega\to\pr{[a,b]^{\dim}}^{\batch_k}$
	be $\fil_k$-measurable,
	assume that $X_n$ is independent of $\fil_{n-1}$,
	let $\MOM\colon\R^{n+1}\to[0,\infty)$ be measurable,
	let $\MOMs[1]\colon\Omega\to\R$ and $\MOMs[2]\colon\Omega\to\R$ be $\fil_0$-measurable,
	let $\alpha_k\in[0,1]$, $k\in\{1,2,\dots,n\}$,
	for every $k\in\{1,2,\dots,n\}$ let
	$G_k\colon\R^{\pars}\times\pr{[a,b]^{\dim}}^{\batch_k}\to\R$ be measurable,
	assume for all $\theta\in\R^{\pars}$ that
	\begin{equation}
		\label{eq:1st:mom:bd:grad}
		\textstyle
		\E\PRb{\vass{G_n(\theta,X_n)}^2}<\infty,
	\end{equation}
	let $\inc\colon\R^2\times\pr{\R^{\pars}}^n
	\times\pr{[a,b]^{\dim}}^{\batch_1}
	\times\pr{[a,b]^{\dim}}^{\batch_2}
	\times\dots\times\pr{[a,b]^{\dim}}^{\batch_n}\to\R$
	satisfy for all $m,s\in\R$, $\theta_1,\theta_2,\dots,\theta_n\in\R^{\pars}$,
	$y_1\in\pr{[a,b]^{\dim}}^{\batch_1}$, 
	$y_2\in\pr{[a,b]^{\dim}}^{\batch_2}$, $\dots$, 
	$y_n\in\pr{[a,b]^{\dim}}^{\batch_n}$ that
	\begin{equation}
		\label{eq:def:Phi:for:var:fact}
		\begin{split}
			&\inc(m,s,\theta_1,\dots,\theta_n,y_1,\dots,y_n)
			=-\frac{\gamma\PR{m+\sum_{k=1}^n\alpha_kG_k(\theta_k,y_k)}}
			{\eps+\MOM(s,G_1(\theta_1,y_1),\dots,G_n(\theta_n,y_n))},
		\end{split}
	\end{equation}
	for every $k\in\N_0\cap[0,n]$ let
	$\Theta_k=(\Theta_k^{(1)},\dots,\Theta_k^{(\pars)})\colon\Omega\to\R^{\pars}$
	be $\fil_k$-measurable,
	let $i\in\{1,2,\dots,\pars\}$ and
	$\cva_n\colon\R^{\pars}\to[0,\infty)$ satisfy for all
	$\theta\in\R^{\pars}$ that
	\begin{equation}
		\label{eq:setup:fact:lemma:applied:Adam:class}
		\begin{split}
			&\Theta_n^{(i)}=\Theta_{n-1}^{(i)}
			+\inc(\MOMs[1],\MOMs[2],\Theta_0,\dots,\Theta_{n-1},X_1,\dots,X_n)
			\qandq
			\cva_n(\theta)=\var\pr{G_n(\theta,X_n)},
		\end{split}
	\end{equation}
	let $\UB\colon\Omega\to[0,\infty)$ be $\fil_{n-1}$-measurable,
	and assume that it holds $\P$-\as\ that
	\begin{equation}
		\label{eq:UB:second:moment:Adam:class}
		\UB\geq\MOM\prb{\MOMs[2],G_1(\Theta_0,X_1),\dots,G_n(\Theta_{n-1},X_n)}.
	\end{equation}
	Then
	\begin{enumerate}[label=(\roman*)]
		\item \label{it:lem:factorization:applied:Adam:class:1}
		it holds that $\inc$ is measurable,
		\item \label{it:lem:factorization:applied:Adam:class:2}
		it holds for all $m,s\in\R$,  $\theta_{1},\theta_2,\dots,\theta_{n}\in\R^{\pars}$, $y_1\in\pr{[a,b]^{\dim}}^{\batch_1}$, 
		$y_2\in\pr{[a,b]^{\dim}}^{\batch_2}$, $\dots$, 
		$y_n\in\pr{[a,b]^{\dim}}^{\batch_n}$ that
		\begin{equation}
			\begin{split}
				\E\PRb{\vass{\inc(m,s,\theta_1,\dots,\theta_n,y_1,\dots,y_{n-1},X_n)}}
				<\infty,
			\end{split}
		\end{equation}
		\item \label{it:lem:factorization:applied:Adam:class:2.1}
		it holds that \improper{\Theta_n^{(i)}-\Theta_{n-1}^{(i)}}{\fil_{n-1}},
		and
		\item \label{it:lem:factorization:applied:Adam:class:3}
		it holds $\P$-\as\ that
		\begin{equation}
			\label{eq:conditional:second:moment:Adam:class}
			\begin{split}
				\cob{\prb{\Theta_n^{(i)}-\Theta_{n-1}^{(i)}}^2}{\fil_{n-1}}
				\geq
				\frac{\gamma^2(\alpha_n)^2\cva_n(\Theta_{n-1})}{\pr{\eps+\UB}^2}
				.
			\end{split}
		\end{equation}
	\end{enumerate}
	\cfout.
\end{athm}

\begin{proof}[Proof of \cref{lem:factorization:applied:Adam:class}]
	Throughout this proof let
	$C\colon\Omega\to\R$, 	
	$T\colon\Omega\to[0,\infty)$, 	
	$V\colon\Omega\to[0,\infty)$, and
	$\mu\colon\R^{\pars}\to\R$ satisfy that
	\begin{equation}
		\label{eq:history:second:moment:Adam:class}
		\begin{split}
			C=\textstyle\MOMs[1]+\sum_{k=1}^{n-1}\alpha_kG_k(\Theta_{k-1},X_k),
			\qquad
			V=\MOM\prb{\MOMs[2],G_1(\Theta_0,X_1),\dots,G_n(\Theta_{n-1},X_n)},
			\\
			\mu(\theta)=\E\PR{G_n(\theta,X_n)},
			\qandq
			T=\frac{\gamma^2}{\eps^2}
			\PRb{\pr{C+\alpha_n\mu(\Theta_{n-1})}^2+(\alpha_n)^2\cva_n(\Theta_{n-1})}
		\end{split}
	\end{equation}
	and for every $j\in\N$ let $A_j\subseteq\Omega$ satisfy $A_j=\{T\leq j\}$.
	\Nobs that 
	\eqref{eq:def:Phi:for:var:fact},
	 the assumption that $\MOM, G_1, G_2, \dots, G_n$ are measurable, and the assumption that $\eps>0$
	\prove\ that $\inc$ is measurable.
	This \proves[pei]\ \cref{it:lem:factorization:applied:Adam:class:1}.
	\Moreover \eqref{eq:1st:mom:bd:grad} and \eqref{eq:def:Phi:for:var:fact}
	\prove\ for all $m,s\in\R$,
	$\theta_1,\theta_2,\dots,\theta_n\in\R^{\pars}$,
	$y_k\in\pr{[a,b]^{\dim}}^{\batch_k}$, $k\in\{1,2,\dots,n-1\}$, that
	\begin{equation}
		\label{eq:finite:factor:exp}
		\begin{split}
			&\E\PRb{\vass{\inc(m,s,\theta_1,\dots,\theta_n,y_1,\dots,y_{n-1},X_n)}}
			\\&=
			\E\PRbbb{\frac{\gamma \vass{m+\alpha_n G_{n}(\theta_{n},X_n)+\sum_{k=1}^{n-1} \alpha_{k}G_{k}(\theta_{k},y_k)}}{\eps+\MOM(s,G_1(\theta_1,y_1),G_2(\theta_2,y_2),\dots,G_{n-1}(\theta_{n-1},y_{n-1}),G_n(\theta_n,X_n))}}
			\\&\leq
			\frac{\gamma \vass{m+\sum_{k=1}^{n-1} \alpha_{k}G_{k}(\theta_{k},y_k)}}{\eps}
			+
						\frac{\gamma\alpha_n\E\PR{\vass{ G_{n}(\theta_{n},X_n)}}}{\eps}
			<\infty.
		\end{split}
	\end{equation}
	This \proves[pei]\ \cref{it:lem:factorization:applied:Adam:class:2}.
	\Nobs that
	\eqref{eq:def:Phi:for:var:fact},
	\eqref{eq:setup:fact:lemma:applied:Adam:class}, and
	\eqref{eq:history:second:moment:Adam:class} \prove\ that
	\begin{equation}
		\label{eq:increment:second:moment:Adam:class}
		\Theta_n^{(i)}-\Theta_{n-1}^{(i)}=-\frac{\gamma(C+\alpha_nG_n(\Theta_{n-1},X_n))}{\eps+V}.
	\end{equation}
	\Moreover integration with respect to the distribution of $X_n$
	\proves\ that $\mu$ and $\cva_n$ are measurable.
	\Nobs that for all $c\in\R$, $\theta\in\R^{\pars}$ it holds that
	\begin{equation}
		\label{eq:fixed:history:second:moment:Adam:class}
		\begin{split}
			\E\PRb{\pr{c+\alpha_nG_n(\theta,X_n)}^2}
			=\pr{c+\alpha_n\mu(\theta)}^2+(\alpha_n)^2\cva_n(\theta).
		\end{split}
	\end{equation}
	This, the fact that $C$ and $\Theta_{n-1}$ are $\fil_{n-1}$-measurable,
	the assumption that $X_n$ is independent of $\fil_{n-1}$,
	and \cref{lem:2.9:from:1609.07031} (applied with
	$\gil\curvearrowleft\fil_{n-1}$,
	$D\curvearrowleft\R\times\R^{\pars}$,
	$E\curvearrowleft\pr{[a,b]^{\dim}}^{\batch_n}$,
	$X\curvearrowleft(C,\Theta_{n-1})$,
	$Y\curvearrowleft X_n$,
	$\Phi\curvearrowleft\prb{(\R\times\R^{\pars})\times\pr{[a,b]^{\dim}}^{\batch_n}
		\ni((c,\theta),x)\mapsto\pr{c+\alpha_nG_n(\theta,x)}^2\in[0,\infty)}$
	in the notation of \cref{lem:2.9:from:1609.07031})
	\prove\ that
	$\cva_n$ is measurable and it holds $\P$-\as\ that
	\begin{equation}
		\label{eq:numerator:conditional:second:moment:Adam:class}
		\begin{split}
			\cob{\pr{C+\alpha_nG_n(\Theta_{n-1},X_n)}^2}{\fil_{n-1}}
			&=\pr{C+\alpha_n\mu(\Theta_{n-1})}^2
			+(\alpha_n)^2\cva_n(\Theta_{n-1})
			<\infty.
		\end{split}
	\end{equation}
	This, the fact that $\mu$ is measurable, and the fact that $C$ and $\Theta_{n-1}$ are $\fil_{n-1}$-measurable \prove\ that $T$ is finite and $\fil_{n-1}$-measurable.
	\Moreover \eqref{eq:increment:second:moment:Adam:class},
	\eqref{eq:numerator:conditional:second:moment:Adam:class}, and the fact that $V\geq0$
	\prove\ that it holds $\P$-\as\ that
	\begin{equation}
		\label{eq:conditional:finite:second:moment:Adam:class}
		\begin{split}
		\cob{\prb{\Theta_n^{(i)}-\Theta_{n-1}^{(i)}}^2}{\fil_{n-1}}
		&=
		\cob{\gamma^2(C+\alpha_nG_n(\Theta_{n-1},X_n))^2\pr{\eps+V}^{-2}}{\fil_{n-1}}
		\\&\leq
		\tfrac{\gamma^2}{\eps^2}\cob{(C+\alpha_nG_n(\Theta_{n-1},X_n))^2}{\fil_{n-1}}
		= T.
		\end{split}
	\end{equation}
	\Nobs that
	\eqref{eq:history:second:moment:Adam:class} and \eqref{eq:numerator:conditional:second:moment:Adam:class} \prove\ that 
	\begin{equation}
		\llabel{covering:set:and:mb}
		\Omega=\cup_{j\in\N}A_j
		\qandq
		\PRb{\forall\,j\in\N\colon A_j\in\fil_{n-1}}.
	\end{equation}
	This and
	\eqref{eq:conditional:finite:second:moment:Adam:class}
	\prove\ that for all $j\in\N$ it holds that
	\begin{equation}
		\begin{split}
			\E\PR{\prb{\Theta_n^{(i)}-\Theta_{n-1}^{(i)}}^2\indicator{A_j}}
			\leq\E\PR{T\indicator{A_j}}\leq j.
		\end{split}
	\end{equation}
	This and the Cauchy--Schwarz inequality \prove\ that for all $j\in\N$ it holds that
	\begin{equation}
		\begin{split}
			\prb{\E\PRb{\vass{\Theta_n^{(i)}-\Theta_{n-1}^{(i)}}\indicator{A_j}}}^2
			&\leq\prb{\E\PRb{\prb{\Theta_n^{(i)}-\Theta_{n-1}^{(i)}}^2\indicator{A_j}}}
			\pr{\E\PR{\indicator{A_j}}}
			\leq j
			\P(A_j)
			\leq j<\infty.
		\end{split}
	\end{equation}
	This and \lref{covering:set:and:mb} \prove[pei]\ \cref{it:lem:factorization:applied:Adam:class:2.1}.
	\Nobs that \eqref{eq:UB:second:moment:Adam:class},
	\eqref{eq:increment:second:moment:Adam:class},
	\eqref{eq:numerator:conditional:second:moment:Adam:class}, and the fact that
	$\UB$ is $\fil_{n-1}$-measurable \prove\ that it holds $\P$-\as\ that
	\begin{equation}
		\label{eq:second:moment:lemma:applied}
		\begin{split}
			\cob{\prb{\Theta_n^{(i)}-\Theta_{n-1}^{(i)}}^2}{\fil_{n-1}}
			&=\cob{\gamma^2(C+\alpha_nG_n(\Theta_{n-1},X_n))^2\pr{\eps+V}^{-2}}{\fil_{n-1}}
			\\&\geq
			\cob{\gamma^2(C+\alpha_nG_n(\Theta_{n-1},X_n))^2\pr{\eps+\UB}^{-2}}{\fil_{n-1}}
			\\&=
			\frac{\gamma^2\cob{\pr{C+\alpha_nG_n(\Theta_{n-1},X_n)}^2}{\fil_{n-1}}}{\pr{\eps+\UB}^2}
			\\&=
			\frac{\gamma^2\PRb{\pr{C+\alpha_n\mu(\Theta_{n-1})}^2
					+(\alpha_n)^2\cva_n(\Theta_{n-1})}}{\pr{\eps+\UB}^2}
			\\&\geq
			\frac{\gamma^2(\alpha_n)^2\cva_n(\Theta_{n-1})}{\pr{\eps+\UB}^2}.
		\end{split}
	\end{equation}
	This \proves[pei]\ \cref{it:lem:factorization:applied:Adam:class:3}.
	\finishproofthus
\end{proof}

\cfclear

\newcommand{\g}{\mathfrak{g}}

\cfclear
\begin{athm}{prop}{prop:prop:non_convergence_modified_Adam:specific:setup:SCOPE}
	For every $n\in\N$ let $\MOM_n\colon\R^{n+1}\to[0,\infty)$ be measurable, 
		let $\gamma\colon\N\to[0,\infty)$ and $\batch\colon\N\to\N$ be functions,
		let $\pars,\dim\in\N$, 
		$a\in\R$,
		$b\in[a,\infty)$,
	for every $n\in\N$ let $G_{n}=\pr{G_{n}^{(1)},\dots,G_{n}^{(\pars)}}\colon\R^\pars \times\pr{[a,b]^{\dim}}^{\batch_n}\to\R^\pars$ be measurable,
	let $ ( \Omega, \cF,(\mathbb{F}_n)_{n\in\N_0}, \P ) $ be a filtered probability space,
	for every $n\in\N$ let $X_n=(X_{n,1},\dots,X_{n,\batch_n})\colon\Omega \allowbreak\to \pr{[ a, b ]^{\dim}}^{\batch_n}$ be $\fil_n$-measurable, 
	assume for all $n\in\N$ that $X_n$ is independent of $\fil_{n-1}$,
	let
	$\eps\in(0,\infty)$,
	$\alpha\in[0,1)$,
	let
	$ \Theta=(\Theta^{(1)},\ldots,\Theta^{(\pars)}) \colon\N_0\times\Omega \to \R^{\pars}$ and
	$\MOMs=(\MOMs^{(1)},\ldots,\MOMs^{(\pars)})\colon\N\times\Omega\to\R^\pars$ 
	satisfy
	for all $n\in\N$, $i\in\{1,2,\dots,\pars\}$ that
	\begin{equation}
		\label{eq:incr:rule:Adam}
		\begin{split}
			\Theta_{ n }^{(i)}
			&=\Theta_{ n-1  }^{(i)}-\frac{\gamma_{n}  \PR{\alpha^n\MOMs[1]^{(i)}+(1-\alpha)\sum_{k=1}^n\alpha^{n-k}G_k^{(i)}(\Theta_{k-1},X_k)}}{\eps+\MOM_n\prb{\MOMs[2]^{(i)},G_1^{(i)}(\Theta_0,X_1),G_2^{(i)}(\Theta_1,X_2),\dots,G_n^{(i)}(\Theta_{n-1},X_n)}},
		\end{split}
	\end{equation}
	assume that $\Theta_0$, $\MOMs[1]$, and $\MOMs[2]$ are $\fil_0$-measurable,
	let $i\in\{1,2,\dots,\pars\}$ 
	and
	$\cva_n\colon\R^\pars\to\R$, $n\in\N$,
	satisfy
	for all $\theta=(\theta_1,\dots,\theta_\pars)\in \R^{\pars}$, $n\in\N$ that
	\begin{equation}
		\label{eq:setup:starting:expectations:bounded}
		\textstyle
		\E\PRb{\vass{G_n^{(i)}(\theta,X_n)}^2}<\infty
				\qqandqq
		\cva_n\pr{\theta}=\var\prb{G_n^{(i)}\pr{\theta,X_{n}}}
		,
	\end{equation}
		for every $n\in\N$ let $\UB_n\colon\Omega\to\R$ be $\fil_{n-1}$-measurable, assume for all $n\in\N$ that
	\begin{equation}
		\llabel{eq:UB}
		\UB_n\geq \MOM_n\prb{\MOMs[2]^{(i)},G_1^{(i)}(\Theta_0,X_1),G_2^{(i)}(\Theta_1,X_2),\dots,G_{n-1}^{(i)}(\Theta_{n-2},X_{n-1}),G_n^{(i)}(\Theta_{n-1},X_n)}.
	\end{equation}
	Then it holds for all $n\in\N$ that
			\begin{equation}
				\label{it:increment:ADAM:L2}
			\begin{split}
				&\E\PRb{\prb{\Theta_{n}^{(i)}-\Theta_{n-1}^{(i)}}^2}
				\geq
				\E\PRbbb{
					\frac{(\gamma_n)^2(1-\alpha)^2\cva_n(\Theta_{n-1})}
					{\pr{\eps+\UB_n}^2}}
				.
			\end{split}
		\end{equation}

\end{athm}

\begin{proof}[Proof of \cref{prop:prop:non_convergence_modified_Adam:specific:setup:SCOPE}]
	\Nobs that \eqref{eq:incr:rule:Adam}, the assumption that for all $n\in\N$ it holds that $G_n$ and $\MOM_n$ are measurable and $X_n$ is $\fil_n$-measurable, the assumption that
	$\Theta_0$, $\MOMs[1]$, and $\MOMs[2]$ are $\fil_0$-measurable,
	and induction
	\prove\ that for all
	$n\in\N_0$ it holds that $\Theta_n$ is $\fil_n$-measurable.
	This,
	\lref{eq:UB},
	 \eqref{eq:incr:rule:Adam},
	\eqref{eq:setup:starting:expectations:bounded},
	the assumption that for all $n\in\N$ it holds that $\UB_n$ is $\fil_{n-1}$-measurable and that $X_n$ is independent of $\fil_{n-1}$, and
	\cref{it:lem:factorization:applied:Adam:class:3} in
	\cref{lem:factorization:applied:Adam:class}
	(applied with
	$\gamma\curvearrowleft\gamma_n$,
	$\MOMs[1]\curvearrowleft\alpha^n\MOMs[1]^{(i)}$,
	$\MOMs[2]\curvearrowleft\MOMs[2]^{(i)}$,
	$\MOM\curvearrowleft\MOM_n$,
	$\UB\curvearrowleft\UB_n$,
	$G_k\curvearrowleft G_k^{(i)}$,
	$\alpha_k\curvearrowleft(1-\alpha)\alpha^{n-k}$
	for 
	$n\in\N$,
	$k\in\{1,2,\dots,n\}$ in the notation of
	\cref{lem:factorization:applied:Adam:class})
	\prove\ that for all $n\in\N$ it holds that
	\begin{equation}
		\label{eq:LB:cond:exp:prep}
		\begin{split}
		\E\PRb{\prb{\Theta_{n}^{(i)}-\Theta_{n-1}^{(i)}}^2}
		=
		\E\PRb{\cob{\prb{\Theta_n^{(i)}-\Theta_{n-1}^{(i)}}^2}{\fil_{n-1}}}
		&\geq
		\E\PRbbb{
			\frac{(\gamma_n)^2(1-\alpha)^2\cva_n(\Theta_{n-1})}
			{\pr{\eps+\UB_n}^2}}
.
		\end{split}
	\end{equation}
	This \proves\ \eqref{it:increment:ADAM:L2}.
	\finishproofthus
\end{proof}

\subsection{Non-convergence for adaptive SGD optimization methods}
\label{subsection:3.4}

\newcommand{\const}{\mathfrak{C}}
\cfclear
\begin{athm}{theorem}{cor:limsup:simple}
	Let $q\in\R$, 
	let $\gamma\colon\N\to[0,\infty)$ and $\batch\colon \N\to\N$ satisfy
		\begin{equation}
				\label{eq:LR:simple}
		\textstyle\limsup_{n\to\infty}\PRb{(\gamma_n)^2(\batch_n)^{-q}}>0
		,
	\end{equation}
	let $\pars,\dim\in\N$, $\const,a\in\R$, $b\in[a,\infty)$, $p\in(0,\infty)$,
	for every $n\in\N$ let
	$G_{n}=\pr{G_{n}^{(1)},\dots,G_{n}^{(\pars)}}\colon\R^\pars \times\pr{[a,b]^{\dim}}^{\batch_n}\to\R^\pars$ be measurable, 
	assume for all $n\in\N$, $\theta\in\R^{\pars}$, $x\in\pr{[a,b]^{\dim}}^{\batch_n}$ that
		\begin{equation}
			\label{prop:eq:setup:simple}
		\textstyle
		\norm{G_n(\theta,x)}\leq \const(\norm{\theta}+1)^p
		,
	\end{equation}
	let 
	$\beta\in[0,1)$,
	$\fc\in[0,\infty)$,
		for every $n\in\N$ let $\MOM_n\colon\R^{n+1}\to[0,\infty)$ be measurable, 
		assume 
	for all $n\in\N$, $x=(x_1,\dots,x_{n+1})\in\R^{n+1}$ that
	\begin{equation}
		\label{eq:cond1:MOM:simple}
		\textstyle
		\vass{\MOM_n(x)}^2
		\leq
		\fc(1-\beta)\sum_{k=0}^n(x_{k+1})^2\beta^{n-k},
	\end{equation}
	let $ ( \Omega, \cF,(\mathbb{F}_n)_{n\in\N_0}, \P ) $ be a filtered probability space,
	for every $n\in\N$ let $X_n=(X_{n,1},\dots,X_{n,\batch_n})\colon\allowbreak\Omega \allowbreak\to \pr{[ a, b ]^{\dim}}^{\batch_n}$ be $\fil_n$-measurable,
	assume for all $n\in\N$ that $X_{n}$ is independent of $\fil_{n-1}$,
	let 
	$\alpha\in[0,1)$,
	$\eps\in(0,\infty)$,
	let
	$ \Theta=(\Theta^{(1)},\ldots,\Theta^{(\pars)}) \colon\N_0\times\Omega \to \R^{\pars}$
	and
	$\MOMs=(\MOMs^{(1)},\ldots,\MOMs^{(\pars)})\colon\N\times\Omega\to\R^\pars$ 
	satisfy
	for all $n\in\N$, $i\in\{1,2,\dots,\pars\}$ that
	\begin{equation}
		\label{eq:incr:rule:Adam:simple}
		\begin{split}
			\Theta_{ n }^{(i)}
			&=\Theta_{ n-1  }^{(i)}-\frac{\gamma_{n}  \PRb{\alpha^n\MOMs[1]^{(i)}+(1-\alpha)\sum_{k=1}^n\alpha^{n-k}G_k^{(i)}(\Theta_{k-1},X_k)}}{\eps+\MOM_n\prb{\MOMs[2]^{(i)},G_1^{(i)}(\Theta_0,X_1),G_2^{(i)}(\Theta_1,X_2),\dots,G_n^{(i)}(\Theta_{n-1},X_n)}},
		\end{split}
	\end{equation}
	assume that $\Theta_0$, $\MOMs[1]$, and $\MOMs[2]$ are $\fil_0$-measurable, and assume 
			\begin{equation}
		\label{eq:setup:simple:starting:expectations:bounded}
		\inf_{n\in\N}\inf_{\theta\in\R^\pars}
		\textstyle\PRb{(\batch_n)^q\sum_{i=1}^\pars\var\prb{G_n^{(i)}\pr{\theta,X_n}}}>0
		.
	\end{equation}
	Then
	\begin{equation}
		\label{statement:easy}
		\begin{split}
			\inf_{\substack{\vartheta\colon \Omega \to \R^\pars\\\text{measurable}}}
			\PRbbb{\limsup_{ n \to \infty }\E\PRb{ \norm{ \Theta_n - \vartheta }^{2} }}
			>
			0
			.
		\end{split}
	\end{equation}
	\cfout
\end{athm}

\begin{proof}[Proof of \cref{cor:limsup:simple}]
	Throughout this proof assume without loss of generality that there exists a random variable $\xi\colon\Omega\to\R^{\pars}$ such that $\limsup_{ n \to \infty }\E\PRb{ \norm{ \Theta_n - \xi }^{2} }<\infty$,
	for every $n\in\N$,
	$i\in\{1,2,\dots,\pars\}$ let $\cva_{n,i}\colon\R^\pars\to[0,\infty)$
	satisfy for all $\theta\in\R^\pars$ that
	\begin{equation}
		\label{eq:jth:step:var}
		\cva_{n,i}(\theta)=\var\prb{G_n^{(i)}(\theta,X_n)},
	\end{equation}
	let $\xi\colon\Omega\to\R^{\pars}$ be measurable, let $r\in[0,\infty)$, $N\in\N$, $\delta\in(0,\infty)$ satisfy
	\begin{equation}
		\llabel{eq:setup:asymp}
		\textstyle
		r=\sup_{n\in\pr{\N\cap[N,\infty)}}\E\PRb{ \norm{ \Theta_n - \xi }^{2} }
		\qandq
		\delta=\inf_{n\in\N}\inf_{\theta\in\R^\pars}
		\textstyle\PRb{(\batch_n)^q\sum_{i=1}^\pars\cva_{n,i}(\theta)}
		,
	\end{equation} 
	let $s\in(0,1]$, $R\in\R$, $B\in\cF$ satisfy
	\begin{equation}
		\label{eq:limsup:simple:localization}
		\textstyle
		s=\min\{1,1/p\},
		\qquad
		B=\pRb{\norm{\xi}+\norm{\MOMs[2]}+
			\max_{0\leq k<N}\norm{\Theta_k}\leq R},
		\qqandqq \P(B)>0,
	\end{equation}
	and for every
	$n\in\N$ let $\UB_n\colon\Omega\to[0,\infty)$ satisfy
	\begin{equation}
		\label{eq:setup:UB:nth:moment:simple:case}
		(\UB_n)^2=\fc(1-\beta)\textstyle\PRb{
			\beta^n\norm{\MOMs[2]}^2+
			\const^2\sum_{k=1}^n\beta^{n-k}(1+\norm{\Theta_{k-1}})^{2p}}.
	\end{equation}
	\Nobs that \eqref{prop:eq:setup:simple}, \eqref{eq:cond1:MOM:simple}, and
	\eqref{eq:setup:UB:nth:moment:simple:case}
	\prove\ that for all $n\in\N$, $i\in\{1,2,\dots,\pars\}$ it holds that $\UB_n$ is $\fil_{n-1}$-measurable and
	\begin{equation}
		\label{eq:limsup:simple:UB:predictable}
		\begin{split}
			&\MOM_n\prb{\MOMs[2]^{(i)},G_1^{(i)}(\Theta_0,X_1),
				\dots,G_n^{(i)}(\Theta_{n-1},X_n)}
				\\&\leq
				\textstyle
				\prb{\fc(1-\beta)\PRb{\vass{\MOMs[2]^{(i)}}^2\beta^{n}+\sum_{k=1}^n\vass{G_k^{(i)}(\Theta_{k-1},X_k)}^2\beta^{n-k}}}^{\nicefrac{1}{2}}
				\\&\leq
				\prb{\fc(1-\beta)\textstyle\PRb{
					\beta^n\norm{\MOMs[2]}^2+
					\const^2\sum_{k=1}^n\beta^{n-k}(1+\norm{\Theta_{k-1}})^{2p}}}^{\nicefrac{1}{2}}
				=\UB_n.
		\end{split}
	\end{equation}
	\Moreover \eqref{prop:eq:setup:simple} \proves\ that for all
	$n\in\N$, $i\in\{1,2,\dots,\pars\}$, $\theta\in\R^\pars$ it holds that
	\begin{equation}
		\label{eq:limsup:simple:fixed:second:moment}
		\E\PRb{\vass{G_n^{(i)}(\theta,X_n)}^2}
		\leq \const^2(1+\norm{\theta})^{2p}<\infty.
	\end{equation}
	This, \eqref{eq:incr:rule:Adam:simple}, \eqref{eq:jth:step:var},
	\eqref{eq:limsup:simple:UB:predictable}, the assumption that for all $n\in\N$
	it holds that $X_n$ is independent of $\fil_{n-1}$, and
	\cref{prop:prop:non_convergence_modified_Adam:specific:setup:SCOPE}
	(applied with $i\curvearrowleft i$, $\cva_n\curvearrowleft\cva_{n,i}$,
	$\UB_n\curvearrowleft\UB_n$ for $n\in\N$, $i\in\{1,2,\dots,\pars\}$
	in the notation of \cref{prop:prop:non_convergence_modified_Adam:specific:setup:SCOPE})
	\prove\ that for all $n\in\N$, $i\in\{1,2,\dots,\pars\}$ it holds that
	\begin{equation}
		\label{eq:limsup:simple:coordinate:increment}
		\E\PRb{\prb{\Theta_n^{(i)}-\Theta_{n-1}^{(i)}}^2}
		\geq \E\PRbbb{\frac{(\gamma_n)^2(1-\alpha)^2\cva_{n,i}(\Theta_{n-1})}
			{(\eps+\UB_n)^2}}.
	\end{equation}
	Combining this with \lref{eq:setup:asymp}
	\proves\ that for all $n\in\N$ it holds that
	\begin{equation}
		\label{eq:limsup:simple:vector:increment}
		\begin{split}
			\textstyle
			\E\PRb{\norm{\Theta_n-\Theta_{n-1}}^2}
			=
			\E\PRb{\sum_{i=1}^\pars\prb{\Theta_n^{(i)}-\Theta_{n-1}^{(i)}}^2}
			&\geq \E\PRbbb{\frac{(\gamma_n)^2(1-\alpha)^2
					\sum_{i=1}^\pars\cva_{n,i}(\Theta_{n-1})}
				{(\eps+\UB_n)^2}}\\
			&\geq \delta(\gamma_n)^2(1-\alpha)^2(\batch_n)^{-q}
			\E\PRb{(\eps+\UB_n)^{-2}}.
		\end{split}
	\end{equation}
	This, 
	\lref{eq:setup:asymp},
	\eqref{eq:limsup:simple:localization}, and the fact that $\forall\,x,y\in\R\colon 2x^2+2y^2\geq (x+y)^2$
	\prove\ that 
	\begin{equation}
		\label{eq:limsup:simple:localized:second:moments}
		\begin{split}
			&\textstyle\sup_{n\in\N_0}\E\PRb{\norm{\Theta_n}^2\indicator{B}}
			\\&=\textstyle
			\max\pRb{
				\sup_{n\in\pr{\N_0\cap[0,N)}}\E\PRb{\norm{\Theta_n}^2\indicator{B}},
				\sup_{n\in\pr{\N\cap[N,\infty)}}\E\PRb{\norm{\Theta_n}^2\indicator{B}}
				}
			\\&\leq\textstyle
				\max\pRb{R^2,
					\sup_{n\in\pr{\N\cap[N,\infty)}}
					\prb{
					2\E\PRb{\norm{\Theta_n-\xi}^2\indicator{B}}
					+
					2\E\PRb{\norm{\xi}^2\indicator{B}}}
				}
						\\&\leq\textstyle
			\max\pRb{R^2,
					2r
					+
					2R^2
			}
			= 2r+2R^2.
		\end{split}
	\end{equation}
	\Moreover the fact that $0<s\leq1$ and $0<ps\leq1$ \proves\ that for all
	$x,y\in[0,\infty)$ it holds that
	\begin{equation}
		\label{eq:limsup:simple:fractional:power}
		(x+y)^s\leq x^s+y^s
		\qqandqq (1+x)^{2ps}\leq (1+x)^2\leq2(1+x^2).
	\end{equation}
	Combining
	\eqref{eq:setup:UB:nth:moment:simple:case} and
	\eqref{eq:limsup:simple:localization} with
	\eqref{eq:limsup:simple:localized:second:moments}
	hence
	\proves\ that for all $n\in\N$ it holds that
	\begin{equation}
		\label{eq:limsup:simple:localized:UB:moment}
		\begin{split}
			\textstyle
			\E\PRb{(\UB_n)^{2s}\indicator{B}}
			&=\E\PRb{\prb{\fc(1-\beta)\textstyle\PRb{
						\beta^n\norm{\MOMs[2]}^2+
						\const^2\sum_{k=1}^n\beta^{n-k}(1+\norm{\Theta_{k-1}})^{2p}}}^{s}\indicator{B}}
			\\&\leq
			\fc^s(1-\beta)^s
			\E\PRb{\textstyle\PRb{
						\beta^{ns}\norm{\MOMs[2]}^{2s}+
						\const^{2s}\sum_{k=1}^n\beta^{s(n-k)}(1+\norm{\Theta_{k-1}})^{2ps}}\indicator{B}}
			\\&\leq
					\fc^s(1-\beta)^s
					\E\PRb{\textstyle\PRb{
							\beta^{ns}\norm{\MOMs[2]}^{2s}+
							2\const^{2s}\sum_{k=1}^n\beta^{s(n-k)}(1+\norm{\Theta_{k-1}}^2)}\indicator{B}}
			\\&\leq
			\fc^s(1-\beta)^s
			\textstyle\prb{
					\beta^{ns}\E\PRb{\norm{\MOMs[2]}^{2s}\indicator{B}}+
					2\const^{2s}\sum_{k=1}^n\beta^{s(n-k)}\prb{1+\E\PRb{\norm{\Theta_{k-1}}^2\indicator{B}}}}
			\\&\leq
			\fc^s(1-\beta)^s
			\textstyle\prb{
				\beta^{ns}R^{2s}+
				2\const^{2s}\sum_{k=1}^n\beta^{s(n-k)}\prb{1+2r+2R^2}}
			\\&\leq
			\fc^s(1-\beta)^s
			\textstyle\prb{
				\beta^{ns}R^{2s}+
				2\const^{2s}\prb{1+2r+2R^2}\sum_{k=0}^{\infty}\pr{\beta^s}^{k}}
			\\&\leq \fc^s(1-\beta)^s\prbbb{R^{2s}+
				\dfrac{2\const^{2s}(1+2r+2R^2)}{1-\beta^s}}.
		\end{split}
	\end{equation}
	\Moreover
	\eqref{eq:limsup:simple:localization}, and Markov's inequality
	 \prove\ that for all $n\in\N$, $L\in(0,\infty)$ with $L^{2s}\P(B)\geq 2\prb{\sup_{n\in\N}\E\PRb{(\UB_n)^{2s}\indicator{B}}}$ it holds that
	\begin{equation}
		\label{eq:limsup:simple:UB:probability}
		\begin{split}
		\P(\UB_n\leq L)
		\geq
		\P(B\cap\{\UB_n\leq L\})
		&=\P(B)-\P(B\cap\{\UB_n>L\})
		\\&=\P(B)-\P(\pr{\UB_n}\indicator{B}>L)
		\\&\geq\P(B)-\P(\pr{\UB_n}\indicator{B}\geq L)
		\\&\geq\P(B)-\frac{\sup_{n\in\N}\E\PRb{(\UB_n)^{2s}\indicator{B}}}{L^{2s}}\geq\frac{\P(B)}{2}.
		\end{split}
	\end{equation}
	This, \eqref{eq:limsup:simple:vector:increment}, and
	\eqref{eq:limsup:simple:localized:UB:moment}
	\prove\ that there exists $L\in(0,\infty)$ such that for all $n\in\N$ it holds that
	\begin{equation}
		\label{eq:limsup:simple:positive:increment}
		\begin{split}
			\E\PRb{\norm{\Theta_n-\Theta_{n-1}}^2}
			&\geq \delta(\gamma_n)^2(1-\alpha)^2(\batch_n)^{-q}
			\E\PRb{(\eps+\UB_n)^{-2}}
			\\&\geq \delta(\gamma_n)^2(1-\alpha)^2(\batch_n)^{-q}
			\E\PRb{\indicator{\{\UB_n\leq L\}}(\eps+\UB_n)^{-2}}
			\\&\geq\delta(\gamma_n)^2(1-\alpha)^2(\batch_n)^{-q}\frac{\P(\UB_n\leq L)}
			{(\eps+L)^2}
			\\&\geq\frac{\delta(1-\alpha)^2\P(B)}{2(\eps+L)^2}
			(\gamma_n)^2(\batch_n)^{-q}.
		\end{split}
	\end{equation}
	\Moreover the triangle inequality and the fact that $\forall\,x,y\in\R\colon 2x^2+2y^2\geq (x+y)^2$ \prove\ that for all measurable
	$\vartheta\colon\Omega\to\R^\pars$ it holds that
	\begin{equation}
		\label{eq:limsup:simple:adjacent:errors}
		\begin{split}
		4\limsup_{n\to\infty}\E\PRb{\norm{\Theta_n-\vartheta}^2}
		&\geq
		\limsup_{n\to\infty}\prb{2\E\PRb{\norm{\Theta_n-\vartheta}^2}+2\E\PRb{\norm{\Theta_{n-1}-\vartheta}^2}}
		\\&\geq
		\limsup_{n\to\infty}
		\E\PRb{\norm{\Theta_n-\Theta_{n-1}}^2}
		.
		\end{split}
	\end{equation}
	This, \eqref{eq:LR:simple}, \eqref{eq:limsup:simple:localization}, and
	\eqref{eq:limsup:simple:positive:increment} \prove\ that there exists $L\in(0,\infty)$ such that
	\begin{equation}
		\begin{split}
			\inf_{\substack{\vartheta\colon\Omega\to\R^\pars\\\text{measurable}}}
			\PRbbb{\limsup_{n\to\infty}\E\PRb{\norm{\Theta_n-\vartheta}^2}}
			&\geq\tfrac14\limsup_{n\to\infty}\E\PRb{\norm{\Theta_n-\Theta_{n-1}}^2}\\
			&\geq \frac{\delta(1-\alpha)^2\P(B)}{8(\eps+L)^2}
			\limsup_{n\to\infty}\PRb{(\gamma_n)^2(\batch_n)^{-q}}>0.
		\end{split}
	\end{equation}
	This \proves[p]\ \eqref{statement:easy}.
	\finishproofthus
\end{proof}

\subsection{Non-accumulation for adaptive SGD optimization methods}
\label{subsection:4.3}

\begin{athm}{prop}{prop:prop:non_convergence_modified_Adam:specific:setup:SCOPE3}
	Let 
	$a\in\R$,
	$b\in[a,\infty)$,
	$\eps,\cst\in(0,\infty)$,
	$\Cst\in[\cst,\infty)$,
	$\alpha\in[0,1)$,
	$\beta\in(\alpha^2,1)$,
	$\fc\in[\max\{1,\vass{a},\vass{b}\},\infty)$,
	$\const\in\R$ satisfy
	\begin{equation}
		\label{eq:const:def:D:new}
		\const
		=\frac{\fc^3}{\min\{1,\eps^3\}}
		\prbbb{\frac{17(\Cst+\eps)\max\{1,\Cst\}}
			{\min\{1,\cst\}(1-\beta)}}^2,
	\end{equation}
	for every $n\in\N$ let $\MOM_n\colon\R^{n+1}\to\R$ be measurable, assume 
	for all $n\in\N$, $x_0,x_1,\dots,x_n\in\R$, $y\in(-\infty,x_n)$ that
	\begin{equation}
		\label{eq:cond2:MOM:incr:SCOPE3}
		\textstyle
		\frac{1}{\fc}\PRb{(1-\beta)\sum_{k=0}^n\beta^{n-k}(x_k)^2
		}^{\nicefrac{1}{2}}\!
		\leq
		\MOM_n(x_0,\dots,x_n)
		\leq
		\sqrt{\fc}
		\max\pR{\vass{x_0},\vass{x_1},\dots,\vass{x_n}}
	\end{equation}
	\begin{equation}
		\label{eq:cond1:MOM:incr:SCOPE3}
		\begin{split}
			\qqandqq
			y\MOM_n(x_0,\dots, x_n)\leq x_n\MOM_n(x_0,x_1,\dots,x_{n-1},y),
		\end{split}
	\end{equation}
	let
	$\pars,\dim\in\N$, 
	let $g=(g_1,\dots,g_{\pars})\colon\R^{\pars}\times[a,b]^{\dim}\to\R^{\pars}$ be measurable, 
	for every $n\in\N$ let
	$\batch_n\in\N$ and $G_{n}=\pr{G_{n}^{(1)},\dots,G_{n}^{(\pars)}}\colon\R^\pars \times\pr{[a,b]^{\dim}}^{\batch_n}\to\R^\pars$ satisfy for all $\theta\in\R^{\pars}$, $x=(x_1,\dots,x_{\batch_n})\in\pr{[a,b]^{\dim}}^{\batch_n}$ that
	\begin{equation}
		\label{prop:eq:setup:for:proof:Adam:no:alpha:scaling:SCOPE3}
		\textstyle
		G_n(\theta,x)=\frac{1}{\batch_{n}}\PRb{\sum_{j=1}^{\batch_{n}}g(\theta,x_j)}
		,
	\end{equation}
	let $ ( \Omega, \cF,(\mathbb{F}_n)_{n\in\N_0}, \P ) $ be a filtered probability space,
	for every $n\in\N$ let $X_n=(X_{n,1},\dots,X_{n,\batch_n})\colon\Omega \allowbreak\to \pr{[ a, b ]^{\dim}}^{\batch_n}$ be $\fil_n$-measurable, 
	assume for all $n\in\N$ that $X_n$ is independent of $\fil_{n-1}$,
	let
	$\gamma\colon\N\to[0,\infty)$, 
	$ \Theta=(\Theta^{(1)},\ldots,\Theta^{(\pars)}) \colon\allowbreak\N_0\times\Omega \to \R^{\pars}$, and
	$\MOMs=(\MOMs^{(1)},\ldots,\MOMs^{(\pars)})\colon\N\times\Omega\to\R^\pars$ 
	satisfy for all $n\in\N$, $i\in\{1,2,\dots,\pars\}$ that
	\begin{equation}
		\label{prop:eq:def:theta:SCOPE3}
		\begin{split}
			\Theta_{ n }^{(i)}
			&=\Theta_{ n-1  }^{(i)}-\frac{\gamma_{n}  \PR{\alpha^n\MOMs[1]^{(i)}+(1-\alpha)\sum_{k=1}^n\alpha^{n-k}G_k^{(i)}(\Theta_{k-1},X_k)}}{\eps+\MOM_n\prb{\MOMs[2]^{(i)},G_1^{(i)}(\Theta_0,X_1),G_2^{(i)}(\Theta_1,X_2),\dots,G_n^{(i)}(\Theta_{n-1},X_n)}},
		\end{split}
	\end{equation}
	assume that $X_{n,j}$, $(n,j)\in\{(m,i)\in\N^2\colon i\leq\batch_m\}$, are \iid,
	assume that $\Theta_0$, $\MOMs[1]$, and $\MOMs[2]$ are $\fil_0$-measurable, 
	let $i\in\{1,2,\dots,\pars\}$ satisfy
	\begin{equation}
		\label{eq:start:mom:bound:new}
		\textstyle
		\max_{j\in\N}\vass{\MOMs[j]^{(i)}}
		\leq \Cst(1-\alpha)\prb{\vass{\Theta_{0}^{(i)}}+\fc}
		\qqandqq
		\E\PRb{\vass{\Theta^{(i)}_0}}+\sup_{n\in\N}\gamma_n<\infty
		,
	\end{equation}
	and assume for all
	$\theta=(\theta_1,\dots,\theta_\pars)\in\R^\pars$, $x\in[a,b]^\dim$
	that
	\begin{equation}
		\label{prop:eq:new:R:cond}
		\pr{\theta_i-\fc}
		\prb{\cst+(\Cst-\cst)\indicator{(-\infty,\fc]}(\theta_i)}
		\leq
		g_i(\theta,x)
		\leq
		\pr{\theta_i+\fc}
		\prb{\cst+(\Cst-\cst)\indicator{[-\fc,\infty)}(\theta_i)}
		.
	\end{equation} 
	Then 
	\begin{enumerate}[label=(\roman*)]
		\item \label{it:first:mom:of:X:fin} it holds for all $n\in\N_0$ that $\E\PRb{\vass{\Theta_{n}^{(i)}}}<\infty$,
		\item 		
		\label{it:unif:bound:of:ith:component}
		it holds that there exists $c\in\R$ such that for all $n\in\N_0$ it holds that
		$
			\vass{\Theta_n^{(i)}}\leq c \vass{\Theta_0^{(i)}}+c
		$,
		\item\label{it:second:mom:g:in:X}
		it holds for all $\theta\in\R^\pars$ that
		$
			\E\PRb{\vass{g_i(\theta,X_{1,1})}^2}<\infty
		$, and
		\item \label{it:LB:gen:Adam:INCR} it holds for all $n\in\N$ with $\alpha=0$ that
		\begin{equation}
			\begin{split}
				&\E\PRb{\prb{\Theta_{n}^{(i)}-\cob{\Theta_{n}^{(i)}}{\fil_{n-1}}}^2}	
				\geq
				\frac{(\gamma_n)^2 \prb{\inf_{\theta\in\R^{\pars}}\var\prb{g_i(\theta,X_{1,1})}}}
				{\batch_n \const^{2}
					\prb{\max\{1,\textstyle\sup_{n\in\N}\gamma_n\}\textstyle\E\PRb{\max\pRb{1,\vass{\Theta_0^{(i)}}}}}^4} 
				.
			\end{split}
		\end{equation}
	\end{enumerate}
\end{athm}

\begin{proof}[Proof of \cref{prop:prop:non_convergence_modified_Adam:specific:setup:SCOPE3}]
	Throughout this proof for every $n\in\N$ let 
	$\UB_n\colon\Omega\to[0,\infty)$ and $\cva_n\colon\R^n\times\R^\pars\to[0,\infty)$
	satisfy for all $g_1,g_2,\dots,g_n\in\R$, $\theta\in\R^\pars$ that
	\begin{equation}
		\llabel{eq:setup}
		\begin{split}
			\UB_n&=\sqrt{\fc}\max\pRb{\vass{\MOMs[2]^{(i)}},
				\Cst\prb{\max\pRb{
							\vass{\Theta_0^{(i)}},
							\vass{\Theta_1^{(i)}},
						\dots,\vass{\Theta_{n-1}^{(i)}}}+\fc}}\qandq\\
			\cva_n(g_1,g_2,\dots,g_n,\theta)&=\var\prbbb{\frac{G_n^{(i)}(\theta,X_n)}
				{\eps+\MOM_n(g_1,g_2,\dots,g_n,G_n^{(i)}(\theta,X_n))}}.
		\end{split}
	\end{equation}
	\Nobs that \eqref{prop:eq:new:R:cond} \proves\ that for all
	$\theta=(\theta_1,\dots,\theta_{\pars})\in\R^{\pars}$ it holds that
	\begin{equation}
		\label{eq:abs:in:data:bd}
		\textstyle\sup_{\omega\in\Omega}\vass{g_i(\theta,X_{1,1}(\omega))}
		\leq
		\textstyle\sup_{\omega\in\Omega}\Cst\pr{\vass{\theta_i}+\fc}
		<\infty
		.
	\end{equation}
	This \proves\ \cref{it:second:mom:g:in:X}.
	\Moreover
	\eqref{prop:eq:new:R:cond} and \eqref{prop:eq:setup:for:proof:Adam:no:alpha:scaling:SCOPE3}
	\prove\ that for all $n\in\N$, $\theta=(\theta_1,\dots,\theta_\pars)\in\R^{\pars}$,
	$x\in\pr{[a,b]^{\dim}}^{\batch_n}$ it holds that
	\begin{equation}
		\label{eq:prop:gen:grad:mean}
		\pr{\theta_i-\fc}
		\pr{\cst+(\Cst-\cst)\indicator{(-\infty,\fc]}(\theta_i)}
		\leq
		G_n^{(i)}(\theta,x)
		\leq
		\pr{\theta_i+\fc}
		\pr{\cst+(\Cst-\cst)\indicator{[-\fc,\infty)}(\theta_i)}
		.
	\end{equation}
	This \proves\ that for all $n\in\N$ with
	$\E\PRb{\vass{\Theta_{n-1}^{(i)}}}<\infty$ it holds that
	\begin{equation}
		\label{eq:inductionstep:first:mom:inf}
		\E\PRb{\vass{G_n^{(i)}(\Theta_{n-1},X_{n})}}
		\leq
		\E\PRb{\Cst\prb{\vass{\Theta_{n-1}^{(i)}}+\fc}}
		=
		\Cst\prb{\E\PRb{\vass{\Theta_{n-1}^{(i)}}}+\fc}
		<\infty
		.
	\end{equation}
		\Moreover \eqref{eq:start:mom:bound:new} \proves\ that
	\begin{equation}
		\label{eq:first:half:for:max:val}
		\begin{split}
			&\fc+3\prbbb{\fc
				+\frac{\PR{\sup_{m\in\N}\gamma_m}\vass{\MOMs[1]^{(i)}}}{\eps+\fc^{-1}(1-\beta)^{\nicefrac{1}{2}}\vass{\MOMs[2]^{(i)}}}
				+\frac{\PR{\sup_{m\in\N}\gamma_m}\max\{1,\Cst\}(2+\alpha)\beta^{\nicefrac{1}{2}}}{\PR{(1-\beta)\fc^{-2}}^{\nicefrac{1}{2}}\cst\pr{\beta^{\nicefrac{1}{2}}-\alpha}}}
			\\&\leq 4\fc
			+\frac{3\PR{\sup_{m\in\N}\gamma_m}\vass{\Cst(\vass{\Theta_{0}^{(i)}}+\fc)}}{\eps}
			+\frac{3\fc\PR{\sup_{m\in\N}\gamma_m}\max\{1,\Cst\}(2+\alpha)\beta^{\nicefrac{1}{2}}}{(1-\beta)^{\nicefrac{1}{2}}\cst\pr{\beta^{\nicefrac{1}{2}}-\alpha}}
			\\&\leq
			\prbbb{4\fc
				+
				\frac{6\Cst\fc}{\eps}
				+\frac{3\fc\max\{1,\Cst\}(2+\alpha)\beta^{\nicefrac{1}{2}}}{(1-\beta)^{\nicefrac{1}{2}}\cst\pr{\beta^{\nicefrac{1}{2}}-\alpha}}}
			\max\{1,\textstyle\sup_{m\in\N}\gamma_m\}
			\max\{1,\vass{\Theta_0^{(i)}}\}
			\\&\leq
			\prbbb{4
				+\frac{6\max\{1,\Cst\}(2+\alpha)\beta^{\nicefrac{1}{2}}}{\min\{1,\cst\}(1-\beta)^{\nicefrac{1}{2}}\pr{\beta^{\nicefrac{1}{2}}-\alpha}}}\fc
			\max\{1,\eps^{-1}\}
			\max\{1,\textstyle\sup_{m\in\N}\gamma_m\}
			\max\{1,\vass{\Theta_0^{(i)}}\}
			\\&\leq
			\prbbb{\frac{8\max\{1,\Cst\}(2+\alpha)\beta^{\nicefrac{1}{2}}}{\min\{1,\cst\}(1-\beta)^{\nicefrac{1}{2}}\pr{\beta^{\nicefrac{1}{2}}-\alpha}}}\fc
			\max\{1,\eps^{-1}\}
			\max\{1,\textstyle\sup_{m\in\N}\gamma_m\}
			\max\{1,\vass{\Theta_0^{(i)}}\}
			.
		\end{split}
	\end{equation}
	\Moreover the fact that $\fc\geq 1$ \proves\ that
	\begin{equation}
		\label{eq:first:half:for:max:val2}
		\begin{split}
			\fc+3\max\pRbbb{
				\frac{(\alpha\Cst+(1-\alpha)\cst) \fc}{(1-\alpha)\cst},
				\vass{\Theta_0^{(i)}}}
			\leq \prbbb{
				\frac{4(\alpha\Cst+(1-\alpha)\cst) \fc}{(1-\alpha)\cst}}
			\max\{1,\vass{\Theta_0^{(i)}}\}
			.
		\end{split}
	\end{equation}
	\Moreover \eqref{eq:cond2:MOM:incr:SCOPE3} \proves\ for all $n\in\N$ that
	\begin{equation}
		\begin{split}
			&\PRb{\MOM_n\prb{\MOMs[2]^{(i)},G_1^{(i)}(\Theta_0,X_1),\allowbreak G_2^{(i)}(\Theta_1,X_2),\allowbreak\dots\allowbreak,G_n^{(i)}(\Theta_{n-1},X_n)}}^2
			\\&\geq\textstyle
			\beta^n\PRb{\fc^{-1}(1-\beta)^{\nicefrac{1}{2}}\MOMs[2]^{(i)}}^2+\sum_{k=1}^n(1-\beta)\fc^{-2}\beta^{n-k}\prb{G_k^{(i)}(\Theta_{k-1},X_k)}^2.
		\end{split}
	\end{equation}
	This,
	\eqref{eq:cond2:MOM:incr:SCOPE3},
	\eqref{prop:eq:def:theta:SCOPE3},
	\eqref{eq:prop:gen:grad:mean},
	\eqref{eq:first:half:for:max:val},
	\eqref{eq:first:half:for:max:val2}, and
	\cite[Corollary~2.6]{DereichGraeberJentzenRiekert2025}
	(applied for every $\omega\in\Omega$ with 
	$\pars \curvearrowleft \pars$,
	$i \curvearrowleft i$,
	$\eps \curvearrowleft \eps$,
	$\nu \curvearrowleft \cst$,
	$\mu \curvearrowleft \Cst$,
	$\alpha \curvearrowleft \alpha$,
	$\beta\curvearrowleft \beta$,
	$\fc \curvearrowleft \fc$,
	$(g_n)_{n\in\N_0} \curvearrowleft \pr{\R^{\pars}\ni\theta\mapsto \indicator{\{0\}}(n)(1-\alpha)^{-1}\MOMs[1]^{(i)}(\omega)+\indicator{\N}(n)G_n^{(i)}(\theta,X_n(\omega))\in\R}_{n\in\N_0}$,
	$\mathbb{V} \curvearrowleft \prb{\N_0\ni n\mapsto\indicator{\{0\}}(n)\PRb{\fc^{-1}(1-\beta)^{\nicefrac{1}{2}}\MOMs[2]^{(i)}(\omega)}^2+\indicator{\N}(n)\PRb{\MOM_n\prb{\MOMs[2]^{(i)}(\omega),G_1^{(i)}(\Theta_0(\omega),X_1(\omega)),\allowbreak G_2^{(i)}(\Theta_1(\omega),\allowbreak X_2(\omega)),\allowbreak\dots\allowbreak,G_n^{(i)}(\Theta_{n-1}(\omega),X_n(\omega))}}^2\allowbreak\in\allowbreak[0,\allowbreak\infty)}$,
	$\bscl \curvearrowleft \pr{\N\ni n\mapsto(1-\beta)\fc^{-2}\allowbreak\in(0,\infty)}$,
	$\gamma \curvearrowleft \gamma$,
	$\Theta \curvearrowleft \prb{\Z\ni n\mapsto \indicator{\N_0}(n)\Theta_n(\omega)\in\R^{\pars}}$
	in the notation of \cite[Corollary~2.6]{DereichGraeberJentzenRiekert2025}) \prove\ that
	\begin{align}
		\label{eq:max:val:Adam:seq}
		&\nonumber\textstyle\sup_{ n \in \N_0 }
		\vass{\Theta_n^{(i)}}
		\\&\nonumber\leq
		\fc+3\max\pRbbb{\fc
			+\frac{\PR{\sup_{m\in\N}\gamma_m}\vass{\MOMs[1]^{(i)}}}{\eps+\fc^{-1}(1-\beta)^{\nicefrac{1}{2}}\vass{\MOMs[2]^{(i)}}}
			\\&\quad+\frac{\PR{\sup_{m\in\N}\gamma_m}\max\{1,\Cst\}(2+\alpha)\beta^{\nicefrac{1}{2}}}
			{\PR{(1-\beta)\fc^{-2}}^{\nicefrac{1}{2}}\cst\pr{\beta^{\nicefrac{1}{2}}-\alpha}},
			\frac{(\alpha\Cst+(1-\alpha)\cst) \fc}{(1-\alpha)\cst},
			\vass{\Theta_0^{(i)}}}
		\\&\nonumber\leq
		\max\pRbbb{\frac{8\max\{1,\Cst\}(2+\alpha)\beta^{\nicefrac{1}{2}}}{\min\{1,\cst\}(1-\beta)^{\nicefrac{1}{2}}\pr{\beta^{\nicefrac{1}{2}}-\alpha}},
			\frac{4(\alpha\Cst+(1-\alpha)\cst)}{(1-\alpha)\cst}}
		\frac{\fc\max\{1,\textstyle\sup_{m\in\N}\gamma_m\}
			\max\{1,\vass{\Theta_0^{(i)}}\}}{\min\pR{1,\eps}}
		.
	\end{align}
		This and \eqref{eq:start:mom:bound:new} \prove\ \cref{it:first:mom:of:X:fin,it:unif:bound:of:ith:component}.
		Throughout the rest of the proof assume $\alpha=0$.
		\Moreover \eqref{eq:const:def:D:new}, the assumption that $\alpha=0$,
	and the assumption that $\fc\geq1$, $\beta\in(0,1)$, and
	$(\eps(\min\{1,\eps^3\})^{-1})^{\nicefrac{1}{2}}\geq\max\{1,\eps^{-1}\}$
	\prove\ that
	\begin{equation}
		\label{eq:transition:D:to:spec:const}
		\begin{split}
			&\Cst^{-1}\prb{(\const\fc^{-1}\eps)^{\nicefrac{1}{2}}-\eps}-\fc
			\\&=\Cst^{-1}\PRbbb{
				\PRbbb{\frac{\fc^2}{\min\{1,\eps^3\}}
					\prbbb{\frac{17(\Cst+\eps)\max\{1,\Cst\}}
						{\min\{1,\cst\}(1-\beta)}}^2\eps}^{\nicefrac{1}{2}}
				-\eps}-\fc
			\\&=\fc\prb{1+\tfrac{\eps}{\Cst}}
			\frac{17\max\{1,\Cst\}}{\min\{1,\cst\}(1-\beta)}
			\prbbb{\frac{\eps}{\min\{1,\eps^3\}}}^{\nicefrac{1}{2}}
			-\tfrac{\eps}{\Cst}-\fc
			\\&\geq
			\fc\max\{1,\eps^{-1}\}
			\frac{17\max\{1,\Cst\}}{\min\{1,\cst\}(1-\beta)}
			+\tfrac{\eps}{\Cst}\prbbb{
				\fc\max\{1,\eps^{-1}\}
				\frac{17\max\{1,\Cst\}}{\min\{1,\cst\}(1-\beta)}-1}-\fc
			\\&\geq
			\fc\max\{1,\eps^{-1}\}
			\prbbb{\frac{16\max\{1,\Cst\}}{\min\{1,\cst\}(1-\beta)}
				+\frac{\max\{1,\Cst\}}{\min\{1,\cst\}(1-\beta)}}-\fc
			\\&\geq
			\fc\max\{1,\eps^{-1}\}
			\frac{16\max\{1,\Cst\}}{\min\{1,\cst\}(1-\beta)^{\nicefrac{1}{2}}}
			\\&=
			\max\{1,\eps^{-1}\}\fc
			\max\pRbbb{
				\frac{8\max\{1,\Cst\}(2+\alpha)\beta^{\nicefrac{1}{2}}}
				{\min\{1,\cst\}(1-\beta)^{\nicefrac{1}{2}}
					\pr{\beta^{\nicefrac{1}{2}}-\alpha}},
				\frac{4(\alpha\Cst+(1-\alpha)\cst)}{(1-\alpha)\cst}}
			.
		\end{split}
	\end{equation}
	Combining this and \eqref{eq:max:val:Adam:seq} with the assumption that $\fc\geq 1$ \proves\ that
		\begin{align}
			\llabel{eq:UB:traj:van;alpha}
		&\textstyle\sup_{ n \in \N_0 }
		\vass{\Theta_n^{(i)}}
		\\&\nonumber\leq
		\max\pRbbb{\frac{8\max\{1,\Cst\}(2+\alpha)\beta^{\nicefrac{1}{2}}}{\min\{1,\cst\}(1-\beta)^{\nicefrac{1}{2}}\pr{\beta^{\nicefrac{1}{2}}-\alpha}},
			\frac{4(\alpha\Cst+(1-\alpha)\cst)}{(1-\alpha)\cst}}
		\frac{\fc\max\{1,\textstyle\sup_{m\in\N}\gamma_m\}
			\max\{1,\vass{\Theta_0^{(i)}}\}}{\min\pR{1,\eps}}
			\\&\nonumber\leq
		\PR{\Cst^{-1}((\const\fc^{-1}\eps)^{\nicefrac{1}{2}}-\eps)-\fc}
		\max\{1,\textstyle\sup_{m\in\N}\gamma_m\}
		\max\{1,\vass{\Theta_0^{(i)}}\}
		.
	\end{align}
	\Moreover
	\eqref{eq:cond2:MOM:incr:SCOPE3}, \eqref{eq:prop:gen:grad:mean}, and \lref{eq:setup}
	\prove\ that for all $n\in\N$, $x\in([a,b]^\dim)^{\batch_n}$ it holds that
	\begin{align}
		\label{eq:L:UB:fro:Momentum.0}
			&\nonumber\MOM_n\prb{\MOMs[2]^{(i)},G_1^{(i)}(\Theta_0,X_1),G_2^{(i)}(\Theta_1,X_2),\dots,G_{n-1}^{(i)}(\Theta_{n-2},X_{n-1}),G_n^{(i)}(\Theta_{n-1},x)}
			\\&\nonumber\leq \sqrt{\fc}\max\pRb{
				\vass{\MOMs[2]^{(i)}},
				\vass{G_1^{(i)}(\Theta_0,X_1)},
				\vass{G_2^{(i)}(\Theta_1,X_2)},\dots,
				\vass{G_{n-1}^{(i)}(\Theta_{n-2},X_{n-1})},
				\vass{G_n^{(i)}(\Theta_{n-1},x)}}
			\\&\leq \sqrt{\fc}\max\pRb{\vass{\MOMs[2]^{(i)}},
				\Cst\prb{\max\pRb{
						\vass{\Theta_0^{(i)}},
						\vass{\Theta_1^{(i)}},
						\dots,\vass{\Theta_{n-1}^{(i)}}}+\fc}}
			=\UB_n
			.
		\end{align}
		\Moreover
		\eqref{eq:start:mom:bound:new},
		\lref{eq:setup}, and
		\lref{eq:UB:traj:van;alpha}
		\prove\ that for all $n\in\N$ it holds that
		\begin{align}
			\label{eq:L:UB:fro:Momentum.1}
			\UB_n
			&=\sqrt{\fc}\max\pRb{\vass{\MOMs[2]^{(i)}},
				\Cst\prb{\max\pRb{
						\vass{\Theta_0^{(i)}},
						\vass{\Theta_1^{(i)}},
						\dots,\vass{\Theta_{n-1}^{(i)}}}+\fc}}
			\\&\nonumber\leq\textstyle
			\sqrt{\fc}\PRb{\sup_{k\in\N}\Cst(\vass{\Theta_{k-1}^{(i)}}+\fc)}
			\\&\leq\textstyle
			\sqrt{\fc}\PR{(\const\fc^{-1}\eps)^{\nicefrac{1}{2}}-\eps}
			\max\{1,\textstyle\sup_{m\in\N}\gamma_m\}\max\{1,\vass{\Theta_0^{(i)}}\}
			.
		\end{align}
	This
	and the assumption that $\fc\geq 1$ \prove\ that for all $n\in\N$ it holds that
	\begin{equation}
		\label{eq:nominator:UB:via:apriori}
		\begin{split}
			\textstyle\prb{\eps+\E\PRb{\UB_n}}^4
			&\leq\textstyle\prb{\eps+\E\PRb{\sqrt{\fc}\PR{(\const\fc^{-1}\eps)^{\nicefrac{1}{2}}-\eps}
					\max\{1,\textstyle\sup_{m\in\N}\gamma_m\}\max\pRb{1,\vass{\Theta_0^{(i)}}}}}^4
			\\&=\textstyle
			\prb{\eps+\sqrt{\fc}\PR{(\const\fc^{-1}\eps)^{\nicefrac{1}{2}}-\eps}
				\max\{1,\textstyle\sup_{m\in\N}\gamma_m\}\E\PRb{\max\pRb{1,\vass{\Theta_0^{(i)}}}}}^4
			\\&\leq\textstyle
			\prb{\sqrt{\fc}\PR{(\const\fc^{-1}\eps)^{\nicefrac{1}{2}}}
				\max\{1,\textstyle\sup_{m\in\N}\gamma_m\}\E\PRb{\max\pRb{1,\vass{\Theta_0^{(i)}}}}}^4
			\\&=\textstyle
			(\const\eps)^{2}
			\prb{\max\{1,\textstyle\sup_{m\in\N}\gamma_m\}\E\PRb{\max\pRb{1,\vass{\Theta_0^{(i)}}}}}^4
			.
		\end{split}
	\end{equation}
	\Nobs that \eqref{prop:eq:setup:for:proof:Adam:no:alpha:scaling:SCOPE3}
	and the assumption that $X_{n,j}$, $(n,j)\in\{(m,i)\in\N^2\colon i\leq\batch_m\}$, are \iid\
	\prove\ that for all $n\in\N$, $\theta\in\R^\pars$ it holds that
	\begin{equation}
		\var\prb{G_n^{(i)}(\theta,X_n)}
		=\textstyle\var\prb{\frac{1}{\batch_{n}}\PRb{\sum_{j=1}^{\batch_{n}}g_i(\theta,X_{n,j})}}
		=\var\prb{g_i(\theta,X_{1,1})}\pr{\batch_n}^{-1}.
	\end{equation}
	This, 
	\eqref{eq:cond1:MOM:incr:SCOPE3},
	\eqref{eq:L:UB:fro:Momentum.1},
	\lref{eq:setup},
	\eqref{eq:prop:gen:grad:mean}, 
	\eqref{eq:L:UB:fro:Momentum.0},
 and \cref{lem:variance:cont:it2} in \cref{lem:variance:cont}
	(applied for every $\omega\in\Omega$, $n\in\N$ with 
	$X\curvearrowleft G_n^{(i)}(\Theta_{n-1}(\omega),X_n)$,
	$\MOM\curvearrowleft(\R\ni x\mapsto\MOM_n\prb{\MOMs[2]^{(i)}(\omega),G_1^{(i)}(\Theta_0(\omega),X_1(\omega)),\allowbreak G_2^{(i)}(\Theta_1(\omega),X_2(\omega)),\dots,G_{n-1}^{(i)}(\Theta_{n-2}(\omega),\allowbreak X_{n-1}(\omega)),\allowbreak x}\in[0,\infty))$
	in the notation of \cref{lem:variance:cont})
	\prove\ for all $n\in\N$ that
	\begin{equation}
		\begin{split}
		&\cva_n(\MOMs[2]^{(i)},G_1^{(i)}(\Theta_0,X_1),G_2^{(i)}(\Theta_1,X_2),\dots,G_{n-1}^{(i)}(\Theta_{n-2},X_{n-1}),\Theta_{n-1})
		\\&\geq
		\frac{\eps^2\prb{\inf_{\theta\in\R^\pars}\var\prb{G_n^{(i)}(\theta,X_{n})}}}
		{(\eps+\UB_n)^4}
		\\&\geq
		\frac{\eps^2\prb{\inf_{\theta\in\R^\pars}\var\prb{g_i(\theta,X_{1,1})}}}
		{\batch_n(\eps+\UB_n)^4}.
		\end{split}
	\end{equation}
	\Nobs that \eqref{prop:eq:def:theta:SCOPE3},
	the assumption that $\MOMs[2]^{(i)}$ is $\fil_0$-measurable,
	the assumption that $g$ is measurable, and induction \prove\ for all $n\in\N$
	that $\MOMs[2]^{(i)}$, $G_1^{(i)}(\Theta_0,X_1)$, $G_2^{(i)}(\Theta_1,X_2)$, $\dots$, $G_{n-1}^{(i)}(\Theta_{n-2},\allowbreak X_{n-1})$, and $\Theta_{n-1}$ are $\fil_{n-1}$-measurable.
	Combining this and the independence of $X_n$ and $\fil_{n-1}$
	with \eqref{prop:eq:def:theta:SCOPE3}, \eqref{eq:nominator:UB:via:apriori},
	\cref{it:mb:fact:for:var2} in \cref{lem:factorization2:spec},
	and
	Jensen's inequality (cf., \eg, \cite[Theorem~7.9]{klenkeprobability}) \proves\ that for all $n\in\N$ it holds that
	\begin{equation}
		\label{eq:increment:LB:before:limits}
		\begin{split}
			&\E\PRb{\prb{\Theta_n^{(i)}-\cob{\Theta_n^{(i)}}{\fil_{n-1}}}^2}
			\\&=(\gamma_n)^2\E\PRb{\cva_n(\MOMs[2]^{(i)},G_1^{(i)}(\Theta_0,X_1),G_2^{(i)}(\Theta_1,X_2),\dots,G_{n-1}^{(i)}(\Theta_{n-2},X_{n-1}),\Theta_{n-1})}
			\\&\geq \frac{(\gamma_n)^2\eps^2\prb{\inf_{\theta\in\R^\pars}\var\prb{g_i(\theta,X_{1,1})}}}
			{\batch_n}\E\PRb{(\eps+\UB_n)^{-4}}\\
			&\geq \frac{(\gamma_n)^2\eps^2\prb{\inf_{\theta\in\R^\pars}\var\prb{g_i(\theta,X_{1,1})}}}
			{\batch_n(\eps+\E\PRb{\UB_n})^4}
			\\&\geq
			\frac{(\gamma_n)^2 \prb{\inf_{\theta\in\R^{\pars}}\var\prb{g_i(\theta,X_{1,1})}}}
			{\batch_n \const^{2}
				\prb{\max\{1,\textstyle\sup_{n\in\N}\gamma_n\}\textstyle\E\PRb{\max\{1,\vass{\Theta_0^{(i)}}\}}}^4}
			.
		\end{split}
	\end{equation}
	This \proves[pei] \cref{it:LB:gen:Adam:INCR}.
	\finishproofthus
\end{proof}

\begin{athm}{cor}{cor:liminf:adaptive:alpha:zero}
	Let $a\in\R$, $b\in[a,\infty)$, $\eps,\cst\in(0,\infty)$,
	$\Cst\in[\cst,\infty)$, $\beta\in(0,1)$,
	$\fc\in[\max\{1,\allowbreak\vass{a},\vass{b}\},\infty)$,
	for every $n\in\N$ let $\MOM_n\colon\R^{n+1}\to\R$ be measurable,
	assume for all $n\in\N$, $x_0,x_1,\dots,\allowbreak x_n\in\R$,
	$y\in(-\infty,x_n)$ that
	\begin{equation}
		\llabel{eq:bounds:for:denom:fct}
		\textstyle
		\frac{1}{\fc}
		\PRb{(1-\beta)\sum_{k=0}^n\beta^{n-k}(x_k)^2}^{\nicefrac{1}{2}}
		\leq \MOM_n(x_0,\dots,x_n)
		\leq \sqrt{\fc}\max\pR{\vass{x_0},\vass{x_1},\dots,\vass{x_n}}
	\end{equation}
	\begin{equation}
		\qqandqq
		y\MOM_n(x_0,\dots,x_n)
		\leq x_n\MOM_n(x_0,\dots,x_{n-1},y),
	\end{equation}
	let $\pars,\dim\in\N$,
	let $g=(g_1,\dots,g_{\pars})\colon\R^\pars\times[a,b]^\dim
	\to\R^\pars$ be measurable,
	for every $n\in\N$ let $\batch_n\in\N$ and
	$G_n=(G_n^{(1)},\dots,G_n^{(\pars)})\colon
	\R^\pars\times([a,b]^\dim)^{\batch_n}\to\R^\pars$
	satisfy for all $\theta\in\R^\pars$,
	$x=(x_1,\dots,x_{\batch_n})\in([a,b]^\dim)^{\batch_n}$ that
	\begin{equation}
		\textstyle
		G_n(\theta,x)=\frac{1}{\batch_n}
		\PRb{\sum_{j=1}^{\batch_n}g(\theta,x_j)},
	\end{equation}
	let $(\Omega,\cF,(\fil_n)_{n\in\N_0},\P)$ be a filtered probability space,
	for every $n\in\N$ let
	$X_n=(X_{n,1},\dots,X_{n,\batch_n})\colon\allowbreak
	\Omega\to([a,b]^\dim)^{\batch_n}$ be $\fil_n$-measurable,
	assume for all $n\in\N$ that $X_n$ is independent of $\fil_{n-1}$,
	let $\gamma\colon\N\to[0,\infty)$,
	$\Theta=\prb{\Theta^{(1)},\dots,\Theta^{(\pars)}}\colon
	\N_0\times\Omega\to\R^\pars$, and
	$\MOMs=\prb{\MOMs^{(1)},\dots,\MOMs^{(\pars)}}\colon\Omega\to\R^\pars$ satisfy for all $n\in\N$,
	$i\in\{1,2,\dots,\pars\}$ that
	\begin{equation}
		\llabel{eq:recurs:alpha0}
		\Theta_n^{(i)}
		=\Theta_{n-1}^{(i)}
		-\frac{\gamma_n G_n^{(i)}(\Theta_{n-1},X_n)}
		{\eps+\MOM_n\prb{\MOMs^{(i)},
				G_1^{(i)}(\Theta_0,X_1),\dots,
				G_n^{(i)}(\Theta_{n-1},X_n)}},
	\end{equation}
	assume that $X_{n,j}$,
	$(n,j)\in\{(m,k)\in\N^2\colon k\leq\batch_m\}$, are \iid,
	assume that $\Theta_0$ and $\MOMs$ are $\fil_0$-measurable,
	let $i\in\{1,2,\dots,\pars\}$ satisfy
	\begin{equation}
		\vass{\MOMs^{(i)}}\leq\Cst\prb{\vass{\Theta_0^{(i)}}+\fc}
		\qqandqq
		\E\PRb{\vass{\Theta_0^{(i)}}}<\infty,
	\end{equation}
	assume for all $\theta=(\theta_1,\dots,\theta_{\pars})\in\R^\pars$,
	$x\in[a,b]^\dim$ that
	\begin{equation}
		(\theta_i-\fc)
		\prb{\cst+(\Cst-\cst)\indicator{(-\infty,\fc]}(\theta_i)}
		\leq g_i(\theta,x)\leq
		(\theta_i+\fc)
		\prb{\cst+(\Cst-\cst)\indicator{[-\fc,\infty)}(\theta_i)},
	\end{equation}
	and assume that
	\begin{equation}
		\label{eq:liminf:alpha:zero:assumptions}
		\textstyle
		\liminf_{n\to\infty}\gamma_n>0,
		\qquad
		\sup_{n\in\N}(\gamma_n+\batch_n)<\infty,
		\qandq
		\inf_{\theta\in\R^\pars}\var\prb{g_i(\theta,X_{1,1})}>0.
	\end{equation}
Then it holds that
\begin{equation}
	\label{eq:liminf:alpha:zero:claim}
	\inf_{\substack{\vartheta\colon\Omega\to\R\\\text{measurable}}}
	\PRbb{\liminf_{n\to\infty}
	\E\PRb{\vass{\Theta_n^{(i)}-\vartheta}}}>0.
\end{equation}
\end{athm}

\begin{proof}[Proof of \cref{cor:liminf:adaptive:alpha:zero}]
	Throughout this proof let $\const\in(0,\infty)$ satisfy
		\begin{equation}
		\llabel{eq:const:def:D:new}
		\const
		=\frac{\fc^3}{\min\{1,\eps^3\}}
		\prbbb{\frac{17(\Cst+\eps)\max\{1,\Cst\}}
			{\min\{1,\cst\}(1-\beta)}}^2,
	\end{equation} 
	and let $R,C\in(0,\infty)$
	satisfy
	\begin{equation}
		R=\frac{\fc\sup_{k\in\N}\gamma_k}{(1-\beta)^{\nicefrac{1}{2}}}
		\qqandqq
		C=
		\frac{\textstyle\inf_{\theta\in\R^\pars}
			\var\prb{g_i(\theta,X_{1,1})}}
		{\const^2
			\prb{\max\{1,\textstyle\sup_{k\in\N}\gamma_k\}
				\E\PRb{\max\{1,\vass{\Theta_0^{(i)}}\}}}^4}.
	\end{equation}
	\Nobs that \cref{it:first:mom:of:X:fin,it:LB:gen:Adam:INCR} in
	\cref{prop:prop:non_convergence_modified_Adam:specific:setup:SCOPE3}
	(applied with $\alpha\curvearrowleft0$,
	$\MOMs\curvearrowleft
	(\N\times\Omega\ni(j,\omega)\mapsto
	\indicator{\{2\}}(j)\MOMs(\omega)\in\R^\pars)$
	in the notation of
	\cref{prop:prop:non_convergence_modified_Adam:specific:setup:SCOPE3})
	\prove\ that for all $n\in\N$ it holds that
	\begin{equation}
		\label{eq:liminf:zero:innovation}
		\E\PRb{\vass{\Theta_n^{(i)}}}<\infty
		\qqandqq
		\E\PRb{\prb{\Theta_n^{(i)}-\cob{\Theta_n^{(i)}}{\fil_{n-1}}}^2}\geq C(\gamma_n)^2(\batch_n)^{-1}.
	\end{equation}
	\Moreover \lref{eq:recurs:alpha0} and \lref{eq:bounds:for:denom:fct} \prove\ that for all $n\in\N$ it holds that
	\begin{equation}
		\begin{split}
			\vass{\Theta_n^{(i)}-\Theta_{n-1}^{(i)}}
			&=
			\frac{\gamma_n \vass{G_n^{(i)}(\Theta_{n-1},X_n)}}
			{\eps+\MOM_n\prb{\MOMs^{(i)},
					G_1^{(i)}(\Theta_0,X_1),\dots,
					G_n^{(i)}(\Theta_{n-1},X_n)}}
			\\&\leq
			\frac{\gamma_n\vass{G_n^{(i)}(\Theta_{n-1},X_n)}}
			{\eps+\fc^{-1}(1-\beta)^{\nicefrac{1}{2}}
				\vass{G_n^{(i)}(\Theta_{n-1},X_n)}}
			\leq
			\frac{\fc\gamma_n}{(1-\beta)^{\nicefrac{1}{2}}}
			\leq R.
		\end{split}
	\end{equation}
	This and the fact that $\Theta_{n-1}^{(i)}$ is $\fil_{n-1}$-measurable 
	\prove\ that for all $n\in\N$ it holds $\P$-\as\ that
	\begin{equation}
		\begin{split}
			\vass{\Theta_n^{(i)}-\cob{\Theta_n^{(i)}}{\fil_{n-1}}}
			&=\vass{\Theta_n^{(i)}-\Theta_{n-1}^{(i)}
				-\cob{\Theta_n^{(i)}-\Theta_{n-1}^{(i)}}{\fil_{n-1}}}
			\leq 2R.
		\end{split}
	\end{equation}
	Combining this and \eqref{eq:liminf:zero:innovation} with the fact
	that it holds $\P$-\as\ that $\cob{\prb{\Theta_n^{(i)}-\cob{\Theta_n^{(i)}}{\fil_{n-1}}}}{\fil_{n-1}}=0$ holds and the fact that for all $m,n\in\N$ with $m<n$ it holds that $\Theta_m^{(i)}$ is $\fil_{n-1}$-measurable
	\proves\ that for all $m,n\in\N$ with $m<n$ it holds that
	\begin{equation}
		\label{eq:liminf:zero:pairwise}
		\begin{split}
			C(\gamma_n)^2(\batch_n)^{-1}
			&\leq \E\PRb{\prb{\Theta_n^{(i)}-\cob{\Theta_n^{(i)}}{\fil_{n-1}}}^2}
			\\&=\E\PRb{\prb{\Theta_n^{(i)}-\cob{\Theta_n^{(i)}}{\fil_{n-1}}}\prb{\Theta_n^{(i)}-\Theta_m^{(i)}}}
			\\&\quad+
			\E\PRb{\prb{\Theta_m^{(i)}-\cob{\Theta_n^{(i)}}{\fil_{n-1}}}\prb{\Theta_n^{(i)}-\cob{\Theta_n^{(i)}}{\fil_{n-1}}}}
			\\&=\E\PRb{\prb{\Theta_n^{(i)}-\cob{\Theta_n^{(i)}}{\fil_{n-1}}}\prb{\Theta_n^{(i)}-\Theta_m^{(i)}}}
				\\&\quad+
			\E\PRb{\prb{\Theta_m^{(i)}-\cob{\Theta_n^{(i)}}{\fil_{n-1}}}\prb{\cob{\prb{\Theta_n^{(i)}-\cob{\Theta_n^{(i)}}{\fil_{n-1}}}}{\fil_{n-1}}}}
			\\&=\E\PRb{\prb{\Theta_n^{(i)}-\cob{\Theta_n^{(i)}}{\fil_{n-1}}}\prb{\Theta_n^{(i)}-\Theta_m^{(i)}}}
			\leq 2R\E\PRb{\vass{\Theta_n^{(i)}-\Theta_m^{(i)}}}.
		\end{split}
	\end{equation}
	This and \cref{metrix:cauchy:limsup}
	(applied for every random variable $\vartheta\colon\Omega\to\R$ with
	$E\curvearrowleft \{V\colon\Omega\to\R\colon\text{$V$ is measurable}\}$,
	 $z\curvearrowleft-1$,
	$\lambda\curvearrowleft1$,
	$d\curvearrowleft(\{V\colon\Omega\to\R\colon\text{$V$ is measurable}\}^2\ni(V,W)
	\mapsto\E\PRb{\vass{V-W}}\in[0,\infty])$,
	$x\curvearrowleft(\N_0\ni n\mapsto \indicator{\{0\}}(n)\vartheta+\indicator{\N}(n)\Theta_n^{(i)}\in\{V\colon\Omega\to\R\colon\text{$V$ is measurable}\})$
	in the notation of \cref{metrix:cauchy:limsup})
	\prove\ that for every random variable $\vartheta\colon\Omega\to\R$ it holds that
	\begin{equation}
		\begin{split}
			\liminf_{n\to\infty}
			\E\PRb{\vass{\Theta_n^{(i)}-\vartheta}}
			&\geq
			\tfrac{1}{2}\PRbb{\liminf_{n\to\infty}\liminf_{m\to\infty}
			\E\PRb{\vass{\Theta_{n+m}^{(i)}-\Theta_n^{(i)}}}}\\
			&\geq\textstyle
			\frac{C}{4R}\PRb{\liminf_{n\to\infty}
			(\gamma_n)^2(\batch_n)^{-1}}.
		\end{split}
	\end{equation}
	This and \eqref{eq:liminf:alpha:zero:assumptions} \prove\ that
	\begin{equation}
		\inf_{\substack{\vartheta\colon\Omega\to\R\\\text{measurable}}}\PRbb{
		\liminf_{n\to\infty}\E\PRb{\vass{\Theta_n^{(i)}-\vartheta}}}
		\geq
		\frac{C\PRb{\liminf_{n\to\infty}\gamma_n}^2}
		{4R\PRb{\sup_{n\in\N}\batch_n}}>0.
	\end{equation}
	\finishproofthus
\end{proof}

\begin{athm}{lemma}{lem:variance:Adam:batch:mean}
	Let $\batch\in\N\cap[2,\infty)$, $\alpha\in[0,1)$,
	$\eps\in(0,\infty)$, $K\in[1,\infty)$, and $\nu\in(0,K^2]$.
	Then there exists $c\in(0,\infty)$ such that for every probability
	space $(\Omega,\cF,\P)$, all \iid\ random variables
	$X_n\colon\Omega\to[-K,K]$, $n\in\N$, with $\var(X_1)\geq\nu$,
	and all $u\in[-K,K]$, $r,q\in[K^{-1},K]$ it holds that
	\begin{equation}
		\label{eq:variance:Adam:batch:mean}
		\var\prbbb{\frac{u+(1-\alpha)\PR{\frac{1}{\batch}\sum_{k=1}^{\batch}X_k}}
			{\eps+(r+q\PR{\frac{1}{\batch}\sum_{k=1}^{\batch}X_k}^2)^{\nicefrac{1}{2}}}}
		\geq c.
	\end{equation}
\end{athm}

\begin{proof}[Proof of \cref{lem:variance:Adam:batch:mean}]
	Throughout this proof 
	let $h\in[0,\infty)$ satisfy $(h\batch)^2=\frac{\nu}{2}$, let
	$G_\ell\colon\Omega\to[-K,K]$, $\ell\in\{0,1,2\}$, and $A\in\cF$ satisfy for all $\ell\in\{0,1,2\}$ that
	\begin{equation}
		\llabel{eq:setup}
		\textstyle
		G_\ell=\frac{1}{\batch}\sum_{k=1}^{\batch}X_{k+\ell}
		\qqandqq
		A=\{G_1-G_0\geq h\}\cap\{G_2-G_1\geq h\},
	\end{equation}
	for every $u\in[-K,K]$, $r,q\in[K^{-1},K]$ let
	$f_{u,r,q}\colon[-K,K]\to\R$ satisfy for all $x\in[-K,K]$ that
	\begin{equation}
		\llabel{eq:setup:frac:fct}
		f_{u,r,q}(x)
		=\frac{u+(1-\alpha)x}{\eps+(r+qx^2)^{\nicefrac{1}{2}}},
	\end{equation}
	and let $C\in[0,\infty)$ satisfy
	\begin{equation}
		\llabel{eq:three:values}
		C=
		\min_{\substack{
				u\in[-K,K],\ r,q\in[K^{-1},K],\\ t\in[-2K\eps^{-1},2K\eps^{-1}]
				}
			}
			\pr*{
				\min_{\substack{
				x_0,x_1,x_2\in[-K,K],\\
				x_1-x_0\geq h\leq x_2-x_1}}
		\textstyle\sum_{\ell=0}^2\prb{f_{u,r,q}(x_\ell)-t}^2}.
	\end{equation}
	\Nobs that for all $u\in[-K,K]$, $r,q\in[K^{-1},K]$,
	$x\in[-K,K]$, $t\in\R$ with $f_{u,r,q}(x)=t$ it holds that
	\begin{equation}
		\prb{(1-\alpha)^2-t^2q}x^2
		+2(1-\alpha)(u-t\eps)x+(u-t\eps)^2-t^2r=0.
	\end{equation}
	This and
	 the fact that for all 
	$u\in[-K,K]$, $r,q\in[K^{-1},K]$, $t\in\R$ with $(u-t\eps)^2-t^2r=0$ and $2(1-\alpha)(u-t\eps)=0$ it holds that
	\begin{equation}
		(1-\alpha)^2-t^2q\neq 0,
	\end{equation}
	\prove\ that for all 
	$u\in[-K,K]$, $r,q\in[K^{-1},K]$,
	$x_1,x_2,x_3\in[-K,K]$, $t\in\R$ with $x_1\neq x_2$ and $f_{u,r,q}(x_1)=f_{u,r,q}(x_2)=f_{u,r,q}(x_3)=t$ it holds that
	\begin{equation}
		x_3\in\{x_1,x_2\}.
	\end{equation}
	This, 
	\lref{eq:setup:frac:fct},
	\lref{eq:three:values}, 
	the fact that $[-K,K]\times[K^{-1},K]^2\times[-2K\eps^{-1},2K\eps^{-1}]\times\{(x_0,x_1,x_2)\in[-K,K]^3\colon x_1-x_0\geq h \leq x_2-x_1\}$ is not empty and compact,
	and the fact that $h>0$
	\prove\ that
	\begin{equation}
		C>0.
	\end{equation}
	\Moreover \cref{lem:variance:representations:0}
	(applied with $X\curvearrowleft X_1$ and
	$Y\curvearrowleft X_{\batch+1}$ in the notation of
	\cref{lem:variance:representations:0}), the fact that
	$\vass{X_{\batch+1}-X_1}\leq 2K$, and the definition of $h$
	\prove\ that
	\begin{equation}
		\begin{split}
			\nu
			\leq\var(X_1)
			&=\E\PRb{(X_{\batch+1}-X_1)^2
				\indicator{\{X_1\leq X_{\batch+1}\}}}\\
			&\leq (\batch h)^2
			+4K^2\P(X_{\batch+1}-X_1\geq\batch h)
			=\tfrac{\nu}{2}
			+4K^2\P(X_{\batch+1}-X_1\geq\batch h).
		\end{split}
	\end{equation}
	This, \lref{eq:setup}, the assumption that $\batch\geq2$, and the assumption that
	$X_1,X_2,\dots,X_{\batch+2}$ are \iid\ \prove\ that
	\begin{equation}
		\begin{split}
			\P(A)
			=\P\prb{X_{\batch+1}-X_1\geq\batch h,\
				X_{\batch+2}-X_2\geq\batch h}
			=\prb{\P(X_{\batch+1}-X_1\geq\batch h)}^2
			\geq\frac{\nu^2}{64K^4}.
		\end{split}
	\end{equation}
	\Nobs that \lref{eq:setup:frac:fct} and the fact that $\vass{G_0}\leq K$
	\prove\ that for all $u\in[-K,K]$, $r,q\in[K^{-1},K]$ it holds that
	\begin{equation}
		\vass{\E[f_{u,r,q}(G_0)]}
		=
		\vass{\E\PRbbb{\frac{u+(1-\alpha)G_0}{\eps+(r+q(G_0)^2)^{\nicefrac{1}{2}}}}}
		\leq
		\eps^{-1}\pr{K+\E\PR{\vass{G_0}}}
		\leq 2K\eps^{-1}
		.
	\end{equation}
	Combining this,
	the fact that $G_0,G_1,G_2$ have the same distribution, and \lref{eq:three:values} with \lref{eq:setup}
	\proves\ that for all $u\in[-K,K]$, $r,q\in[K^{-1},K]$ it holds that
	\begin{equation}
		\begin{split}
			3\var\prb{f_{u,r,q}(G_0)}
			&=\textstyle\sum_{\ell=0}^2
				\E\PRb{\prb{f_{u,r,q}(G_0)-\E[f_{u,r,q}(G_0)]}^2}
			\\&=\textstyle\sum_{\ell=0}^2
				\E\PRb{\prb{f_{u,r,q}(G_\ell)-\E[f_{u,r,q}(G_\ell)]}^2}
			\\&=\textstyle\E\PRbb{\sum_{\ell=0}^2
				\prb{f_{u,r,q}(G_\ell)-\E[f_{u,r,q}(G_\ell)]}^2}
			\\&\geq\textstyle\E\PRbb{\sum_{\ell=0}^2
				\prb{f_{u,r,q}(G_\ell)-\E[f_{u,r,q}(G_\ell)]}^2\indicator{A}}
			\geq C\P(A)
			\geq\frac{C\nu^2}{64K^4}.
		\end{split}
	\end{equation}
	This \proves\
	\eqref{eq:variance:Adam:batch:mean}.
	\finishproofthus
\end{proof}

\begin{athm}{theorem}{thm:new:rob:cond}
	Let $z\in\{-1,1\}$, let $\batch\colon\N\to\N$ and
	$\gamma\colon\N\to[0,\infty)$ satisfy
	\begin{equation}
		\label{eq:bounded:batches:and:LR}
		\textstyle
		z\PRb{\limsup_{n\to\infty}\pr{z\gamma_n}}>0,
		\quad
		\sup_{n\in\N}(\gamma_n+\batch_n)<\infty,
		\qqandqq
		\liminf_{n\to\infty}\batch_n>-z,
	\end{equation}
	let $\fc,\eps,\cst\in(0,\infty)$, $\Cst\in[\cst,\infty)$,
	$\alpha\in[0,1)$, $\beta\in(\alpha^2,1)$,
	for every $n\in\N$ let $\MOM_n\colon\R^{n+1}\to[0,\infty)$
	be measurable, assume for all $n\in\N$, $x=(x_0,\dots,x_n)\in\R^{n+1}$,
	$y\in(-\infty,x_n)$ that
	\begin{equation}
		\label{eq:cond1:MOM:incr:thm:inf}
		\textstyle
		\frac{1}{\fc}
		\PRb{(1-\beta)\sum_{k=0}^n\beta^{n-k}(x_k)^2}^{\nicefrac{1}{2}}
		\leq \MOM_n(x)
		\leq \sqrt{\fc}\max\pR{\vass{x_0},\dots,\vass{x_n}},
	\end{equation}
	\begin{equation}
		\label{eq:cond2:MOM:incr:thm:inf}
		y\MOM_n(x)
		\leq x_n\MOM_n(x_0,\dots,x_{n-1},y),
				\qqandqq
	\end{equation}
	\begin{equation}
		\label{eq:cond2:MOM:incr:thm:inf.1}
		\begin{gathered}
			\alpha(1-z)\PRb{
				\MOM_n(x)^2-\MOM_n(x_0,\dots,x_{n-1},0)^2
				-(x_n)^2\MOM_n(0,\dots,0,1)^2
			}=0,
		\end{gathered}
	\end{equation}
	let $\pars,\dim\in\N$, $a\in\R$, $b\in[a,\infty)$,
	let $g=(g_1,\dots,g_{\pars})\colon
	\R^\pars\times[a,b]^\dim\to\R^\pars$ be measurable,
	for every $n\in\N$ let
	$G_n=(G_n^{(1)},\dots,G_n^{(\pars)})\colon
	\R^\pars\times([a,b]^\dim)^{\batch_n}\to\R^\pars$
	satisfy for all $\theta\in\R^\pars$,
	$x=(x_1,\dots,x_{\batch_n})\in([a,b]^\dim)^{\batch_n}$ that
	\begin{equation}
		\label{prop:eq:setup:for:proof:Adam:no:alpha:scaling:thm:inf}
		\textstyle
		G_n(\theta,x)=\frac{1}{\batch_n}
		\PRb{\sum_{j=1}^{\batch_n}g(\theta,x_j)},
	\end{equation}
	let $(\Omega,\cF,(\fil_n)_{n\in\N_0},\P)$ be a filtered probability
	space, for every $n\in\N$ let
	$X_n=(X_{n,1},\dots,X_{n,\batch_n})\colon\allowbreak
	\Omega\to([a,b]^\dim)^{\batch_n}$ be $\fil_n$-measurable,
	assume for all $n\in\N$ that $X_n$ is independent of $\fil_{n-1}$,
	assume that $X_{n,j}$,
	$(n,j)\in\{(m,k)\in\N^2\colon k\leq\batch_m\}$, are \iid,
	let $\MOMs[1]\colon\Omega\to\R^\pars$ and $\MOMs[2]\colon\Omega\to\R^\pars$ be
	$\fil_0$-measurable, let
	$\Theta=(\Theta^{(1)},\dots,\Theta^{(\pars)})\colon
	\N_0\times\Omega\to\R^\pars$ satisfy for all $n\in\N$,
	$k\in\{1,2,\dots,\pars\}$ that
	\begin{equation}
		\label{eq:main:THM:recursion}
		\Theta_n^{(k)}
		=\Theta_{n-1}^{(k)}
		-\frac{\gamma_n
			\PRb{\alpha^n\MOMs[1]^{(k)}
				+(1-\alpha)\sum_{j=1}^n
				\alpha^{n-j}G_j^{(k)}(\Theta_{j-1},X_j)}}
		{\eps+\MOM_n\prb{\MOMs[2]^{(k)},
				G_1^{(k)}(\Theta_0,X_1),\dots,
				G_n^{(k)}(\Theta_{n-1},X_n)}},
	\end{equation}
	assume that $\Theta_0$ is $\fil_0$-measurable,
	let $i\in\{1,2,\dots,\pars\}$ satisfy
	\begin{equation}
		\label{eq:start:mom:bound:new:rob:cond}
		\textstyle
		\max_{j\in\{1,2\}}\vass{\MOMs[j]^{(i)}}
		\leq\Cst\prb{\vass{\Theta_0^{(i)}}+\fc},
		\quad
		\E\PRb{\vass{\Theta_0^{(i)}}}<\infty,
		\qandq
		\inf_{\theta\in\R^\pars}\var\prb{g_i(\theta,X_{1,1})}>0,
	\end{equation}
	and assume for all $\theta=(\theta_1,\dots,\theta_{\pars})\in\R^\pars$,
	$x\in[a,b]^\dim$ that
	\begin{equation}
		\label{prop:eq:setup:new:rob:cond}
		(\theta_i-\fc)
		\prb{\cst+(\Cst-\cst)\indicator{(-\infty,\fc]}(\theta_i)}
		\leq g_i(\theta,x)\leq
		(\theta_i+\fc)
		\prb{\cst+(\Cst-\cst)\indicator{[-\fc,\infty)}(\theta_i)}.
	\end{equation}
	Then it holds that
	\begin{equation}
		\label{eq:main:Thm}
		\inf_{\substack{\vartheta\colon\Omega\to\R\\\text{measurable}}}
		\PRbbb{z\PRbb{\limsup_{n\to\infty}
				\prb{z\E\PRb{\vass{\Theta_n^{(i)}-\vartheta}}}}}
		>0.
	\end{equation}
\end{athm}

\begin{proof}[Proof of \cref{thm:new:rob:cond}]
	Throughout this proof assume without loss of generality that
	$\fc\geq\max\{1,\vass{a},\vass{b}\}$, let $\mathcal{X}=\{V\colon\Omega\to\R\colon\text{$V$ is measurable}\}$, 
	let 
	$\nu,\delta,R\in(0,\infty)$ satisfy
	\begin{equation}
		\textstyle
		\nu=\inf_{\theta\in\R^\pars}\var(g_i(\theta,X_{1,1})),
		\qquad
		\delta=\frac{\nu}{2\PR{\sup_{n\in\N}\batch_n}},
		\qquad
		\P\prb{\vass{\Theta_0^{(i)}}\leq R}>0,
	\end{equation}
	let $H_n\colon\Omega\to\R$, $n\in\N_0$, and $m_n\colon\Omega\to\R$, $n\in\N_0$, satisfy
	 for all $n\in\N$ that
	\begin{equation}
		\llabel{eq:momentum}
		H_0=\MOMs[2]^{(i)},
		\quad
		H_n=G_n^{(i)}(\Theta_{n-1},X_n),
		\quad
		m_0=\MOMs[1]^{(i)},
		\qandq
		m_n=\alpha m_{n-1}+(1-\alpha)H_n,
	\end{equation}
	for every $n\in\N$ let $q_n\in(0,\infty)$,
	$r_n\colon\Omega\to[0,\infty)$, and
	$\Phi_n\colon\R^\pars\times\R\times[0,\infty)
	\times([a,b]^\dim)^{\batch_n}\to\R$ satisfy for all
	$\theta\in\R^\pars$, $u\in\R$, $r\in[0,\infty)$,
	$x\in([a,b]^\dim)^{\batch_n}$ that
	\begin{equation}
		\llabel{eq:frozen:update}
		\begin{split}
			q_n&=\prb{\MOM_n(0,\dots,0,1)}^2,
			\qquad
			r_n=\prb{\MOM_n(H_0,\dots,H_{n-1},0)}^2,
			\\
			\Phi_n(\theta,u,r,x)
			&=\frac{u+(1-\alpha)G_n^{(i)}(\theta,x)}
			{\eps+\prb{r+q_n(G_n^{(i)}(\theta,x))^2}^{\nicefrac12}},
		\end{split}
	\end{equation}
	for every $n\in\N\cap[2,\infty)$ let $A_n\in\fil_{n-1}$ satisfy $A_n=\{\vass{\Theta_0^{(i)}}\leq R,(H_{n-1})^2\geq\delta\}$, 
	and for every $n\in\N\cap[2,\infty)$ let $v_n\colon\Omega\to[0,\infty)$ satisfy for all $\omega\in\Omega$ that
	\begin{equation}
		\llabel{eq:past:event}
		v_n(\omega)=
		\var\prb{\Phi_n\prb{\Theta_{n-1}(\omega),
				\alpha m_{n-1}(\omega),r_n(\omega),X_n}}.
	\end{equation}
	\Nobs that \eqref{eq:main:THM:recursion}, \lref{eq:momentum},
	and induction \prove\ that for all $n\in\N_0$ it holds that
	$\Theta_n$, $H_n$, and $m_n$ are $\fil_n$-measurable and that
	for all $n\in\N$ it holds that
	\begin{equation}
		\llabel{eq:coordinate:update}
		\begin{split}
		\Theta_n^{(i)}
		&=\Theta_{n-1}^{(i)}-\frac{\gamma_n
			\PRb{\alpha^n\MOMs[1]^{(i)}
				+(1-\alpha)\sum_{j=1}^n
				\alpha^{n-j}G_j^{(i)}(\Theta_{j-1},X_j)}}
		{\eps+\MOM_n\prb{\MOMs[2]^{(i)},
				G_1^{(i)}(\Theta_0,X_1),\dots,
				G_n^{(i)}(\Theta_{n-1},X_n)}}
		\\&=\Theta_{n-1}^{(i)}
			-\frac{\gamma_n m_n}{\eps+\MOM_n(H_0,\dots,H_n)}
		.
		\end{split}
	\end{equation}
	\Moreover \eqref{eq:cond1:MOM:incr:thm:inf},
	\eqref{eq:cond2:MOM:incr:thm:inf},
	\eqref{eq:start:mom:bound:new:rob:cond},
	\eqref{prop:eq:setup:new:rob:cond}, and
	\cref{it:first:mom:of:X:fin,it:unif:bound:of:ith:component}
	in \cref{prop:prop:non_convergence_modified_Adam:specific:setup:SCOPE3}
	(applied with $\Cst\curvearrowleft(1-\alpha)^{-1}\Cst$ and
	$(\MOMs[j])_{j\in\N}\curvearrowleft
	(\MOMs[\min\{j,2\}])_{j\in\N}$ in the notation of
	\cref{prop:prop:non_convergence_modified_Adam:specific:setup:SCOPE3})
	\prove\ that there exists $C\in(1,\infty)$ such that for all
	$n\in\N_0$ it holds that
	\begin{equation}
		\llabel{eq:uniform:coordinate}
		\vass{\Theta_n^{(i)}}\leq C(1+\vass{\Theta_0^{(i)}})
		\qqandqq
		\E\PRb{\vass{\Theta_n^{(i)}}}<\infty.
	\end{equation}
	This, \eqref{eq:cond1:MOM:incr:thm:inf},
	\eqref{eq:start:mom:bound:new:rob:cond}, \lref{eq:momentum},
	\lref{eq:frozen:update}, and the fact that
	$\forall\,\theta\in\R^\pars,\,x\in[a,b]^\dim\colon\vass{g_i(\theta,x)}\leq\Cst(\vass{\theta_i}+\fc)$ 
 \prove\ that there exists
	$K\in[1,\infty)$ which satisfies for all $n\in\N$, $x\in[a,b]^\dim$ that
	\begin{equation}
		\llabel{eq:choice:K}
		\nu\leq K^2,
		\quad
		K^{-1}\leq\fc^{-2}(1-\beta)\min\{1,\beta\delta\},
	\qqandqq
	\end{equation}
	\begin{equation}
		\llabel{eq:local:bounds}
		\begin{split}
			&\max\bigl\{
			\vass{\Theta_{n-1}^{(i)}},
			\vass{g_i(\Theta_{n-1},x)},\vass{H_n},
			\vass{m_{n-1}},\vass{m_n},\eps+\MOM_n(H_0,\dots,H_n),r_n,q_n
			\bigr\}\indicator{\{\vass{\Theta_0^{(i)}}\leq R\}}
			\\&\leq\eps+\fc\prb{\max\{1,\Cst\}\prb{C(1+R)+\fc}}^2
			\leq K.
		\end{split}
	\end{equation}
	\Nobs that \eqref{prop:eq:setup:for:proof:Adam:no:alpha:scaling:thm:inf}
	and the assumption that the $X_{n,j}$ are \iid\ \prove\ for all
	$n\in\N$, $\theta\in\R^\pars$ that
	\begin{equation}
		\textstyle
		\var\prb{G_n^{(i)}(\theta,X_n)}
		=\var(g_i(\theta,X_{1,1}))\pr{\batch_n}^{-1}
		\geq\nu\PRb{\sup_{n\in\N}\batch_n}^{-1}=2\delta.
	\end{equation}
	This, \lref{eq:momentum}, the fact that $\{\vass{\Theta_0^{(i)}}\leq R\}\in\fil_0$, and
	\cref{it:mb:fact:for:var1} in \cref{lem:factorization2:spec}
	(applied for every $n\in\N$ with $\gil\curvearrowleft\fil_{n-1}$,
	$X\curvearrowleft\Theta_{n-1}$, $Y\curvearrowleft X_n$,
	$\Phi\curvearrowleft G_n^{(i)}$ in the notation of
	\cref{lem:factorization2:spec})
	\prove\ that for all $n\in\N$ it holds that
	\begin{equation}
		\llabel{eq:local:second:moments.1}
		\begin{split}
			\E\PRb{\indicator{\{\vass{\Theta_0^{(i)}}\leq R\}}(H_n)^2}
			&\geq
			\E\PRb{\indicator{\{\vass{\Theta_0^{(i)}}\leq R\}}
				\prb{H_n-\cob{H_n}{\fil_{n-1}}}^2}
			\geq 2\delta\P(\vass{\Theta_0^{(i)}}\leq R)
		\end{split}
	\end{equation}
	and
		\begin{equation}
		\llabel{eq:local:second:moments.2}
		\begin{split}
			\E\PRb{\indicator{\{\vass{\Theta_0^{(i)}}\leq R\}}(m_n)^2}
			&\geq
			(1-\alpha)^2\E\PRb{\indicator{\{\vass{\Theta_0^{(i)}}\leq R\}}
			\prb{H_n-\cob{H_n}{\fil_{n-1}}}^2}
			\\&\geq 2(1-\alpha)^2\delta\P\prb{\vass{\Theta_0^{(i)}}\leq R}.
		\end{split}
	\end{equation}
	In the remainder of our proof of \eqref{eq:main:Thm} we distinguish
	between the case $z=1$, the case $(z,\alpha)=(-1,0)$, and the
	case $(z,\alpha)\in\{-1\}\times(0,1)$.
	We first \prove[pe]\ \eqref{eq:main:Thm} in the case of $z=1$.
	\Nobs that \lref{eq:coordinate:update}, \lref{eq:local:bounds},
	and \lref{eq:local:second:moments.2} \prove\ for all $n\in\N$ that
	\begin{equation}
		\begin{split}
			\E\PRb{\vass{\Theta_n^{(i)}-\Theta_{n-1}^{(i)}}}
			\geq
			\E\PRbbb{\frac{\indicator{\{\vass{\Theta_0^{(i)}}\leq R\}}\gamma_n\vass{m_n}}
				{\eps+\MOM_n(H_0,\dots,H_n)}}
			&\geq\frac{\gamma_n}{K^2}
			\E\PRb{\indicator{\{\vass{\Theta_0^{(i)}}\leq R\}}(m_n)^2}
			\\&\geq\frac{2\gamma_n(1-\alpha)^2\delta\P\prb{\vass{\Theta_0^{(i)}}\leq R}}{K^2}.
		\end{split}
	\end{equation}
	This, the triangle inequality, and \eqref{eq:bounded:batches:and:LR}
	\prove\ that
	\begin{equation}
		\begin{split}
			\inf_{\vartheta\in\mathcal{X}}
			\PRbbb{\limsup_{n\to\infty}\E\PRb{\vass{\Theta_n^{(i)}-\vartheta}}}
			&\geq\tfrac12\limsup_{n\to\infty}
			\E\PRb{\vass{\Theta_n^{(i)}-\Theta_{n-1}^{(i)}}}
			\\
			&\geq\textstyle(1-\alpha)^2\delta\P\prb{\vass{\Theta_0^{(i)}}\leq R}\PRb{\limsup_{n\to\infty}\gamma_n}K^{-2}
			>0.
		\end{split}
	\end{equation}
	This \proves[pe]\ \eqref{eq:main:Thm} in the case of $z=1$.
	In the next step we \prove[pe]\ \eqref{eq:main:Thm} in the case
	of $(z,\alpha)=(-1,0)$.
	\Nobs that \eqref{eq:bounded:batches:and:LR},
	\eqref{eq:start:mom:bound:new:rob:cond}, 
	\lref{eq:momentum},
	\lref{eq:coordinate:update},
	the assumption that $(z,\alpha)=(-1,0)$, and \cref{cor:liminf:adaptive:alpha:zero}
	(applied with $\MOMs\curvearrowleft\MOMs[2]$ in the notation of
	\cref{cor:liminf:adaptive:alpha:zero}) \prove\ that
	\begin{equation}
		\inf_{\vartheta\in\mathcal{X}}
		\PRbb{\liminf_{n\to\infty}\E\PRb{\vass{\Theta_n^{(i)}-\vartheta}}}>0.
	\end{equation}
	This \proves[pe]\ \eqref{eq:main:Thm} in the case of $(z,\alpha)=(-1,0)$.
	Next, we \prove[pe]\ \eqref{eq:main:Thm} in the case of
	$(z,\alpha)\in\{-1\}\times(0,1)$.
	\Nobs that 
	\lref{eq:past:event}, 
	\lref{eq:local:bounds},
	\lref{eq:local:second:moments.1}, and the fact that for all $n\in\N\cap[2,\infty)$ it holds that $A_n=\{\vass{\Theta_0^{(i)}}\leq R,(H_{n-1})^2\geq\delta\}$ \prove\ that for all
	$n\in\N\cap[2,\infty)$ it holds that
	\begin{equation}
		2\delta\P\prb{\vass{\Theta_0^{(i)}}\leq R}
		\leq\E\PRb{\indicator{\{\vass{\Theta_0^{(i)}}\leq R\}}(H_{n-1})^2}
		\leq\delta\P\prb{\vass{\Theta_0^{(i)}}\leq R}+K^2\P(A_n).
	\end{equation}
	Combining this with the fact that $K>0$ \proves\ that for all $n\in\N\cap[2,\infty)$ it holds that
		\begin{equation}
			\llabel{eq:non:van:prob:cool:set}
		\P(A_n)\geq\delta\P\prb{\vass{\Theta_0^{(i)}}\leq R}K^{-2}>0.
	\end{equation}
	\Moreover \eqref{eq:cond1:MOM:incr:thm:inf},
	\lref{eq:frozen:update}, \lref{eq:past:event}, and
	\lref{eq:choice:K} \prove\ for all $n\in\N\cap[2,\infty)$, $\omega\in A_n$
	that it holds on  that
	\begin{equation}
		q_n\geq\fc^{-2}(1-\beta)\geq K^{-1}
		\qandq
		r_n(\omega)\geq\fc^{-2}(1-\beta)\beta(H_{n-1}(\omega))^2
		\geq\fc^{-2}(1-\beta)\beta\delta\geq K^{-1}.
	\end{equation}
	This, 
	\lref{eq:past:event},
	\lref{eq:local:bounds}, the fact that
	$\liminf_{n\to\infty}\batch_n\in\N\cap(1,\infty)$,
	the fact that for all $n\in\N$, $\theta\in\R^\pars$
	$G_n^{(i)}(\theta,X_n)$ and
	$(\batch_n)^{-1}\sum_{k=1}^{\batch_n}g_i(\theta,X_{k,1})$
	have the same distribution, 
	and \cref{lem:variance:Adam:batch:mean}
	(applied for every $n\in\{m\in\N\colon\batch_m\geq 2,\,m\geq 2\}$, $\omega\in A_n$ with 
	$\batch\curvearrowleft\batch_n$,
	$\pr{X_k}_{k\in\N}\curvearrowleft \pr{g_i(\Theta_{n-1}(\omega),X_{k,1})}_{k\in\N}$,
	$u\curvearrowleft\alpha m_{n-1}(\omega)$,
	$r\curvearrowleft r_n(\omega)$, 
	$q\curvearrowleft q_n$.
	$K\curvearrowleft K$
	in the notation of \cref{lem:variance:Adam:batch:mean})
	\prove\ that for all $j\in\N\cap[2,\sup_{m\in\N}\batch_m]$
	there exists $c_j\in(0,\infty)$ such that for all
	$n\in\{m\in\N\colon\batch_m=j,\,m\geq2\}$ it holds that
	\begin{equation}
		\llabel{eq:pre:positive:frozen:variance}
		\begin{split}
		\inf_{\omega\in A_n}\var\prbbb{\frac{\alpha m_{n-1}(\omega)+(1-\alpha)G_n^{(i)}(\Theta_{n-1}(\omega),X_{n})}
				{\eps+\prb{r_n(\omega)+q_n(G_n^{(i)}(\Theta_{n-1}(\omega),X_{n}))^2}^{\nicefrac12}}}
				\geq c_j.
		\end{split}
	\end{equation}
	This, \eqref{eq:bounded:batches:and:LR}, and
	the fact that the set $\N\cap[2,\sup_{n\in\N}\batch_n]$ is finite \prove\ that
		\begin{equation}
		\llabel{eq:positive:frozen:variance}
		\begin{split}
			&\textstyle\liminf_{n\to\infty}\PRb{\inf_{\omega\in A_n}v_n(\omega)}
			\\&\geq
			\liminf_{n\to\infty}
			\PRbbb{\inf_{\omega\in A_n}\var\prbbb{\frac{\alpha m_{n-1}(\omega)+(1-\alpha)G_n^{(i)}(\Theta_{n-1}(\omega),X_{n})}{\eps+\prb{r_n(\omega)+q_n(G_n^{(i)}(\Theta_{n-1}(\omega),X_{n}))^2}^{\nicefrac12}}}}
			>0.
		\end{split}
	\end{equation}
	\Moreover 
	\eqref{eq:cond2:MOM:incr:thm:inf.1},
	\lref{eq:momentum}, \lref{eq:frozen:update},
	\lref{eq:coordinate:update}, and the fact that $\alpha(1-z)=2\alpha>0$ \prove\ for that all $n\in\N$ it holds that
	\begin{equation}
		\begin{split}
		&\indicator{\{\vass{\Theta_0^{(i)}}\leq R\}}\Theta_n^{(i)}-\indicator{\{\vass{\Theta_0^{(i)}}\leq R\}}\Theta_{n-1}^{(i)}
		\\&=-\frac{\indicator{\{\vass{\Theta_0^{(i)}}\leq R\}}\gamma_n m_n}{\eps+\MOM_n(H_0,H_1,\dots,H_n)}
		\\&=-\indicator{\{\vass{\Theta_0^{(i)}}\leq R\}}\prbbb{
			\frac{\gamma_n m_n}
			{\eps+\prb{\PRb{\MOM_n(H_0,H_1,\dots,H_{n-1},0)}^2+(H_n)^2\PRb{\MOM_n(0,0,\dots,0,1)}^2}^{\nicefrac{1}{2}}}}
		\\&=-\indicator{\{\vass{\Theta_0^{(i)}}\leq R\}}\gamma_n\prbbb{
		\frac{\alpha m_{n-1}+(1-\alpha)G_n^{(i)}(\Theta_{n-1},X_n)}
		{\eps+\prb{r_n+q_n(G_n^{(i)}(\Theta_{n-1},X_n))^2}^{\nicefrac12}}}
		\\&=-\indicator{\{\vass{\Theta_0^{(i)}}\leq R\}}\gamma_n
		\Phi_n(\Theta_{n-1},\alpha m_{n-1},r_n,X_n)
		.
		\end{split}
	\end{equation}
	This, the fact that $\{\vass{\Theta_0^{(i)}}\leq R\}\in\fil_0$,
	the fact that for all $n\in\N$ it holds that $(\Theta_{n-1},\alpha m_{n-1},r_n)$ is $\fil_{n-1}$ measurable, and
	\cref{it:mb:fact:for:var1} in \cref{lem:factorization2:spec}
	(applied for every $n\in\N\cap[2,\infty)$ with
	$\gil\curvearrowleft\fil_{n-1}$,
	$X\curvearrowleft(\Theta_{n-1},\alpha m_{n-1},r_n)$,
	$Y\curvearrowleft X_n$,
	$\Phi\curvearrowleft\prb{\prb{\R^{\pars}\times\R\times[0,\infty)}\times([a,b]^d)^{\batch_n}\ni((\theta,u,r),x)\mapsto\Phi_n(\theta,u,r,x)\in\R}$
	in the notation of \cref{lem:factorization2:spec})
	\prove\ that for all $n\in\N\cap[2,\infty)$ it holds that
	\begin{equation}
		\begin{split}
			&\E\PRb{\prb{\indicator{\{\vass{\Theta_0^{(i)}}\leq R\}}\Theta_n^{(i)}-\cob{\indicator{\{\vass{\Theta_0^{(i)}}\leq R\}}\Theta_n^{(i)}}{\fil_{n-1}}}^2}
			\\&=(\gamma_n)^2\E\PRb{\prb{\indicator{\{\vass{\Theta_0^{(i)}}\leq R\}}
					\Phi_n(\Theta_{n-1},\alpha m_{n-1},r_n,X_n)
					\\&\quad-\cob{\indicator{\{\vass{\Theta_0^{(i)}}\leq R\}}
						\Phi_n(\Theta_{n-1},\alpha m_{n-1},r_n,X_n)}{\fil_{n-1}}}^2}
			\\&=(\gamma_n)^2\E\PRb{\indicator{\{\vass{\Theta_0^{(i)}}\leq R\}}v_n}
			\\&\textstyle\geq(\gamma_n)^2\E\PRb{\indicator{\{\vass{\Theta_0^{(i)}}\leq R,(H_{n-1})^2\geq\delta\}}}\PRb{\inf_{\omega\in A_n}v_n(\omega)}
			=(\gamma_n)^2\P(A_n)\PRb{\inf_{\omega\in A_n}v_n(\omega)}.
			\end{split}
		\end{equation}
	Combining this with 
	\eqref{eq:bounded:batches:and:LR},
	\lref{eq:non:van:prob:cool:set},
	\lref{eq:positive:frozen:variance}, and the assumption that $z=-1$ \proves\ that
	\begin{equation}
		\llabel{eq:positive:innovation}
		\begin{split}
		&\liminf_{n\to\infty}
		\E\PRb{\prb{\indicator{\{\vass{\Theta_0^{(i)}}\leq R\}}\Theta_n^{(i)}-\cob{\indicator{\{\vass{\Theta_0^{(i)}}\leq R\}}\Theta_n^{(i)}}{\fil_{n-1}}}^2}
		\\&\geq
		\liminf_{n\to\infty}
		\textstyle
		\PRb{(\gamma_n)^2\P(A_n)\PRb{\inf_{\omega\in A_n}v_n(\omega)}}
		\\&\geq
		\textstyle
		\PRb{\liminf_{n\to\infty}(\gamma_n)^2}
		\delta\P\prb{\vass{\Theta_0^{(i)}}\leq R}K^{-2}\prb{\liminf_{n\to\infty}\PRb{\inf_{\omega\in A_n}v_n(\omega)}}
		>0.
		\end{split}
	\end{equation}
	\Moreover \lref{eq:local:bounds}, the fact that for all $m,n\in\N$
	with $m<n$ it holds that $\indicator{\{\vass{\Theta_0^{(i)}}\leq R\}}\Theta_m^{(i)}$ is $\fil_{n-1}$-measurable,
	and \cref{lem:cond:exp:vanish3}
	(applied for every $m,n\in\N$ with $m<n$ with
	$X\curvearrowleft \indicator{\{\vass{\Theta_0^{(i)}}\leq R\}}\Theta_n^{(i)}$, $Y\curvearrowleft \indicator{\{\vass{\Theta_0^{(i)}}\leq R\}}\Theta_m^{(i)}$, and
	$\gil\curvearrowleft\fil_{n-1}$ in the notation of
	\cref{lem:cond:exp:vanish3})
	\prove\ for all $m,n\in\N$ with $m<n$ that
	\begin{equation}
		\begin{split}
			2K\E\PRb{\vass{\Theta_n^{(i)}-\Theta_m^{(i)}}}
			&\geq\E\PRb{(\indicator{\{\vass{\Theta_0^{(i)}}\leq R\}}\Theta_n^{(i)}-\indicator{\{\vass{\Theta_0^{(i)}}\leq R\}}\Theta_m^{(i)})^2}
			\\&\geq\E\PRb{\prb{\indicator{\{\vass{\Theta_0^{(i)}}\leq R\}}\Theta_n^{(i)}-\cob{\indicator{\{\vass{\Theta_0^{(i)}}\leq R\}}\Theta_n^{(i)}}{\fil_{n-1}}}^2}.
		\end{split}
	\end{equation}
	This, \lref{eq:positive:innovation}, and
	\cref{metrix:cauchy:limsup}
	(applied for every $\vartheta\in\mathcal{X}$ with
	$E\curvearrowleft\mathcal{X}$, $z\curvearrowleft-1$,
	$\lambda\curvearrowleft1$,
	$d\curvearrowleft\prb{\mathcal{X}\times\mathcal{X}\ni(V,W)\mapsto\E\PRb{\vass{V-W}}\in[0,\infty]}$,
	$x_0\curvearrowleft\vartheta$,
	$(x_k)_{k\in\N}\curvearrowleft\pr{\Theta_k^{(i)}}_{k\in\N}$
	in the notation of \cref{metrix:cauchy:limsup})
	\prove\ that
	\begin{equation}
		\begin{split}
			&\inf_{\vartheta\in\mathcal{X}}
			\PRbb{\liminf_{n\to\infty}\E\PRb{\vass{\Theta_n^{(i)}-\vartheta}}}
			\\&\geq\tfrac12\PRbb{\liminf_{m\to\infty}\liminf_{n\to\infty}
			\E\PRb{\vass{\Theta_n^{(i)}-\Theta_m^{(i)}}}}
			\\
			&\geq\tfrac{1}{4K}\PRbb{\liminf_{n\to\infty}
			\E\PRb{\prb{\indicator{\{\vass{\Theta_0^{(i)}}\leq R\}}\Theta_n^{(i)}-\cob{\indicator{\{\vass{\Theta_0^{(i)}}\leq R\}}\Theta_n^{(i)}}{\fil_{n-1}}}^2}}>0.
		\end{split}
	\end{equation}
	This \proves[pe]\ \eqref{eq:main:Thm} in the case of $z=-1$
	and $\alpha\in(0,1)$.
	\finishproofthus
\end{proof}

\newcommand{\momf}{\mathcal{M}}

\subsection{Independent random variables and filtrations}
\label{subsection:4.4}
In this subsection we prove
the well-known properties for independent random variables in \cref{lemma:independence:on:generators} (cf., \eg,  \cite[Theorem~2.16]{klenkeprobability}), 
\cref{lemma:basic:independ:fct} (cf., \eg, \cite[Problem~7.7.b~in~Section~7.3]{Krishna2006}),
\cref{cor:basic:independ:fct},
\cref{cor:idpd:times}, and
\cref{lemma:independ:rv:sigma}.
Moreover, we employ \cref{cor:basic:independ:fct} and \cref{lemma:independ:rv:sigma} to prove the well known relationship between independent random variables and filtrations
\cref{lem:end:filtration}.
Only for completeness we include here in this subsection detailed proofs for \cref{lemma:independence:on:generators},
\cref{lemma:basic:independ:fct},
\cref{cor:basic:independ:fct},
\cref{cor:idpd:times},
\cref{lemma:independ:rv:sigma}, and
\cref{lem:end:filtration}.

\begin{lemma}[{Independent generators}]
	\label{lemma:independence:on:generators}
	Let $I$ be a set,
	let $(S_i,\mathcal{S}_i)$, $i\in I$, be measurable spaces,
	let $ ( \Omega, \cF,\P ) $ be a probability space,
	let $X_i\colon\Omega\to S_i$, $i\in I$, be random variables,
	for every $i\in I$ let $\emptyset\neq\mathcal{E}_i\subseteq\mathcal{S}_i$ satisfy for every $A,B\in\mathcal{E}_i$ 
	and every sigma-algebra $\mathcal{A}$ on $S_i$ with $\mathcal{E}_i\subseteq \mathcal{A}$ that
	\begin{equation}
		\label{eq:gen:schnittstabil}
		(A\cap B)\in \mathcal{E}_i
		\qqandqq
		\mathcal{S}_i\subseteq\mathcal{A},
	\end{equation}
	and assume for all $n\in\N$, $i_1,i_2,\dots,i_n\in I$, $A_1\in\mathcal{E}_{i_1}$,
	$A_2\in\mathcal{E}_{i_2}$,
	$\dots$,
	$A_n\in\mathcal{E}_{i_n}$ with
	$\forall\, k,l\in\{1,2,\dots,n\}\colon (i_k=i_l)\Rightarrow(k=l)$ that
	\begin{equation}
		\label{eq:indep:on:gen}
		\textstyle
		\P\prb{X_{i_1}\in A_1, X_{i_2}\in A_2,\dots,X_{i_n}\in A_n}=\prod_{k=1}^n\P(X_{i_k}\in A_k).
	\end{equation}
	Then $X_i$, $i\in I$, are independent.
\end{lemma}

\begin{proof}[Proof of \cref{lemma:independence:on:generators}]
	\Nobs that \eqref{eq:indep:on:gen}
	\proves\ that for all $n\in\N$, $i_1,i_2,\dots,i_n\in I$ it holds that $\cup_{B\in\mathcal{E}_i}\{\{\omega\in\Omega\colon X_i(\omega)\in B\}\}  $, $i\in \{i_1,i_2,\dots,i_n\}$, are independent classes of events (see, \eg, \cite[Definition~2.11]{klenkeprobability}). 
	Combining this with \cite[Theorem~2.16]{klenkeprobability} and \eqref{eq:gen:schnittstabil} \proves\ that $X_i$, $i\in I$, are independent.
	The proof of \cref{lemma:independence:on:generators} is thus complete.
\end{proof}

\begin{athm}{lemma}{lemma:basic:independ:fct}
	Let $N\in\N$,
	let $(D_n,\mathcal{D}_n)$, $n\in\N$, and $(E_n,\mathcal{E}_n)$, $n\in\N$, be measurable spaces,
	let	$ ( \Omega, \cF,\P ) $ be a probability space, 
	let	$
	X_{n}\colon \Omega \to D_n
	$, $n\in\N$, be independent random variables,
	and let
	$f_n\colon D_n\to E_n$, $n\in\N$, be measurable.
	Then $f_k(X_k)$, $k\in\{1,2,\dots,N\}$, are independent.
\end{athm}

\begin{proof}[Proof of \cref{lemma:basic:independ:fct}]
	\Nobs that for all $A_n\in\mathcal{E}_n$, $n\in\N$, and all $B_n\in\mathcal{D}_n$, $n\in\N$, with $\forall\,n\in\N\colon B_n=\{b\in D_n\colon f_n(b)\in A_n\}$ it holds that
	\begin{equation}
		\begin{split}
			&\P\prb{f_1(X_1)\in A_1,f_2(X_2)\in A_2,\dots,f_N(X_N)\in A_N}
			\\&=
			\P\prb{X_1\in B_1,X_2\in B_2,\dots,X_N\in B_N}
			=\textstyle\prod_{k=1}^N \P(X_k\in B_k)
			=\textstyle\prod_{k=1}^N \P(f_k(X_k)\in A_k)
			.
		\end{split}
	\end{equation}
	This \proves\ that $f_k(X_k)$, $k\in\{1,2,\dots,N\}$, are independent.
	\finishproofthus
\end{proof}

\begin{athm}{cor}{cor:basic:independ:fct}
	Let $I$ be a set,
	let $(D_i,\mathcal{D}_i)$, $i\in I$, and $(E_i,\mathcal{E}_i)$, $i\in I$, be measurable spaces,
	let	$ ( \Omega, \cF,\P ) $ be a probability space, 
	let	$
	X_{i}\colon \Omega \to D_i
	$, $i\in I$, be independent random variables,
	and let
	$f_i\colon D_i\to E_i$, $i\in I$, be measurable.
	Then $f_i(X_i)$, $i\in I$, are independent.
\end{athm}

\begin{proof}[Proof of \cref{cor:basic:independ:fct}]
	\Nobs \cref{lemma:basic:independ:fct} \proves\ that for every finite subset $J\subseteq I$ it holds that $f_i(X_i)$, $i\in J$, are independent.
	\finishproofthus
\end{proof}

\newcommand{\mbs}{S}
\newcommand{\pow}{P}
\newcommand{\power}[1]{\textit{2}^{\mspace{2mu}#1}}

\begin{athm}{lemma}{cor:idpd:times}
	Let	$ ( \Omega, \cF,\P ) $ be a probability space,
	let $I$ and $J$ be sets,
	let $(\mbs_i,\mathcal{\mbs}_i)$, $i\in I$, be measurable spaces,
	let	$
	X_{i}\colon \Omega \to \mbs_i
	$, $i\in I$, be independent random variables,
	let $\call_j\subseteq I$, $j\in J$, be disjoint subsets of $I$,
	and for every $j\in J$ let
	$Y_j\colon \Omega\to\pr{\times_{i\in\call_j}\mbs_i}$ satisfy for all $\omega\in\Omega$ that $Y_j(\omega)=(X_i(\omega))_{i\in\call_j}$.
	Then $Y_j$, $j\in J$, are independent random variables.
\end{athm}

\begin{proof}[Proof of \cref{cor:idpd:times}]
	Throughout this proof\footnote{Note that for all sets $A$ and $B$ it holds that $A\in \power{B}$ if and only if $A\subseteq B$ (Note that for all sets $A$ and $B$ it holds that $A$ is an element of the power set of $B$ if and only if $A\subseteq B$). } 
	for every $L\subseteq I$ let $\pow_L=\power{\pr{\times_{i\in L}\mbs_i}}$
	and let
	$\mathcal{G}\colon\power{I}\to\cup_{L\subseteq I}\power{\pow_L}$ and 
	$\mathcal{E}\colon\power{I}\to\cup_{L\subseteq I}\power{\pow_L}$ satisfy for all $L\subseteq I$ that $\mathcal{G}(L)\subseteq\pow_L$ is the product sigma-algebra on $\times_{i\in L}\mbs_i$ 
	and 
	\begin{equation}
		\label{eq:def:generator}
		\begin{split}
		\mathcal{E}(L)
		&=\cup_{F\subseteq L\text{ finite}}\{\times_{k\in L}A_k\colon(A_k)_{k\in L}\in\pr{\times_{k\in L}\mathcal{S}_k}\colon\PR{\forall\, k\in L\backslash F\colon A_k=S_k}\}.
		\end{split}
	\end{equation}
	\Nobs that \eqref{eq:def:generator} \proves\ that for every $L\subseteq I$, $A,B\in\mathcal{E}(L)$
	and every sigma-algebra $\mathcal{A}$ on $\times_{i\in L}\mbs_i$ with $\mathcal{E}(L)\subseteq \mathcal{A}$
	it holds that
	\begin{equation}
		\label{eq:schnittstabil:und:gen}
		(A\cap B)\in \mathcal{E}(L)
		\qqandqq
		\mathcal{G}(L)\subseteq\mathcal{A}
		.
	\end{equation}
	\Moreover the fact that for all $i\in I$ it holds that $X_i$ is 
	measurable
	\proves\ that for all $j\in J$ it holds that
	$
		Y_j\text{ is measurable} 
	$.
	This,
	the assumption that $X_i$, $i\in I$, are independent, and
	the assumption that for all $j\in J$ it holds that $Y_j=(X_i)_{i\in\call_j}$
	\prove\ that for all $j\in J$, $L\subseteq I\backslash K_j$ and all
	$A_k\in\mathcal{S}_k$, $k\in K_j\cup L$, such that $\{k\in K_j\cup L\colon A_k\neq S_k\}$ is finite it holds that
	\begin{equation}
		\begin{split}
			&\P\prb{Y_j\in \pr{\times_{i\in \call_j}A_i},
				(X_i)_{i\in L}\in \pr{\times_{i\in  L}A_i}}
			\\&=\P\prb{(X_i)_{i\in\call_j}\in \pr{\times_{i\in \call_j}A_i},
				(X_i)_{i\in L}\in \pr{\times_{i\in  L}A_i}}
			\\&=\P\prb{(X_i)_{i\in\pr{\call_j\cup L}}\in \pr{\times_{i\in \pr{\call_j\cup L}}A_i}}
			\\&=\textstyle\prod_{i\in\{k\in\pr{\call_j\cup L}\colon A_k\neq S_k\}}\P(X_i\in A_i)
			\\&=\textstyle\prb{\prod_{i\in\{k\in\call_j\colon A_k\neq S_k\}}\P(X_i\in A_i)}\prb{\prod_{i\in\{k\in L\colon A_k\neq S_k\}}\P(X_i\in A_i)}
			\\&=\P\prb{(X_i)_{i\in\call_j}\in \pr{\times_{i\in \call_j}A_i}}
			\P\prb{
				(X_i)_{i\in L}\in \pr{\times_{i\in  L}A_i}}
			\\&=\P\prb{Y_j\in \pr{\times_{i\in \call_j}A_i}}
			\P\prb{
				(X_i)_{i\in L}\in \pr{\times_{i\in  L}A_i}}
			.
		\end{split}
	\end{equation}
	This,
	\eqref{eq:def:generator}, and
	the assumption that for all $j\in J$ it holds that $Y_j=(X_i)_{i\in\call_j}$, \prove\ that for all $m\in\N$, $j_1,j_2,\dots,j_m\in J$, $A_1\in\mathcal{E}(\call_{j_1})$, $A_2\in\mathcal{E}(\call_{j_2})$, $\dots$,
	$A_m\in\mathcal{E}(\call_{j_m})$
	with
	$\forall\, k,l\in\{1,2,\dots,m\}\colon (j_k=j_l)\Rightarrow (k=l)$
	it holds that 
	\begin{equation}
		\begin{split}
			&\P\prb{(Y_{j_1},Y_{j_2},\dots,Y_{j_m})\in\pr{\times_{k=1}^m A_k}}
			\\&=
			\P\prb{Y_{j_1}\in A_1,(Y_{j_2},Y_{j_3},\dots,Y_{j_m})\in\pr{\times_{k=2}^m A_k}}
			\\&=
			\P\prb{Y_{j_1}\in A_1,(X_i)_{i\in\pr{\call_{j_2}\cup\call_{j_3}\cup\dots\cup\call_{j_m}}}\in\pr{\times_{k=2}^m A_k}}
			\\&=
			\P(Y_{j_1}\in A_1)\P\prb{(X_i)_{i\in\pr{\call_{j_2}\cup\call_{j_3}\cup\dots\cup\call_{j_m}}}\in\pr{\times_{k=2}^m A_k}}
			\\&=
			\P(Y_{j_1}\in A_1)\P\prb{(Y_{j_2},Y_{j_3},\dots,Y_{j_m})\in\pr{\times_{k=2}^m A_k}}
			\\&=
			\P(Y_{j_1}\in A_1)\P\prb{Y_{j_2}\in A_2,(Y_{j_3},Y_{j_4},\dots,Y_{j_m})\in\pr{\times_{k=3}^m A_k}}
			\\&=
			\P(Y_{j_1}\in A_1)\P(Y_{j_2}\in A_2)\P\prb{(Y_{j_3},Y_{j_4},\dots,Y_{j_m})\in\pr{\times_{k=3}^m A_k}}
			\\&=\ldots
			=\textstyle\prod_{k=1}^m\P(Y_{j_k}\in A_k)
			.
		\end{split}
	\end{equation}
	Combining \eqref{eq:schnittstabil:und:gen} with \cref{lemma:independence:on:generators} (applied with 
	$I\curvearrowleft J$,
	$(S_i,\mathcal{S}_i)_{i\in I}\curvearrowleft (\times_{i\in\call_j} S_i,\mathcal{G}(\call_j))_{j\in J}$,
	$( \Omega, \cF,\P ) \curvearrowleft ( \Omega, \cF,\P ) $,
	$(X_i)_{i\in I}\curvearrowleft (Y_j)_{j\in J}$,
	$(\mathcal{E}_i)_{i\in I}\curvearrowleft(\mathcal{E}(\call_j))_{j\in J}$
	in the notation of \cref{lemma:independence:on:generators}) hence
	\proves[pis]\ that $Y_j$, $j\in J$, are independent.
	\finishproofthus
\end{proof}

\begin{athm}{lemma}{lemma:independ:rv:sigma}
	Let $N,\pars,\dim\in\N$, $a\in\R$, $b\in[a,\infty)$,
	let	$ ( \Omega, \cF,\P ) $ be a probability space, 
	for every $n\in\N$
	let $\batch_n,\Batch_n,\BATCH_n\in\pr{\N\cup\{\infty\}}$ satisfy $\max\{\batch_n,\Batch_n\}\leq\BATCH_n$,
	let	$
	X_{n,\mm}\colon \Omega \to [ a, b ]^{\dim}
	$, $(n,m)\in\{(k,j)\in\N^2\colon j\leq\BATCH_k\}$, be independent random variables,
	let $Y\colon\Omega\to\R^{\pars}$ be a random variable,
	and assume that $(X_{n,m})_{(n,m)\in\{(k,j)\in\N^2\colon j\leq\BATCH_k\}}$ and $Y$ are independent.
	Then 
	$(X_{N+1,m})_{m\in\pr{\N\cap[1,\batch_{N+1}]}}$ and
	$\pr{Y,\allowbreak(X_{n,m})_{(n,m)\in\{(k,j)\in\N^2\colon (k\leq N)\wedge (j\leq\Batch_k)\}}}$
	 are independent.
\end{athm}

\begin{proof}[Proof of \cref{lemma:independ:rv:sigma}]
	Throughout this proof for every
	$(n,k)\in\{(m,l)\in\N\times(\N\cup\{\infty\})\colon l\leq\BATCH_m\}$ let $Z_{n,k}\colon\Omega \to \times_{m=1}^{k}(\R^{\dim})$ satisfy for all $\omega\in\Omega$ that
	\begin{equation}
		\label{eq:def:Zs}
		Z_{n,k}(\omega)=(X_{n,m}(\omega))_{m\in\pr{\N\cap[1,k]}}.
	\end{equation}
	\Nobs that \eqref{eq:def:Zs}, the assumption that $
	X_{n,m}$, $(n,m)\in\{(k,j)\in\N^2\colon j\leq\BATCH_k\}$, are independent, and \cref{cor:idpd:times} (applied with 
	$I\curvearrowleft\{(k,j)\in\N^2\colon j\leq\BATCH_k\}$,
	$J\curvearrowleft\{1,2,\dots,N+1\}$,
	$(K_j)_{j\in J}\curvearrowleft\prb{\{j\}\times(\N\cap[1,\BATCH_j])}_{j\in\{1,2,\dots,N+1\}}$,
	$(X_i)_{i\in I}\curvearrowleft (X_{n,m})_{(n,m)\in\{(k,j)\in\N^2\colon j\leq\BATCH_k\}}$
	in the notation of \cref{cor:idpd:times}),  \prove\ that 
	\begin{equation}
		\label{eq:Z:as:coords:idpd}
		Z_{1,\BATCH_1},Z_{2,\BATCH_2},\dots,Z_{N+1,\BATCH_{N+1}}
	\end{equation}
	are independent.
	\Moreover the assumption that $(X_{n,m})_{(n,m)\in\{(k,j)\in\N^2\colon j\leq\BATCH_k\}}$ and $Y$ are independent, \eqref{eq:def:Zs}, and \cref{cor:basic:independ:fct} (applied with 
	$I\curvearrowleft\{1,2\}$,
	$f_1\curvearrowleft\pr{\R^{\pars}\ni x\mapsto x\in\R^{\pars}}$,
	$f_2\curvearrowleft\pr{(\R^{\dim})^{\{(k,j)\in\N^2\colon j\leq\BATCH_k\}}\ni (x_{n,m})_{(n,m)\in\{(k,j)\in\N^2\colon j\leq\BATCH_k\}}\mapsto ((x_{n,m})_{m\in\pr{\N\cap[1,\BATCH_n]}})_{n\in\{1,2,\dots,N+1\}}\allowbreak\in(\R^{\dim})^{\{(k,j)\in\N^2\colon (k\leq N+1)\wedge (j\leq\BATCH_k)\}}}$
	in the notation of \cref{cor:basic:independ:fct}) \prove\ that
	\begin{equation}
		Y\qqandqq(Z_{1,\BATCH_1},Z_{2,\BATCH_2},\dots,Z_{N+1,\BATCH_{N+1}})
	\end{equation}
	are independent.
	This and \eqref{eq:Z:as:coords:idpd} \prove\ that for all $A\in\mathcal{B}(\R^{\pars})$, $B_1\in\mathcal{B}\prb{\times_{k=1}^{\BATCH_1}(\R^{\dim})}$,
	$B_2\in\mathcal{B}\prb{\times_{k=1}^{\BATCH_2}(\R^{\dim})}$,
	$\dots$,
	$B_{N+1}\in\mathcal{B}\prb{\times_{k=1}^{\BATCH_{N+1}}(\R^{\dim})}$ it  holds that
	\begin{equation}
		\begin{split}
			&\P(Y\in A, Z_{1,\BATCH_1}\in B_1, Z_{2,\BATCH_2}\in B_2,\dots,Z_{N+1,\BATCH_{N+1}}\in B_{N+1})
			\\&=\P(Y\in A, (Z_{1,\BATCH_1},Z_{2,\BATCH_2},\dots,Z_{N+1,\BATCH_{N+1}})\in (B_1\times B_2\times\dots\times B_{N+1}))
			\\&=\P(Y\in A)\P( (Z_{1,\BATCH_1},Z_{2,\BATCH_2},\dots,Z_{N+1,\BATCH_{N+1}})\in (B_1\times B_2\times\dots\times B_{N+1}))
			\\&=\P(Y\in A)\P( Z_{1,\BATCH_1}\in B_1)\P(Z_{2,\BATCH_2}\in B_2)\dots\P(Z_{N+1,\BATCH_{N+1}}\in B_{N+1})
			.
		\end{split}
	\end{equation}
	This \proves[pei] that
	$
		Y,Z_{1,\BATCH_1},Z_{2,\BATCH_2},\dots,Z_{N+1,\BATCH_{N+1}}
	$
	are independent. 
	This and \cref{cor:basic:independ:fct} (applied with 
	$I\curvearrowleft\{1,2,\dots,N+2\}$
	in the notation of \cref{cor:basic:independ:fct}) \prove\ that
	\begin{equation}
		Y,
		Z_{1,\Batch_1},
		Z_{2,\Batch_2},
		\dots,
		Z_{N,\Batch_{N}},
		Z_{N+1,\batch_{N+1}}
	\end{equation}
	are independent.
	This and \cref{cor:idpd:times} (applied with 
	$I\curvearrowleft\{1,2,\dots,N+2\}$,
	$J\curvearrowleft\{1,2\}$,
	$(X_i)_{i\in I}\curvearrowleft (Y,
	Z_{1,\Batch_1},
	Z_{2,\Batch_2},
	\dots,
	Z_{N,\Batch_{N}},
	Z_{N+1,\batch_{N+1}})$,
	$Y_1\curvearrowleft (Y,
	Z_{1,\Batch_1},
	Z_{2,\Batch_2},
	\dots,
	Z_{N,\Batch_{N}})$,
	$Y_2\curvearrowleft ,
	Z_{N+1,\batch_{N+1}}$
	in the notation of \cref{cor:idpd:times}) \prove\ that
	\begin{equation}
		(Y,
		Z_{1,\Batch_1},
		Z_{2,\Batch_2},
		\dots,
		Z_{N,\Batch_{N}})\qqandqq Z_{N+1,\batch_{N+1}}
	\end{equation}
	are independent.
	\finishproofthus
\end{proof}

\begin{athm}{lemma}{lem:end:filtration}
	Let $\pars,\dim\in\N$, $a\in\R$, $b\in[a,\infty)$,
	let $ ( \Omega, \cF,\P ) $ be a probability space,
	for every $n\in\N$ let $\batch_n,\Batch_n,\BATCH_n\in\N\cup\{\infty\}$ satisfy $\batch_n\leq\Batch_n\leq\BATCH_n$,
	let	$
	X_{n,\mm}\colon \Omega \to [ a, b ]^{\dim}
	$, $(n,m)\in\{(k,j)\in\N^2\colon j\leq\BATCH_k\}$, be \iid\ random variables,
	let $Y\colon\Omega\to\R^\pars$ be a random variable,
	assume that $Y$ and $\pr{X_{n,m}}_{(n,m)\in\{(k,j)\in\N^2\colon j\leq\BATCH_k\}}$ are independent, and
	for\footnote{Note that for every set $D$,
		every measurable space $(E,\mathcal{E})$, 
		and every function $Y\colon D\to E$ it holds that $Y$ is $\sigma(Y)$-measurable 
		and note that for
		every measurable space $(D,\mathcal{D})$,
		every measurable space $(E,\mathcal{E})$, and 
		every $\cD$-measurable function $Y\colon D\to E$
		it holds that $\sigma(Y)\subseteq\cD$.} every $n\in\N_0$ let $\fil_n\subseteq\cF$ satisfy
	\begin{equation}
		\label{eq:setting:up:filtration}
		\fil_0=\sigma(Y)
		\qqandqq
		\fil_{n+1}=\sigma\prb{Y,(X_{k,m})_{(k,m)\in\{(l,j)\in\N^2\colon (l\leq n+1)\wedge(j\leq\Batch_l)\}}}
		.
	\end{equation}
	Then
	\begin{enumerate}[label=(\roman*)]
		\item \label{it:filtered:prob:space}it holds that $ ( \Omega, \cF,(\mathbb{F}_n)_{n\in\N_0}, \P ) $ is a filtered probability space,
		\item \label{it:filter:measurable2}it holds that $Y$ is $\fil_0$-measurable,
		\item \label{it:filter:measurable}it holds for all $n\in\N$ that $(X_{n,m})_{m\in\pr{\N\cap[1,\batch_n]}}$ is $\fil_n$-measurable, and
		\item \label{it:filter:independence}it holds for all $n\in\N$ that
		$(X_{n,m})_{m\in\pr{\N\cap[1,\batch_n]}}$ is independent of $\fil_{n-1}$.
	\end{enumerate}
	\end{athm}

\begin{proof}[Proof of \cref{lem:end:filtration}]
	\Nobs that \eqref{eq:setting:up:filtration} \proves\ \cref{it:filtered:prob:space,it:filter:measurable2,it:filter:measurable}.
	\Nobs that
	the assumption that $Y$ and $\pr{X_{n,m}}_{(n,m)\in\{(k,j)\in\N^2\colon j\leq\BATCH_k\}}$ are independent and
	\cref{cor:basic:independ:fct} (applied with 
	$I\curvearrowleft\{1,2\}$,
	$f_2\curvearrowleft\PRb{\times_{n\in\N}\prb{\times_{m=1}^{\BATCH_n}([a,b]^\dim)}\ni (x_{n,m})_{(n,m)\in\{(k,j)\in\N^2\colon j\leq\BATCH_k\}}\mapsto ((x_{1,m})_{m\in\pr{\N\cap[1,\batch_1]}}\in\pr{\times_{m=1}^{\batch_1}([a,b]^\dim)}}$,
	$f_1\curvearrowleft\PRb{\R^{\pars}\ni x\mapsto x\in\R^{\pars}}$
	in the notation of \cref{cor:basic:independ:fct}) \prove\ that
	\begin{equation}
		\label{eq:independence:base:case}
		Y\qqandqq(X_{1,m})_{m\in\pr{\N\cap[1,\batch_1]}}
	\end{equation}
	are independent.
	\Moreover
	the assumption that $Y$ and $\pr{X_{n,m}}_{(n,m)\in\{(k,j)\in\N^2\colon j\leq\BATCH_k\}}$  are independent and \cref{lemma:independ:rv:sigma} (applied with 
	$N\curvearrowleft n$,
	$\batch\curvearrowleft \batch$,
	$\Batch\curvearrowleft \Batch$,
	$\BATCH\curvearrowleft \BATCH$,
	$\dim\curvearrowleft \dim$,
	$\pars\curvearrowleft \pars$,
	$ ( \Omega, \cF, \P ) \curvearrowleft ( \Omega, \cF, \P )$,
	$(X_{n,m})_{(n,m)\in\{(k,j)\in\N^2\colon j\leq\BATCH_k\}}\curvearrowleft(X_{n,m})_{(n,m)\in\{(k,j)\in\N^2\colon j\leq\BATCH_k\}}$,
	$Y\curvearrowleft Y$
	for $n\in\N$
	in the notation of \cref{lemma:independ:rv:sigma}) \prove\ that 
	for all $n\in\N$ it holds that 
	\begin{equation}
		\label{eq:independence:other:cases}
		\prb{Y,(X_{k,m})_{(k,m)\in\{(l,j)\in\N^2\colon (l\leq n)\wedge (j\leq\Batch_l)\}}}\qqandqq
	(X_{n+1,m})_{m\in\pr{\N\cap[1,\batch_{n+1}]}}
	\end{equation}
	are independent. This, \eqref{eq:setting:up:filtration}, and \eqref{eq:independence:base:case} \prove\ that 
	for all $n\in\N$ it holds that 
	$(X_{n,m})_{m\in\pr{\N\cap[1,\batch_n]}}$ is independent of $\fil_{n-1}$. This \proves[pei] \cref{it:filter:independence}.
	\finishproofthus
\end{proof}

\subsection{Non-convergence of Adam in the training of ANNs}

In this subsection we establish in \cref{lem:for:EX:ANNs:1} non-convergence of \ADAM\ and other adaptive \SGD\ optimization methods in the training of deep \ANNs.
In our proof of \cref{lem:for:EX:ANNs:1} we employ the non-convergence result in \cref{cor:limsup:simple} above
and 
representations of the partial derivatives of objective functions in the training of deep \ANNs\ (see, \eg, \cite[Corollary~8.2.4]{jentzen2023mathematical}).

\begin{athm}{cor}{lem:for:EX:ANNs:1}	
	Let $\dim,\DIM,\pars,L\in\N$, 
	$\fc\in\R$,
	$l_0,l_1,\dots,l_L\in\N$,
	$s_0,s_1,\dots,s_L\in\N_0$
	satisfy $l_0=\dim$, $l_L=\DIM$, $s_L=\pars$, and $\sum_{k=0}^L \vass{s_k-\sum_{j=1}^k l_j(l_{j-1}+1)}=0$,
	let $\act\in C^1(\R,\R)$ satisfy for all $x\in\R$ that $\vass{\act'(x)}\leq\fc(\vass{x}+1)^{\fc}$,
	for every $\theta=(\theta_1,\dots,\theta_{\pars})\in\R^{\pars}$
	let 
	$\realization{\theta}_{k}=(\realization{\theta}_{k,1},\dots,\realization{\theta}_{k,l_k})\colon\R^\dim\to\R^{l_k}$, $k\in\{0,1,\dots,L\}$, satisfy for all $k\in\{0,1,\dots,L-1\}$, $x=(x_1,\dots,x_\dim)\in\R^\dim$, $j\in\{1,2,\dots,l_{k+1}\}$ that
		\begin{equation}
			\label{eq:vect:ANN:real}
		\realization{\theta}_{k+1,j}(x)
		=
		\theta_{j+l_{k+1}l_{k}+s_{k}}
		+
		\sum_{i=1}^{l_{k}}
		\theta_{i+(j-1) l_{k}+s_{k}}
		\prb{
			\act\prb{\realization{\theta}_{k,i}(x)}\indicator{\N}(k)
			+x_{\min\{i,d\}}\indicator{\{0\}}(k)
		},
	\end{equation}
	let $\loss\colon\R^{\pars}\times\pr{\R^{\dim}\times\R^\DIM}\to\R$, $\batch\colon\N\to\N$, and
	$\gamma\colon\N\to[0,\infty)$ 
	satisfy
	for all 
	$\theta\in\R^\pars$, 
	$x\in\R^\dim$,
	$y\in\R^\DIM$
	that
	\begin{equation}
		\label{eq:ANN:EX:loss:1}	
		\textstyle
		\loss(\theta,(x,y))=\norm{\realization{\theta}_{L}(x)-y}^2
		\qqandqq
		\textstyle\limsup_{n\to\infty}\PRb{(\gamma_n)^2(\batch_n)^{-1}}>0
		,
	\end{equation} 
	let $a\in\R$, $b\in(a,\infty)$, $\alpha,\beta\in[0,1)$,
	for every $n\in\N$ let
	$G_{n}=\pr{G_{n}^{(1)},\dots,G_{n}^{(\pars)}}\colon\R^\pars \times\pr{[ a, b ]^{\dim}\times[ a, b ]^{\DIM}}^{\batch_n}\to\R^\pars$ satisfy for all $\theta\in\R^{\pars}$, $\dpa=(\dpa_1,\dots,\dpa_{\batch_n})\in\pr{[ a, b ]^{\dim}\times[ a, b ]^{\DIM}}^{\batch_n}$ that
	\begin{equation}
		\label{eq:nom:EX1:0:1}
		\textstyle
		G_n(\theta,\dpa)=\frac{1}{\batch_{n}}\PRb{\sum_{j=1}^{\batch_{n}}\pr{\nabla_{\theta}\loss}(\theta,\dpa_j)}
		,
	\end{equation}
	for every $n\in\N$ let $\MOM_n\colon\R^{n+1}\to[0,\infty)$ be measurable, assume 
	for all $n\in\N$, $x=(x_1,\dots,x_{n+1})\in\R^{n+1}$ that
	\begin{equation}
		\label{eq:nom:EX1:1:1}
		\textstyle
		\vass{\MOM_n(x)}^2
		\leq
		\fc(1-\beta)\sum_{k=0}^n(x_{k+1})^2\beta^{n-k},
	\end{equation}
	let $ ( \Omega, \cF, \P ) $ be a probability space,
	for every $n\in\N$ let $\DPA_n=\pr{\pr{X_{n,j},Y_{n,j}}_{j\in\{1,2,\dots,\batch_n\}}}\colon\Omega \allowbreak\to \pr{[ a, b ]^{\dim}\times[ a, b ]^{\DIM}}^{\batch_n}$ and 
	$\MOMs[n]=(\MOMs[n]^{(1)},\ldots,\MOMs[n]^{(\pars)})\colon\Omega\to\R^\pars$ be random variables, 
	let $\eps\in(0,\infty)$,
	let
	$ \Theta=(\Theta^{(1)},\ldots,\Theta^{(\pars)}) \colon\allowbreak\N_0\times\Omega \to \R^{\pars}$
	satisfy
	for all $n\in\N$, $i\in\{1,2,\dots,\pars\}$ that
	\begin{equation}
		\label{eq:EX:recursion:1}
		\begin{split}
			\Theta_{ n }^{(i)}
			&=\Theta_{ n-1  }^{(i)}-\frac{\gamma_{n}  \PRb{\alpha^n \MOMs[1]^{(i)}+(1-\alpha)\sum_{k=1}^n\alpha^{n-k}G_k^{(i)}(\Theta_{k-1},\DPA_k)}}{\eps+\MOM_n\prb{\MOMs[2]^{(i)},G_1^{(i)}(\Theta_0,\DPA_1),G_2^{(i)}(\Theta_1,\DPA_2),\dots,G_n^{(i)}(\Theta_{n-1},\DPA_n)}},
		\end{split}
	\end{equation}
	assume that $(X_{n,j},Y_{n,j})$, $(n,j)\in\{(m,i)\in\N^2\colon i\leq\batch_m\}$, are \iid,
	assume that $(\Theta_0,\MOMs[1],\MOMs[2])$ and $\pr{\DPA_n}_{n\in\N}$ are independent,
	assume
		\begin{equation}
		\label{eq:ANN:EX:properties:of:setup:1}
		\textstyle
		\inf_{\theta\in\R^\pars}
		\E\PRb{\norm{\realization{\theta}_{L}(X_{1,1})-Y_{1,1}}^2}
		>0
		,
	\end{equation}
	and assume that $(\Theta_0,\MOMs[1],\MOMs[2])$ is measurable.
	Then
	\begin{equation}
		\label{eq:SANDBOX:RESULT:1}
		\begin{split}
			\inf_{\substack{\xi\colon \Omega \to \R^\pars\\\text{measurable}}}
			\PRbbb{\limsup_{ n \to \infty }\E\PRb{ \norm{ \Theta_n - \xi }^2 }}
			>
			0
			.
		\end{split}
	\end{equation}
\end{athm}

\begin{proof}[Proof of \cref{lem:for:EX:ANNs:1}]
	Throughout this proof
	assume without loss of generality for all $x\in\R^\dim$, $\theta\in\R^\pars$ that 
	$\realization{\theta}_{0}(x)=x$,
	$\max\{1,\vass{\act(0)}\}\leq\fc$,
	for every $k\in\N_0$ let $c_k,\const_k\in\R$ satisfy $c_k=(\fc+1)^k$,
	 let $\vertiii{\cdot}\colon\cup_{k\in\N}\R^k\to\R$ satisfy for all $k\in\N_0$, $x=(x_1,\dots,x_{k+1})\in\R^{k+1}$ that
	\begin{equation}
		\label{eq:setup:coefficients:and:supn}
		\vertiii{x}=\max\{\vass{x_1},\vass{x_2},\dots,\vass{x_{k+1}}\}
		\qqandqq
			\const_k=\prb{2^{\fc+1}\pars\fc\vertiii{(1,a,b)}}^{c_k},
	\end{equation}
	let $D_{k}=\prb{D_{k}^{(1)},\dots,D_{k}^{(l_{k-1})}}\colon\R^{\pars}\times\pr{[a,b]^\dim\times[a,b]^\DIM}\to\R^{l_{k-1}}$, $k\in\{1,2,\dots,L+1\}$, satisfy for all 
	$\theta\in\R^\pars$,
	$x\in[a,b]^\dim$,
	$y\in[a,b]^\DIM$,
	$\dpa\in\{(x,y)\}$,
	$k\in\{1,2,\dots,L\}$, 
	$j\in\{1,2,\dots,l_{k-1}\}$ that
				\begin{equation}
					\label{eq:setup:backprop1}
		\textstyle
		D_{k}^{(j)}(\theta,\dpa)
		=
		\sum_{i=1}^{l_k}\theta_{s_{k-1}+(i-1)l_{k-1}+j}
		D_{k+1}^{(i)}(\theta,\dpa)
		\prb{
			\act'\prb{\realization{\theta}_{k,i}(x)}\indicator{(0,L)}(k)
			+
			\indicator{\{L\}}(k)
		}
	\end{equation}
	\begin{equation}
		\label{eq:setup:backprop2}
		\qqandqq
		D_{L+1}(\theta,\dpa)=2\prb{\realization{\theta}_{L}(x)-y}.
	\end{equation}
	\Nobs that 
	the assumption that $(X_{n,j},Y_{n,j})$, $(n,j)\in\{(m,i)\in\N^2\colon i\leq\batch_m\}$, are \iid,
	the assumption that
	 $(\Theta_0,\MOMs[1],\MOMs[2])$ and $\pr{\DPA_{n}}_{n\in\N}$ are independent,
	 and \cref{lem:end:filtration} 
	(applied with 
	$\pars\curvearrowleft 3\pars$,
	$\dim\curvearrowleft\dim+\DIM$,
	$Y\curvearrowleft (\Theta_0,\MOMs[1],\MOMs[2])$,
	$\batch\curvearrowleft\batch$,
	$\Batch\curvearrowleft\batch$,
	$\BATCH\curvearrowleft\batch$,
	$(X_{n,j})_{(n,j)\in\{(m,i)\in\N^2\colon i\leq\BATCH_m\}}\curvearrowleft(\pr{X_{n,j},Y_{n,j}})_{(n,j)\in\{(m,i)\in\N^2\colon i\leq\batch_m\}}$
	 in the notation of \cref{lem:end:filtration}) \prove\ that there exist $\fil_n\subseteq\cF$, $n\in\N_0$, which satisfy that
	\begin{enumerate}[label=(\Roman*)]
		\item \label{it:filtered:prob:space:simple}it holds that $ ( \Omega, \cF,(\mathbb{F}_n)_{n\in\N_0}, \P ) $ is a filtered probability space,
		\item \label{it:filter:independence:simple}it holds for all $n\in\N$ that
		$\DPA_n$ is independent of
		 $\fil_{n-1}$,
		\item \label{it:filter:measurable:simple}it holds for all $n\in\N$ that $\DPA_n$ is $\fil_n$-measurable, and
		\item \label{it:filter:measurable2:simple}it holds that $(\Theta_0,\MOMs[1],\MOMs[2])$ is $\fil_0$-measurable.
	\end{enumerate}
	\Moreover
	\eqref{eq:vect:ANN:real},
	\eqref{eq:setup:backprop1},
	\eqref{eq:setup:backprop2}, 
	and
	 \cite[
	 Proposition~1.3.10 and 
	 Corollary~8.2.4]{jentzen2023mathematical} (applied with
	$M\curvearrowleft1$,
	$(l_0,\dots,l_L)\curvearrowleft (l_0,\dots,l_L)$,
	$(D_k^{(1)})_{k\in\{1,2,\dots,L+1\}}\curvearrowleft\allowbreak(D_k(\theta,(x,y)))_{k\in\{1,2,\dots,L+1\}}$,
	$a\curvearrowleft \act$,
	$(\mathfrak{x}_k^{(1)})_{k\in\{0,1,\dots,L\}}\curvearrowleft \prb{\indicator{(0,L)}(k)\prb{
		\act\prb{\realization{\theta}_{k,1}(x)},
		\dots,
		\act\prb{\realization{\theta}_{k,l_k}(x)}}
		+
		\indicator{\{0,L\}}(k)\realization{\theta}_{k}(x)
		}_{k\in\{0,1,\dots,L\}}$,
	$\mathfrak{y}^{(1)}\curvearrowleft y$
	for $\theta\in\R^{\pars}$, $x\in[a,b]^\dim$, $y\in[a,b]^\DIM$
	  in the notation of \cite[Corollary~8.2.4]{jentzen2023mathematical}) \prove\ that for all 
	  $\theta\in\R^{\pars}$, 
	  $x\in[a,b]^\dim$, 
	  $y\in[a,b]^\DIM$,
	  $\dpa\in\{(x,y)\}$ it holds that
	  \begin{enumerate}[label=(\Alph*)]
	  	\item\label{it:diff:loss} it holds that $\R^\pars\ni\vartheta\mapsto\loss(\vartheta,\dpa)\in\R$ is differentiable,
	  	\item\label{it:last:bias} it holds for all $p\in\N$, $j\in\{1,2,\dots,l_L\}$ with $p=s_{L-1}+l_Ll_{L-1}+j$ that 
	  	\begin{equation}
				\tfrac{\partial\loss}{\partial\theta_{p}}(\theta,\dpa)
	  		=D_{L+1}^{(j)}(\theta,\dpa),
	  	\end{equation}
	  	\item\label{it:kth:bias} it holds for all $k\in\{1,2,\dots,L-1\}$, $p\in\N$,
	  	$j\in\{1,2,\dots,l_k\}$ with
	  	$p=s_{k-1}+l_kl_{k-1}+j$ that 
	  	\begin{equation}
	  		\tfrac{\partial\loss}{\partial\theta_{p}}(\theta,\dpa)
	  		=\act'\prb{\realization{\theta}_{k,j}(x)}D_{k+1}^{(j)}(\theta,\dpa),
	  	\end{equation}
	  		  	\item\label{it:last:weight} it holds for all 
	  		  	$p\in\N$, 
	  		  	$i\in\{1,2,\dots,l_{L-1}\}$, 
	  		  	$j\in\{1,2,\dots,l_L\}$ with
	  	$p=s_{L-1}+(j-1)l_{L-1}+i$ that 
	  	\begin{equation}
	  		\tfrac{\partial\loss}{\partial\theta_{p}}(\theta,\dpa)
	  		=
	  		\prb{\indicator{(1,\infty)}(L)\act\prb{\realization{\theta}_{L-1,i}(x)}
	  			+
	  			\indicator{\{1\}}(L)\realization{\theta}_{0,i}(x)}
	  		D_{L+1}^{(j)}(\theta,\dpa),
	  	\end{equation}
	  	and
	    \item\label{it:kth:weight} it holds for all 
	    $k\in\{1,2,\dots,L-1\}$,
	  	$p\in\N$, 
	  	$i\in\{1,2,\dots,l_{k-1}\}$, 
	  	$j\in\{1,2,\dots,l_k\}$ with
	  	$p=s_{k-1}+(j-1)l_{k-1}+i$ that 
	  	\begin{equation}
	  		\tfrac{\partial\loss}{\partial\theta_{p}}(\theta,\dpa)
	  		=\act'\prb{\realization{\theta}_{k,j}(x)}
	  		\prb{\indicator{(1,L)}(k)\act\prb{\realization{\theta}_{k-1,i}(x)}
	  		+
	  		\indicator{\{1\}}(k)\realization{\theta}_{0,i}(x)}
	  		D_{k+1}^{(j)}(\theta,\dpa)
	  		.
	  	\end{equation}
	  \end{enumerate}
	\Nobs that \eqref{eq:setup:coefficients:and:supn} and the assumption that $\fc\geq 1$ \prove\ that for all $n\in\N$ it holds that
	\begin{equation}
		\label{eq:coeff:growth}
		\textstyle
		c_{n+1}=(1+\fc)^{n+1}\geq\frac{(1+\fc)^{n+1}-1}{\fc}=\sum_{j=0}^n(1+\fc)^j=\sum_{j=0}^nc_j
		\qqandqq
	\end{equation}
	\begin{equation}
		\label{eq:coeff:growth:1}
		\textstyle
		\const_{n+1}
		=\prb{2^{\fc+1}\pars\fc\vertiii{(1,a,b)}}^{c_{n+1}}
		\geq
		\prb{2^{\fc+1}\pars\fc\vertiii{(1,a,b)}}^{\sum_{j=0}^nc_j}
		\geq\prod_{j=0}^n\const_j.
	\end{equation}
	\Moreover the fundamental theorem of calculus \proves\ that for all $x\in[0,\infty)$ it holds that
	\begin{equation}
		\act(x)=\act(0)+\int_{0}^x \act'(y)\d y
		\qqandqq
		\act(-x)=\act(0)-\int_{-x}^0 \act'(y)\d y
	\end{equation}
	This and the assumption that for all $x\in\R$ it holds that $\vass{\act'(x)}\leq\fc(\vass{x}+1)^{\fc}$ \prove\ that for all $x\in\R$ it holds that
	\begin{equation}
		\label{eq:poly:act:growth}
		\textstyle
		\vass{\act\pr{x}}
		\leq\vass{\act(0)}+\vass{x}\sup_{y\in[-\vass{x},\vass{x}]}\vass{\act'(y)}
		\leq\fc+\vass{x}\fc(\vass{x}+1)^{\fc}
		\leq \fc(\vass{x}+1)^{\fc+1}
		.
	\end{equation}
		This and \eqref{eq:vect:ANN:real} \prove\ that for all $\theta\in\R^{\pars}$,
	$x\in[a,b]^\dim$, $k\in\{1,2,\dots,L\}$ it holds that
	\begin{equation}
		\begin{split}
		\max\pRb{\vertiii{\realization{\theta}_{k}(x)},1}
		&\leq
		\max\pRb{
		\vertiii{\theta}(l_{k-1}+1)\fc\prb{\vertiii{\realization{\theta}_{k-1}(x)}+1}^{\fc+1},1}
		\\&\leq
		\max\{\vertiii{\theta},1\}\pars\fc 2^{\fc+1}\max\prb{\pRb{\vertiii{\realization{\theta}_{k-1}(x)},1}}^{\fc+1}
		.
		\end{split}
	\end{equation}
	Combining this,
	\eqref{eq:setup:coefficients:and:supn},
	and
	\eqref{eq:coeff:growth} with induction \proves\ that for all $\theta\in\R^{\pars}$,
	$x\in[a,b]^\dim$, 
	$k\in\{1,2,\dots,L\}$ it holds that
	\begin{equation}
		\label{eq:realization:kth:layer:bound}
		\begin{split}
		\vertiii{\realization{\theta}_{k}(x)}
		&\leq 2^{\fc+1}\max\{\vertiii{\theta},1\}\pars\fc \max\prb{\pRb{\vertiii{\realization{\theta}_{k-1}(x)},1}}^{\fc+1}
		\\&\leq \pr{2^{\fc+1}\max\{\vertiii{\theta},1\}\pars\fc}^{\PR{\sum_{j=0}^{k-1}(\fc+1)^j}}\prb{\max\pRb{\vertiii{\realization{\theta}_{0}(x)},1}}^{((\fc+1)^k)}
		\\&= \pr{2^{\fc+1}\max\{\vertiii{\theta},1\}\pars\fc}^{\PR{\sum_{j=0}^{k-1}c_j}}\prb{\max\pRb{\vertiii{x},1}}^{c_k}
		\\&\leq \pr{2^{\fc+1}\max\{\vertiii{\theta},1\}\pars\fc}^{c_k}\pr{\max\{1,\vass{a},\vass{b}\}}^{c_k}
		\\&\leq \pr{2^{\fc+1}\pars\fc\max\{1,\vass{a},\vass{b}\}}^{c_k}\pr{\vertiii{\theta}+1}^{c_k}
		=
		\const_k\pr{\vertiii{\theta}+1}^{c_k}.
		\end{split}
	\end{equation}
	This, 
	\eqref{eq:setup:coefficients:and:supn},
	 and \eqref{eq:setup:backprop2} \prove\ that for all
	$\theta\in\R^\pars$,
	$x\in[a,b]^\dim$,
	$y\in[a,b]^\DIM$,
	$\dpa\in\{(x,y)\}$ it holds that 
	\begin{equation}
		\label{eq:UB:last:backprop:vect}
		\vertiii{D_{L+1}(\theta,\dpa)}
		=\vertiii{2\prb{\realization{\theta}_{L}(x)-y}}
		\leq2\vertiii{\realization{\theta}_{L}(x)}+2\vertiii{y}
		\leq 4 \const_L\pr{\vertiii{\theta}+1}^{c_L}.
	\end{equation}
	This and \eqref{eq:setup:backprop1} \prove\ that for all
	$\theta\in\R^\pars$,
	$x\in[a,b]^\dim$,
	$y\in[a,b]^\DIM$,
	$\dpa\in\{(x,y)\}$
	it holds that
		\begin{equation}
			\label{eq:pre:end:coefficients}
		\vertiii{D_{L}(\theta,\dpa)}
		\leq l_L\vertiii{\theta}\vertiii{D_{L+1}(\theta,\dpa)}
		\leq 4 \const_L\pars\pr{\vertiii{\theta}+1}^{1+c_L}.
	\end{equation}
	\Moreover 
	\eqref{eq:setup:coefficients:and:supn},
	\eqref{eq:setup:backprop1},
	\eqref{eq:realization:kth:layer:bound},
	 and the assumption that for all $x\in\R$ it holds that $\vass{\act'(x)}\leq\fc(\vass{x}+1)^{\fc}$ \prove\ that 
	for all 
	$\theta\in\R^\pars$,
	$x\in[a,b]^\dim$,
	$y\in[a,b]^\DIM$,
	$\dpa\in\{(x,y)\}$,
	$k\in\{1,2,\dots,L-1\}$
	it holds that
			\begin{equation}
				\begin{split}
		\vertiii{D_{k}(\theta,\dpa)}
		&\leq l_k\vertiii{\theta}\vertiii{\prb{
				\act'\pr{\realization{\theta}_{k,1}(x)},
				\act'\pr{\realization{\theta}_{k,2}(x)},\dots,
				\act'\pr{\realization{\theta}_{k,l_k}(x)}}}\vertiii{D_{k+1}(\theta,\dpa)}
		\\&\leq l_k\vertiii{\theta}\fc\prb{\vertiii{
				\realization{\theta}_{k}(x)}+1}^{\fc}\vertiii{D_{k+1}(\theta,\dpa)}
		\\&\leq \pars\vertiii{\theta}\fc\prb{\const_k\pr{\vertiii{\theta}+1}^{c_k}+1}^{\fc}\vertiii{D_{k+1}(\theta,\dpa)}
		\\&\leq (\const_k)^\fc 2^\fc\pars\vertiii{\theta}\fc\pr{\vertiii{\theta}+1}^{\fc c_k}\vertiii{D_{k+1}(\theta,\dpa)}
		\\&\leq (\const_k)^\fc \const_k\pr{\vertiii{\theta}+1}^{\fc c_k+1}\vertiii{D_{k+1}(\theta,\dpa)}
		\\&\leq (\const_k)^{\fc+1} \pr{\vertiii{\theta}+1}^{(\fc+1)c_{k}}\vertiii{D_{k+1}(\theta,\dpa)}
		= \const_{k+1}\pr{\vertiii{\theta}+1}^{ c_{k+1}}\vertiii{D_{k+1}(\theta,\dpa)}
			.
			\end{split}
	\end{equation}
	Combining this,
	\eqref{eq:coeff:growth},
	\eqref{eq:coeff:growth:1},
	\eqref{eq:pre:end:coefficients}, and the fact that $\min_{j\in\N_0}\const_j\geq1$ with induction \proves\ that for all 
	$\theta\in\R^\pars$,
	$x\in[a,b]^\dim$,
	$y\in[a,b]^\DIM$,
	$\dpa\in\{(x,y)\}$,
	$k\in\{1,2,\dots,L\}$ it holds that
	\begin{equation}
		\label{eq:UB:for:backprop:vect}
		\begin{split}
		\textstyle
		\vertiii{D_{k}(\theta,\dpa)}
		&\textstyle\leq \PRb{\prod_{j=k+1}^L \const_j\pr{\vertiii{\theta}+1}^{ c_{j}}}\vertiii{D_{L}(\theta,\dpa)}
		\\&\textstyle\leq \pr{\vertiii{\theta}+1}^{ \sum_{j=k+1}^L c_j} \PRb{\prod_{j=k+1}^L \const_j}4 \const_L\pars\pr{\vertiii{\theta}+1}^{ 1+c_{L}} 
		\\&\textstyle\leq \pr{\vertiii{\theta}+1}^{ \sum_{j=0}^L c_j} \PRb{\prod_{j=0}^L \const_j}4 \const_L\pars\pr{\vertiii{\theta}+1}^{ 1+c_{L}} 
		\\&\leq 4 \const_L\pars\pr{\vertiii{\theta}+1}^{ 1+c_L+c_{L+1}}  \const_{L+1}
		\leq  \const_{L+2}\pr{\vertiii{\theta}+1}^{ c_{L+2}} 
		.
		\end{split}
	\end{equation}
	\Moreover 
	\eqref{eq:setup:coefficients:and:supn},
	\eqref{eq:coeff:growth},
	\eqref{eq:coeff:growth:1},
	\eqref{eq:realization:kth:layer:bound}, and \eqref{eq:poly:act:growth} \prove\ that for all
	$\theta\in\R^\pars$,
	$x\in[a,b]^\dim$,
	$y\in[a,b]^\DIM$,
	$k\in\{1,2,\dots,L\}$, 
	$j\in\{1,2,\dots,l_{k}\}$ it holds that
	\begin{equation}
		\label{eq:UB:activation:kth:layer}
		\begin{split}
		\vass{\act\prb{\realization{\theta}_{k,j}(x)}}
		\leq
		\fc(\vass{\realization{\theta}_{k,j}(x)}+1)^{\fc+1}
		&\leq
		\fc(\const_k\pr{\vertiii{\theta}+1}^{c_k}+1)^{\fc+1}
		\\&\leq
		2^{\fc+1}\fc (\const_k)^{\fc+1}\pr{\vertiii{\theta}+1}^{c_k(\fc+1)}
		\\&\leq
		\const_{1}\const_{k+1}\pr{\vertiii{\theta}+1}^{c_{k+1}}
		\leq
		\const_{L+2}\pr{\vertiii{\theta}+1}^{c_{L+1}}
		.
		\end{split}
	\end{equation}
		\Moreover 
		\eqref{eq:setup:coefficients:and:supn},
		\eqref{eq:coeff:growth},
		\eqref{eq:realization:kth:layer:bound}, and 
		the assumption that for all $x\in\R$ it holds that $\vass{\act'(x)}\leq\fc(\vass{x}+1)^{\fc}$
		 \prove\ that for all
	$\theta\in\R^\pars$,
	$x\in[a,b]^\dim$,
	$y\in[a,b]^\DIM$,
	$k\in\{1,2,\dots,L\}$, 
	$j\in\{1,2,\dots,l_{k}\}$ it holds that
	\begin{equation}
		\begin{split}
			\vass{\act'\prb{\realization{\theta}_{k,j}(x)}}
			\leq
			\fc(\vass{\realization{\theta}_{k,j}(x)}+1)^{\fc}
			\leq
			\fc(\const_k\pr{\vertiii{\theta}+1}^{c_k}+1)^{\fc}
			&\leq
			2^\fc\fc (\const_k)^{\fc}\pr{\vertiii{\theta}+1}^{c_k\fc}
			\\&\leq
			\const_{k+1}\pr{\vertiii{\theta}+1}^{c_{k+1}}
			\\&\leq
			\const_{L+1}\pr{\vertiii{\theta}+1}^{c_{L+1}}
			.
		\end{split}
	\end{equation}
	Combining this,
	\eqref{eq:setup:coefficients:and:supn},
	\eqref{eq:UB:last:backprop:vect},
	\eqref{eq:UB:for:backprop:vect},
	\eqref{eq:UB:activation:kth:layer}
	with
	\cref{it:diff:loss,it:last:weight,it:last:bias,it:kth:bias,it:kth:weight} \proves\ that for all 
		$\theta\in\R^\pars$,
	$x\in[a,b]^\dim$,
	$y\in[a,b]^\DIM$ it holds that
	\begin{equation}
		\label{eq:bounded:grad:loss:ANN}
		\begin{split}
		\norm{\pr{\nabla_{\theta}\loss}(\theta,(x,y))}
		\leq \sqrt{\pars}\vertiii{\pr{\nabla_{\theta}\loss}(\theta,(x,y))}
		&=\textstyle
		\sqrt{\pars}\max_{p\in\{1,2,\dots,\pars\}}\vass{\tfrac{\partial\loss}{\partial\theta_{p}}(\theta,(x,y))}
		\\&\leq\textstyle
		\sqrt{\pars}
		\const_{L+1}\pr{\const_{L+2}}^2\pr{\vertiii{\theta}+1}^{2c_{L+1}+c_{L+2}}
		.
		\end{split}
	\end{equation}
	\Moreover
	\eqref{eq:vect:ANN:real},
	\eqref{eq:realization:kth:layer:bound},
	and the fact that $s_{L-1}+l_{L}l_{L-1}=\pars-\DIM$
	\prove\ 
	that 
	for every random variable $Z=(Z_1,\dots,Z_\DIM)\colon\Omega\to[a,b]^\DIM$ and
	for all $\theta=(\theta_1,\dots,\theta_\pars),\vartheta=(\vartheta_1,\dots,\vartheta_\pars)\in\R^{\pars}$ which satisfy for all $i\in\{1,2,\dots,\pars-\DIM\}$, $j\in\{1,2,\dots,\DIM\}$ that
	\begin{equation}
		\vartheta_i=\theta_i
		\qqandqq
		\vartheta_{\pars-\DIM+j}
		\textstyle
		=
		\theta_{\pars-\DIM+j}-
		\E\PRb{\realization{\theta}_{L,j}(X_{1,1})
			-Z_{j}
		}
	\end{equation}
	it holds for all $k\in\{0,1,\dots,L-1\}$, $j\in\{1,2,\dots,\DIM\}$ that
	\begin{equation}
		\realization{\vartheta}_{k}=\realization{\theta}_{k}
		\qqandqq
		\realization{\vartheta}_{L,j}=\realization{\theta}_{L,j}-\E\PRb{\realization{\theta}_{L,j}(X_{1,1})-Z_j}.
	\end{equation}
	This and \eqref{eq:vect:ANN:real} \prove\ that 
	for every random variable $Z=(Z_1,\dots,Z_\DIM)\colon\Omega\to[a,b]^\DIM$ and
	for all $\theta\in\R^{\pars}$ there exists $\vartheta\in\R^{\pars}$ such that for all $j\in\{1,2,\dots,\DIM\}$ it holds that
	$\E\PRb{\realization{\vartheta}_{L,j}(X_{1,1})-Z_j}=0$
	and
	\begin{equation}
		\begin{split}
			\textstyle
			\E\PRb{\vass{\realization{\theta}_{L,j}(X_{1,1})-Z_j}^2}
			&\geq
			\var\prb{\realization{\theta}_{L,j}(X_{1,1})-Z_j}
			\\&=\var\prb{\realization{\vartheta}_{L,j}(X_{1,1})+\E\PRb{\realization{\theta}_{L,j}(X_{1,1})-Z_j}-Z_j}
			\\&=
			\textstyle
			\var\prb{\realization{\vartheta}_{L,j}(X_{1,1})-Z_j}
			=
			\E\PRb{\vass{\realization{\vartheta}_{L,j}(X_{1,1})-Z_j}^2}
			.
		\end{split}
	\end{equation}
		This \proves\ that for all $\theta\in\R^{\pars}$ there exists $\vartheta\in\R^{\pars}$ with  $\E\PRb{\realization{\vartheta}_{L}(X_{1,1})-Y_{1,1}}=0$ such that 
	\begin{equation}
		\begin{split}
			\textstyle
			\E\PRb{\norm{\realization{\theta}_{L}(X_{1,1})-Y_{1,1}}^2}
			\geq
			\operatorname{Trace}( \operatorname{Cov}\pr{\realization{\vartheta}_{L}(X_{1,1})-Y_{1,1}})
			=
			\E\PRb{\norm{\realization{\vartheta}_{L}(X_{1,1})-Y_{1,1}}^2}
			.
		\end{split}
	\end{equation}
	This and \eqref{eq:ANN:EX:properties:of:setup:1} \prove\ that
	\begin{equation}
		\begin{split}
			\textstyle
			0<
			\textstyle\inf_{\theta\in\R^{\pars}}\E\PRb{\norm{\realization{\theta}_{L}(X_{1,1})-Y_{1,1}}^2}
			=
			\textstyle
			\inf_{\theta\in\R^\pars}
			\operatorname{Trace}( \operatorname{Cov}\pr{\realization{\theta}_{L}(X_{1,1})-Y_{1,1}})
			.
		\end{split}
	\end{equation}
	This,
	\eqref{eq:nom:EX1:0:1},
	\eqref{eq:setup:backprop2},
	 and \cref{it:last:bias} \prove\ that 
	\begin{equation}
		\begin{split}
			&\textstyle\inf_{n\in\N,\theta\in\R^\pars}
			\textstyle\prb{\batch_n\sum_{i=1}^\pars\var\prb{G_n^{(i)}\pr{\theta,\DPA_n}}}
			\\&\textstyle=
			\inf_{n\in\N,\theta\in\R^\pars}
			\textstyle\prb{\sum_{i=1}^\pars\batch_n\var\prb{			\frac{1}{\batch_{n}}\PRb{\sum_{j=1}^{\batch_{n}}\prb{\tfrac{\partial\loss}{\partial\theta_{i}}}(\theta,(X_{n,j},Y_{n,j}))}}}
			\\&\textstyle=
			\inf_{\theta\in\R^\pars}
			\textstyle\prb{\sum_{i=1}^\pars\var\prb{			\prb{\tfrac{\partial\loss}{\partial\theta_{i}}}(\theta,(X_{1,1},Y_{1,1}))}}
			\\&\textstyle\geq
			\inf_{\theta\in\R^\pars}
			\textstyle\prb{\sum_{i=\pars-\DIM+1}^\pars\var\prb{			\prb{\tfrac{\partial\loss}{\partial\theta_{i}}}(\theta,(X_{1,1},Y_{1,1}))}}
			\\&\textstyle=
			\inf_{\theta\in\R^\pars}
			\textstyle\prb{\sum_{i=\pars-\DIM+1}^\pars\var\prb{2\prb{\realization{\theta}_{L,i-\pars+\DIM}(X_{1,1})-Y_{1,1}^{(i-\pars+\DIM)}}}}
		\\&\textstyle=
		4\inf_{\theta\in\R^\pars}\operatorname{Trace}\prb{ \operatorname{Cov}\pr{\realization{\theta}_{L}(X_{1,1})-Y_{1,1}}}
		>0.
		\end{split}
	\end{equation}
	This, \eqref{eq:ANN:EX:loss:1}, \eqref{eq:nom:EX1:0:1},
	\eqref{eq:nom:EX1:1:1}, \eqref{eq:EX:recursion:1},
	\eqref{eq:bounded:grad:loss:ANN},
	\cref{it:filtered:prob:space:simple,it:filter:independence:simple,it:filter:measurable:simple,it:filter:measurable2:simple}, and
	\cref{cor:limsup:simple}
	(applied with
	$q\curvearrowleft 1$,
	$\dim\curvearrowleft\dim+\DIM$,
	$\const\curvearrowleft\sqrt{\pars}\const_{L+1}(\const_{L+2})^2$,
	$p\curvearrowleft 2c_{L+1}+c_{L+2}$, and
	$(X_n)_{n\in\N}\curvearrowleft(\DPA_n)_{n\in\N}$
	in the notation of \cref{cor:limsup:simple})
	\prove\ that
	\begin{equation}
		\inf_{\substack{\xi\colon\Omega\to\R^\pars\\\text{measurable}}}
		\PRbbb{\limsup_{n\to\infty}
			\E\PRb{\norm{\Theta_n-\xi}^{2}}}>0.
	\end{equation}
	This \proves\ \eqref{eq:SANDBOX:RESULT:1}.
	\finishproofthus
\end{proof}

\section{Non-convergence for specific SGD optimization methods}
\label{section:4}

In this section we utilize the results of \cref{sec:results} to explicitly formulate non-convergence results for various optimization methods.
In particular, we employ \cref{prop:SGD:mom} to establish non-convergence results for 
the \SGD\ optimizer in \cref{cor:SGD}, 
the momentum optimizer in \cref{cor:SGD:mom}, and
the simplified Nesterov accelerated optimizer in \cref{cor:Nesterov}.
For each component satisfying the respective hypotheses,
\cref{cor:ADAM} and \cref{cor:AMSgrad} establish non-convergence in the expected absolute
error to any random variable.
Under stronger assumptions,
\cref{cor:ADAM} also establishes non-accumulation in the sense that
this kind of non-convergence holds along every subsequence of the
iterates.
We utilize \cref{thm:new:rob:cond} to show non-convergence results for 
the \ADAM\ optimizer in \cref{cor:ADAM}
and the AMSGrad optimizer in \cref{cor:AMSgrad}.
Moreover, we employ \cref{cor:limsup:simple} to prove non-convergence results for 
the \ADAM\ optimizer in \cref{cor:ADAM:New},
the \RMSprop\ optimizer in \cref{cor:RMSprop:New},
and the Adamax optimizer in \cref{cor:AdaMax:New}.

\Cref{prop:prop:non_convergence_modified_Adam:easy:mat} is an immediate consequence of \cref{cor:ADAM:New} and shows non-convergence when \ADAM\ is applied to a simple class of quadratic optimization problems.
\Cref{prop:prop:non_convergence_modified_Adam:easy:mat}, in turn, implies \cref{whole:vec:thm:sup}.
We note that \cref{whole:vec:thm:sup} covers \ADAM\ with the first moment bias adjustment although this is not explicitly stated and to illustrate this point better to the reader we also present the corresponding result in \cref{prop:weak:opt:error:non-cvg}.

\subsection{Standard SGD optimization}

\cfclear
\begin{athm}{prop}{cor:SGD}
	Let $\pars,\dim\in\N$, 
	$\fc,a\in\R$,
	$b\in[a,\infty)$,
	$\Cst\in(0,\infty)$, $z\in\{-1,1\}$,
	let $\gamma\colon\N\to[0,\infty)$ and $\batch\colon\N\to\N$ satisfy
	\begin{equation}
		\label{eq:setup:cor:SGD}
		\textstyle
		z\PRb{\limsup_{n\to\infty}\prb{z(\gamma_n)^2\pr{\batch_n}^{-1}}}>0,
	\end{equation}
	let $g=(g_1,\dots,g_{\pars})\colon\R^{\pars}\times[a,b]^{\dim}\to\R^{\pars}$ be measurable, 
	let $ ( \Omega, \cF, \P ) $ be a probability space,
	let $X_{n,j}\colon\Omega\to[a,b]^\dim$, $(n,j)\in\N^2$, be \iid\ random variables,
	let 
	$ \Theta=(\Theta^{(1)},\ldots,\Theta^{(\pars)}) \colon\N_0\times\Omega \to \R^{\pars}$ be a stochastic process which 
	satisfies
	for all $n\in\N$ that
	\begin{equation}
		\label{eq:incr:rule:cor:SGD}
		\begin{split}
			\textstyle
			\Theta_{ n }
			&\textstyle=\Theta_{ n-1  }-\gamma_{n}  \PRbb{\frac{1}{\batch_{n}}\PRb{\sum_{j=1}^{\batch_{n}}g(\Theta_{n-1},X_{n,j})}},
		\end{split}
	\end{equation}
	assume that $\Theta_0$ and $\pr{X_{n,m}}_{(n,\mm)\in\{(k,j)\in\N^2\colon j\leq\batch_k\}}$ are independent,
	and let $i\in\{1,2,\dots,\pars\}$ satisfy for all $\theta=(\theta_1,\dots,\theta_\pars)\in \R^{\pars}$, $x\in[a,b]^d$ that
	\begin{equation}
		\label{eq:setup:starting:expectations:bounded:cor:SGD}
		\E\PRb{\vass{\Theta_0^{(i)}}}
		<\infty,
		\qquad
		\textstyle
		\vass{g_i(\theta,x)}\leq \Cst(\vass{\theta_i}+\fc),
		\qqandqq
		\inf_{\vartheta\in\R^{\pars}}\var\pr{g_i(\vartheta,X_{1,1})}>0
		.
	\end{equation}
	Then 
	\begin{equation}
		\label{eq:SGD:LB}
		\begin{split}
			\inf_{\substack{\xi\colon \Omega \to \R\\\text{measurable}}}
			\prbbb{z\PRbbb{\limsup_{n\to\infty}\prb{z\E\PRb{\vass{\Theta_n^{(i)}-\xi}^2}}}}
			>
			0
			.
		\end{split}
	\end{equation}
\end{athm}

\begin{proof}[Proof of \cref{cor:SGD}]
	Throughout this proof 
	for every $n\in\N$ let $X_n=(X_{n,1},\dots,X_{n,\batch_n})$ and
	for every $n\in\N$ let
	$G_{n}\colon\R^\pars \times\pr{[a,b]^{\dim}}^{\batch_n}\to\R^\pars$ satisfy for all $\theta\in\R^{\pars}$, $x=(x_1,\dots,x_{\batch_n})\in\pr{[a,b]^{\dim}}^{\batch_n}$ that
	\begin{equation}
		\label{eq:representation:SGD:gradient}
		\textstyle
		G_n(\theta,x)=\PRb{\frac{1}{\batch_{n}}\sum_{j=1}^{\batch_{n}}g(\theta,x_j)}
		.
	\end{equation}
	\Nobs that \cref{lem:end:filtration} 
	(applied with 
	$(X_{n,j})_{(n,j)\in\{(m,i)\in\N^2\colon i\leq\BATCH_m\}}\curvearrowleft\pr{X_{n,j}}_{(n,j)\in\{(m,i)\in\N^2\colon i\leq\batch_m\}}$,
	$\batch\curvearrowleft\batch$,
	$\Batch\curvearrowleft\batch$,
	$\BATCH\curvearrowleft\batch$,
	$Y\curvearrowleft \Theta_0$ in the notation of \cref{lem:end:filtration}) \prove\ that there exist $\fil_n\subseteq\cF$, $n\in\N_0$, which satisfy that
	\begin{enumerate}[label=(\Roman*)]
		\item \label{it:filtered:prob:space:SGD}it holds that $ ( \Omega, \cF,(\mathbb{F}_n)_{n\in\N_0}, \P ) $ is a filtered probability space,
		\item \label{it:filter:independence:SGD}it holds for all $n\in\N$ that
		 $(X_{n,1},X_{n,2},\dots,X_{n,\batch_n})$ is independent of $\fil_{n-1}$,
		\item \label{it:filter:measurable:SGD}it holds for all $n\in\N$ that $(X_{n,1},X_{n,2},\dots,X_{n,\batch_n})$ is $\fil_n$-measurable, and
		\item \label{it:filter:measurable2:SGD}it holds that $\Theta_0$ is $\fil_0$-measurable.
	\end{enumerate}
	\Nobs that \eqref{eq:incr:rule:cor:SGD} and \eqref{eq:representation:SGD:gradient} \prove\ for all $n\in\N$ that
	\begin{equation}
		\begin{split}
			\textstyle
			\Theta_{ n }
			&\textstyle=\Theta_{ n-1  }-\gamma_n\PRb
			{\frac{1}{\batch_{n}}\sum_{j=1}^{\batch_{n}}g(\Theta_{n-1},X_{n,j})}
			\\&\textstyle=\Theta_{ n-1  }-\gamma_n\PRb
			{G_n(\Theta_{n-1},X_n)}
			\\&\textstyle=\Theta_{ n-1  }-\gamma_n
			\PRb{\sum_{k=1}^nG_k(\Theta_{k-1},X_k)\indicator{\{n\}}(k)}
			.
		\end{split}
	\end{equation}
	Combining 
	\cref{it:SGD:mom:lims} in
	\cref{prop:SGD:mom}
	(applied with 
	$\pars \curvearrowleft \pars$,
	$\dim \curvearrowleft \dim$,
	$z\curvearrowleft z$,
	$\gamma \curvearrowleft \gamma$,
	$\batch\curvearrowleft \batch$,
	$(\alpha_{n,k})_{(n,k)\in(\N_0)^2} \curvearrowleft 
	\pr{(\N_0)^2\ni(n,k)\mapsto\indicator{\{n\}}(k)\in[0,\infty)}$,
	$(G_n)_{n\in\N} \curvearrowleft (G_n)_{n\in\N}$,
	$ ( \Omega, \cF,(\mathbb{F}_n)_{n\in\N_0}, \P ) \curvearrowleft( \Omega, \cF,(\mathbb{F}_n)_{n\in\N_0}, \P ) $,
	$(X_n)_{n\in\N}\curvearrowleft (X_{n,1},X_{n,2},\dots,X_{n,\batch_n})_{n\in\N}$,
	$\mom \curvearrowleft \prb{\Omega\ni\omega\mapsto (1,\dots,1)\in\R^\pars}$,
	$\Theta \curvearrowleft \Theta$,
	$i \curvearrowleft i$
	in the notation of \cref{prop:SGD:mom}),
	\eqref{eq:setup:cor:SGD},
	\eqref{eq:setup:starting:expectations:bounded:cor:SGD}, and
	\eqref{eq:representation:SGD:gradient}
	with
	\cref{it:filtered:prob:space:SGD,it:filter:independence:SGD,it:filter:measurable:SGD,it:filter:measurable2:SGD}
	hence \proves\ that
	\begin{equation}
		\begin{split}
			\inf_{\substack{\xi\colon \Omega \to \R\\\text{measurable}}}
			\prbbb{z\PRbbb{\limsup_{n\to\infty}\prb{z\E\PRb{\vass{\Theta_n^{(i)}-\xi}^2}}}}
			>
			0
			.
		\end{split}
	\end{equation}
	This \proves\ \cref{eq:SGD:LB}. \finishproofthus
\end{proof}

\subsection{Momentum optimization}

\cfclear
\begin{athm}{prop}{cor:SGD:mom}
	Let $\pars,\dim\in\N$, 
	$\fc,a\in\R$,
	$b\in[a,\infty)$,
	$\Cst\in(0,\infty)$, $z\in\{-1,1\}$,
	let $\gamma\colon\N\to[0,\infty)$,
	 $\batch\colon\N\to\N$, and
	 $\beta\colon\N\to[0,1]$ satisfy
	\begin{equation}
		\label{eq:setup:cor:SGD:mom}
		\textstyle
		z\PRb{\limsup_{n\to\infty}\prb{z((1-\beta_n)\gamma_n)^2\pr{\batch_n}^{-1}}}>0,
	\end{equation}
	let $g=(g_1,\dots,g_{\pars})\colon\R^{\pars}\times[a,b]^{\dim}\to\R^{\pars}$ be measurable, 
	let $ ( \Omega, \cF, \P ) $ be a probability space,
	let $X_{n,j}\colon\Omega\to[a,b]^\dim$, $(n,j)\in\N^2$, be \iid\ random variables,
	let 
	$\mom=(\mom^{(1)},\ldots,\mom^{(\pars)})\colon\N_0\times\Omega\to\R^{\pars}$
	and
	$ \Theta \colon\N_0\times\Omega \to \R^{\pars}$
	be stochastic processes which 
	satisfy
	for all $n\in\N$ that
		\begin{equation}
	\label{eq:representation:SGD:mom:1}
	\begin{split}
		\mom[n]&\textstyle=\beta_n \mom[n-1]+(1-\beta_n)\PRb
		{\frac{1}{\batch_{n}}\sum_{j=1}^{\batch_{n}}g(\Theta_{n-1},X_{n,j})}
	\end{split}
\end{equation}
\begin{equation}
	\label{eq:representation:SGD:mom:2}
	\begin{split}
		\textstyle
		\qqandqq
		\Theta_{ n }
		=\Theta_{ n-1  }-\gamma_n\mom[n]
		,
	\end{split}
\end{equation}
	let $i\in\{1,2,\dots,\pars\}$, 
	assume that $(\Theta_0,\mom[0])$ and $\pr{X_{n,m}}_{(n,\mm)\in\{(k,j)\in\N^2\colon j\leq\batch_k\}}$ are independent,
	assume for all $\theta=(\theta_1,\dots,\theta_\pars)\in \R^{\pars}$, $x\in[a,b]^d$ that
	\begin{equation}
		\label{eq:setup:starting:expectations:bounded:cor:SGD:mom}
				\max\pRb{
			\E\PRb{\vass{\Theta_0^{(i)}}},
			\E\PRb{\vass{\mom[0]^{(i)}}}
		}
		<\infty
		\qqandqq
		\textstyle
		\vass{g_i(\theta,x)}\leq \Cst(\vass{\theta_i}+\fc),
		\end{equation}
		and assume $\inf_{\theta\in\R^{\pars}}\var\pr{g_i(\theta,X_{1,1})}>0$.
	Then 
	\begin{equation}
		\label{eq:SGD:mom:LB}
		\begin{split}
			\inf_{\substack{\xi\colon \Omega \to \R\\\text{measurable}}}
			\prbbb{z\PRbbb{\limsup_{n\to\infty}\prb{z\E\PRb{\vass{\Theta_n^{(i)}-\xi}^2}}}}
			>
			0
			.
		\end{split}
	\end{equation}
\end{athm}

\begin{proof}[Proof of \cref{cor:SGD:mom}]
	Throughout this proof
	for every $n\in\N$ let $X_n=(X_{n,1},\dots,X_{n,\batch_n})$, 
	for every $n\in\N$ let
	$G_{n}\colon\R^\pars \times\pr{[a,b]^{\dim}}^{\batch_n}\to\R^\pars$ satisfy for all $\theta\in\R^{\pars}$, $x=(x_1,\dots,x_{\batch_n})\in\pr{[a,b]^{\dim}}^{\batch_n}$ that
	\begin{equation}
		\label{eq:representation:SGD:mom:gradient}
		\textstyle
		G_n(\theta,x)=\PRb{\frac{1}{\batch_{n}}\sum_{j=1}^{\batch_{n}}g(\theta,x_j)},
	\end{equation}
	and let $\alpha\colon(\N_0)^2\to[0,\infty)$ satisfy for all $n,k\in\N_0$ that
	\begin{equation}
		\alpha_{n,k}=
		\begin{cases}
			(1-\beta_k)\PRb{\prod_{j=k+1}^{n}\beta_j}&\colon 0<k\leq n\\
			\PRb{\prod_{j=1}^{n}\beta_j}&\colon k=0\\
			0&\colon n<k.
		\end{cases}
	\end{equation}
	\Nobs that \cref{lem:end:filtration} 
	(applied with 
	$(X_{n,j})_{(n,j)\in\{(m,i)\in\N^2\colon i\leq\BATCH_m\}}\curvearrowleft\pr{X_{n,j}}_{(n,j)\in\{(m,i)\in\N^2\colon i\leq\batch_m\}}$,
	$Y\curvearrowleft (\Theta_0,\mom[0])$,
	$\batch\curvearrowleft\batch$,
	$\Batch\curvearrowleft\batch$,
	$\BATCH\curvearrowleft\batch$
	 in the notation of \cref{lem:end:filtration}) \prove\ that there exist $\fil_n\subseteq\cF$, $n\in\N_0$, which satisfy that
	\begin{enumerate}[label=(\Roman*)]
		\item \label{it:filtered:prob:space:SGD:mom}it holds that $ ( \Omega, \cF,(\mathbb{F}_n)_{n\in\N_0}, \P ) $ is a filtered probability space,
		\item \label{it:filter:independence:SGD:mom}it holds for all $n\in\N$ that $(X_{n,1},X_{n,2},\dots,X_{n,\batch_n})$ is independent of $\fil_{n-1}$,
		\item \label{it:filter:measurable:SGD:mom}it holds for all $n\in\N$ that $(X_{n,1},X_{n,2},\dots,X_{n,\batch_n})$ is $\fil_n$-measurable, and
		\item \label{it:filter:measurable2:SGD:mom}it holds that $(\Theta_0,\mom[0])$ is $\fil_0$-measurable.
	\end{enumerate}
	\Nobs that \eqref{eq:representation:SGD:mom:1},
	\eqref{eq:representation:SGD:mom:gradient}, and \cite[Lemma~2.7]{DereichGraeberJentzenRiekert2025}
	(applied with
	$\pars\curvearrowleft\pars$,
	$\beta\curvearrowleft\beta$,
	$\mom\curvearrowleft\mom$,
	$(g_n)_{n\in\N}\curvearrowleft \pr{\frac{1-\beta_n}{\batch_{n}}\sum_{j=1}^{\batch_{n}}g(\Theta_{n-1},X_{n,j})}_{n\in\N}$
	in the notation of \cite[Lemma~2.7]{DereichGraeberJentzenRiekert2025})
	\prove\ that for all $n\in\N$ it holds that
	\begin{equation}
		\begin{split}
			\mom[n]
			&\textstyle=\PRb{\prod_{j=1}^{n}\beta_j}\mom[0]
			+\sum_{k=1}^{n}\PRb{\prod_{j=k+1}^{n}\beta_j}\PRb{\frac{1-\beta_k}{\batch_{k}}\sum_{j=1}^{\batch_{k}}g(\Theta_{k-1},X_{k,j})}
			\\&\textstyle=\PRb{\prod_{j=1}^{n}\beta_j}\mom[0]
			+\sum_{k=1}^{n}(1-\beta_k)\PRb{\prod_{j=k+1}^{n}\beta_j}
			G_k\pr{\Theta_{k-1},X_k}
			\\&\textstyle=\alpha_{n,0}\mom[0]
			+\sum_{k=1}^{n}\alpha_{n,k}
			G_k\pr{\Theta_{k-1},X_k}
			.
		\end{split}
	\end{equation}
	This, \eqref{eq:representation:SGD:mom:2}, and \eqref{eq:representation:SGD:mom:gradient} \prove\ for all $n\in\N$ that
	\begin{equation}
		\begin{split}
		\textstyle
		\Theta_{ n }&=\Theta_{ n-1  }-\gamma_n\mom[n]
		\textstyle=\Theta_{ n-1  }-\gamma_n\prb{\alpha_{n,0}\mom[0]
			+\sum_{k=1}^{n}\alpha_{n,k}
			G_k\pr{\Theta_{k-1},X_k}}
		.
		\end{split}
	\end{equation}
	Combining 
	\cref{it:SGD:mom:lims} in
	\cref{prop:SGD:mom}
	(applied with 
	$\pars \curvearrowleft \pars$,
	$\dim \curvearrowleft \dim$,
	$z\curvearrowleft z$,
	$\gamma \curvearrowleft \gamma$,
	$\batch\curvearrowleft \batch$,
	$(\alpha_{n,k})_{(n,k)\in(\N_0)^2} \curvearrowleft 
	(\alpha_{n,k})_{(n,k)\in(\N_0)^2}$,
	$(G_n)_{n\in\N} \curvearrowleft (G_n)_{n\in\N}$,
	$\mom \curvearrowleft \mom[0]$,
	$ ( \Omega, \cF,(\mathbb{F}_n)_{n\in\N_0}, \P ) \curvearrowleft( \Omega, \cF,(\mathbb{F}_n)_{n\in\N_0}, \P ) $,
	$(X_n)_{n\in\N}\curvearrowleft (X_{n,1},X_{n,2},\dots,X_{n,\batch_n})_{n\in\N}$,
	$\Theta \curvearrowleft \Theta$,
	$i \curvearrowleft i$
	in the notation of \cref{prop:SGD:mom}),
	\eqref{eq:setup:cor:SGD:mom},
	\eqref{eq:setup:starting:expectations:bounded:cor:SGD:mom}, and
	\eqref{eq:representation:SGD:mom:gradient}
	with
	\cref{it:filtered:prob:space:SGD:mom,it:filter:independence:SGD:mom,it:filter:measurable:SGD:mom,it:filter:measurable2:SGD:mom}
	hence \proves\ that
	\begin{equation}
		\begin{split}
			\inf_{\substack{\xi\colon \Omega \to \R\\\text{measurable}}}
			\prbbb{z\PRbbb{\limsup_{n\to\infty}\prb{z\E\PRb{\vass{\Theta_n^{(i)}-\xi}^2}}}}
			>
			0
			.
		\end{split}
	\end{equation}
	This \proves\ \cref{eq:SGD:mom:LB}. \finishproofthus
\end{proof}

\subsection{Simplified Nesterov accelerated optimization}

\cfclear
\begin{athm}{prop}{cor:Nesterov}
	Let $\pars,\dim\in\N$, 
	$\fc,a\in\R$,
	$b\in[a,\infty)$,
	$\Cst\in(0,\infty)$, $z\in\{-1,1\}$,
	let $\gamma\colon\N\to[0,\infty)$,
	$\batch\colon\N\to\N$, and
	$\beta\colon\N\to[0,1]$ satisfy
	\begin{equation}
		\label{eq:setup:cor:Nesterov}
		\textstyle
		z\PRb{\limsup_{n\to\infty}\prb{z(\gamma_n)^2\pr{\batch_n}^{-1}}}>0,
	\end{equation}
	let $g=(g_1,\dots,g_{\pars})\colon\R^{\pars}\times[a,b]^{\dim}\to\R^{\pars}$ be measurable, 
	let $ ( \Omega, \cF, \P ) $ be a probability space,
	let $X_{n,j}\colon\Omega\to[a,b]^\dim$, $(n,j)\in\N^2$, be \iid\ random variables,
	let 
	$\mom=(\mom^{(1)},\ldots,\mom^{(\pars)})\colon\N_0\times\Omega\to\R^{\pars}$
	and
	$ \Theta \colon\N_0\times\Omega \to \R^{\pars}$
	be stochastic processes which satisfy
	for all $n\in\N$ that
	\begin{equation}
		\label{eq:representation:Nesterov:1}
		\begin{split}
			\mom[n]&\textstyle=\beta_n \mom[n-1]+\PRb
			{\frac{1}{\batch_{n}}\sum_{j=1}^{\batch_{n}}g(\Theta_{n-1},X_{n,j})}
		\end{split}
	\end{equation}
	\begin{equation}
		\label{eq:representation:Nesterov:2}
		\begin{split}
			\textstyle
			\qqandqq
			\Theta_{ n }
			=\Theta_{ n-1  }-\gamma_n\beta_n\mom[n]-\gamma_n\PRb
			{\frac{1}{\batch_{n}}\sum_{j=1}^{\batch_{n}}g(\Theta_{n-1},X_{n,j})}
			,
		\end{split}
	\end{equation}
	assume that $(\Theta_0,\mom[0])$ and $\pr{X_{n,m}}_{(n,\mm)\in\{(k,j)\in\N^2\colon j\leq\batch_k\}}$ are independent,
	let $i\in\{1,2,\dots,\pars\}$ satisfy for all $\theta=(\theta_1,\dots,\theta_\pars)\in \R^{\pars}$, $x\in[a,b]^d$ that
	\begin{equation}
		\label{eq:setup:starting:expectations:bounded:cor:Nesterov}
		\max\pRb{
			\E\PRb{\vass{\Theta_0^{(i)}}},
		\E\PRb{\vass{\mom[0]^{(i)}}}
		}
		<\infty
		\qqandqq
		\textstyle
		\vass{g_i(\theta,x)}\leq \Cst(\vass{\theta_i}+\fc),
	\end{equation}
	and assume $\inf_{\theta\in\R^{\pars}}\var\pr{g_i(\theta,X_{1,1})}>0$.
	Then 
	\begin{equation}
		\label{eq:Nesterov:LB}
		\begin{split}
			\inf_{\substack{\xi\colon \Omega \to \R\\\text{measurable}}}
			\prbbb{z\PRbbb{\limsup_{n\to\infty}\prb{z\E\PRb{\vass{\Theta_n^{(i)}-\xi}^2}}}}
			>
			0
			.
		\end{split}
	\end{equation}
\end{athm}

\begin{proof}[Proof of \cref{cor:Nesterov}]
	Throughout this proof 
	for every $n\in\N$ let $X_n=(X_{n,1},\dots,X_{n,\batch_n})$, 
	for every $n\in\N$ let
	$G_{n}\colon\R^\pars \times\pr{[a,b]^{\dim}}^{\batch_n}\to\R^\pars$ satisfy for all $\theta\in\R^{\pars}$, $x=(x_1,\dots,x_{\batch_n})\in\pr{[a,b]^{\dim}}^{\batch_n}$ that
	\begin{equation}
		\label{eq:representation:Nesterov:gradient}
		\textstyle
		G_n(\theta,x)=\PRb{\frac{1}{\batch_{n}}\sum_{j=1}^{\batch_{n}}g(\theta,x_j)},
	\end{equation}
	and let $\alpha\colon(\N_0)^2\to[0,\infty)$ satisfy for all $n,k\in\N_0$ that
	\begin{equation}
		\alpha_{n,k}=
		\begin{cases}
			\beta_n\PRb{\prod_{j=k+1}^{n}\beta_j}&\colon 0\leq k< n\\
			(1+\beta_n)&\colon 0<n=k\\
			0&\colon (n<k)\vee(n=k=0).
		\end{cases}
	\end{equation}
	\Nobs that \eqref{eq:setup:cor:Nesterov} \proves\ that
	\begin{equation}
		\label{eq:setup:cor:LB:rates:Nesterov}
		\textstyle
		z\PRb{\limsup_{n\to\infty}\prb{z((1+\beta_n)\gamma_n)^2\pr{\batch_n}^{-1}}}
		\geq
		z\PRb{\limsup_{n\to\infty}\prb{z(\gamma_n)^2\pr{\batch_n}^{-1}}}
		>0
		.
	\end{equation}
	\Moreover \cref{lem:end:filtration} 
	(applied with 
	$\batch\curvearrowleft\batch$,
	$\Batch\curvearrowleft\batch$,
	$\BATCH\curvearrowleft\batch$,
	$Y\curvearrowleft (\Theta_0,\mom[0])$,
	$(X_{n,j})_{(n,j)\in\{(m,i)\in\N^2\colon i\leq\BATCH_m\}}\curvearrowleft\pr{X_{n,j}}_{(n,j)\in\{(m,i)\in\N^2\colon i\leq\batch_m\}}$ in the notation of \cref{lem:end:filtration}) \prove\ that there exist $\fil_n\subseteq\cF$, $n\in\N_0$, which satisfy that
	\begin{enumerate}[label=(\Roman*)]
		\item \label{it:filtered:prob:space:Nesterov}it holds that $ ( \Omega, \cF,(\mathbb{F}_n)_{n\in\N_0}, \P ) $ is a filtered probability space,
		\item \label{it:filter:independence:Nesterov}it holds for all $n\in\N$ that 
		$(X_{n,1},X_{n,2},\dots,X_{n,\batch_n})$ is independent of
		$\fil_{n-1}$,
		\item \label{it:filter:measurable:Nesterov}it holds for all $n\in\N$ that $(X_{n,1},X_{n,2},\dots,X_{n,\batch_n})$ is $\fil_n$-measurable, and
		\item \label{it:filter:measurable2:Nesterov}it holds that $(\Theta_0,\mom[0])$ is $\fil_0$-measurable.
	\end{enumerate}
	\Nobs that \eqref{eq:representation:Nesterov:1},
	\eqref{eq:representation:Nesterov:gradient}, and \cite[Lemma~2.7]{DereichGraeberJentzenRiekert2025}
	(applied with
	$\pars\curvearrowleft\pars$,
	$\beta\curvearrowleft\beta$,
	$\mom\curvearrowleft\mom$,
	$(g_n)_{n\in\N}\curvearrowleft \pr{\frac{1}{\batch_{n}}\sum_{j=1}^{\batch_{n}}g(\Theta_{n-1},X_{n,j})}_{n\in\N}$
	in the notation of \cite[Lemma~2.7]{DereichGraeberJentzenRiekert2025})
	\prove\ that for all $n\in\N$ it holds that
	\begin{equation}
		\begin{split}
			\mom[n]
			&\textstyle=\PRb{\prod_{j=1}^{n}\beta_j}\mom[0]
			+\sum_{k=1}^{n}\PRb{\prod_{j=k+1}^{n}\beta_j}\PRb{\frac{1}{\batch_{k}}\sum_{j=1}^{\batch_{k}}g(\Theta_{k-1},X_{k,j})}
			\\&\textstyle=\PRb{\prod_{j=1}^{n}\beta_j}\mom[0]
			+\sum_{k=1}^{n}\PRb{\prod_{j=k+1}^{n}\beta_j}
			G_k\pr{\Theta_{k-1},X_k}
			.
		\end{split}
	\end{equation}
	This, \eqref{eq:representation:Nesterov:2}, and \eqref{eq:representation:Nesterov:gradient} \prove\ for all $n\in\N$ that
	\begin{align}
		\textstyle
		&\nonumber\Theta_{ n-1  }-\Theta_{ n }
		\\&\nonumber\textstyle=\gamma_n\beta_n\mom[n]+\gamma_n\PRb
		{\frac{1}{\batch_{n}}\sum_{j=1}^{\batch_{n}}g(\Theta_{n-1},X_{n,j})}
		\\&\nonumber\textstyle=\gamma_n\beta_n\prb{\PRb{\prod_{j=1}^{n}\beta_j}\mom[0]
			+\sum_{k=1}^{n}\PRb{\prod_{j=k+1}^{n}\beta_j}
			G_k\pr{\Theta_{k-1},X_k}}+\gamma_nG_n(\Theta_{n-1},X_n)
		\\&\textstyle=\gamma_n\prb{\beta_n\PRb{\prod_{j=1}^{n}\beta_j}\mom[0]
			+
			G_n(\Theta_{n-1},X_n)
			+\sum_{k=1}^{n}\beta_n\PRb{\prod_{j=k+1}^{n}\beta_j}
			G_k\pr{\Theta_{k-1},X_k}}
		\\&\nonumber\textstyle=\gamma_n
		\prb{
			\beta_n\PRb{\prod_{j=1}^{n}\beta_j}\mom[0]
			+\sum_{k=1}^{n} \prb{(1+\beta_n)\indicator{\{n\}}(k)+\beta_n\PRb{\prod_{j=k+1}^{n}\beta_j}\indicator{(0,n)}(k)}
			G_k\pr{\Theta_{k-1},X_k}}
		\\&\nonumber\textstyle
		=\gamma_{n}  \prb{\alpha_{n,0}\mom[0]+\sum_{k=1}^n\alpha_{n,k}G_k(\Theta_{k-1},X_k)}
		.
	\end{align}
	Combining 
	\cref{it:SGD:mom:lims} in
	\cref{prop:SGD:mom}
	(applied with 
	$\pars \curvearrowleft \pars$,
	$\dim \curvearrowleft \dim$,
	$z\curvearrowleft z$,
	$(\alpha_{n,k})_{(n,k)\in(\N_0)^2} \curvearrowleft 
	(\alpha_{n,k})_{(n,k)\in(\N_0)^2}$,
	$\gamma \curvearrowleft \gamma$,
	$\batch\curvearrowleft \batch$,
	$(G_n)_{n\in\N} \curvearrowleft (G_n)_{n\in\N}$,
	$ ( \Omega, \cF,(\mathbb{F}_n)_{n\in\N_0}, \P ) \curvearrowleft( \Omega, \cF,(\mathbb{F}_n)_{n\in\N_0}, \P ) $,
	$(X_n)_{n\in\N}\curvearrowleft (X_{n,1},X_{n,2},\dots,X_{n,\batch_n})_{n\in\N}$,
	$\mom \curvearrowleft \mom[0]$,
	$\Theta \curvearrowleft \Theta$,
	$i \curvearrowleft i$
	in the notation of \cref{prop:SGD:mom}),
	\eqref{eq:setup:cor:Nesterov},
	\eqref{eq:setup:starting:expectations:bounded:cor:Nesterov},
	\eqref{eq:representation:Nesterov:gradient}, and
	\eqref{eq:setup:cor:LB:rates:Nesterov}
	with
	\cref{it:filtered:prob:space:Nesterov,it:filter:independence:Nesterov,it:filter:measurable:Nesterov,it:filter:measurable2:Nesterov}
	hence \proves\ that
		\begin{equation}
		\begin{split}
			\inf_{\substack{\xi\colon \Omega \to \R\\\text{measurable}}}
			\prbbb{z\PRbbb{\limsup_{n\to\infty}\prb{z\E\PRb{\vass{\Theta_n^{(i)}-\xi}^2}}}}
			>
			0
			.
		\end{split}
	\end{equation}
	This \proves\ \cref{eq:Nesterov:LB}. \finishproofthus
\end{proof}

\subsection{Adam optimization}

\begin{athm}{lemma}{lem:repr:adam:and:prop}
	Let $\pars,\dim\in\N$, 
	let $ \Omega$ be a set,
	let 
	$\bscl\colon\N\to[0,\infty)$,
	$\batch\colon\N\to\N$, 
	$\gamma\colon\N\to\R$, and
	$G=(G_n)_{n\in\N}=(G_n^{(1)},\dots,G_n^{(\pars)})_{n\in\N}\colon\N\times\Omega\to\R^{\pars}$ be functions,
	let $a\in\R$,
	$b\in(a,\infty)$,
	$\eps\in(0,\infty)$,
	$\alpha,\beta\in[0,1)$,
	for every $n\in\N$
	let	
	$\momf_n\colon\R^{n+1}\to[0,\infty)$ satisfy for all $x=(x_1,\dots,x_{n+1})\in\R^{n+1}$ that
	\begin{equation}
		\label{eq:def:2nd:mom:fct}
		\textstyle
		\momf_n(x)
		=
		\prb{\bscl_n\PRb{\sum_{k=0}^n(1-\beta)\beta^{n-k}(x_{k+1})^2}}^{\nicefrac{1}{2}}
		,
	\end{equation}
	let
	$ \Theta=(\Theta^{(1)},\ldots,\Theta^{(\pars)}) \colon\N_0\times\Omega \to \R^{\pars}$,
	$\mom=(\mom^{(1)},\ldots,\mom^{(\pars)})\colon\N_0\times\Omega\to\R^{\pars}$, and 
	$\MOM=(\MOM^{(1)},\ldots,\MOM^{(\pars)})\colon\N_0\times\Omega\to[0,\infty)^{\pars}$ satisfy
	for all $n\in\N$, $i\in\{1,2,\dots,\pars\}$ that
	\begin{equation}
		\label{lem:repr:adam:and:prop:moms}
		\begin{split}
			\mom[n]&\textstyle=\alpha \mom[n-1]+(1-\alpha)G_n,
			\qquad\qquad
			\MOM_{n}^{(i)}\textstyle=\beta \MOM_{n-1}^{(i)}+(1-\beta)\prb
			{G_n^{(i)}}^2,
		\end{split}
	\end{equation}
	\begin{equation}
		\label{lem:repr:adam:and:prop:recursion}
		\begin{split}
			\text{and}
			\qquad
			\Theta_{ n }^{(i)}
			&=\Theta_{ n-1  }^{(i)}-\gamma_n \prb{\eps+\PR{\bscl_n \MOM_{n}^{(i)}}^{\nicefrac{1}{2}}}^{-1}\mom[n]^{(i)}.
		\end{split}
	\end{equation}
	Then 
	\begin{enumerate}[label=(\roman*)]
		\item\label{lem:repr:adam:and:prop:repr} it holds for all $n\in\N$, $i\in\{1,2,\dots,\pars\}$ that
		\begin{equation}
			\begin{split}
				\textstyle
				\Theta_{ n }^{(i)}
				&=\Theta_{ n-1  }^{(i)}-\frac{\gamma_{n}  \PRb{\alpha^n\mom[0]^{(i)}+(1-\alpha)\sum_{k=1}^n\alpha^{n-k}G_k^{(i)}}}{\eps+\momf_n
					\prb{\prb{\pr{1-\beta}^{-1}\MOM_0^{(i)}}^{\nicefrac{1}{2}},G_1^{(i)},G_2^{(i)},\dots,G_n^{(i)}}}
				,	
			\end{split}
		\end{equation}
		\item\label{lem:repr:adam:and:prop:LB} it holds
		for all $n\in\N$, $x=(x_1,\dots,x_{n+1})\in\R^{n+1}$ that
		\begin{equation}
			\label{eq:momf:1}
			\textstyle
			\PRb{\inf_{j\in\N}\bscl_j}^{\nicefrac{1}{2}}\PRb{(1-\beta)\sum_{k=0}^n\beta^{n-k}(x_{k+1})^2
			}^{\nicefrac{1}{2}}
			\leq
			\momf_n(x)
			,
		\end{equation}
		\item\label{lem:repr:adam:and:prop:UB} it holds
		for all $n\in\N$, $x=(x_1,\dots,x_{n+1})\in\R^{n+1}$ that
		\begin{equation}
			\label{eq:momf:2}
			\textstyle
			\momf_n(x)
			\leq
			\sqrt{\bscl_n}
			\max\pR{\vass{x_1},\vass{x_2},\dots,\vass{x_{n+1}}},
		\end{equation}
		\item\label{lem:repr:adam:and:prop:growth} it holds
		for all $n\in\N$, $x=(x_1,\dots,x_{n+1})\in\R^{n+1}$, $y\in(-\infty,x_{n+1})$ that
		\begin{equation}
			\label{eq:momf:3}
			\begin{split}
				y\momf_n(x)\leq x_{n+1}\momf_n(x_1,x_2,\dots,x_{n},y),
			\end{split}
		\end{equation}
		and
		\item\label{lem:repr:adam:and:prop:square} it holds
		for all $n\in\N$, $x=(x_1,\dots,x_{n+1})\in\R^{n+1}$ that
		\begin{equation}
			\PRb{\momf_n(x)}^2
			=
			\PRb{\momf_n(x_1,\dots,x_n,0)}^2
			+(x_{n+1})^2\PRb{\momf_n(0,\dots,0,1)}^2.
		\end{equation}
	\end{enumerate}
\end{athm}

\begin{aproof}
	\Nobs that
	\eqref{lem:repr:adam:and:prop:moms} and
	\cite[Lemma~2.7]{DereichGraeberJentzenRiekert2025} (applied with
	$\pars\curvearrowleft\pars$,
	$(\beta_n)_{n\in\N}\curvearrowleft(\alpha)_{n\in\N}$,
	$\mom\curvearrowleft\mom$,
	$(g_n)_{n\in\N}\curvearrowleft \pr{\pr{1-\alpha}G_n}_{n\in\N}$
	in the notation of \cite[Lemma~2.7]{DereichGraeberJentzenRiekert2025}) \prove\ that for all $n\in\N$ it holds that 
	\begin{equation}
		\label{eq:rep:adam:nominator}
		\begin{split}
			\mom[n]
			=\textstyle
			\alpha^n\mom[0]
			+\sum_{k=1}^{n}\alpha^{n-k}\pr{\pr{1-\alpha}G_k}
			&=\textstyle
			\alpha^n\mom[0]
			+(1-\alpha)\sum_{k=1}^{n}\alpha^{n-k}G_k
			.
		\end{split}
	\end{equation}
	\Moreover 
	\eqref{lem:repr:adam:and:prop:moms}
	and
	\cite[Lemma~2.7]{DereichGraeberJentzenRiekert2025} (applied with
	$\pars\curvearrowleft 1$,
	$(\beta_n)_{n\in\N}\curvearrowleft(\beta)_{n\in\N}$,
	$\mom\curvearrowleft\MOM^{(i)}$,
	$(g_n)_{n\in\N}\curvearrowleft \prb{\pr{1-\beta}\prb{G_n^{(i)}}^2}_{n\in\N}$
	for $i\in\{1,2,\dots,\pars\}$
	in the notation of \cite[Lemma~2.7]{DereichGraeberJentzenRiekert2025}) \prove\ that for all $n\in\N$, $i\in\{1,2,\dots,\pars\}$ it holds that 
	\begin{equation}
		\label{eq:rep:adam:denominator}
		\begin{split}
			\bscl_n\MOM_n^{(i)}
			&=\textstyle\bscl_n\PRb{\beta^n \MOM_0^{(i)}+\sum_{k=1}^n \beta^{n-k}(1-\beta)\prb
				{G_k^{(i)}}^2}
			\\&=\textstyle\bscl_n\PRb{\beta^n \MOM_0^{(i)}+\sum_{k=1}^n (1-\beta)\beta^{n-k}\prb
				{G_k^{(i)}}^2}
			\\&=
			\textstyle
			\bscl_n
			\PRb{
				(1-\beta)\beta^n\PRb{\prb{\pr{1-\beta}^{-1}\MOM_0^{(i)}}^{\nicefrac{1}{2}}}^2+\sum_{k=1}^n (1-\beta)\beta^{n-k}\prb
				{G_k^{(i)}}^2}
			\\&=\PRb{\momf_n
				\prb{\prb{\pr{1-\beta}^{-1}\MOM_0^{(i)}}^{\nicefrac{1}{2}},G_1^{(i)},G_2^{(i)},\dots,G_n^{(i)}}}^2
			.
		\end{split}
	\end{equation}
	This,
	\eqref{lem:repr:adam:and:prop:recursion}, and
	\eqref{eq:rep:adam:nominator} \prove\ that for all $n\in\N$, $i\in\{1,2,\dots,\pars\}$ it holds that
	\begin{equation}
		\begin{split}
			\textstyle
			\Theta_{ n }^{(i)}
			&=\Theta_{ n-1  }^{(i)}-\gamma_n \prb{\eps+\PR{\bscl_n \MOM_{n}^{(i)}}^{\nicefrac{1}{2}}}^{-1}\mom[n]^{(i)}
			\\&=\Theta_{ n-1  }^{(i)}-\gamma_n \prb{\eps+\momf_n
				\prb{\prb{\pr{1-\beta}^{-1}\MOM_0^{(i)}}^{\nicefrac{1}{2}},G_1^{(i)},G_2^{(i)},\dots,G_n^{(i)}}}^{-1}\mom[n]^{(i)}
			\\&=\Theta_{ n-1  }^{(i)}-\frac{\gamma_{n}  \PRb{\alpha^n\mom[0]^{(i)}+(1-\alpha)\sum_{k=1}^n\alpha^{n-k}G_k^{(i)}}}{\eps+\momf_n
				\prb{\prb{\pr{1-\beta}^{-1}\MOM_0^{(i)}}^{\nicefrac{1}{2}},G_1^{(i)},G_2^{(i)},\dots,G_n^{(i)}}}
			.		
		\end{split}
	\end{equation}
	This \proves[pei]\ \cref{lem:repr:adam:and:prop:repr}.
	\Nobs that \eqref{eq:def:2nd:mom:fct} \proves\ for all $n\in\N$, $x=(x_1,\dots,x_{n+1})\in\R^{n+1}$ that
	\begin{equation}
		\begin{split}
			\momf_n(x)
			&=
			\textstyle
			\prb{\bscl_n
				\PRb{\sum_{k=0}^n (1-\beta)\beta^{n-k}\pr
					{x_{k+1}}^2}}^{\nicefrac{1}{2}}
			\\&\geq
			\textstyle
			\PRb{\inf_{j\in\N}\bscl_j}^{\nicefrac{1}{2}}
			\prb{
				\PRb{(1-\beta)\sum_{k=0}^n \beta^{n-k}\pr
					{x_{k+1}}^2}}^{\nicefrac{1}{2}}
			.
		\end{split}
	\end{equation}
	This \proves[pei]\ \cref{lem:repr:adam:and:prop:LB}.
	\Nobs that \eqref{eq:def:2nd:mom:fct} \proves\
	for all $n\in\N$, $x=(x_1,\dots,x_{n+1})\in\R^{n+1}$ that
	\begin{equation}
		\begin{split}
			\momf_n(x)
			&=
			\textstyle
			\prb{\bscl_n
				\PRb{\sum_{k=0}^n (1-\beta)\beta^{n-k}\pr
					{x_{k+1}}^2}}^{\nicefrac{1}{2}}
			\\&\textstyle\leq
			\prb{\bscl_n
				\PRb{\sum_{k=0}^n (1-\beta)\beta^{n-k}\pr
					{\max\pR{\vass{x_1},\vass{x_2},\dots,\vass{x_{n+1}}}}^2}}^{\nicefrac{1}{2}}
			\\&\textstyle=
			\sqrt{\bscl_n}
			\prb{
				\PRb{(1-\beta)\sum_{k=0}^n \beta^{n-k}}}^{\nicefrac{1}{2}}
			\max\pR{\vass{x_1},\vass{x_2},\dots,\vass{x_{n+1}}}
			\\&\textstyle\leq
			\sqrt{\bscl_n}
			\prb{
				\PRb{(1-\beta)\sum_{k=0}^{\infty} \beta^{k}}}^{\nicefrac{1}{2}}
			\max\pR{\vass{x_1},\vass{x_2},\dots,\vass{x_{n+1}}}
			\\&\textstyle=
			\sqrt{\bscl_n}
			\max\pR{\vass{x_1},\vass{x_2},\dots,\vass{x_{n+1}}}
			.
		\end{split}
	\end{equation}
	This \proves[pei]\ \cref{lem:repr:adam:and:prop:UB}.
	\Nobs that for all 
	$r\in[0,\infty)$,
	$x\in\R$, 
	$y\in(-\infty,x]$, 
	$s,t\in\{-1,1\}$ with $x=s\vass{x}$ and $y=t\vass{y}$ it holds that
	\begin{equation}
		\begin{split}
			x\sqrt{ y^2 + r }
			-y\sqrt{ x^2 + r }
			=
			s\sqrt{ x^2 y^2 + r x^2 }
			-t\sqrt{ x^2 y^2 + r y^2 }
			\geq 0
			.
		\end{split}
	\end{equation}
	This and \eqref{eq:def:2nd:mom:fct} \prove\ that for all $n\in\N$, $x=(x_1,\dots,x_{n+1})\in\R^{n+1}$, $y\in(-\infty,x_{n+1})$ it holds that
	\begin{equation}
		\begin{split}
			y\momf_n(x)
			&=
			\textstyle
			y\prb{
				\bscl_n
				\PRb{\sum_{k=0}^n (1-\beta)\beta^{n-k}\pr
					{x_{k+1}}^2}}^{\nicefrac{1}{2}}
			\\&=
			\PR{\pr{1-\beta}\bscl_n}^{\nicefrac{1}{2}}
			y\PRb{(x_{n+1})^2+
				(1-\beta)^{-1}\PRb{\textstyle\sum_{k=0}^{n-1} (1-\beta)\beta^{n-k}\pr
					{x_{k+1}}^2}}^{\nicefrac{1}{2}}
			\\&\leq
			\PR{\pr{1-\beta}\bscl_n}^{\nicefrac{1}{2}}
			x_{n+1}\PRb{y^2+
				(1-\beta)^{-1}\PRb{\textstyle\sum_{k=0}^{n-1} (1-\beta)\beta^{n-k}\pr
					{x_{k+1}}^2}}^{\nicefrac{1}{2}}
			\\&=
			\textstyle
			x_{n+1}\prb{
				\bscl_n
				\PRb{(1-\beta)y^2+\sum_{k=0}^{n-1} (1-\beta)\beta^{n-k}\pr
					{x_{k+1}}^2}}^{\nicefrac{1}{2}}
			\\&= x_{n+1}\momf_n(x_1,x_2\dots,x_{n},y).
		\end{split}
	\end{equation}
	This \proves[pei]\ \cref{lem:repr:adam:and:prop:growth}.
	\Nobs that \eqref{eq:def:2nd:mom:fct} \proves\ that for all $n\in\N$,
	$x=(x_1,\dots,x_{n+1})\in\R^{n+1}$ it holds that
	\begin{equation}
		\begin{split}
			\PRb{\momf_n(x)}^2
			&\textstyle
			=\bscl_n\PRb{\sum_{k=0}^{n-1}
				(1-\beta)\beta^{n-k}(x_{k+1})^2}
			+(x_{n+1})^2(1-\beta)\bscl_n
			\\
			&=\PRb{\momf_n(x_1,\dots,x_n,0)}^2
			+(x_{n+1})^2\PRb{\momf_n(0,\dots,0,1)}^2.
		\end{split}
	\end{equation}
	This \proves[pei]\ \cref{lem:repr:adam:and:prop:square}.
	\finishproofthus
\end{aproof}

\begin{athm}{cor}{ANNS:ADAM}[Non-convergence of \ADAM\ in the training of deep \ANNs]
	Let $\dim,\DIM,\pars,L\in\N$, 
	$\eps,\fc\in(0,\infty)$,
	$l_0,l_1,l_2,\ldots\in\N$,
	$s_0,s_1,s_2,\ldots\in\N_0$
	satisfy $l_0=\dim$, $l_L=\DIM$, $s_L=\pars$, and $\sum_{k=0}^L \vass{s_k-\sum_{j=1}^k l_j(l_{j-1}+1)}=0$,
	let $\act\in C^1(\R,\R)$ satisfy for all $x\in\R$ that $\vass{\act'(x)}\leq\fc(1+\vass{x})^{\fc}$,
	for every $\theta=(\theta_1,\dots,\theta_{\pars})\in\R^{\pars}$
	let 
	$\realization{\theta}_{k}=(\realization{\theta}_{k,1},\dots,\realization{\theta}_{k,l_k})\colon\R^\dim\to\R^{l_k}$, $k\in\N_0$, satisfy for all $k\in\{0,1,\dots,L-1\}$, $x=(x_1,\dots,x_\dim)\in\R^\dim$, $j\in\{1,2,\dots,l_{k+1}\}$ that
	\begin{equation}
		\label{eq:vect:ANNS:ADAM:real}
		\realization{\theta}_{k+1,j}(x)
		=
		\theta_{j+l_{k+1}l_{k}+s_{k}}
		+
		\sum_{i=1}^{l_{k}}
		\theta_{i+(j-1) l_{k}+s_{k}}
		\prb{
			\act\prb{\realization{\theta}_{k,i}(x)}\indicator{\N}(k)
			+x_{\min\{i,d\}}\indicator{\{0\}}(k)
		},
	\end{equation}
	let $\loss\colon\R^{\pars}\times\pr{\R^{\dim}\times\R^\DIM}\to\R$, 
	$\batch\colon\N\to\N$,
	$\gamma\colon\N\to[0,\infty)$,
	 and $\bscl\colon\N\to[0,\infty)$ 
	satisfy
	for all 
	$\theta\in\R^\pars$, 
	$v\in\R^\dim$,
	$w\in\R^\DIM$
	that
	\begin{equation}
		\label{eq:ANNS:ADAM:EX:loss:1}	
		\loss(\theta,(v,w))=\norm{\realization{\theta}_{L}(v)-w}^2,
		\quad
		\limsup_{n\to\infty}\PRb{(\gamma_n)^2\pr{\batch_n}^{-1}}>0,
		\qandq
		\limsup_{n\to\infty}\bscl_n<\infty,
	\end{equation} 
	let $ ( \Omega, \cF, \P ) $ be a probability space,
	for every $n\in\N$ let $\DPA_n=\pr{\pr{X_{n,j},Y_{n,j}}_{j\in\{1,2,\dots,\batch_n\}}}\colon\Omega \allowbreak\to \pr{\R^{\dim}\times\R^{\DIM}}^{\batch_n}$
	be a bounded random variable,
		let 
	$\alpha,\beta\in[0,1)$,
	let
	$ \Theta=(\Theta^{(1)},\ldots,\Theta^{(\pars)}) \colon\N_0\times\Omega \to \R^{\pars}$,
	$\mom=(\mom^{(1)},\ldots,\mom^{(\pars)})\colon\N_0\times\Omega\to\R^{\pars}$, and 
	$\MOM=(\MOM^{(1)},\ldots,\MOM^{(\pars)})\colon\N_0\times\Omega\to[0,\infty)^{\pars}$ be stochastic processes which satisfy
	for all $n\in\N$, $i\in\{1,2,\dots,\pars\}$ that
	\begin{equation}
		\label{eq:representation:ANNS:ADAM:1:New:spec:loss}
		\begin{split}
			\mom[n]&\textstyle=\alpha \mom[n-1]+(1-\alpha)\PRb
			{\frac{1}{\batch_{n}}\sum_{j=1}^{\batch_{n}}\pr{\nabla_\theta\loss}(\Theta_{n-1},\pr{X_{n,j},Y_{n,j}})},
		\end{split}
	\end{equation}
	\begin{equation}
		\label{eq:representation:ANNS:ADAM:2:New:spec:loss}
		\begin{split}
			\MOM_n^{(i)}&\textstyle=\beta \MOM_{n-1}^{(i)}+(1-\beta)\PRb
			{\frac{1}{\batch_{n}}\sum_{j=1}^{\batch_{n}}\prb{\tfrac{\partial\loss}{\partial\theta_{i}}}(\Theta_{n-1},\pr{X_{n,j},Y_{n,j}})}^2,
		\end{split}
	\end{equation}
	\begin{equation}
		\label{eq:representation:ANNS:ADAM:3:New:spec:loss}
		\begin{split}
			\qqandqq
			\Theta_{ n }^{(i)}
			&=\Theta_{ n-1  }^{(i)}-\gamma_n\prb{\eps+\PRb{\bscl_n\MOM_n^{(i)}}^{\nicefrac{1}{2}}}^{-1}
			\mom[n]^{(i)}
			,
		\end{split}
	\end{equation}
	assume that $(X_{n,j},Y_{n,j})$, $(n,j)\in\{(m,i)\in\N^2\colon i\leq\batch_m\}$, are \iid,
	assume that $(\Theta_0,\mom[0],\MOM_0)$ and $\pr{\DPA_n}_{n\in\N}$ are independent, and
	assume
	\begin{equation}
		\label{eq:ANNS:ADAM:properties:of:setup:1}
		\textstyle
		\inf_{\theta\in\R^\pars}
		\E\PRb{\norm{\realization{\theta}_{L}(X_{1,1})-Y_{1,1}}^2}
		>0
		.
	\end{equation}
	Then
	\begin{equation}
		\label{eq:ANNS:ADAM:1}
		\begin{split}
			\inf_{\substack{\xi\colon \Omega \to \R^\pars\\\text{measurable}}}
			\PRbbb{\limsup_{ n \to \infty }\E\PRb{ \norm{ \Theta_n - \xi }^2 }}
			>
			0
			.
		\end{split}
	\end{equation}
\end{athm}

\begin{aproof}
	Throughout this proof 
	assume without loss of generality that for all 
	$n\in\N$, 
	$\omega\in\Omega$ it holds that
	$\fc>0$,
	$\DPA_n(\omega)\in\pr{[ -\fc, \fc ]^{\dim}\times[ -\fc, \fc ]^{\DIM}}^{\batch_n}$,
	let $\const\in\R$ satisfy
	 $\const\geq\max\{1,\fc,\sup_{m\in\N}\bscl_m\}$,
	let
	$\MOMalt=\prb{\MOMalt[1],\dots,\MOMalt[\pars]}\colon\Omega\to\R^\pars$ satisfy for all $\mathfrak{i}\in\{1,2,\dots,\pars\}$ that
	$
	\MOMalt[\mathfrak{i}]=\PRb{\pr{1-\beta}^{-1}\MOMALT[0]^{(\mathfrak{i})}}^{\nicefrac{1}{2}}
	$, and
	for every $n\in\N$ let
	$G_{n}\colon\R^\pars\times\pr{[-\fc,\fc]^\dim\times[-\fc,\fc]^\DIM}^{\batch_n}\to\R^\pars$
	and
	$\MOMs[n]\colon\R^{n+1}\to\R$ satisfy for all $\theta\in\R^{\pars}$, $\dpa=(\dpa_1,\dots,\dpa_{\batch_n})\in\pr{[ -\fc, \fc ]^{\dim}\times[ -\fc, \fc ]^{\DIM}}^{\batch_n}$,
	$y=(y_1,\dots,y_{n+1})\in\R^{n+1}$ that
	\begin{equation}
		\label{eq:representation:ANNS:ADAM:gradient:New}
		\textstyle
		G_n(\theta,\dpa)=\frac{1}{\batch_{n}}\PRb{\sum_{j=1}^{\batch_{n}}\pr{\nabla_{\theta}\loss}(\theta,\dpa_j)}
		\qandq
		\MOMs[n](y)=
		\textstyle
		\PRb{\bscl_n
			\PRb{\sum_{k=0}^n (1-\beta)\beta^{n-k}\pr
				{y_{k+1}}^2}}^{\nicefrac{1}{2}}
		.
	\end{equation}
	\Nobs that 
	\eqref{eq:representation:ANNS:ADAM:1:New:spec:loss},
	\eqref{eq:representation:ANNS:ADAM:2:New:spec:loss},
	\eqref{eq:representation:ANNS:ADAM:3:New:spec:loss},
	\eqref{eq:representation:ANNS:ADAM:gradient:New}, and 
	\cref{lem:repr:adam:and:prop} (applied with
	$\Theta\curvearrowleft\Theta$,
	$G\curvearrowleft\prb{\N\times\Omega\ni(n,\omega)\mapsto G_n(\Theta_{n-1}(\omega),\DPA_n(\omega))\in\R^{\pars}}$,
	$(\momf_n)_{n\in\N}\curvearrowleft(\MOMs[n])_{n\in\N}$ in the notation of \cref{lem:repr:adam:and:prop}) \prove\ that
		it holds for all $n\in\N$, $\mathfrak{i}\in\{1,2,\dots,\pars\}$ that
		\begin{equation}
			\label{it:lem:repr:adam:and:prop:repr:ANNS:ADAM}
			\begin{split}
				\textstyle
				\Theta_{ n }^{(\mathfrak{i})}
				&=\Theta_{ n-1  }^{(\mathfrak{i})}-\frac{\gamma_{n}  \PR{\alpha^n\mom[0]^{(\mathfrak{i})}+(1-\alpha)\sum_{k=1}^n\alpha^{n-k}G_k^{(\mathfrak{i})}(\Theta_{k-1},\DPA_k)}}{\eps+\MOMs[n]
					\prb{\MOMalt[\mathfrak{i}],G_1^{(\mathfrak{i})}(\Theta_{0},\DPA_1),G_2^{(\mathfrak{i})}(\Theta_{1},\DPA_2),\dots,G_n^{(\mathfrak{i})}(\Theta_{n-1},\DPA_n)}}.
			\end{split}
		\end{equation}
	\Moreover \eqref{eq:representation:ANNS:ADAM:gradient:New} \proves\
	for all $n\in\N$, $x=(x_1,\dots,x_{n+1})\in\R^{n+1}$ that
	\begin{equation}
		\label{eq:LB:ANNS:ADAM:denominator:UB:New}
		\begin{split}
			\vass{\MOMs[n](x)}^2
			&=
			\textstyle
			\bscl_n
				\PRb{\sum_{k=0}^n (1-\beta)\beta^{n-k}\pr
					{x_{k+1}}^2}
			\leq
			\const(1-\beta)
				\sum_{k=0}^n \pr
				{x_{k+1}}^2\beta^{n-k}
			.
		\end{split}
	\end{equation}
	Combining 
	\cref{lem:for:EX:ANNs:1}
	(applied with 
	$\act \curvearrowleft \act$,
	$\fc\curvearrowleft\const$,
	$a\curvearrowleft -\fc$,
	$b\curvearrowleft \fc$,
	$\pr{\realization{\theta}_{k}}_{(k,\theta)\in\{0,1,\dots,L\}\times\R^{\pars}}\curvearrowleft\pr{\realization{\theta}_{k}}_{(k,\theta)\in\{0,1,\dots,L\}\times\R^{\pars}}$,
	$\loss\curvearrowleft\loss$,
	$(G_n)_{n\in\N}\curvearrowleft(G_n)_{n\in\N}$,
	$(\MOMs[n])_{n\in\N}\curvearrowleft(\mom[0]\indicator{\{1\}}(n)+\MOMalt\indicator{\{2\}}(n))_{n\in\N}$,
	$(\MOM_n)_{n\in\N}\curvearrowleft(\MOMs[n])_{n\in\N}$,
	$(\DPA_n)_{n\in\N}\curvearrowleft(\DPA_n)_{n\in\N}$,
	$\Theta \curvearrowleft \Theta$
	in the notation of \cref{lem:for:EX:ANNs:1}),
	\eqref{eq:representation:ANNS:ADAM:gradient:New}, and
	\eqref{eq:LB:ANNS:ADAM:denominator:UB:New}
	with
	\eqref{it:lem:repr:adam:and:prop:repr:ANNS:ADAM}
	hence \proves\ that 
	\begin{equation}
		\begin{split}
			\inf_{\substack{\xi\colon \Omega \to \R^\pars\\\text{measurable}}}
			\PRbbb{\limsup_{ n \to \infty }\E\PRb{ \norm{ \Theta_n - \xi }^{2} }}
			>
			0
			.
		\end{split}
	\end{equation}
	This \proves\ \cref{eq:ANNS:ADAM:1}.
	\finishproofthus
\end{aproof}

\cfclear
\begin{athm}{cor}{cor:ADAM:New}
	Let $\pars,\dim\in\N$, 
	$a\in\R$,
	$b\in[a,\infty)$,
	let $g=(g_1,\dots,g_{\pars})\colon\R^{\pars}\times[a,b]^{\dim}\to\R^{\pars}$ be measurable, 
	let $p,\const\in(0,\infty)$,
	let 
	$\batch\colon\N\to\N$,
	$\gamma\colon\N\to[0,\infty)$,
	 and $\bscl\colon\N\to[0,\infty)$ satisfy for all $\theta\in\R^\pars$, $x\in[a,b]^\dim$ that
			\begin{equation}
		\label{eq:bounded:batches:and:LR:ADAM:New}
			\norm{
	g(\theta,x)}
\leq
\const
\pr{\norm{\theta}+1}^p,
\qquad
\limsup_{n\to\infty}\PRb{(\gamma_n)^2\pr{\batch_n}^{-1}}>0,
\qqandqq
\limsup_{n\to\infty}\bscl_n<\infty,
	\end{equation}
	let $ ( \Omega, \cF, \P ) $ be a probability space,
	let $X_{n,m}\colon\Omega\to[a,b]^\dim$, $(n,m)\in\N^2$, be \iid\ random variables,
	let $\eps\in(0,\infty)$,
	$\alpha,\beta\in[0,1)$,
	let
	$ \Theta=(\Theta^{(1)},\ldots,\Theta^{(\pars)}) \colon\N_0\times\Omega \to \R^{\pars}$,
	$\mom=(\mom^{(1)},\ldots,\mom^{(\pars)})\colon\N_0\times\Omega\to\R^{\pars}$, and 
	$\MOMALT=(\MOMALT^{(1)},\ldots,\MOMALT^{(\pars)})\colon\N_0\times\Omega\to[0,\infty)^{\pars}$ be stochastic processes which satisfy
	for all $n\in\N$, $i\in\{1,2,\dots,\pars\}$ that
	\begin{equation}
		\label{eq:representation:ADAM:1:New}
		\begin{split}
			\mom[n]&\textstyle=\alpha \mom[n-1]+(1-\alpha)\PRb
			{\frac{1}{\batch_{n}}\sum_{j=1}^{\batch_{n}}g(\Theta_{n-1},X_{n,j})},
		\end{split}
	\end{equation}
	\begin{equation}
		\label{eq:representation:ADAM:2:New}
		\begin{split}
			\MOMALT[n]^{(i)}&\textstyle=\beta \MOMALT[n-1]^{(i)}+(1-\beta)\PRb
			{\frac{1}{\batch_{n}}\sum_{j=1}^{\batch_{n}}g_i(\Theta_{n-1},X_{n,j})}^2,
		\end{split}
	\end{equation}
	\begin{equation}
		\label{eq:representation:ADAM:3:New}
		\begin{split}
			\qqandqq
			\Theta_{ n }^{(i)}
			&=\Theta_{ n-1  }^{(i)}-\gamma_n\prb{\eps+\PRb{\bscl_n\MOMALT[n]^{(i)}}^{\nicefrac{1}{2}}}^{-1}
			\mom[n]^{(i)}
			,
		\end{split}
	\end{equation}
	assume that
		\begin{equation}
			\label{eq:start:mom:bound:new:rob:cond:ADAM:New}
		\textstyle
		\inf_{\theta\in\R^\pars}\PRb{\sum_{i=1}^\pars\var\pr{g_i\pr{\theta,X_{1,1}}}}>0,
	\end{equation}
	and assume that $(\Theta_0,\mom[0],\MOMALT[0])$ and $\pr{X_{n,j}}_{(n,j)\in\{(k,l)\in\N^2\colon l\leq\batch_k\}}$ are independent.
	Then 
	\begin{equation}
		\label{eq:ADAM:LB:New}
		\begin{split}
			\inf_{\substack{\xi\colon \Omega \to \R^\pars\\\text{measurable}}}
			\PRbbb{\limsup_{ n \to \infty }\E\PRb{ \norm{\Theta_n - \xi }^{2} }}
			>
			0
			.
		\end{split}
	\end{equation}
\end{athm}

\begin{proof}[Proof of \cref{cor:ADAM:New}]
	Throughout this proof 
	assume without loss of generality that $\const\geq\max\{1,\sup_{n\in\N}\bscl_n\}$,
	let
	$\MOMalt=\prb{\MOMalt[1],\dots,\MOMalt[\pars]}\colon\Omega\to\R^\pars$ satisfy for all $\mathfrak{i}\in\{1,2,\dots,\pars\}$ that
	$
	\MOMalt[\mathfrak{i}]=\PRb{\pr{1-\beta}^{-1}\MOMALT[0]^{(\mathfrak{i})}}^{\nicefrac{1}{2}}
	$,
	for every $n\in\N$ let $X_n=(X_{n,1},\dots,X_{n,\batch_n})$,
	 and
	for every $n\in\N$ let
	$G_{n}\colon\R^\pars \times\pr{[a,b]^{\dim}}^{\batch_n}\to\R^\pars$
	and
	$\MOMs[n]\colon\R^{n+1}\to\R$ satisfy for all $\theta\in\R^{\pars}$, $x=(x_1,\dots,x_{\batch_n})\in\pr{[a,b]^{\dim}}^{\batch_n}$,
	$y=(y_1,\dots,y_{n+1})\in\R^{n+1}$ that
	\begin{equation}
		\label{eq:representation:ADAM:gradient:New}
		\textstyle
		G_n(\theta,x)=\frac{1}{\batch_{n}}\sum_{j=1}^{\batch_{n}}g(\theta,x_j)
		\qqandqq
		\MOMs[n](y)=
		\textstyle
		\PRb{\bscl_n
			\PRb{\sum_{k=0}^n (1-\beta)\beta^{n-k}\pr
				{y_{k+1}}^2}}^{\nicefrac{1}{2}}
		.
	\end{equation}
	\Nobs that 
	\eqref{eq:representation:ADAM:1:New},
	\eqref{eq:representation:ADAM:2:New},
	\eqref{eq:representation:ADAM:3:New},
	\eqref{eq:representation:ADAM:gradient:New}, and 
	\cref{lem:repr:adam:and:prop} (applied with
	$\Theta\curvearrowleft\Theta$,
	$G\curvearrowleft\prb{\N\times\Omega\ni(n,\omega)\mapsto G_n(\Theta_{n-1}(\omega),X_n(\omega))\in\R^{\pars}}$,
	$(\momf_n)_{n\in\N}\curvearrowleft(\MOMs[n])_{n\in\N}$ in the notation of \cref{lem:repr:adam:and:prop}) \prove\ that
		for all $n\in\N$, $\mathfrak{i}\in\{1,2,\dots,\pars\}$ it holds that
		\begin{equation}
			\label{it:lem:repr:adam:and:prop:repr}
			\begin{split}
				\textstyle
				\Theta_{ n }^{(\mathfrak{i})}
				&=\Theta_{ n-1  }^{(\mathfrak{i})}-\frac{\gamma_{n}  \PR{\alpha^n\mom[0]^{(\mathfrak{i})}+(1-\alpha)\sum_{k=1}^n\alpha^{n-k}G_k^{(\mathfrak{i})}(\Theta_{k-1},X_k)}}{\eps+\MOMs[n]
					\prb{\MOMalt[\mathfrak{i}],G_1^{(\mathfrak{i})}(\Theta_{0},X_1),G_2^{(\mathfrak{i})}(\Theta_{1},X_2),\dots,G_n^{(\mathfrak{i})}(\Theta_{n-1},X_n)}}.
			\end{split}
		\end{equation}
	\Moreover \eqref{eq:bounded:batches:and:LR:ADAM:New},
	\eqref{eq:start:mom:bound:new:rob:cond:ADAM:New}, and \eqref{eq:representation:ADAM:gradient:New} \prove\ that 
	\begin{equation}
		\label{eq:double:inf:var:for:ADAM}
		\begin{split}
		\inf_{n\in\N}\inf_{\theta\in\R^\pars}
		\textstyle\PRb{\batch_n\sum_{i=1}^\pars\var\prb{G_n^{(i)}\pr{\theta,X_n}}}
		&=
		\inf_{n\in\N}\inf_{\theta\in\R^\pars}
		\textstyle\PRb{\batch_n\sum_{i=1}^\pars\var\prb{\frac{1}{\batch_{n}}\sum_{j=1}^{\batch_{n}}g_i(\theta,X_{n,j})}}
		\\&=
		\inf_{n\in\N}\inf_{\theta\in\R^\pars}
		\textstyle\PRb{\prb{\batch_n}^{-1}\sum_{i=1}^\pars\var\prb{\sum_{j=1}^{\batch_{n}}g_i(\theta,X_{n,j})}}
		\\&=
		\inf_{n\in\N}\inf_{\theta\in\R^\pars}
		\textstyle\PRb{\sum_{i=1}^\pars\var\prb{g_i(\theta,X_{1,1})}}
		>0
		.
		\end{split}
	\end{equation}
	\Moreover \eqref{eq:representation:ADAM:gradient:New} and \cref{lem:end:filtration} 
	(applied with 
	$\pars\curvearrowleft 3\pars$,
	$\batch\curvearrowleft\batch$,
	$\Batch\curvearrowleft\batch$,
	$\BATCH\curvearrowleft\batch$,
	$Y\curvearrowleft (\Theta_0,\mom[0],\MOMalt)$,
	$(X_{n,j})_{(n,j)\in\{(m,i)\in\N^2\colon i\leq\BATCH_m\}}\curvearrowleft\pr{X_{n,j}}_{(n,j)\in\{(m,i)\in\N^2\colon i\leq\batch_m\}}$
	 in the notation of \cref{lem:end:filtration}) \prove\ that there exist $\fil_n\subseteq\cF$, $n\in\N_0$, which satisfy that
	\begin{enumerate}[label=(\Alph*)]
		\item \label{it:filtered:prob:space:ADAM:New}it holds that $ ( \Omega, \cF,(\mathbb{F}_n)_{n\in\N_0}, \P ) $ is a filtered probability space,
		\item \label{it:filter:independence:ADAM:New}it holds for all $n\in\N$ that 
		$(X_{n,1},X_{n,2},\dots,X_{n,\batch_n})$ is independent of
		$\fil_{n-1}$,
		\item \label{it:filter:measurable:ADAM:New}it holds for all $n\in\N$ that $(X_{n,1},X_{n,2},\dots,X_{n,\batch_n})$ is $\fil_n$-measurable, and
		\item \label{it:filter:measurable2:ADAM:New}it holds that $(\Theta_0,\mom[0],\MOMalt)$ is $\fil_0$-measurable.
	\end{enumerate}
	\Moreover \eqref{eq:representation:ADAM:gradient:New} \proves\
	for all $n\in\N$, $x=(x_1,\dots,x_{n+1})\in\R^{n+1}$ that
	\begin{equation}
		\label{eq:LB:ADAM:denominator:UB:New}
		\begin{split}
			\vass{\MOMs[n](x)}^2
			&=
			\textstyle
			\bscl_n
			\PRb{\sum_{k=0}^n (1-\beta)\beta^{n-k}\pr
				{x_{k+1}}^2}
			\leq
			\const(1-\beta)
			\sum_{k=0}^n \beta^{n-k}\pr
			{x_{k+1}}^2
			.
		\end{split}
	\end{equation}
	Combining 
	\eqref{eq:bounded:batches:and:LR:ADAM:New},
	\eqref{eq:start:mom:bound:new:rob:cond:ADAM:New},
	\eqref{eq:representation:ADAM:gradient:New}, 
	\eqref{eq:double:inf:var:for:ADAM}, and
	\eqref{eq:LB:ADAM:denominator:UB:New}
	with
	\cref{it:lem:repr:adam:and:prop:repr,it:filtered:prob:space:ADAM:New,it:filter:independence:ADAM:New,it:filter:measurable:ADAM:New,it:filter:measurable2:ADAM:New} and
	\cref{cor:limsup:simple}
	(applied with 
	$q \curvearrowleft 1$,
	$\fc\curvearrowleft \const$,
	$\const \curvearrowleft \const$,
	$\batch\curvearrowleft \batch$,
	$(\MOM_n)_{n\in\N} \curvearrowleft (\MOMs[n])_{n\in\N}$,
	$(G_n)_{n\in\N} \curvearrowleft (G_n)_{n\in\N}$,
	$ ( \Omega, \cF,(\mathbb{F}_n)_{n\in\N_0}, \P ) \curvearrowleft( \Omega, \cF,(\mathbb{F}_n)_{n\in\N_0}, \P ) $,
	$(X_n)_{n\in\N}\curvearrowleft (X_{n,1},\allowbreak X_{n,2},\dots,\allowbreak X_{n,\batch_n})_{n\in\N}$,
	$\MOMs \curvearrowleft \pr{\N\times\Omega\ni(n,\omega)\mapsto (2-n)\mom[0](\omega)+(n-1)\MOMalt(\omega)\in\R^\pars}$,
	$\Theta \curvearrowleft \Theta$
	in the notation of \cref{cor:limsup:simple})
	 hence \proves\ that
	\begin{equation}
		\begin{split}
			\inf_{\substack{\xi\colon \Omega \to \R^\pars\\\text{measurable}}}
			\PRbbb{\limsup_{ n \to \infty }\E\PRb{ \norm{ \Theta_n - \xi }^{2} }}
			>
			0
			.
		\end{split}
	\end{equation}
	This \proves\ \cref{eq:ADAM:LB:New}. \finishproofthus
\end{proof}

\begin{athm}{cor}{prop:prop:non_convergence_modified_Adam:easy:mat}	
	Let $\pars,\dim\in\N$, 
	$A\in\R^{\dim\times\pars}$,
	let 
	$\loss\colon\R^{\pars}\times\R^{\dim}\to\R$,
	$\batch\colon\N\to\N$, and
	$\gamma\colon\N\to[0,\infty)$
	satisfy
	for all 
	$\theta\in\R^\pars$, $x\in\R^\dim$
	that
	\begin{equation}
		\label{eq:matrix:loss:and:bounds:easy:mat}
		\textstyle
		\loss(\theta,x)=\norm{A\theta-x}^2
		\qqandqq
		\limsup_{n\to\infty}\PRb{(\gamma_n)^2\pr{\batch_n}^{-1}}>0,
	\end{equation} 
	let $ ( \Omega, \cF, \P ) $ be a probability space,
	let $a\in\R$,
	$b\in(a,\infty)$,
	let	$
	X_{n,j} \colon \Omega \to [ a, b ]^{\dim}
	$, $(n,j)\in\N^2$, be \iid\ random variables,
	let $\eps\in(0,\infty)$,
	$\alpha,\beta\in[0,1)$,
	let
	$ \Theta=(\Theta^{(1)},\ldots,\Theta^{(\pars)}) \colon\N_0\times\Omega \to \R^{\pars}$,
	$\mom=(\mom^{(1)},\ldots,\mom^{(\pars)})\colon\N_0\times\Omega\to\R^{\pars}$, and 
	$\MOM=(\MOM^{(1)},\ldots,\MOM^{(\pars)})\colon\N_0\times\Omega\to[0,\infty)^{\pars}$ be stochastic processes satisfying
	for all $n\in\N$, $i\in\{1,2,\dots,\pars\}$ that
	\begin{equation}
		\label{prop:eq:def:momentum:main:2.1:easy:mat}
		\begin{split}
			\mom[n]&\textstyle=\alpha \mom[n-1]+(1-\alpha)\PRb
			{\frac{1}{\batch_n}\sum_{j=1}^{\batch_n}\pr{\nabla_{\theta}\loss}(\Theta_{n-1},X_{n,j})},
		\end{split}
	\end{equation}
	\begin{equation}
		\label{prop:eq:def:RMS:factor:main:2.2:easy:mat}
		\begin{split}
			\MOM_{n}^{(i)}&\textstyle=\beta \MOM_{n-1}^{(i)}+(1-\beta)\PRb
			{\frac{1}{\batch_n}\sum_{j=1}^{\batch_n}\prb{\frac{\partial\loss}{\partial\theta_i}}(\Theta_{n-1},X_{n,j})}^2,
		\end{split}
	\end{equation}
	\begin{equation}
		\label{prop:eq:matrix:Adam:recursion:easy:mat}
		\begin{split}
			\text{and}
			\qquad
			\Theta_{ n }^{(i)}
			&=\Theta_{ n-1  }^{(i)}-\gamma_n \prb{\eps+\PR{(1-\beta^n)^{-1} \MOM_{n}^{(i)}}^{\nicefrac{1}{2}}}^{-1}\mom[n]^{(i)},
		\end{split}
	\end{equation}
	assume that $(\Theta_0,\mom[0],\MOM_0)$ and $\pr{X_{n,j}}_{(n,j)\in\{(k,l)\in\N^2\colon l\leq\batch_k\}}$ are independent,
	and assume that
	$0<\allowbreak\operatorname{Trace}( \operatorname{Cov}\pr{A^*X_{1,1}})$.
	Then 
	\begin{equation}
		\label{eq:matrix:Adam:easy:mat}
		\begin{split}
			\inf_{\substack{\xi\colon \Omega \to \R^\pars\\\text{measurable}}}\PRbbb{\limsup_{ n \to \infty }
			\E\PRb{ \norm{ \Theta_n - \xi }^2 }}
			>0.
		\end{split}
	\end{equation}
\end{athm}

\begin{proof}[Proof of \cref{prop:prop:non_convergence_modified_Adam:easy:mat}]
	Throughout this proof 
	for every $j\in\{1,2,\dots,\pars\}$ let $e_j\in\R^\pars$ satisfy for all $x=(x_1,\dots,x_\pars)\in\R^\pars$ that $\langle x,e_j\rangle = x_j$,
	and let $\fc\in\R$ satisfy for all $\theta\in\R^\pars$, $\vartheta\in\R^\dim$ that
	\begin{equation}
		\label{eq:matrix:op:norm:bd}
		\sqrt{\dim}\max\{\vass{a},\vass{b}\}\leq\fc,
		\qquad
		\norm{A\theta}\leq \fc\norm{\theta},
		\qqandqq
		\norm{A^*\vartheta}\leq \fc\norm{\vartheta}.
	\end{equation}
		\Nobs that \eqref{eq:matrix:loss:and:bounds:easy:mat} \proves\ that
	for all 
	$j\in\{1,2,\dots,\pars\}$,
	$\theta=(\theta_1,\dots,\theta_\pars)\in\R^\pars$, $x\in[a,b]^\dim$ it holds that
	\begin{equation}
		\label{eq:EX:i_th:partial}
		\begin{split}
			\textstyle
			\prb{\tfrac{\partial\loss}{\partial\theta_{j}}}(\theta,x)
			=2\langle Ae_j,A\theta-x\rangle
			.
		\end{split}
	\end{equation}
		\Nobs that \eqref{eq:matrix:loss:and:bounds:easy:mat}, \eqref{eq:matrix:op:norm:bd}, and
		\eqref{eq:EX:i_th:partial} \prove\ that
  for all $\theta\in\R^\pars$, $x\in[a,b]^{\dim}$
	it holds that
	\begin{equation}
		\label{eq:UB:gradient:gen:loss:sup}
		\begin{split}
			\norm{\pr{\nabla_{\theta}\loss}(\theta,x)}
			=\norm{2A^*(A\theta-x)}
			\leq
			2\norm{A^*A\theta}+2\norm{A^*x}
			&\leq
			2\fc\norm{A\theta}+2\fc\norm{x}
			\\&\leq
			2\fc^2\norm{\theta}+2\fc\sqrt{\dim}\max\{\vass{a},\vass{b}\}
			\\&\leq
			2\fc^2(\norm{\theta}+1)
			.
		\end{split}
	\end{equation} 
	\Moreover 
	\eqref{eq:EX:i_th:partial} and the assumption that $\operatorname{Trace}( \operatorname{Cov}\pr{A^*X_{1,1}})>0$
	\prove\ that
	\begin{equation}
		\label{eq:ANN:EX0:properties:of:setup:1:constant}
		\begin{split}
			\textstyle
				\inf_{\theta\in\R^\pars}\PRb{
					\sum_{i=1}^\pars\var\prb{\prb{\tfrac{\partial\loss}{\partial\theta_{i}}}\pr{\theta,X_{1,1}}}}
				&=
				\textstyle
				\inf_{\theta\in\R^\pars}\PRb{
					\sum_{i=1}^\pars\var\prb{2\langle Ae_i,A\theta-X_{1,1}\rangle}}
				\\&=
				\textstyle
				\inf_{\theta\in\R^\pars}4\PRb{
					\sum_{i=1}^\pars\var\prb{\langle Ae_i,X_{1,1}\rangle}}
			\\&=
			\textstyle
			4\PRb{
			\inf_{\theta\in\R^\pars}
			\operatorname{Trace}\prb{ \operatorname{Cov}\prb{A^*X_{1,1}}}
		}
			>0
			.
		\end{split}
	\end{equation}
	This, 
	\eqref{eq:matrix:loss:and:bounds:easy:mat},
	\eqref{prop:eq:matrix:Adam:recursion:easy:mat},
	\eqref{eq:UB:gradient:gen:loss:sup},
	\eqref{eq:ANN:EX0:properties:of:setup:1:constant},
	and
	\cref{cor:ADAM:New} (applied with
	$\pars \curvearrowleft \pars$,
	$\dim \curvearrowleft \dim$,
	$g\curvearrowleft\nabla_{\theta}\loss$,
	$\const\curvearrowleft 2\fc^2$,
	$p\curvearrowleft 1$,
	$\batch \curvearrowleft \batch$,
	$\gamma \curvearrowleft \gamma$,
	$\bscl\curvearrowleft\pr{\N\ni n\mapsto (1-\beta^n)^{-1}\in[0,\infty)}$,
	$(X_{n,m})_{(n,m)\in\N^2} \curvearrowleft (X_{n,m})_{(n,m)\in\N^2}$,
	$\mom \curvearrowleft \mom$,
	$\MOM \curvearrowleft \MOM$,
	$\Theta\curvearrowleft\Theta$
	in the notation of \cref{cor:ADAM:New}) hence
	\prove\ that
	\begin{equation}
		\begin{split}
			\inf_{\substack{\xi\colon \Omega \to \R^\pars\\\text{measurable}}}
			\PRbbb{\limsup_{ n \to \infty }\E\PRb{ \norm{ \Theta_n - \xi }^2 }}
			>
			0
			.
		\end{split}
	\end{equation}
	This \proves\ \eqref{eq:matrix:Adam:easy:mat}.
	\finishproofthus
\end{proof}

\begin{athm}{cor}{prop:weak:opt:error:non-cvg}	
	Let $\pars\in\N$, 
	let $ ( \Omega, \cF, \P ) $ be a probability space,
	let $Y\colon\Omega\to\R^{\pars}$ be a random variable,
	let 
	$\loss\colon\R^{\pars}\times\R^{\pars}\to\R$,
	$\cL\colon\R^{\pars}\to\R$,
	$\batch\colon\N\to\N$, and
	$\gamma\colon\N\to[0,\infty)$
	satisfy
	for all 
	$\theta,x\in\R^\pars$
	that
	\begin{equation}
		\label{eq:matrix:loss:and:bounds:easy:mat:BA:weak}
		\textstyle
		\loss(\theta,x)=\norm{\theta-x}^2,
		\qquad
		\cL(\theta)=\E\PR{\loss(\theta,Y)},	
		\qqandqq
		\limsup_{n\to\infty}\PRb{(\gamma_n)^2\pr{\batch_n}^{-1}}>0,
	\end{equation} 
	let $a\in\R$,
	$b\in(a,\infty)$,
	let	$
	X_{n,j} \colon \Omega \to [ a, b ]^{\pars}
	$, $(n,j)\in\N^2$, be \iid\ random variables,
	let $\eps\in(0,\infty)$,
	$\alpha,\beta\in[0,1)$,
	let
	$ \Theta=(\Theta^{(1)},\ldots,\Theta^{(\pars)}) \colon\N_0\times\Omega \to \R^{\pars}$,
	$\mom=(\mom^{(1)},\ldots,\mom^{(\pars)})\colon\N_0\times\Omega\to\R^{\pars}$, and 
	$\MOM=(\MOM^{(1)},\ldots,\MOM^{(\pars)})\colon\N_0\times\Omega\to[0,\infty)^{\pars}$ be stochastic processes satisfying
	for all $n\in\N$, $i\in\{1,2,\dots,\pars\}$ that
	\begin{equation}
		\label{prop:eq:def:momentum:main:2.1:easy:mat:BA:weak}
		\begin{split}
			\mom[n]&\textstyle=\alpha \mom[n-1]+(1-\alpha)\PRb
			{\frac{1}{\batch_n}\sum_{j=1}^{\batch_n}\pr{\nabla_{\theta}\loss}(\Theta_{n-1},X_{n,j})},
		\end{split}
	\end{equation}
	\begin{equation}
		\label{prop:eq:def:RMS:factor:main:2.2:easy:mat:BA:weak}
		\begin{split}
			\MOM_{n}^{(i)}&\textstyle=\beta \MOM_{n-1}^{(i)}+(1-\beta)\PRb
			{\frac{1}{\batch_n}\sum_{j=1}^{\batch_n}\prb{\frac{\partial\loss}{\partial\theta_i}}(\Theta_{n-1},X_{n,j})}^2,
		\end{split}
	\end{equation}
	\begin{equation}
		\label{prop:eq:matrix:Adam:recursion:easy:mat:BA:weak}
		\begin{split}
			\text{and}
			\qquad
			\Theta_{ n }^{(i)}
			&=\Theta_{ n-1  }^{(i)}-\gamma_n \prb{\eps+\PR{(1-\beta^n)^{-1} \MOM_{n}^{(i)}}^{\nicefrac{1}{2}}}^{-1}\PRb{(1-\alpha^n)^{-1}\mom[n]^{(i)}},
		\end{split}
	\end{equation}
	assume that $(\Theta_0,\mom[0],\MOM_0)$ and $\pr{X_{n,j}}_{(n,j)\in\{(k,l)\in\N^2\colon l\leq\batch_k\}}$ are independent,
	and assume that
	$0<\allowbreak\operatorname{Trace}( \operatorname{Cov}\pr{X_{1,1}})$.
	Then 
	\begin{enumerate}[label=(\roman*)]
		\item 
		\label{it:old:cor}
		it holds that
		\begin{equation}
			\begin{split}
				\inf_{\substack{\xi\colon \Omega \to \R^\pars\\\text{measurable}}}\PRbbb{\limsup_{ n \to \infty }
					\E\PRb{ \norm{ \Theta_n - \xi }^2 }}
				>0
			\end{split}
		\end{equation}
		and
		\item 
		\label{it:new:cor}
		it holds that
			\begin{equation}
			\label{eq:matrix:Adam:easy:mat:BA:weak}
			\begin{split}
				\textstyle
				\limsup_{n\to\infty} \E\PR{ \cL( \Theta_n )}
				- \inf_{ \theta \in\R^{\pars}} \cL( \theta ) 
				>0.
			\end{split}
		\end{equation}
	\end{enumerate}
\end{athm}

\begin{aproof}
	Throughout this proof let $A\in\R^{\pars\times\pars}$ satisfy for all $v\in\R^{\pars}$ that $Av=v$.
	\Nobs that \eqref{eq:matrix:loss:and:bounds:easy:mat:BA:weak} \proves\ that
	\begin{equation}
		\textstyle
		\limsup_{n\to\infty}\PRb{((1-\alpha^n)^{-1}\gamma_n)^2\pr{\batch_n}^{-1}}
		\geq
		\limsup_{n\to\infty}\PRb{(\gamma_n)^2\pr{\batch_n}^{-1}}
		>0.
	\end{equation} 
	This and
	\cref{prop:prop:non_convergence_modified_Adam:easy:mat} (applied with
	$\pars\curvearrowleft\pars$,
	$\dim\curvearrowleft \pars$,
	$A\curvearrowleft A$,
	$\loss\curvearrowleft\loss$,
	$\batch\curvearrowleft\batch$,
	$\gamma\curvearrowleft\prb{\N\ni n\mapsto (1-\alpha^n)^{-1}\gamma_n\in[0,\infty)}$,
	$\Theta\curvearrowleft\Theta$,
	$\mom\curvearrowleft\mom$,
	$\MOM\curvearrowleft\MOM$
	in the notation of \cref{prop:prop:non_convergence_modified_Adam:easy:mat}) \prove\ that
	\begin{equation}
		\begin{split}
			\inf_{\substack{\xi\colon \Omega \to \R^\pars\\\text{measurable}}}\PRbbb{\limsup_{ n \to \infty }
				\E\PRb{ \norm{ \Theta_n - \xi }^2 }}
			>0.
		\end{split}
	\end{equation}
	This \proves[eip] \cref{it:old:cor}.
	\Nobs that \eqref{eq:matrix:loss:and:bounds:easy:mat:BA:weak} \proves\ that for all $\theta\in\R^{\pars}$ it holds that
	\begin{equation}
		\label{eq:setup:parial:stoch:loss}
		\pr{\nabla\cL}(\theta)
		=\E\PR{\pr{\nabla_{\theta}\loss}(\theta,Y)}
		=2(\theta-\E\PR{Y}).
	\end{equation}
	This \proves\ that for all
	$\theta\in\R^{\pars}$ it holds that
	\begin{equation}
		\langle\theta-\E\PR{Y},\pr{\nabla\cL}(\theta)\rangle
		=
		\langle\theta-\E\PR{Y},2(\theta-\E\PR{Y})\rangle
		=
		2\norm{\theta-\E\PR{Y}}^2
		.
	\end{equation}
	This, 
	\eqref{eq:matrix:loss:and:bounds:easy:mat:BA:weak}, and
	\cite[Lemma~5.7.29]{jentzen2023mathematical} (applied with
	$\pars\curvearrowleft\pars$,
	$c\curvearrowleft 2$,
	$r\curvearrowleft \infty$,
	$\vartheta\curvearrowleft\E\PR{Y}$,
	$\mathbb{B}\curvearrowleft\R^{\pars}$,
	$\mathscr{L} \curvearrowleft \cL$
	in the notation of \cite[Lemma~5.7.29]{jentzen2023mathematical}) \prove\ that for all
	$\theta\in\R^{\pars}$ it holds that
	\begin{equation}
		\label{eq:statment:com:weak:book}
		\textstyle
		\cL(\theta)-\inf_{\vartheta\in\R^{\pars}}\cL(\vartheta)
		=
		\cL(\theta)-\cL(\E\PR{Y})
		\geq
		\norm{\theta-\E\PR{Y}}^2
		.
	\end{equation}
	\Moreover
	\eqref{eq:matrix:loss:and:bounds:easy:mat:BA:weak},
	\eqref{prop:eq:def:momentum:main:2.1:easy:mat:BA:weak},
	\eqref{prop:eq:def:RMS:factor:main:2.2:easy:mat:BA:weak},
	\eqref{prop:eq:matrix:Adam:recursion:easy:mat:BA:weak}, and
	\cref{it:old:cor} \prove\ that
		\begin{equation}
				\limsup_{ n \to \infty }
					\E\PRb{ \norm{ \Theta_n - \E\PR{Y} }^2 }
					\geq
						\inf_{\substack{\xi\colon \Omega \to \R^\pars\\\text{measurable}}}\PRbbb{\limsup_{ n \to \infty }
						\E\PRb{ \norm{ \Theta_n - \xi }^2 }}
				>0
				.
	\end{equation}
	Combining this with \eqref{eq:statment:com:weak:book} \proves\ that
	\begin{equation}
		\begin{split}
		\textstyle
		\limsup_{n\to\infty} \prb{\E\PR{ \cL( \Theta_n )}
		- \inf_{ \theta \in\R^{\pars}} \cL( \theta ) 
	}
	&\textstyle=
	\limsup_{n\to\infty} \E\PRb{ \cL( \Theta_n )
		- \inf_{ \theta \in\R^{\pars}} \cL( \theta ) 
	}
	\\&\textstyle\geq
	\limsup_{n\to\infty} \E\PRb{
			\norm{\Theta_n-\E\PR{Y}}^2
		}
		>0
		.
		\end{split}
	\end{equation}
	This \proves[pei] \cref{it:new:cor}.
	\finishproofthus
\end{aproof}

\cfclear
\begin{athm}{cor}{cor:ADAM}
	Let 
	$\pars,\dim\in\N$, 
	$a\in\R$,
	$b\in[a,\infty)$,
	$\eps\in(0,\infty)$,
	$\alpha\in[0,1)$,
	$\beta\in(\alpha^2,1)$,
	$z\in\{-1,1\}$,
	let 
	$\batch\colon\N\to\N$,
	$\gamma\colon\N\to[0,\infty)$,
	 and $\bscl\colon\N\to(0,\infty)$ satisfy for all $s\in\Z$ that
	\begin{equation}
		\label{eq:bounded:batches:and:LR:ADAM}
		z\PRbb{\limsup_{n\to\infty}\pr{z\gamma_n}}>0,
		\quad
		\limsup_{n\to\infty}\PRbb{\gamma_n+\batch_n+(\bscl_n)^{s}}<\infty,
		\qandq
		\liminf_{n\to\infty}\batch_n>-z,
	\end{equation}
	let $g=(g_1,\dots,g_{\pars})\colon\R^{\pars}\times[a,b]^{\dim}\to\R^{\pars}$ be measurable, 
	let $ ( \Omega, \cF, \P ) $ be a probability space,
	let $X_{n,m}\colon\Omega\to[a,b]^\dim$, $(n,m)\in\N^2$, be \iid\ random variables,
	let
	$ \Theta=(\Theta^{(1)},\ldots,\Theta^{(\pars)}) \colon\N_0\times\Omega \to \R^{\pars}$,
	$\mom=(\mom^{(1)},\ldots,\mom^{(\pars)})\colon\N_0\times\Omega\to\R^{\pars}$, and 
	$\MOMALT=(\MOMALT^{(1)},\ldots,\MOMALT^{(\pars)})\colon\N_0\times\Omega\to[0,\infty)^{\pars}$ be stochastic processes which satisfy
	for all $n\in\N$, $i\in\{1,2,\dots,\pars\}$ that
	\begin{equation}
		\label{eq:representation:ADAM:1}
		\begin{split}
			\mom[n]&\textstyle=\alpha \mom[n-1]+(1-\alpha)\PRb
			{\frac{1}{\batch_{n}}\sum_{j=1}^{\batch_{n}}g(\Theta_{n-1},X_{n,j})},
		\end{split}
	\end{equation}
	\begin{equation}
		\label{eq:representation:ADAM:2}
		\begin{split}
			\MOMALT[n]^{(i)}&\textstyle=\beta \MOMALT[n-1]^{(i)}+(1-\beta)\PRb
			{\frac{1}{\batch_{n}}\sum_{j=1}^{\batch_{n}}g_i(\Theta_{n-1},X_{n,j})}^2,
		\end{split}
	\end{equation}
	\begin{equation}
		\label{eq:representation:ADAM:3}
		\begin{split}
			\qqandqq
			\Theta_{ n }^{(i)}
			&=\Theta_{ n-1  }^{(i)}-\gamma_n\prb{\eps+\PRb{\bscl_n\MOMALT[n]^{(i)}}^{\nicefrac{1}{2}}}^{-1}
			\mom[n]^{(i)}
			,
		\end{split}
	\end{equation}
	assume that $(\Theta_0,\mom[0],\MOMALT[0])$ and $\pr{X_{n,j}}_{(n,j)\in\{(k,l)\in\N^2\colon l\leq\batch_k\}}$ are independent, 
	let 
	$\fc\in\R$, $\cst,\Cst\in(0,\infty)$,
	$i\in\{1,2,\dots,\pars\}$ satisfy
	for all
	$\theta=(\theta_1,\dots,\theta_\pars)\in\R^\pars$, $x\in[a,b]^\dim$
	that
	\begin{equation}
		\label{prop:eq:setup:new:rob:cond:ADAM}
		\pr{\theta_i-\fc}
		\prb{\cst+(\Cst-\cst)\indicator{(-\infty,\fc]}(\theta_i)}
		\leq
		g_i(\theta,x)
		\leq
		\pr{\theta_i+\fc}
		\prb{\cst+(\Cst-\cst)\indicator{[-\fc,\infty)}(\theta_i)}
		,
	\end{equation}
	\begin{equation}
		\label{eq:start:mom:bound:new:rob:cond:ADAM}
		\max\pRb{
			\vass{\mom[0]^{(i)}},
			\vass{\MOMALT[0]^{(i)}}^{\nicefrac{1}{2}}
		}
		\leq \Cst(\vass{\Theta_{0}^{(i)}}+\fc),
		\qqandqq
		\E\PRb{\vass{\Theta^{(i)}_0}}<\infty,
	\end{equation}
	and assume $\textstyle\inf_{\theta\in\R^{\pars}}\var\prb{g_i(\theta,X_{1,1})}>0$.
	Then 
	\begin{equation}
		\label{eq:ADAM:LB}
		\begin{split}
			\inf_{\substack{\xi\colon \Omega \to \R\\\text{measurable}}}
			\prbbb{
			z\PRbbb{\limsup_{ n \to \infty }
				\prb{z\E\PRb{ | \Theta_n^{(i)} - \xi | }}}
			}>
			0
			.
		\end{split}
	\end{equation}
\end{athm}

\begin{proof}[Proof of \cref{cor:ADAM}]
	Throughout this proof 
	let $\const\in\R$ satisfy $\const=\max\{1,\fc,\sup_{n\in\N}\PR{\bscl_n+(\bscl_n)^{-\nicefrac{1}{2}}}\}$,
	for every $n\in\N$ let $X_n=(X_{n,1},\dots,X_{n,\batch_n})$,
	let
	$\MOMalt=\prb{\MOMalt[1],\dots,\MOMalt[\pars]}\colon\Omega\to\R^\pars$ satisfy for all $\mathfrak{i}\in\{1,2,\dots,\pars\}$ that
	$
	\MOMalt[\mathfrak{i}]=\PRb{\pr{1-\beta}^{-1}\MOMALT[0]^{(\mathfrak{i})}}^{\nicefrac{1}{2}}
	$, and
	for every $n\in\N$ let
	$G_{n}\colon\R^\pars \times\pr{[a,b]^{\dim}}^{\batch_n}\to\R^\pars$
	and
	$\MOMs[n]\colon\R^{n+1}\to\R$ satisfy for all $\theta\in\R^{\pars}$, $x=(x_1,\dots,x_{\batch_n})\in\pr{[a,b]^{\dim}}^{\batch_n}$,
	$y=(y_1,\dots,y_{n+1})\in\R^{n+1}$ that
	\begin{equation}
		\label{eq:representation:ADAM:gradient}
		\textstyle
		G_n(\theta,x)=\PRb{\frac{1}{\batch_{n}}\sum_{j=1}^{\batch_{n}}g(\theta,x_j)}
		\qqandqq
		\MOMs[n](y)=
		\textstyle
		\PRb{\bscl_n
			\PRb{\sum_{k=0}^n (1-\beta)\beta^{n-k}\pr
				{y_{k+1}}^2}}^{\nicefrac{1}{2}}
		.
	\end{equation}
	\eqref{eq:representation:ADAM:1},
	\eqref{eq:representation:ADAM:2},
	\eqref{eq:representation:ADAM:3},
	\eqref{eq:representation:ADAM:gradient}, and
	\cref{lem:repr:adam:and:prop} (applied with
	$\Theta\curvearrowleft\Theta$,
	$G\curvearrowleft\prb{\N\times\Omega\ni(n,\omega)\mapsto G_n(\Theta_{n-1}(\omega),X_n(\omega))\in\R^{\pars}}$,
	$(\momf_n)_{n\in\N}\curvearrowleft(\MOMs[n])_{n\in\N}$ in the notation of \cref{lem:repr:adam:and:prop}) \prove\ that
	\begin{enumerate}[label=(\Roman*)]
		\item\label{it2:lem:repr:adam:and:prop:repr}
		it holds for all $n\in\N$, $\mathfrak{i}\in\{1,2,\dots,\pars\}$ that
		\begin{equation}
			\begin{split}
				\textstyle
				\Theta_{ n }^{(\mathfrak{i})}
				&=\Theta_{ n-1  }^{(\mathfrak{i})}-\frac{\gamma_{n}  \PR{\alpha^n\mom[0]^{(\mathfrak{i})}+(1-\alpha)\sum_{k=1}^n\alpha^{n-k}G_k^{(\mathfrak{i})}(\Theta_{k-1},X_k)}}{\eps+\MOMs[n]
					\prb{\MOMalt[\mathfrak{i}],G_1^{(\mathfrak{i})}(\Theta_{0},X_1),G_2^{(\mathfrak{i})}(\Theta_{1},X_2),\dots,G_n^{(\mathfrak{i})}(\Theta_{n-1},X_n)}},
			\end{split}
		\end{equation}
		\item\label{it2:lem:repr:adam:and:prop:LB} it holds
		for all $n\in\N$, $x=(x_1,\dots,x_{n+1})\in\R^{n+1}$ that
		\begin{equation}
			\label{it2:eq:momf:1}
			\textstyle
			\tfrac{1}{\fC}\PRb{(1-\beta)\sum_{k=0}^n\beta^{n-k}(x_{k+1})^2
			}^{\nicefrac{1}{2}}
			\leq
			\MOMs[n](x)
			,
		\end{equation}
		\item\label{it2:lem:repr:adam:and:prop:UB} it holds
		for all $n\in\N$, $x=(x_1,\dots,x_{n+1})\in\R^{n+1}$ that
		\begin{equation}
			\label{it2:eq:momf:2}
			\textstyle
			\MOMs[n](x)
			\leq
			\sqrt{\fC}
			\max\pR{\vass{x_1},\vass{x_2},\dots,\vass{x_{n+1}}},
		\end{equation}
		\item\label{it2:lem:repr:adam:and:prop:growth} it holds
		for all $n\in\N$, $x=(x_1,\dots,x_{n+1})\in\R^{n+1}$, $y\in(-\infty,x_{n+1})$ that
		\begin{equation}
			\label{it2:eq:momf:3}
			\begin{split}
				y\MOMs[n](x)\leq x_{n+1}\MOMs[n](x_1,x_2,\dots,x_{n},y),
			\end{split}
		\end{equation}
		and
		\item\label{it2:lem:repr:adam:and:prop:square} it holds
		for all $n\in\N$, $x=(x_1,\dots,x_{n+1})\in\R^{n+1}$ that
		\begin{equation}
			\PRb{\MOMs[n](x)}^2
			=
			\PRb{\MOMs[n](x_1,\dots,x_n,0)}^2
			+(x_{n+1})^2\PRb{\MOMs[n](0,\dots,0,1)}^2.
		\end{equation}
	\end{enumerate}
	\Nobs that \eqref{eq:representation:ADAM:gradient} and \cref{lem:end:filtration} 
	(applied with 
	$\batch\curvearrowleft\batch$,
	$\Batch\curvearrowleft\batch$,
	$\BATCH\curvearrowleft\batch$,
	$Y\curvearrowleft (\Theta_0,\mom[0],\MOMalt)$,
	$(X_{n,j})_{(n,j)\in\{(m,i)\in\N^2\colon i\leq\BATCH_m\}}\curvearrowleft\pr{X_{n,j}}_{(n,j)\in\{(m,i)\in\N^2\colon i\leq\batch_m\}}$
	in the notation of \cref{lem:end:filtration}) \prove\ that there exist $\fil_n\subseteq\cF$, $n\in\N_0$, which satisfy that
	\begin{enumerate}[label=(\Alph*)]
		\item \label{it:filtered:prob:space:ADAM}it holds that $ ( \Omega, \cF,(\mathbb{F}_n)_{n\in\N_0}, \P ) $ is a filtered probability space,
		\item \label{it:filter:independence:ADAM}it holds for all $n\in\N$ that
		$(X_{n,1},X_{n,2},\dots,X_{n,\batch_n})$ is independent of $\fil_{n-1}$,
		\item \label{it:filter:measurable:ADAM}it holds for all $n\in\N$ that $(X_{n,1},X_{n,2},\dots,X_{n,\batch_n})$ is $\fil_n$-measurable, and
		\item \label{it:filter:measurable2:ADAM}it holds that $(\Theta_0,\mom[0],\MOMalt)$ is $\fil_0$-measurable.
	\end{enumerate}
	\Nobs that \eqref{eq:start:mom:bound:new:rob:cond:ADAM} and \eqref{eq:representation:ADAM:gradient} \prove\ that
	\begin{equation}
		\label{eq:UB:start:mom:ADAM}
		\begin{split}
			\max\pRb{\vass{\MOMalt[i]},
				\vass{\mom[0]^{(i)}}
			}
			&=
			\max\pRb{\vass{\MOMALT[0]^{(i)}}^{\nicefrac{1}{2}}\pr{1-\beta}^{-\nicefrac{1}{2}},
				\vass{\mom[0]^{(i)}}
			}
			\\&\leq
			\max\pRb{\vass{\MOMALT[0]^{(i)}}^{\nicefrac{1}{2}}\pr{1-\beta}^{-1},
				\vass{\mom[0]^{(i)}}
			}
			\\&\leq
			\pr{1-\beta}^{-1}
			\max\pRb{\vass{\MOMALT[0]^{(i)}}^{\nicefrac{1}{2}},
				\vass{\mom[0]^{(i)}}
			}
			\leq
			\pr{1-\beta}^{-1}
			\Cst(\vass{\Theta_{0}^{(i)}}+\fc)
			.
		\end{split}
	\end{equation}
	Combining 
	\cref{thm:new:rob:cond}
	(applied with 
	$\fc\curvearrowleft \const$,
	$\cst \curvearrowleft \cst$,
	$\Cst \curvearrowleft \pr{1-\beta}^{-1}
	\Cst$,
	$z\curvearrowleft z$,
	$\gamma \curvearrowleft \gamma$,
	$\batch\curvearrowleft \batch$,
	$(\MOM_n)_{n\in\N} \curvearrowleft (\MOMs[n])_{n\in\N}$,
	$(G_n)_{n\in\N} \curvearrowleft (G_n)_{n\in\N}$,
	$ ( \Omega, \cF,(\mathbb{F}_n)_{n\in\N_0}, \P ) \curvearrowleft( \Omega, \cF,(\mathbb{F}_n)_{n\in\N_0}, \P ) $,
	$(X_n)_{n\in\N}\curvearrowleft (X_{n,1},X_{n,2},\dots,X_{n,\batch_n})_{n\in\N}$,
	$\MOMs[1] \curvearrowleft \mom[0]$,
	$\MOMs[2] \curvearrowleft \MOMalt$,
	$\Theta \curvearrowleft \Theta$,
	$i \curvearrowleft i$
	in the notation of \cref{thm:new:rob:cond}),
	\eqref{eq:bounded:batches:and:LR:ADAM},
	\eqref{prop:eq:setup:new:rob:cond:ADAM},
	\eqref{eq:start:mom:bound:new:rob:cond:ADAM},
	\eqref{eq:representation:ADAM:gradient}, and
	\eqref{eq:UB:start:mom:ADAM} with
	\cref{it2:lem:repr:adam:and:prop:repr,it2:lem:repr:adam:and:prop:LB,it2:lem:repr:adam:and:prop:UB,it2:lem:repr:adam:and:prop:growth,it:filtered:prob:space:ADAM,it:filter:independence:ADAM,it:filter:measurable:ADAM,it:filter:measurable2:ADAM,it2:lem:repr:adam:and:prop:square}
	hence \proves\ that 
	\begin{equation}
		\begin{split}
			\inf_{\substack{\xi\colon \Omega \to \R\\\text{measurable}}}
			\prbbb{
				z\PRbbb{\limsup_{ n \to \infty }
					\prb{z\E\PRb{ | \Theta_n^{(i)} - \xi | }}}
			}>
			0
			.
		\end{split}
	\end{equation}
	This \proves\ \cref{eq:ADAM:LB}. \finishproofthus
\end{proof}



\subsection{RMSprop optimization}

\cfclear
\begin{athm}{cor}{cor:RMSprop:New}
	Let $\pars,\dim\in\N$, 
	$a\in\R$,
	$b\in[a,\infty)$,
	let $g=(g_1,\dots,g_{\pars})\colon\R^{\pars}\times[a,b]^{\dim}\to\R^{\pars}$ be measurable, 
	let $p,\const\in(0,\infty)$,
	let 
	$\batch\colon\N\to\N$,
	$\gamma\colon\N\to[0,\infty)$,
	 and $\bscl\colon\N\to[0,\infty)$ satisfy for all $\theta\in\R^\pars$, $x\in[a,b]^\dim$ that
	\begin{equation}
		\label{eq:bounded:batches:and:LR:RMSprop:New}
		\textstyle
		\norm{
			g(\theta,x)}
		\leq
		\const
		\pr{\norm{\theta}+1}^p,
		\!\quad
		\limsup_{n\to\infty}\PRb{(\gamma_n)^2\pr{\batch_n}^{-1}}>0,
		\!\qandq\!
		\limsup_{n\to\infty}\bscl_n<\infty,
	\end{equation}
	let $ ( \Omega, \cF, \P ) $ be a probability space,
	let $X_{n,m}\colon\Omega\to[a,b]^\dim$, $(n,m)\in\N^2$, be \iid\ random variables,
	let $\eps\in(0,\infty)$,
	$\beta\in[0,1)$,
	let
	$ \Theta=(\Theta^{(1)},\ldots,\Theta^{(\pars)}) \colon\N_0\times\Omega \to \R^{\pars}$ and 
	$\MOM=(\MOM^{(1)},\ldots,\MOM^{(\pars)})\colon\N_0\times\Omega\to[0,\infty)^{\pars}$ be stochastic processes which satisfy
	for all $n\in\N$, $i\in\{1,2,\dots,\pars\}$ that
	\begin{equation}
		\label{eq:representation:RMSprop:2:New}
		\begin{split}
			\MOM_n^{(i)}&\textstyle=\beta \MOM_{n-1}^{(i)}+(1-\beta)\PRb
			{\frac{1}{\batch_{n}}\sum_{j=1}^{\batch_{n}}g_i(\Theta_{n-1},X_{n,j})}^2,
		\end{split}
	\end{equation}
	\begin{equation}
		\label{eq:representation:RMSprop:3:New}
		\begin{split}
			\qqandqq
			\textstyle
			\Theta_{ n }^{(i)}
			&\textstyle=\Theta_{ n-1  }^{(i)}-\gamma_n\prb{\eps+\PRb{\bscl_n\MOM_n^{(i)}}^{\nicefrac{1}{2}}}^{-1}
			\PRb{\frac{1}{\batch_{n}}\sum_{j=1}^{\batch_{n}}g_i(\Theta_{n-1},X_{n,j})}
			,
		\end{split}
	\end{equation}
	assume that
	\begin{equation}
		\label{eq:start:mom:bound:new:rob:cond:RMSprop:New}
		\textstyle
		\inf_{\theta\in\R^\pars}\PRb{\sum_{i=1}^\pars\var\pr{g_i\pr{\theta,X_{1,1}}}}>0
		,
	\end{equation}
	and assume that $(\Theta_0,\MOM_0)$ and $\pr{X_{n,j}}_{(n,j)\in\{(k,l)\in\N^2\colon l\leq\batch_k\}}$ are independent.
	Then 
	\begin{equation}
		\label{eq:RMSprop:LB:New}
		\begin{split}
			\inf_{\substack{\xi\colon \Omega \to \R^\pars\\\text{measurable}}}
			\PRbbb{\limsup_{ n \to \infty }\E\PRb{ \norm{\Theta_n - \xi }^{2} }}
			>
			0
			.
		\end{split}
	\end{equation}
\end{athm}

\begin{proof}[Proof of \cref{cor:RMSprop:New}]
 	\Nobs that 
 	\eqref{eq:bounded:batches:and:LR:RMSprop:New},
 	\eqref{eq:representation:RMSprop:2:New},
 	\eqref{eq:representation:RMSprop:3:New},
 	\eqref{eq:start:mom:bound:new:rob:cond:RMSprop:New},
 	 and
 	\cref{cor:ADAM:New} (applied with 
 	$\alpha\curvearrowleft 0$,
 	$\beta \curvearrowleft \beta$,
 	$\mom\curvearrowleft\prb{\prb{\N_0\times\Omega}\ni(n,\omega)\mapsto\indicator{\N}(n)\PRb
 		{\frac{1}{\batch_{n}}\sum_{j=1}^{\batch_{n}}g(\Theta_{n-1}(\omega),X_{n,j}(\omega))}\in\R^\pars}$,
 	$\MOM\curvearrowleft\MOM$,
 	$\Theta\curvearrowleft\Theta$
 	in the notation of \cref{cor:ADAM:New}) \prove\ that
 		\begin{equation}
 		\begin{split}
 			\inf_{\substack{\xi\colon \Omega \to \R^\pars\\\text{measurable}}}
 			\PRbbb{\limsup_{ n \to \infty }\E\PRb{ \norm{\Theta_n - \xi }^{2} }}
 			>
 			0
 			.
 		\end{split}
 	\end{equation}
 	This \proves[pei]\ \eqref{eq:RMSprop:LB:New}.
 	\finishproofthus
\end{proof}

\subsection{Adamax optimization}

\cfclear
\begin{athm}{cor}{cor:AdaMax:New}
	Let $\pars,\dim\in\N$, 
	$a\in\R$,
	$b\in[a,\infty)$,
	let $g=(g_1,\dots,g_{\pars})\colon\R^{\pars}\times[a,b]^{\dim}\to\R^{\pars}$ be measurable, 
	let $p,\const\in(0,\infty)$,
	let 
	$\batch\colon\N\to\N$ and
	$\gamma\colon\N\to[0,\infty)$ satisfy for all $\theta\in\R^\pars$, $x\in[a,b]^\dim$ that
	\begin{equation}
		\label{eq:bounded:batches:and:LR:AdaMax:New}
		\norm{
			g(\theta,x)}
		\leq
		\const
		\pr{\norm{\theta}+1}^p
		\qqandqq
		\limsup_{n\to\infty}\PRb{(\gamma_n)^2\pr{\batch_n}^{-1}}>0,
	\end{equation}
	let $ ( \Omega, \cF, \P ) $ be a probability space,
	let $X_{n,m}\colon\Omega\to[a,b]^\dim$, $(n,m)\in\N^2$, be \iid\ random variables,
	let $\eps\in(0,\infty)$,
	$\alpha,\beta\in[0,1)$,
	let
	$ \Theta=(\Theta^{(1)},\ldots,\Theta^{(\pars)}) \colon\N_0\times\Omega \to \R^{\pars}$,
	$\mom=(\mom^{(1)},\ldots,\mom^{(\pars)})\colon\N_0\times\Omega\to\R^{\pars}$, and 
	$\MOMALT=(\MOMALT^{(1)},\ldots,\MOMALT^{(\pars)})\colon\N_0\times\Omega\to[0,\infty)^{\pars}$ be stochastic processes which satisfy
	for all $n\in\N$, $i\in\{1,2,\dots,\pars\}$ that
	\begin{equation}
		\label{eq:representation:AdaMax:1:New}
		\begin{split}
			\mom[n]&\textstyle=\alpha \mom[n-1]+(1-\alpha)\PRb
			{\frac{1}{\batch_{n}}\sum_{j=1}^{\batch_{n}}g(\Theta_{n-1},X_{n,j})},
		\end{split}
	\end{equation}
	\begin{equation}
	\label{eq:representation:AdaMax:2}
	\begin{split}
		\MOMALT[n]^{(i)}&\textstyle=\max\pRb{\beta\MOMALT[n-1]^{(i)},
			\vass
			{\frac{1}{\batch_{n}}\sum_{j=1}^{\batch_{n}}g_i(\Theta_{n-1},X_{n,j})}},
	\end{split}
\end{equation}
\begin{equation}
	\label{eq:representation:AdaMax:3}
	\begin{split}
		\qqandqq
		\Theta_{ n }^{(i)}
		&=\Theta_{ n-1  }^{(i)}-\gamma_n\prb{\eps+\MOMALT[n]^{(i)}}^{-1}
		\mom[n]^{(i)}
		,
	\end{split}
\end{equation}
	assume that
	\begin{equation}
		\label{eq:start:mom:bound:new:rob:cond:AdaMax:New}
		\textstyle
		\inf_{\theta\in\R^\pars}\PRb{\sum_{i=1}^\pars\var\pr{g_i\pr{\theta,X_{1,1}}}}>0,
	\end{equation}
	and assume that $(\Theta_0,\mom[0],\MOMALT[0])$ and $\pr{X_{n,j}}_{(n,j)\in\{(k,l)\in\N^2\colon l\leq\batch_k\}}$ are independent.
	Then 
	\begin{equation}
		\label{eq:AdaMax:LB:New}
		\begin{split}
			\inf_{\substack{\xi\colon \Omega \to \R^\pars\\\text{measurable}}}
			\PRbbb{\limsup_{ n \to \infty }\E\PRb{ \norm{\Theta_n - \xi }^{2} }}
			>
			0
			.
		\end{split}
	\end{equation}
\end{athm}

\begin{proof}[Proof of \cref{cor:AdaMax:New}]
	Throughout this proof 
	assume without loss of generality that $\const\geq(1-\beta)^{-2}$,
	for every $n\in\N$ let $X_n=(X_{n,1},\dots,X_{n,\batch_n})$, and
	for every $n\in\N$ let
	$G_{n}\colon\R^\pars \times\pr{[a,b]^{\dim}}^{\batch_n}\to\R^\pars$
	and
	$\MOMs[n]\colon\R^{n+1}\to\R$ satisfy for all $\theta\in\R^{\pars}$, $x=(x_1,\dots,x_{\batch_n})\in\pr{[a,b]^{\dim}}^{\batch_n}$,
	$y=(y_1,\dots,y_{n+1})\in\R^{n+1}$ that
	\begin{equation}
		\label{eq:representation:AdaMax:gradient:New}
		\textstyle
		G_n(\theta,x)=\frac{1}{\batch_{n}}\sum_{j=1}^{\batch_{n}}g(\theta,x_j)
		\qqandqq
		\MOMs[n](y)=
		\textstyle
		\max_{k\in\{0,1,\dots,n\}}\beta^{n-k}\vass{y_{k+1}}
		.
	\end{equation}
		\Nobs that \eqref{eq:start:mom:bound:new:rob:cond:AdaMax:New} and \eqref{eq:representation:AdaMax:gradient:New} \prove\ that 
	\begin{equation}
		\label{eq:double:inf:var:for:AdaMax}
		\begin{split}
			\inf_{n\in\N}\inf_{\theta\in\R^\pars}
			\textstyle\PRb{\batch_n\sum_{i=1}^\pars\var\prb{G_n^{(i)}\pr{\theta,X_n}}}
			&=
			\inf_{n\in\N}\inf_{\theta\in\R^\pars}
			\textstyle\PRb{\batch_n\sum_{i=1}^\pars\var\prb{\frac{1}{\batch_{n}}\sum_{j=1}^{\batch_{n}}g_i(\theta,X_{n,j})}}
			\\&=
			\inf_{n\in\N}\inf_{\theta\in\R^\pars}
			\textstyle\PRb{\prb{\batch_n}^{-1}\sum_{i=1}^\pars\var\prb{\sum_{j=1}^{\batch_{n}}g_i(\theta,X_{n,j})}}
			\\&=
			\inf_{n\in\N}\inf_{\theta\in\R^\pars}
			\textstyle\PRb{\sum_{i=1}^\pars\var\prb{g_i(\theta,X_{1,1})}}
			>0
			.
		\end{split}
	\end{equation}
	\Moreover \eqref{eq:representation:AdaMax:gradient:New} \proves\
	for all $n\in\N$, $x=(x_1,\dots,x_{n+1})\in\R^{n+1}$ that
	\begin{equation}
		\label{eq:LB:AdaMax:denominator:UB:New}
		\begin{split}
			\MOMs[n](x)
			=
			\textstyle
			\max_{k\in\{0,1,\dots,n\}}\beta^{n-k}\vass{x_{k+1}}
			&\leq
			\textstyle\sum_{k=0}^n \beta^{\frac{n-k}{2}}\vass{x_{k+1}}\beta^{\frac{n-k}{2}}
			\\&\leq
			\textstyle
			\PRb{\sum_{k=0}^n\beta^{n-k}\pr{x_{k+1}}^2}^{\nicefrac{1}{2}}
			\PRb{\sum_{k=0}^n\beta^{n-k}}^{\nicefrac{1}{2}}
			\\&\leq
			\textstyle
			\PRb{\tfrac{1-\beta}{1-\beta}\sum_{k=0}^n\beta^{n-k}\pr{x_{k+1}}^2}^{\nicefrac{1}{2}}
			\PRb{\sum_{k=0}^{\infty}\beta^{k}}^{\nicefrac{1}{2}}
			\\&=
			\textstyle
			\PRb{(1-\beta)^{-2}\sum_{k=0}^n\pr{1-\beta}\beta^{n-k}\pr{x_{k+1}}^2}^{\nicefrac{1}{2}}
			\\&\leq\textstyle
			\PRb{\const
				\PRb{\sum_{k=0}^n (1-\beta)\beta^{n-k}\pr
					{x_{k+1}}^2}}^{\nicefrac{1}{2}}
			.
		\end{split}
	\end{equation}
	\Nobs that \eqref{eq:representation:AdaMax:gradient:New} and \cref{lem:end:filtration} 
	(applied with 
	$Y\curvearrowleft (\Theta_0,\mom[0],\MOMALT[0])$,
	$\batch\curvearrowleft\batch$,
	$\Batch\curvearrowleft\batch$,
	$\BATCH\curvearrowleft\batch$,
	$(X_{n,j})_{(n,j)\in\{(m,i)\in\N^2\colon i\leq\BATCH_m\}}\curvearrowleft\pr{X_{n,j}}_{(n,j)\in\{(m,i)\in\N^2\colon i\leq\batch_m\}}$
	 in the notation of \cref{lem:end:filtration}) \prove\ that there exist $\fil_n\subseteq\cF$, $n\in\N_0$, which satisfy that
	\begin{enumerate}[label=(\Roman*)]
		\item \label{it:filtered:prob:space:AdaMax:New}it holds that $ ( \Omega, \cF,(\mathbb{F}_n)_{n\in\N_0}, \P ) $ is a filtered probability space,
		\item \label{it:filter:independence:AdaMax:New}it holds for all $n\in\N$ that 
		$(X_{n,1},X_{n,2},\dots,X_{n,\batch_n})$ is independent of
		$\fil_{n-1}$,
		\item \label{it:filter:measurable:AdaMax:New}it holds for all $n\in\N$ that $(X_{n,1},X_{n,2},\dots,X_{n,\batch_n})$ is $\fil_n$-measurable, and
		\item \label{it:filter:measurable2:AdaMax:New}it holds that $(\Theta_0,\mom[0],\MOMALT[0])$ is $\fil_0$-measurable.
	\end{enumerate}
	\Moreover 
	\eqref{eq:representation:AdaMax:1:New},
	\eqref{eq:representation:AdaMax:gradient:New}, and
	\cite[Lemma~2.7]{DereichGraeberJentzenRiekert2025} (applied with
	$\pars\curvearrowleft\pars$,
	$(\beta_n)_{n\in\N}\curvearrowleft(\alpha)_{n\in\N}$,
	$\mom\curvearrowleft\mom$,
	$(g_n)_{n\in\N}\curvearrowleft \pr{\frac{1-\alpha}{\batch_{n}}\sum_{j=1}^{\batch_{n}}g(\Theta_{n-1},X_{n,j})}_{n\in\N}$
	in the notation of \cite[Lemma~2.7]{DereichGraeberJentzenRiekert2025}) \prove\ that for all $n\in\N$ it holds that 
	\begin{equation}
		\label{prop:eq:repr:momentum:seq2:AdaMax:New}
		\begin{split}
			\mom[n]
			&=\textstyle
			\alpha^n\mom[0]
			+\sum_{k=1}^{n}\alpha^{n-k}\pr{\frac{1-\alpha}{\batch_{k}}\sum_{j=1}^{\batch_{k}}g(\Theta_{k-1},X_{k,j})}
			\\&=\textstyle
			\alpha^n\mom[0]
			+(1-\alpha)\sum_{k=1}^{n}\alpha^{n-k}G_k\pr{\Theta_{k-1},X_k}
			.
		\end{split}
	\end{equation}
	\Moreover 
	\eqref{eq:representation:AdaMax:2} and
	\eqref{eq:representation:AdaMax:gradient:New}
	\prove\ that for all $n\in\N$, $\mathfrak{i}\in\{1,2,\dots,\pars\}$ it holds that 
	\begin{equation}
	\begin{split}
		\MOMALT[n]^{(\mathfrak{i})}
		&=\textstyle\max\pRb{\beta\MOMALT[n-1]^{(\mathfrak{i})},
			\vass
			{\frac{1}{\batch_{n}}\sum_{j=1}^{\batch_{n}}g_{\mathfrak{i}}(\Theta_{n-1},X_{n,j})}}
		\\&=\textstyle\max\pRb{\beta\max\pRb{\beta\MOMALT[n-2]^{(\mathfrak{i})},
				\vass
				{\frac{1}{\batch_{n-1}}\sum_{j=1}^{\batch_{n-1}}g_{\mathfrak{i}}(\Theta_{n-2},X_{n-1,j})}},
			\vass
			{G_n^{(\mathfrak{i})}\pr{\Theta_{n-1},X_n}}}
		\\&=\textstyle\max\pRb{\beta^2\MOMALT[n-2]^{(\mathfrak{i})},\beta
			\vass
			{G_{n-1}^{(\mathfrak{i})}\pr{\Theta_{n-2},X_{n-1}}},
			\vass
			{G_n^{(\mathfrak{i})}\pr{\Theta_{n-1},X_n}}}
		\\&=\dots
		\\&=\textstyle\max\pRb{\beta^n\MOMALT[n-n]^{(\mathfrak{i})},
			\beta^{n-1}
			\vass
			{G_{1}^{(\mathfrak{i})}\pr{\Theta_{0},X_{1}}},
			\beta^{n-2}
			\vass
			{G_{2}^{(\mathfrak{i})}\pr{\Theta_{1},X_{2}}},
			\dots,
			\vass
			{G_n^{(\mathfrak{i})}\pr{\Theta_{n-1},X_n}}}
		\\&=\MOMs[n]
		\prb{\MOMALT[0]^{(\mathfrak{i})},G_1^{(\mathfrak{i})}(\Theta_{0},X_1),G_2^{(\mathfrak{i})}(\Theta_{1},X_2),\dots,G_n^{(\mathfrak{i})}(\Theta_{n-1},X_n)}
		.
	\end{split}
\end{equation}
	This,
	\eqref{eq:representation:AdaMax:3}, and \eqref{prop:eq:repr:momentum:seq2:AdaMax:New} \prove\ for all $n\in\N$, $\mathfrak{i}\in\{1,2,\dots,\pars\}$ that
	\begin{equation}
		\begin{split}
			\textstyle
			\Theta_{ n }^{(\mathfrak{i})}
			&=\Theta_{ n-1  }^{(\mathfrak{i})}-\gamma_n\prb{\eps+\MOMALT[n]^{(\mathfrak{i})}}^{-1}
			\mom[n]^{(\mathfrak{i})}
			\\&=\Theta_{ n-1  }^{(\mathfrak{i})}-\frac{\gamma_{n}  \PRb{\alpha^n\mom[0]^{(\mathfrak{i})}+(1-\alpha)\sum_{k=1}^n\alpha^{n-k}G_k^{(\mathfrak{i})}(\Theta_{k-1},X_k)}}{\eps+\MOMs[n]
				\prb{\MOMALT[0]^{(\mathfrak{i})},G_1^{(\mathfrak{i})}(\Theta_{0},X_1),G_2^{(\mathfrak{i})}(\Theta_{1},X_2),\dots,G_n^{(\mathfrak{i})}(\Theta_{n-1},X_n)}}
			.		
		\end{split}
	\end{equation}
	Combining 
	\cref{cor:limsup:simple}
	(applied with 
	$q \curvearrowleft 1$,
	$\fc\curvearrowleft \const$,
	$\const \curvearrowleft \const$,
	$\gamma \curvearrowleft \gamma$,
	$\batch\curvearrowleft \batch$,
	$\Theta \curvearrowleft \Theta$,
	$(\MOM_n)_{n\in\N} \curvearrowleft (\MOMs[n])_{n\in\N}$,
	$(G_n)_{n\in\N} \curvearrowleft (G_n)_{n\in\N}$,
	$ ( \Omega, \cF,(\mathbb{F}_n)_{n\in\N_0}, \P ) \curvearrowleft( \Omega, \cF,(\mathbb{F}_n)_{n\in\N_0}, \P ) $,
	$(X_n)_{n\in\N}\curvearrowleft (X_{n,1},X_{n,2},\dots,\allowbreak X_{n,\batch_n})_{n\in\N}$,
	$\MOMs \curvearrowleft \prb{\N\times\Omega\ni(n,\omega)\mapsto (2-n)\mom[0](\omega)+(n-1)\MOMALT[0](\omega)\in\R^\pars}$
	in the notation of \cref{cor:limsup:simple}),
	\eqref{eq:bounded:batches:and:LR:AdaMax:New},
	\eqref{eq:start:mom:bound:new:rob:cond:AdaMax:New},
	\eqref{eq:representation:AdaMax:gradient:New},
	\eqref{eq:double:inf:var:for:AdaMax}, and
	\eqref{eq:LB:AdaMax:denominator:UB:New}
	with
	\cref{it:filtered:prob:space:AdaMax:New,it:filter:independence:AdaMax:New,it:filter:measurable:AdaMax:New,it:filter:measurable2:AdaMax:New}
	hence \proves\ that
	\begin{equation}
		\begin{split}
			\inf_{\substack{\xi\colon \Omega \to \R^\pars\\\text{measurable}}}
			\PRbbb{\limsup_{ n \to \infty }\E\PRb{ \norm{ \Theta_n - \xi }^{2} }}
			>
			0
			.
		\end{split}
	\end{equation}
	This \proves\ \cref{eq:AdaMax:LB:New}. \finishproofthus
\end{proof}

\subsection{AMSGrad optimization}

\cfclear
\begin{athm}{cor}{cor:AMSgrad}
	Let $\pars,\dim\in\N$, 
	$\fc,a\in\R$,
	$b\in[a,\infty)$,
	$\eps,\cst,\Cst\in(0,\infty)$,
	$\alpha\in[0,1)$,
	$\beta\in(\alpha^{2},1)$, 
	let 
	$\batch\colon\N\to\N$ and
	$\gamma\colon\N\to[0,\infty)$ satisfy
	\begin{equation}
		\label{eq:bounded:batches:and:LR:AMSgrad}
		\textstyle
		0<\limsup_{n\to\infty}\gamma_n<\infty
		\qqandqq
		\limsup_{n\to\infty}\batch_n<\infty,
	\end{equation}
	let $g=(g_1,\dots,g_{\pars})\colon\R^{\pars}\times[a,b]^{\dim}\to\R^{\pars}$ be measurable, 
	let $ ( \Omega, \cF, \P ) $ be a probability space,
	let $X_{n,m}\colon\Omega\to[a,b]^\dim$, $(n,m)\in\N^2$, be \iid\ random variables,
	let
	$ \Theta=(\Theta^{(1)},\ldots,\Theta^{(\pars)}) \colon\N_0\times\Omega \to \R^{\pars}$,
	$\mom=(\mom^{(1)},\ldots,\mom^{(\pars)})\colon\N_0\times\Omega\to\R^{\pars}$,
	$\MOMALT=(\MOMALT^{(1)},\ldots,\MOMALT^{(\pars)})\colon\N_0\times\Omega\to[0,\infty)^{\pars}$,
	and 
	$\MOMalt=\prb{\MOMalt^{(1)},\dots,\MOMalt^{(\pars)}}\colon\N_0\times\Omega\to\R^\pars$ be stochastic processes which satisfy
	for all $n\in\N$, $i\in\{1,2,\dots,\pars\}$ that
	\begin{equation}
		\label{eq:representation:AMSgrad:1}
		\begin{split}
			\mom[n]&\textstyle=\alpha \mom[n-1]+(1-\alpha)\PRb
			{\frac{1}{\batch_{n}}\sum_{j=1}^{\batch_{n}}g(\Theta_{n-1},X_{n,j})},
		\end{split}
	\end{equation}
	\begin{equation}
		\label{eq:representation:AMSgrad:2}
		\begin{split}
			\MOMALT[n]^{(i)}&\textstyle=\beta \MOMALT[n-1]^{(i)}+(1-\beta)\PRb
			{\frac{1}{\batch_{n}}\sum_{j=1}^{\batch_{n}}g_i(\Theta_{n-1},X_{n,j})}^2,
		\end{split}
	\end{equation}
	\begin{equation}
		\label{eq:representation:AMSgrad:3}
		\begin{split}
			\MOMalt[0]=\MOMALT[0],
			\qquad
			\MOMalt[n]^{(i)}\textstyle
			=\max\pRb{\MOMalt[n-1]^{(i)},\MOMALT[n]^{(i)}},
			\qqandqq
			\Theta_{ n }^{(i)}
			&=\Theta_{ n-1  }^{(i)}-\frac{\gamma_n
				\mom[n]^{(i)}
			}{\eps+\PRb{\MOMalt[n]^{(i)}}^{\nicefrac{1}{2}}}
			,
		\end{split}
	\end{equation}
	assume that $(\Theta_0,\mom[0],\MOMALT[0])$ and $\pr{X_{n,m}}_{(n,\mm)\in\{(k,j)\in\N^2\colon j\leq\batch_k\}}$ are independent, 
	let $i\in\{1,2,\dots,\pars\}$ satisfy
	\begin{equation}
		\label{eq:start:mom:bound:new:rob:cond:AMSgrad}
		\max\pRb{
			\vass{\mom[0]^{(i)}},
			\vass{\MOMALT[0]^{(i)}}^{\nicefrac{1}{2}}
		}
		\leq \Cst(\vass{\Theta_{0}^{(i)}}+\fc)
		\qqandqq
		\E\PRb{\vass{\Theta^{(i)}_0}}<\infty,
	\end{equation}
	assume for all
	$\theta=(\theta_1,\dots,\theta_\pars)\in\R^\pars$, $x\in[a,b]^\dim$
	that
	\begin{equation}
		\label{prop:eq:setup:new:rob:cond:AMSgrad}
		\pr{\theta_i-\fc}
		\prb{\cst+(\Cst-\cst)\indicator{(-\infty,\fc]}(\theta_i)}
		\leq
		g_i(\theta,x)
		\leq
		\pr{\theta_i+\fc}
		\prb{\cst+(\Cst-\cst)\indicator{[-\fc,\infty)}(\theta_i)}
		,
	\end{equation}
	and assume $\textstyle\inf_{\theta\in\R^{\pars}}\var\prb{g_i(\theta,X_{1,1})}>0$.
	Then 
	\begin{equation}
		\label{eq:AMSgrad:LB}
		\begin{split}
				\inf_{\substack{\xi\colon \Omega \to \R\\\text{measurable}}}
			\PRbbb{	\limsup_{ n \to \infty }\E\PRb{ | \Theta_n^{(i)} - \xi |}}
			>
			0
			.
		\end{split}
	\end{equation}
\end{athm}

\begin{proof}[Proof of \cref{cor:AMSgrad}]
	Throughout this proof 
	assume without loss of generality that $\fc\geq 1$,
	for every $n\in\N$ let $X_n=(X_{n,1},\dots,X_{n,\batch_n})$,
	 for every $n\in\N$ let
	$G_{n}\colon\R^\pars \times\pr{[a,b]^{\dim}}^{\batch_n}\to\R^\pars$ satisfy for all $\theta\in\R^{\pars}$, $x=(x_1,\dots,x_{\batch_n})\in\pr{[a,b]^{\dim}}^{\batch_n}$ that
	\begin{equation}
		\label{eq:representation:AMSgrad:gradient}
		\textstyle
		G_n(\theta,x)=\PRb{\frac{1}{\batch_{n}}\sum_{j=1}^{\batch_{n}}g(\theta,x_j)}
		,
	\end{equation}
	and let $\MOMs[n]\colon\R^{n+1}\to\R$, $n\in\N$,
	satisfy for all $n\in\N$, $x_0,x_1,\dots,x_n\in\R$ that
	\begin{equation}
		\label{eq:setup:AMSgrad:working:denominator:repr}
		\begin{split}
			\MOMs[n](x_0,\dots,x_n)=
			\textstyle\max_{j\in\{0,1,\dots,n\}}
			\PRb{\sum_{k=0}^j (1-\beta)\beta^{j-k}\pr
				{x_k}^2}^{\nicefrac{1}{2}}
			.
		\end{split}
	\end{equation}
	\Nobs that for all 
$r\in[0,\infty)$,
$x\in\R$, 
$y\in(-\infty,x]$, 
$s,t\in\{-1,1\}$ with $x=s\vass{x}$ and $y=t\vass{y}$ it holds that
\begin{equation}
	\begin{split}
		x\sqrt{ y^2 + r }
		-y\sqrt{ x^2 + r }
		=
		s\sqrt{ x^2 y^2 + r x^2 }
		-t\sqrt{ x^2 y^2 + r y^2 }
		\geq 0
		.
	\end{split}
\end{equation}
This and \eqref{eq:setup:AMSgrad:working:denominator:repr} \prove\ that for all $n\in\N$, $x_0,x_1,\dots,x_n\in\R$, $y\in(-\infty,x_n)$ it holds that
\begin{equation}
	\label{eq:repr:AMSgrad:denominator:first:prop}
	\begin{split}
		\textstyle
		&\textstyle y
		\PRb{\sum_{k=0}^n (1-\beta)\beta^{n-k}\pr
			{x_k}^2}^{\nicefrac{1}{2}}
		\\&=
		\pr{1-\beta}^{\nicefrac{1}{2}}
		y\PRb{(x_n)^2+
			(1-\beta)^{-1}\PRb{\textstyle\sum_{k=0}^{n-1} (1-\beta)\beta^{n-k}\pr
				{x_k}^2}}^{\nicefrac{1}{2}}
		\\&\leq
		\pr{1-\beta}^{\nicefrac{1}{2}}
		x_n\PRb{y^2+
			(1-\beta)^{-1}\PRb{\textstyle\sum_{k=0}^{n-1} (1-\beta)\beta^{n-k}\pr
				{x_k}^2}}^{\nicefrac{1}{2}}
		\\&=
		\textstyle
		x_n
		\PRb{(1-\beta)y^2+\sum_{k=0}^{n-1} (1-\beta)\beta^{n-k}\pr
			{x_k}^2}^{\nicefrac{1}{2}}
		\\&=
		\textstyle
		x_n\PRb{\sum_{k=0}^n (1-\beta)\beta^{n-k}\pr
			{x_k\indicator{[0,n)}(k)+y\indicator{\{n\}}(k)}^2}^{\nicefrac{1}{2}}
		.
	\end{split}
\end{equation}
This and \eqref{eq:setup:AMSgrad:working:denominator:repr} \prove\ that for all $n\in\N$, $i\in\{1,2,\dots,\pars\}$,  $x_0,x_1,\dots,x_n\in\R$, $y\in(-\infty,x_n)$ it holds that
\begin{equation}
	\label{eq:repr:AMSgrad:denominator:first:prop:1}
	\begin{split}
		y\MOMs[n](x_0,\dots, x_n)
		&=\textstyle 
		y\max_{j\in\{0,1,\dots,n\}}
		\PRb{\sum_{k=0}^j (1-\beta)\beta^{j-k}\pr
			{x_k}^2}^{\nicefrac{1}{2}}
		\\&\leq\textstyle 
		x_n\max_{j\in\{0,1,\dots,n\}}
		\PRb{\sum_{k=0}^j (1-\beta)\beta^{j-k}\pr
			{x_k\indicator{[0,n)}(k)+y\indicator{\{n\}}(k)}^2}^{\nicefrac{1}{2}}
		\\&= x_n\MOMs[n](x_0,\dots,x_{n-1},y).
	\end{split}
\end{equation}
\Moreover \eqref{eq:start:mom:bound:new:rob:cond:AMSgrad} \proves\ that
\begin{equation}
	\label{eq:UB:by:start:params}
	\begin{split}
		\textstyle
		\max\pRb{ \vass{\mom[0]^{(i)}},
			\PRb{(1-\beta)^{-1}\vass{\MOMALT[0]^{(i)}}}^{\nicefrac{1}{2}}}
		&\leq
		\max\pRb{
		\Cst(\vass{\Theta_{0}^{(i)}}+\fc)
		,
		(1-\beta)^{-\nicefrac{1}{2}}\Cst(\vass{\Theta_{0}^{(i)}}+\fc)
		}
		\\&\leq
		 (1-\beta)^{-1}\Cst(\vass{\Theta_{0}^{(i)}}+\fc).
	\end{split}
\end{equation}
	\Moreover 
	\eqref{eq:setup:AMSgrad:working:denominator:repr} \proves\
	for all $n\in\N$, $x_0,x_1,\dots,x_n\in\R$ that
	\begin{equation}
		\label{eq:LB:AMSgrad:denominator}
		\begin{split}
		\textstyle\frac{1}{\fc}\PRb{\pr{1-\beta}\sum_{k=0}^n\beta^{n-k}(x_k)^2}^{\nicefrac{1}{2}}
		&\leq
		\textstyle\PRb{\pr{1-\beta}\sum_{k=0}^n\beta^{n-k}(x_k)^2}^{\nicefrac{1}{2}}
		\\&\leq
		\textstyle\max_{k\in\{0,1,\dots,n\}}
		\PRb{\sum_{j=0}^k (1-\beta)\beta^{k-j}\pr
			{x_j}^2}^{\nicefrac{1}{2}}
		\\&=
		\MOMs[n](x_0,\dots,x_n)
		.
		\end{split}
	\end{equation}
	\Moreover \eqref{eq:setup:AMSgrad:working:denominator:repr} \proves\
	for all $n\in\N$, $x_0,x_1,\dots,x_n\in\R$ that
	\begin{equation}
		\label{eq:LB:AMSgrad:denominator:UB}
		\begin{split}
			\MOMs[n](x_0,\dots,x_n)
			&=
			\textstyle\max_{j\in\{0,1,\dots,n\}}
			\PRb{\sum_{k=0}^j (1-\beta)\beta^{j-k}\pr
				{x_k}^2}^{\nicefrac{1}{2}}
			\\&\leq
			\textstyle
			\max\pRb{\vass{x_0},\vass{x_1},\dots,\vass{x_n}}
			\max_{j\in\{0,1,\dots,n\}}
			\PRb{\sum_{k=0}^j (1-\beta)\beta^{j-k}}^{\nicefrac{1}{2}}
			\\&\leq\textstyle
			\sqrt{\fc}
			\max\pRb{\vass{x_0},\vass{x_1},\dots,\vass{x_n}}
			.
		\end{split}
	\end{equation}
	\Nobs that \cref{lem:end:filtration} 
	(applied with 
	$(X_{n,j})_{(n,j)\in\{(m,i)\in\N^2\colon i\leq\BATCH_m\}}\curvearrowleft\pr{X_{n,j}}_{(n,j)\in\{(m,i)\in\N^2\colon i\leq\batch_m\}}$,
	$Y\curvearrowleft (\Theta_0,\mom[0],\MOMALT[0])$,
	$\batch\curvearrowleft\batch$,
	$\Batch\curvearrowleft\batch$,
	$\BATCH\curvearrowleft\batch$ in the notation of \cref{lem:end:filtration}) \prove\ that there exist $\fil_n\subseteq\cF$, $n\in\N_0$, which satisfy that
	\begin{enumerate}[label=(\Roman*)]
		\item \label{it:filtered:prob:space:AMSgrad}it holds that $ ( \Omega, \cF,(\mathbb{F}_n)_{n\in\N_0}, \P ) $ is a filtered probability space,
		\item \label{it:filter:independence:AMSgrad}it holds for all $n\in\N$ that 
		$(X_{n,1},X_{n,2},\dots,X_{n,\batch_n})$ is independent of
		$\fil_{n-1}$,
		\item \label{it:filter:measurable:AMSgrad}it holds for all $n\in\N$ that $(X_{n,1},X_{n,2},\dots,X_{n,\batch_n})$ is $\fil_n$-measurable, and
		\item \label{it:filter:measurable2:AMSgrad}it holds that $(\Theta_0,\mom[0],\MOMALT[0])$ is $\fil_0$-measurable.
	\end{enumerate}
	\Moreover 
	\eqref{eq:representation:AMSgrad:1},
	\eqref{eq:representation:AMSgrad:gradient}, and
	\cite[Lemma~2.7]{DereichGraeberJentzenRiekert2025} (applied with
	$\pars\curvearrowleft\pars$,
	$(\beta_n)_{n\in\N}\curvearrowleft(\alpha)_{n\in\N}$,
	$\mom\curvearrowleft\mom$,
	$(g_n)_{n\in\N}\curvearrowleft \pr{\frac{1-\alpha}{\batch_{n}}\sum_{j=1}^{\batch_{n}}g(\Theta_{n-1},X_{n,j})}_{n\in\N}$
	in the notation of \cite[Lemma~2.7]{DereichGraeberJentzenRiekert2025}) \prove\ that for all $n\in\N$ it holds that 
	\begin{equation}
		\label{prop:eq:repr:momentum:seq2:AMSgrad}
		\begin{split}
			\mom[n]
			&=\textstyle
			\alpha^n\mom[0]
			+\sum_{k=1}^{n}\alpha^{n-k}\pr{\frac{1-\alpha}{\batch_{k}}\sum_{j=1}^{\batch_{k}}g(\Theta_{k-1},X_{k,j})}
			\\&=\textstyle
			\alpha^n\mom[0]
			+(1-\alpha)\sum_{k=1}^{n}\alpha^{n-k}G_k\pr{\Theta_{k-1},X_k}
			.
		\end{split}
	\end{equation}
		\Moreover 
	\eqref{eq:representation:AMSgrad:2},
	\eqref{eq:representation:AMSgrad:gradient},
	and
	\cite[Lemma~2.7]{DereichGraeberJentzenRiekert2025} (applied with
	$\pars\curvearrowleft 1$,
	$(\beta_n)_{n\in\N}\curvearrowleft(\beta)_{n\in\N}$,
	$\mom\curvearrowleft\MOMALT^{(\mathfrak{i})}$,
	$(g_n)_{n\in\N}\curvearrowleft \prb{\pr{1-\beta}\PRb{\frac{1}{\batch_{n}}\sum_{j=1}^{\batch_{n}}g_{\mathfrak{i}}(\Theta_{n-1},X_{n,j})}^2}_{n\in\N}$
	for $\mathfrak{i}\in\{1,2,\dots,\pars\}$
	in the notation of \cite[Lemma~2.7]{DereichGraeberJentzenRiekert2025}) \prove\ that for all $n\in\N$, $\mathfrak{i}\in\{1,2,\dots,\pars\}$ it holds that 
	\begin{equation}
		\label{prop:eq:repr:RMS:seq2:AMSgrad}
		\begin{split}
			\MOMALT[n]^{(\mathfrak{i})}
			&=\textstyle\beta^n \MOMALT[0]^{(\mathfrak{i})}+\sum_{k=1}^n \beta^{n-k}(1-\beta)\PRb
			{\frac{1}{\batch_{k}}\sum_{j=1}^{\batch_{k}}g_{\mathfrak{i}}(\Theta_{k-1},X_{k,j})}^2
			\\&=\textstyle\beta^n \MOMALT[0]^{(\mathfrak{i})}+\sum_{k=1}^n (1-\beta)\beta^{n-k}\prb
			{G_k^{(\mathfrak{i})}(\Theta_{k-1},X_k)}^2
			\\&=\textstyle
			\sum_{k=0}^n (1-\beta)\beta^{n-k}
			\PRb
			{\prb
				{G_k^{(\mathfrak{i})}(\Theta_{k-1},X_k)}\indicator{(0,n]}(k)+\prb{(1-\beta)^{-1}\MOMALT[0]^{(\mathfrak{i})}}^{\nicefrac{1}{2}}\indicator{\{0\}}(k)}^2
			.
		\end{split}
	\end{equation}
	Combining this and
	\eqref{eq:representation:AMSgrad:3} with induction
	\proves\ that for all $n\in\N$, $\mathfrak{i}\in\{1,2,\dots,\pars\}$ it holds that 
	\begin{align}
			&\nonumber\MOMalt[n]^{(\mathfrak{i})}\textstyle
			\\&\nonumber=\max\pRb{\MOMalt[n-1]^{(\mathfrak{i})},\MOMALT[n]^{(\mathfrak{i})}}
			\\&=\max\pRb{\MOMALT[n]^{(\mathfrak{i})},\MOMALT[n-1]^{(\mathfrak{i})},
				\dots,
				\MOMALT[1]^{(\mathfrak{i})}
				,\MOMalt[0]^{(\mathfrak{i})}}
			\\&\nonumber=\max\pRb{\MOMALT[0]^{(\mathfrak{i})},\MOMALT[1]^{(\mathfrak{i})},
					\dots,
					\MOMALT[n]^{(\mathfrak{i})}
				}
			\\&\nonumber=
			\textstyle
			\max_{j\in\{0,1,\dots,n\}}
				\prb{\sum_{k=0}^j (1-\beta)\beta^{j-k}
			\PRb
			{\prb
				{G_k^{(\mathfrak{i})}(\Theta_{k-1},X_k)}\indicator{(0,j]}(k)+\prb{(1-\beta)^{-1}\MOMALT[0]^{(\mathfrak{i})}}^{\nicefrac{1}{2}}\indicator{\{0\}}(k)}^2}
			\\&\nonumber
			=
			\prb{\MOMs[n]\prb{
			\PRb{(1-\beta)^{-1}\MOMALT[0]^{(\mathfrak{i})}}^{\nicefrac{1}{2}},
			G_1^{(\mathfrak{i})}(\Theta_{0},X_1),
			G_2^{(\mathfrak{i})}(\Theta_{1},X_2),
			\dots,
			G_n^{(\mathfrak{i})}(\Theta_{n-1},X_n)
			}}^2
			.
		\end{align}
	This, \eqref{eq:representation:AMSgrad:3}, and
	\eqref{prop:eq:repr:momentum:seq2:AMSgrad} \prove\ for all $n\in\N$, $\mathfrak{i}\in\{1,2,\dots,\pars\}$ that
	\begin{equation}
		\begin{split}
			\textstyle
			&\Theta_{ n }^{(\mathfrak{i})}
			\\&=\Theta_{ n-1  }^{(\mathfrak{i})}-\gamma_n
				\mom[n]^{(\mathfrak{i})}
			\prb{\eps+\PRb{\MOMalt[n]^{(\mathfrak{i})}}^{\nicefrac{1}{2}}}^{-1}
			\\&=\Theta_{ n-1  }^{(\mathfrak{i})}-\frac{\gamma_{n}  \PR{\alpha^n\mom[0]^{(\mathfrak{i})}+(1-\alpha)\sum_{k=1}^n\alpha^{n-k}G_k^{(\mathfrak{i})}(\Theta_{k-1},X_k)}}{\eps+\MOMs[n]
				\prb{
					\PRb{(1-\beta)^{-1}\MOMALT[0]^{(\mathfrak{i})}}^{\nicefrac{1}{2}},
					G_1^{(\mathfrak{i})}(\Theta_{0},X_1),
					G_2^{(\mathfrak{i})}(\Theta_{1},X_2),
					\dots,
					G_n^{(\mathfrak{i})}(\Theta_{n-1},X_n)
					}}
			.		
		\end{split}
	\end{equation}
	Combining 
	\cref{thm:new:rob:cond}
	(applied with 
	$z\curvearrowleft 1$,
	$\fc \curvearrowleft \fc$,
	$(\MOM_n)_{n\in\N} \curvearrowleft (\MOMs[n])_{n\in\N}$,
	$ ( \Omega, \cF,(\mathbb{F}_n)_{n\in\N_0}, \P ) \curvearrowleft( \Omega, \cF,(\mathbb{F}_n)_{n\in\N_0},\allowbreak \P ) $,
	$(X_n)_{n\in\N}\curvearrowleft (X_{n,1},X_{n,2},\dots,X_{n,\batch_n})_{n\in\N}$,
	$\Theta \curvearrowleft \Theta$,
	$\MOMs[1]\curvearrowleft\mom[0]$, $\MOMs[2]\curvearrowleft((\MOMALT[0]^{(j)}(1-\beta)^{-1})^{\nicefrac{1}{2}})_{j\in\{1,2,\dots,\pars\}}$,
	$\Cst\curvearrowleft(1-\beta)^{-1}\Cst$,
	$i \curvearrowleft i$
	in the notation of \cref{thm:new:rob:cond}),
	\eqref{eq:bounded:batches:and:LR:AMSgrad},
	\eqref{eq:start:mom:bound:new:rob:cond:AMSgrad},
	\eqref{prop:eq:setup:new:rob:cond:AMSgrad},
	\eqref{eq:representation:AMSgrad:gradient},
	\eqref{eq:setup:AMSgrad:working:denominator:repr},  
	\eqref{eq:repr:AMSgrad:denominator:first:prop:1},
	\eqref{eq:UB:by:start:params},
	\eqref{eq:LB:AMSgrad:denominator}, and
	\eqref{eq:LB:AMSgrad:denominator:UB}
	with
	\cref{it:filtered:prob:space:AMSgrad,it:filter:independence:AMSgrad,it:filter:measurable:AMSgrad,it:filter:measurable2:AMSgrad}
	hence \proves\ that
	\begin{equation}
		\begin{split}
			\inf_{\substack{\xi\colon \Omega \to \R\\ \text{measurable}}}
			\PRbbb{	\limsup_{ n \to \infty }\E\PRb{ | \Theta_n^{(i)} - \xi |}}
			>
			0
			.
		\end{split}
	\end{equation}
	This \proves\ \cref{eq:AMSgrad:LB}. \finishproofthus
\end{proof}

\subsection*{Acknowledgements}

Shokhrukh Ibragimov is gratefully acknowledged for having brought 
useful related research findings to our attention. 
This work has been partially supported by the National Natural Science Foundation of China (NSFC) under grant number W2531010. 
This work has been partially funded by the European Union (ERC, MONTECARLO, 101045811). 
The views and the opinions expressed in this work are however those of the authors only 
and do not necessarily reflect those of the European Union 
or the European Research Council (ERC). Neither the European Union nor 
the granting authority can be held responsible for them. 
In addition, we gratefully acknowledge the Cluster of Excellence EXC 2044-390685587, 
Mathematics M\"{u}nster: Dynamics-Geometry-Structure funded by the 
Deutsche Forschungsgemeinschaft (DFG, German Research Foundation).

\subsection*{Use of large language models}
\label{intro:llms}
The large language model GPT 6 has helped to identify several errors and inaccuracies within this article.
Moreover, GPT 6 has helped organize and extend the literature review in \cref{subsection:literature} as well as the references below.
The ideas underlying the mathematical results and their proofs originate from the authors of this article, who take full responsibility for the entire content in this article.

\bibliographystyle{acm}
\bibliography{bibfile}

\end{document}